\documentclass[twoside,11pt]{article}

\usepackage{jmlr2e}

\newcommand{\R}{\mathbb{R}}

\newcommand{\cH}{\mathcal{H}}

\newcommand{\norm}[1]{\left|\left| #1 \right|\right|}
\newcommand{\Rad}{\mathtt{Rad}}
\newcommand{\Tang}{\mathtt{Tang}}
\newcommand{\Haus}{\mathtt{Haus}}

\newcommand{\grad}{\mathtt{grad}\,}
\newcommand{\Exp}{\mathtt{Exp}}
\renewcommand{\hat}{\widehat}
\renewcommand{\tilde}{\widetilde}

\allowdisplaybreaks

\newtheorem{innercustomthm}{Theorem}
\newenvironment{customthm}[1]
{\renewcommand\theinnercustomthm{#1}\innercustomthm}
{\endinnercustomthm}

\newtheorem{innercustomlem}{Lemma}
\newenvironment{customlem}[1]
{\renewcommand\theinnercustomlem{#1}\innercustomlem}
{\endinnercustomlem}

\usepackage{lastpage}

\jmlrheading{22}{2021}{1-\pageref{LastPage}}{10/20; Revised 4/21}{6/21}{20-1194}{Yikun Zhang and Yen-Chi Chen}

\ShortHeadings{Learning Theory of Directional Mean Shift Algorithm}{Zhang and Chen}
\firstpageno{1}

\begin{document}
	
	\title{Kernel Smoothing, Mean Shift, and Their Learning Theory with Directional Data}
	
	\author{\name Yikun Zhang \email yikun@uw.edu \\
		\name Yen-Chi Chen \email yenchic@uw.edu \\
		\addr Department of Statistics\\
		University of Washington\\
		Seattle, WA 98195, USA}
	
	\editor{Sayan Mukherjee}
	
	\maketitle
	
	\begin{abstract}
		Directional data consist of observations distributed on a (hyper)sphere, and appear in many applied fields, such as astronomy, ecology, and environmental science.
		This paper studies both statistical and computational problems of kernel smoothing for directional data.
		We generalize the classical mean shift algorithm to directional data, which allows us to identify local modes of the directional kernel density estimator (KDE).
		The statistical convergence rates of the directional KDE and its derivatives are derived, and the problem of mode estimation is examined.
		We also prove the ascending property of the directional mean shift algorithm and investigate a general problem of gradient ascent on the unit hypersphere.
		To demonstrate the applicability of the algorithm, we evaluate it as a mode clustering method on both simulated and real-world data sets.
	\end{abstract}
	
	
	\section{Introduction}
	\label{Sec:Intro}

	A directional data set (or simply directional data) is the collection of observations on a (hyper)sphere.
	It occurs in many scientific problems when measurements are taken on the surface of a spherical object, such as Earth or other planets.
	For instance, the locations of earthquakes are often represented by their longitudes and latitudes
	\citep{taylor2009active, craig2011earthquake};
	thus, the locations can be viewed as random variables on a two-dimensional (2D) sphere.
	In astronomical surveys,
	the locations of galaxies are usually recorded by their angular positions (right ascensions and declinations) in the sky, leading
	to observations on a 2D sphere \citep{york2000sloan,skrutskie2006two, abbott2016dark}. 
	In planetary science, observations often comprise locations on a planet, such as Mars,
	and can also be considered as random variables on a 2D sphere \citep{Lakes_On_Mars,BARLOW2015,Unif_test_hypersphere2020}.
	
	These observations on a sphere can be regarded as independently and identically distributed random variables from a density function supported on the sphere (called a directional density function). 
	The local modes of a density function are often of research interest because they signal high density areas \citep{scott2012multivariate}
	and can be used to cluster data \citep{MS_Density_ridge2018,chacon2020modal}.
	However, identifying the local modes of a directional density function is a nontrivial task that involves both statistical and computational challenges. 
	From a statistical perspective, it is necessary to obtain an accurate estimator of the underlying directional density (as well as its derivatives). From a computational perspective, it is needful to design an algorithm to efficiently compute the local modes of the density estimator. 
	
	To address the aforementioned challenges, we consider the idea of kernel smoothing because the kernel density estimator (KDE; \citealt{Rosenblatt1956, Parzen1962}) in the Euclidean data setting is highly successful. Its statistical properties have been well-studied \citep{All_nonpara2006,Scott2015,KDE_t}, and the local modes of a Euclidean KDE are often good estimators of the local modes of the underlying density function \citep{Parzen1962, Romano1988, vieu1996note,Mode_clu2016}. 
	Moreover, in Euclidean KDEs, there is an elegant algorithm known as the \emph{mean shift} algorithm
	 \citep{MS1975,MS1995,MS2002,carreira2015review}
	that allows us to numerically obtain the local modes at a low cost.
	
	\begin{figure}
		\centering
		\includegraphics[width=0.9\linewidth]{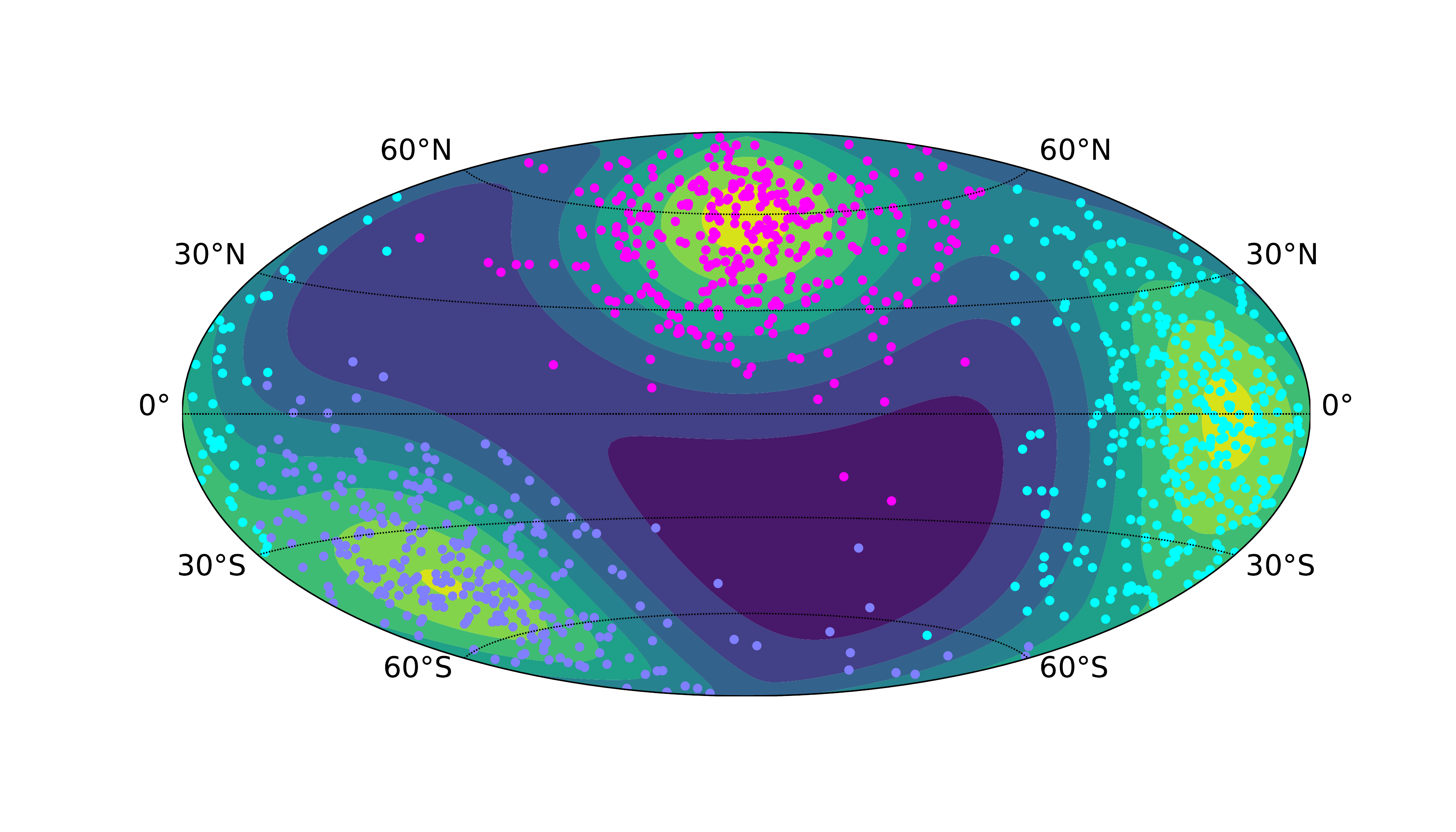}
		\caption{Clustering of directional data using the proposed directional mean shift algorithm (Algorithm~\ref{Algo:MS}). Additional details of the simulated data can be found in Section~\ref{sec:spherical}.}
		\label{fig:Mode_clu_Example}
	\end{figure}
	
	Although kernel smoothing has been applied to directional data since the seminal work of \cite{KDE_Sphe1987}
	and other studies have been conducted on analyzing its performance as a density estimator \citep{KDE_direct1988, Zhao2001,Exact_Risk_bw2013, ley2018applied}, 
	little is known about the behavior of the derivatives of a directional KDE. To the best of our knowledge,
	\cite{KLEMELA2000} was the only work to examine the derivatives of a particular type of directional KDE; however, their estimators are rarely used in practice.
	Thus, the statistical properties of the gradient system induced by a general directional KDE and the resulting local modes are still open problems. 
	
	Computationally, the standard mean shift algorithm was first generalized to directional data setting by \cite{Multi_Clu_Gene2005}. Using the directional mean shift algorithm, we are able to determine the local modes of the directional KDE and perform mode clustering (mean shift clustering) of spherical data. Figure~\ref{fig:Mode_clu_Example} presents an example of mode clustering with our proposed algorithm. However, the algorithmic rate of convergence of the mean shift algorithm with directional data remains unclear. We address this problem by viewing the directional mean shift algorithm as a special case of gradient ascent methods on the $q$-dimensional unit sphere $\Omega_q = \big\{\bm{x} \in \mathbb{R}^{q+1}: \norm{\bm{x}}_2^2 = x_1^2+ \cdots + x_{q+1}^2=1 \big\}$ and develop some linear convergence results for the general gradient ascent method on $\Omega_q$.

	\textbf{\emph{Notation}.} Bold-faced variables (e.g., $\bm{x}, \bm{\mu}$) represent vectors, while capitalized (bold-faced) variables (e.g., $\bm{X}_1,...,\bm{X}_n$) denote random variables (or random vectors). The set of real numbers is denoted by $\mathbb{R}$, while the unit $q$-dimensional sphere embedded in $\mathbb{R}^{q+1}$ is denoted by $\Omega_q$. The norm $\norm{\cdot}_2$ is the usual Euclidean norm (or so-called $L_2$-norm) in $\mathbb{R}^d$ for some positive integer $d$. The directional density is denoted by $f$ unless otherwise specified, and the probability of a set of events is denoted by $\mathbb{P}$. If a random vector $\bm{X}$ is distributed as $f(\cdot)$, the expectations of functions of $\bm{X}$ are denoted by $\mathbb{E}_f$ or $\mathbb{E}$ when the underlying distribution function is clear. We use the big-O notation $h(x)=O(g(x))$ if the absolute value of $h(x)$ is upper bounded by a positive constant multiple of $g(x)$ for all sufficiently large $x$. In contrast, $h(x)=o(g(x))$ when $\lim_{x\to\infty} \frac{|h(x)|}{g(x)}=0$. For random vectors, the notation $o_P(1)$ is short for a sequence of random vectors that converges to zero in probability. The expression $O_P(1)$ denotes a sequence that is bounded in probability. Additional details of stochastic $o$ and $O$ symbols can be found in Section 2.2 of \cite{VDV1998}.
	The notation $a_n \asymp b_n$ indicates that $\frac{a_n}{b_n}$ has lower and upper bounds away from zero and infinity, respectively. 

	\textbf{\emph{Main results}.}
	\begin{enumerate}
		\item We revisit the mean shift algorithm with directional data (Algorithm~\ref{Algo:MS}) and provide some new insights on its iterative formula, which can be expressed in terms of the total gradient of the directional KDE (Sections~\ref{Sec:MS_Dir} and \ref{Sec:grad_Hess_Dir}).
		\item From the perspective of statistical learning theory, we establish uniform convergence rates of the gradient and Hessian of the directional KDE (Theorem~\ref{pw_conv_tang} and \ref{unif_conv_tang}).
		\item Moreover, we derive the asymptotic properties of estimated local modes around the true (population) local modes (Theorem~\ref{Mode_cons}).
		\item With regard to computational learning theory, we prove the ascending and converging properties of the directional mean shift algorithm (Theorems~\ref{MS_asc} and \ref{MS_conv}). 
		\item In addition, we prove that the directional mean shift algorithm converges linearly to an estimated local mode under suitable initialization (Theorem~\ref{Linear_Conv_GA}).
		\item We demonstrate the applicability of the directional mean shift algorithm by using it as a clustering method on both simulated and real-world data sets (Section~\ref{Sec:Experiments}).
	\end{enumerate}
	
	\textbf{\emph{Related work}.}
	The directional KDE has a long history in statistics since the work of \cite{KDE_Sphe1987}. Its statistical convergence rates and asymptotic distributions have been studied by \cite{KDE_direct1988, Zhao2001}.
	In addition, \cite{KDE_Sphe1987,KDE_direct1988,Exact_Risk_bw2013,Dir_Linear2013} considered the problem of selecting the smoothing bandwidth of directional KDEs.
 	A study by \cite{KLEMELA2000} was the first to estimate the derivatives of a directional density. 
 	More generally, \cite{hendriks1990,pelletier2005,berry2017} considered the nonparametric density estimation on (Riemannian) manifolds (with boundary). The uniform convergence rate and asymptotic results of the KDE on Riemannian manifolds have also been investigated in \cite{henry2009,jiang2017,kim2019uniform}. As the unit hypersphere $\Omega_q$ is a $q$-dimensional manifold with constant curvature and positive reach \citep{federer1959}, their analyses and results are applicable to the directional KDE.
 	
 	The standard mean shift algorithm with Euclidean data is a popular approach to various tasks such as clustering \citep{MS1975}, image segmentation \citep{MS2002}, and object tracking \citep{Kernel_Based_Ob2003}; see a comprehensive review in \cite{carreira2015review}. Its convergence properties have been well-studied in \cite{MS1995,MS2007_pf,MS_onedim2013,MS2015_Gaussian,Ery2016,wang2016}.
	The algorithmic convergence rates of mean shift algorithms with Gaussian and Epanechnikov kernels are generally linear, except for some extreme values of the bandwidth \citep{MS_EM2007,huang2018convergence}. It can be improved to be superlinear by dynamically updating the data set for estimating the density \citep{Acc_Dy_MS2006}. There are other methods to accelerate the mean shift algorithm by combining stochastic optimization with blurring or random sampling \citep{Fast_GBMS2006,GBMS2008, Stoc_GKD_Mode_seek2009,Fast_MS2016}. The mean shift algorithm with directional data was studied by \cite{Multi_Clu_Gene2005,DMS_topology2010,vMF_MS2010,MSBC_Dir2010,MSBC_Cir2012,MSC_Dir2014} in the last two decades. More generally, \cite{tuzel2005simultaneous,subbarao2006nonlinear,Nonlinear_MS_man2009,Intrinsic_MS2009,Semi_Intrinsic_MS2012,ashizawa2017least} proposed their mean shift algorithms on manifolds using logarithmic and exponential maps, heat kernel, or direct log-density estimation via least squares. These mean shift algorithms on general manifolds are applicable to directional data, though they are more complicated than our interested method.

	\textbf{\emph{Outline}.}	
	The remainder of the paper is organized as follows. Section~\ref{Sec:Prelim} reviews some background knowledge on directional KDEs and differential geometry, while Section~\ref{Sec:MS_Dir} provides a detailed derivation of the mean shift algorithm with directional data. 
	Section \ref{Sec:grad_Hess_consist} focuses on the statistical learning theory of the directional KDE; we
	formulate the gradient and Hessian estimators of directional KDEs and establish their pointwise and uniform consistency results
	as well as a mode consistency theory.
	Section \ref{Sec:Algo_conv} considers the computational learning theory of the directional mean shift algorithm;
	we study the ascending and converging properties of the algorithm.
	Simulation studies and applications to real-world data sets are unfolded in Section~\ref{Sec:Experiments}. Proofs of theorems and technical lemmas are deferred to Appendix~\ref{Appendix:proofs}. All the code for our experiments is available at \url{https://github.com/zhangyk8/DirMS}.
	
	\section{Preliminaries}
	\label{Sec:Prelim}
	
	This section is devoted to a brief review of the directional KDE and some technical concepts of differential geometry on $\Omega_q$.

	\subsection{Kernel Density Estimation with Directional Data}	\label{sec::KDE}
	
	Let $\bm{X}_1,...,\bm{X}_n \in \Omega_q\subset\R^{q+1}$ be a random sample generated from the underlying directional density function $f$ on $\Omega_q$ with $\int_{\Omega_q} f(\bm{x}) \, \omega_q(d\bm{x})=1$,
	where $\omega_q$ is the Lebesgue measure on $\Omega_q$. A well-known fact about the surface area of $\Omega_q$ is that
	\begin{equation}
	\label{surf_area}
	\bar{\omega}_q\equiv \omega_q\left(\Omega_q \right) = \frac{2\pi^{\frac{q+1}{2}}}{\Gamma(\frac{q+1}{2})} \quad \text{ for any integer } q\geq 1,
	\end{equation}
	where $\Gamma$ is the Gamma function defined as $\Gamma(z)=\int_0^{\infty} x^{z-1} e^{-x} dx$ with the real part of the complex integration variable $z$ (if applicable) being positive. 
	The directional KDE at point $\bm{x}\in \Omega_q$ is often written as \citep{KDE_Sphe1987,KDE_direct1988,Exact_Risk_bw2013}:
	\begin{equation}
	\hat{f}_h(\bm{x}) = \frac{c_{h,q}(L)}{n} \sum_{i=1}^n L\left(\frac{1-\bm{x}^T \bm{X}_i}{h^2} \right),
	\label{Dir_KDE}
	\end{equation}
	where $L$ is a directional kernel (a rapidly decaying function with nonnegative values and defined on $(-\delta_L,\infty) \subset \mathbb{R}$ for some constant $\delta_L >0$)\footnote{Normally, the kernel $L$ is only required to be defined on $[0,\infty)$. We extend its domain to $(-\delta_L,\infty) \subset \mathbb{R}$ so that the usual derivatives of $\hat{f}_h$ can be defined in $\mathbb{R}^{q+1}$ or at least a small neighborhood around $\Omega_q$ in $\mathbb{R}^{q+1}$ under some mild conditions on $L$. See Section~\ref{Diff_Geo_Review} and condition (D2') in Section~\ref{Sec:Consist_Assump} for details.}, $h>0$ is the bandwidth parameter, and $c_{h,q}(L)$ is a normalizing constant satisfying
	\begin{equation}
	\label{asym_norm_const}
	c_{h,q}(L)^{-1} = \int_{\Omega_q} L\left(\frac{1-\bm{x}^T \bm{y}}{h^2} \right) \omega_q(d\bm{y}) =h^q \lambda_{h,q}(L) \asymp h^q \lambda_q(L)
	\end{equation}
	with $\lambda_{h,q}(L) = \bar{\omega}_{q-1} \int_0^{2h^{-2}} L(r) r^{\frac{q}{2}-1} (2-rh^2)^{\frac{q}{2}-1} dr$ and $\lambda_q(L) = 2^{\frac{q}{2}-1} \bar{\omega}_{q-1} \int_0^{\infty} L(r) r^{\frac{q}{2}-1} dr$; see (a) of Lemma~\ref{integ_lemma} in Appendix~\ref{Appendix:Thm2_pf} for details.

	As in Euclidean kernel smoothing, bandwidth selection is a critical component in determining the performance of directional KDEs. There is extensive literature \citep{KDE_Sphe1987,KDE_direct1988,Auto_bw_cir2008,KDE_torus2011,Oliveira2012,Exact_Risk_bw2013,Nonp_Dir_HDR2020} that investigates various reliable bandwidth selection mechanisms. On the contrary, kernel selection is less crucial, and a popular candidate is the so-called von Mises kernel $L(r) = e^{-r}$. Its name originates from the famous $q$-von Mises-Fisher (vMF) distribution on $\Omega_q$, which is denoted by $\text{vMF}(\bm{\mu}, \nu)$ and has the density
	\begin{equation}
	\label{vMF_density}
	f_{\text{vMF}}(\bm{x};\bm{\mu},\nu) = C_q(\nu) \cdot \exp(\nu \bm{\mu}^T \bm{x}) \quad \text{ with } \quad C_q(\nu) = \frac{\nu^{\frac{q-1}{2}}}{(2\pi)^{\frac{q+1}{2}} \mathcal{I}_{\frac{q-1}{2}}(\nu)},
	\end{equation}
	where $\bm{\mu} \in \Omega_q$ is the directional mean, $\nu \geq 0$ is the concentration parameter, and 
	$$\mathcal{I}_{\alpha}(\nu) = \frac{\left(\frac{\nu}{2} \right)^{\alpha}}{\pi^{\frac{1}{2}} \Gamma\left(\alpha +\frac{1}{2} \right)} \int_{-1}^1 (1-t^2)^{\alpha -\frac{1}{2}} \cdot e^{\nu t} dt$$ 
	is the modified Bessel function of the first kind of order $\nu$. See Figure~\ref{fig:vMF_density} for contour plots of a von Mises-Fisher density and a mixture of von Mises-Fisher densities on $\Omega_2$, respectively.
	
	\begin{figure}
		\captionsetup[subfigure]{justification=centering}
		\begin{subfigure}[t]{.5\textwidth}
			\centering
			\includegraphics[width=0.9\linewidth]{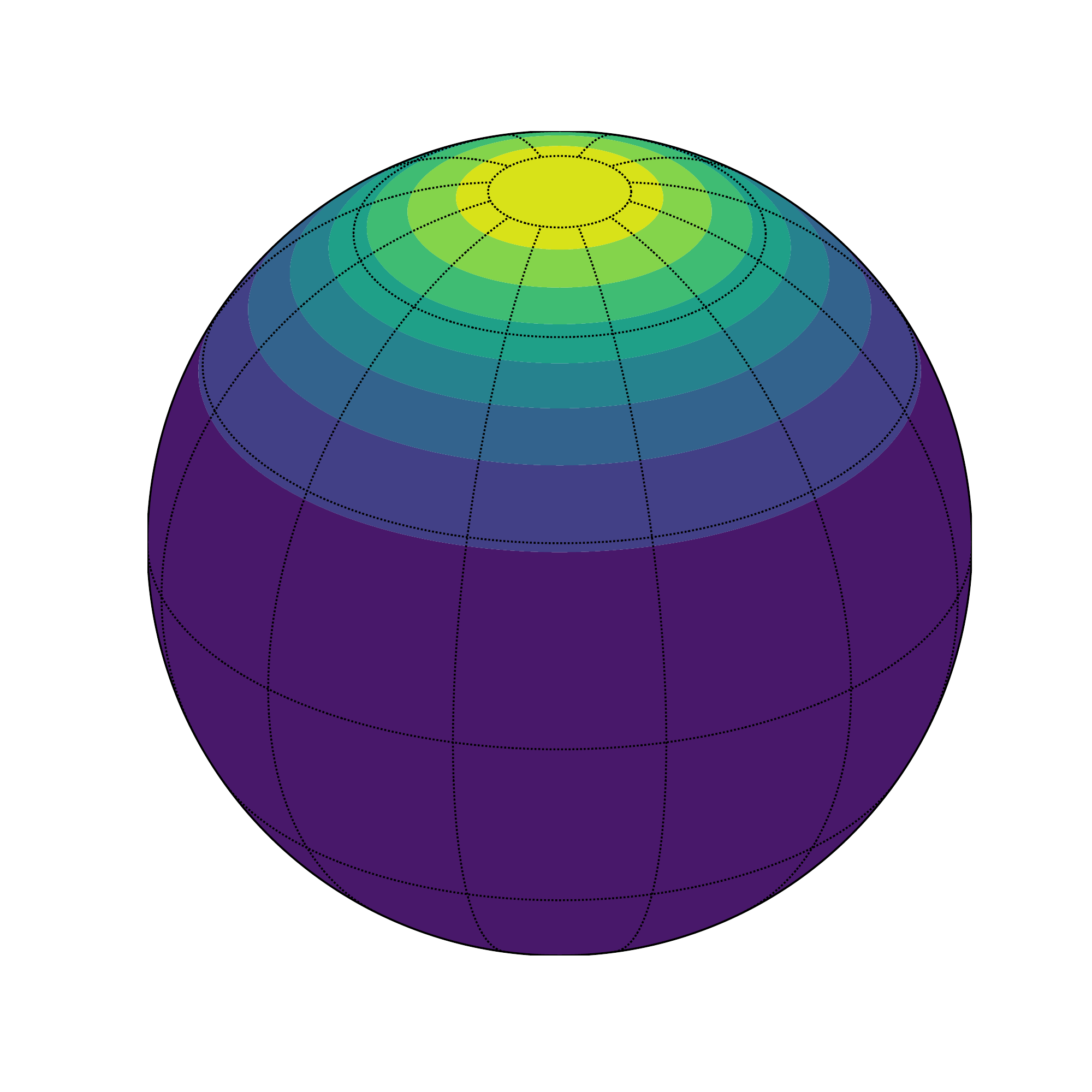}
			\caption{$f_{\text{vMF},2}(\bm{x};\bm{\mu}, \nu)$ with $\bm{\mu}=(0,0,1)$ and $\nu=4.0$}
		\end{subfigure}%
		\hspace{1em}
		\begin{subfigure}[t]{.5\textwidth}
			\centering
			\includegraphics[width=0.9\linewidth]{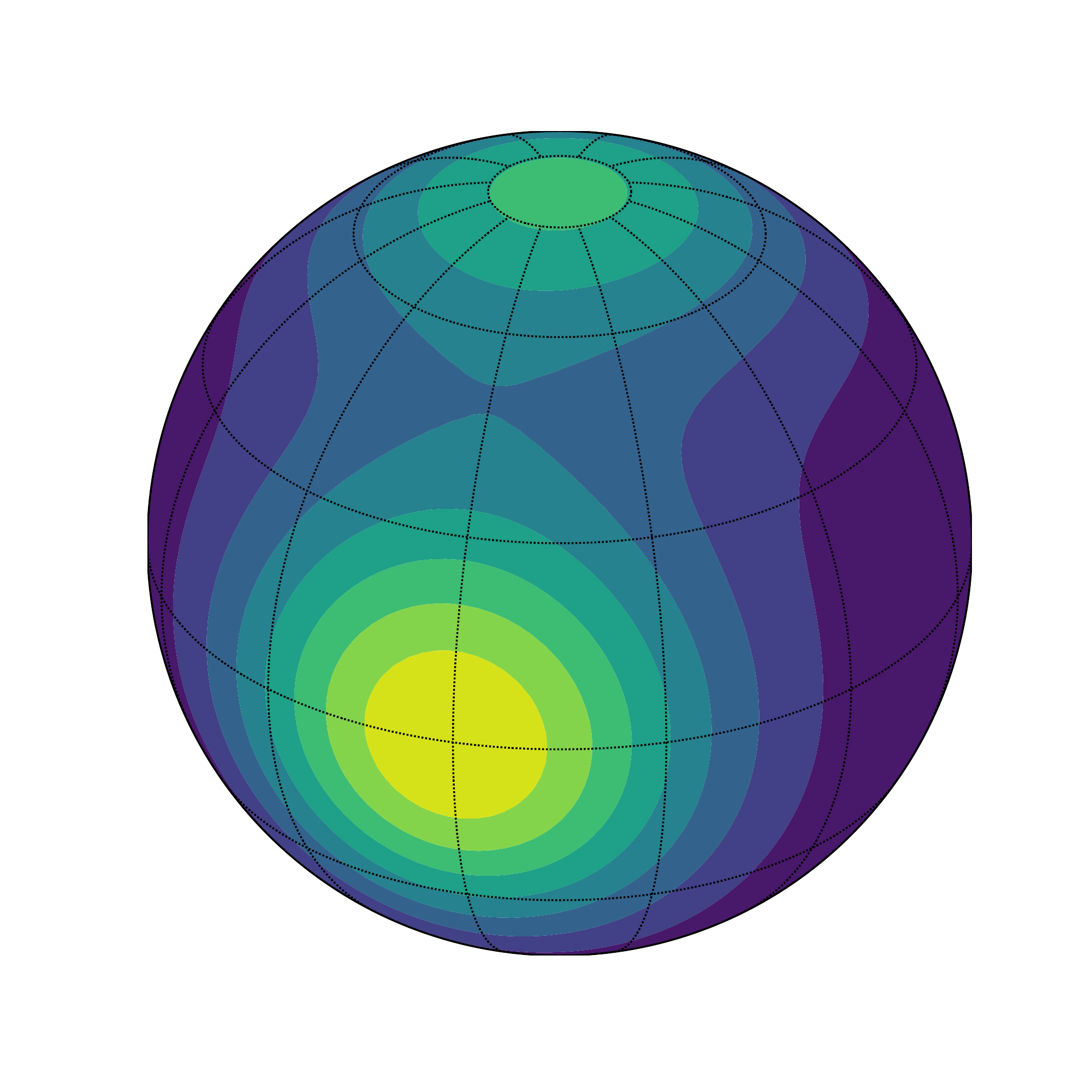}
			\caption{$\frac{2}{5}\cdot f_{\text{vMF},2}(\bm{x};\bm{\mu}_1, \nu_1) + \frac{3}{5} \cdot f_{\text{vMF},2}(\bm{x};\bm{\mu}_2, \nu_2)$ \\ with $\bm{\mu}_1=(0,0,1), \bm{\mu}_2=(1,0,0)$,\\ and $\nu_1=\nu_2=5.0$}
		\end{subfigure}
		\caption{Contour plots of a 2-von Mises-Fisher density and a mixture of 2-vMF densities}
		\label{fig:vMF_density}
	\end{figure}
	
	Using the von-Mises kernel, the directional KDE in (\ref{Dir_KDE}) becomes a mixture of $q$-von Mises-Fisher densities as follows:
	$$\hat{f}_h(\bm{x}) = \frac{1}{n} \sum_{i=1}^n f_{\text{vMF}}\left(\bm{x};\bm{X}_i, \frac{1}{h^2} \right)=\frac{1}{n(2\pi)^{\frac{q+1}{2}} \mathcal{I}_{\frac{q-1}{2}}(1/h^2) h^{q-1}} \sum_{i=1}^n \exp\left(\frac{\bm{x}^T\bm{X}_i}{h^2} \right).$$
	
	For a more detailed discussion of the statistical properties of the von Mises-Fisher distribution and directional KDE, we refer the interested reader to \cite{Mardia2000directional,spherical_EM,pewsey2021recent}.
	
	\subsection{Gradient and Hessian on a Sphere}
	\label{Diff_Geo_Review}
	
	For a function defined on a manifold, its gradient and Hessian are defined through the tangent space of the manifold. Whereas the formal definitions of the gradient and Hessian on a general manifold are often involved (see Appendix~\ref{sec::GH}), their representations are simple when the manifold is a (hyper)sphere $\Omega_q$.
	
	Let $T_{\bm x} \equiv T_{\bm x}(\Omega_q)$ be the tangent space of the sphere $\Omega_q$ at point ${\bm x}\in \Omega_q$.
	For the sphere $\Omega_q$, the tangent space has a simple representation in the ambient space $\mathbb{R}^{q+1}$ as follows:
	\begin{equation}
	T_{\bm x} \simeq \left\{\bm{v}\in \mathbb{R}^{q+1}: \bm{x}^T\bm{v}=0 \right\},
	\label{tangent_new}
	\end{equation}
	where $V_1 \simeq V_2$ signifies that the two vector spaces are isomorphic. In what follows, the expression $\bm{v}\in T_{\bm x}$ indicates that $\bm{v}$ is a vector tangent to $\Omega_q$ at $\bm{x}$.

    A \emph{geodesic} on $\Omega_q$ is a non-constant, parametrized curve $\gamma: [0,1] \to \Omega_q$ of constant speed and (locally) minimum length between two points on $\Omega_q$. It can be represented by part of a great circle on the sphere $\Omega_q$.
	For a smooth function $f: \Omega_q\to \R$, its differential in the (tangent) direction $\bm{v} \in T_{\bm x}$ with $\norm{\bm{v}}_2=1$ at point $\bm{x}\in\Omega_q$ is defined as follows. 
	We first define a geodesic curve $\alpha: (-\epsilon, \epsilon) \to \Omega_q$ with $\alpha(0)=\bm{x}$ and $\alpha'(0)=\bm{v}$.
	Then the differential (at $\bm x$) $df_{\bm{x}}:T_{\bm x}\rightarrow \R$ 
	is given by
	\begin{equation}
	df_{\bm{x}}(\bm{v}) = \frac{d}{dt} f(\alpha(t)) \Big|_{t=0}.
	\label{differential}
	\end{equation}
	With this, the \emph{Riemannian gradient} $\grad  f(\bm x)\in T_{\bm x}\subset \R^{q+1}$ is defined as 
	\begin{equation}
	df_{\bm{x}}(\bm{v}) =\langle \grad  f(\bm x), \bm{v}\rangle= \bm v^T \grad  f(\bm x).
	\label{grad_new}
	\end{equation}
	
	The Riemannian Hessian $\cH f(\bm x)\in T_{\bm x} \times T_{\bm x}$ is the second derivative of $f$ within the tangent space $T_{\bm x}$. We characterize its matrix representation as follows.
	Let ${\bm v},{\bm u}\in T_{\bm x}\subset\R^{q+1}$ be two unit vectors inside the tangent space $T_{\bm x}$.
	We consider two geodesic curves
	$\alpha,\beta : (-\epsilon, \epsilon) \to \Omega_q$ with $\alpha(0) = \beta(0)=\bm{x}$ and $\alpha'(0)=\bm{v}$ and $\beta'(0) = {\bm u}$.
	We define a second-order differential as
	$$d^2f_{\bm{x}}(\bm{v}, \bm{u}) = \frac{d}{dt} df_{\beta(t)}\left(\alpha'(t) \right)\Big|_{t=0}$$
	and the \emph{Riemannian Hessian} $\cH f(\bm x)$ is a $(q+1)\times (q+1)$ matrix satisfying 
	\begin{equation}
	d^2f_{\bm{x}}(\bm{v}, \bm{u}) = \langle\grad \langle \grad f, \bm{v}\rangle(\bm x) , \bm{u}\rangle=  \bm v^T \cH f(\bm x) \bm u
	\label{hess_new}
	\end{equation}
	and belongs to $T_{\bm x} \times T_{\bm x}$. 
	To ensure that $\cH f(\bm x)$ belongs to $T_{\bm x} \times T_{\bm x}$,
	it has to satisfy
	\begin{equation}
	\cH f(\bm x) = (I_{q+1}-\bm x \bm x^T)\cH f(\bm x) = \cH f(\bm x) (I_{q+1} - \bm x \bm x^T),
	\label{projection}
	\end{equation}
	where $I_{q+1}$ is the $(q+1)\times (q+1)$ identity matrix and $(I_{q+1}-\bm x \bm x^T)$ is a projection matrix onto the tangent space $T_{\bm x}$.
	Note that $d^2f_{\bm{x}}(\bm{v}, \bm{u}) = d^2f_{\bm{x}}(\bm{u}, \bm{v})$ can be easily verified.

	Although \eqref{grad_new} and \eqref{hess_new} define the Riemannian gradient and Riemannian Hessian on a sphere, 
	it is unclear how they are related to the total gradient operator $\nabla$, where $\nabla g(\bm x)\in\R^{q+1}$ and the $\ell$-th component is 
	$$[\nabla g(\bm x)]_\ell = \frac{d g(\bm x)}{dx_\ell}$$
	for any differentiable function $g:\R^{q+1}\rightarrow \R$.
	Whereas the total gradient $\nabla$ cannot be applied to a directional density (because it is only supported on $\Omega_q$),
	the directional KDE $\hat f_h$ 	is well-defined outside of $\Omega_q$ (after smoothly extending the domain of the kernel $L$ from $[0,\infty)$ to $\mathbb{R}$), and
	its total gradient $\nabla \hat f_h(\bm x) \in\R^{q+1}$ can be defined for any point $\bm x\in \R^{q+1}$. 
	
	\begin{figure}
		\centering
		\includegraphics[width=0.8\linewidth]{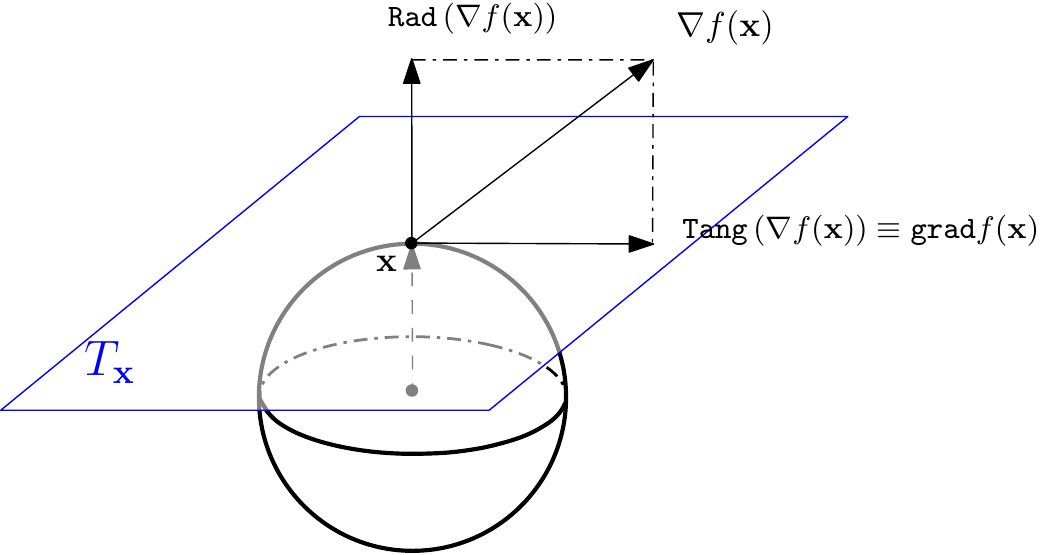}
		\caption{Visualization of a differential of the directional density $f$ on the unit sphere and its gradient}
		\label{fig:grad_proj}
	\end{figure}

	To associate the total gradient with the Riemannian gradient, we consider the following construction. Assume tentatively that $f$ is well-defined and smooth in $\R^{q+1}\backslash\{\bm 0\}$, not limited to $\Omega_q$.
	In this case, $\nabla f(\bm x)$ is well-defined $\R^{q+1}\backslash\{\bm 0\}$ and all subsequent derivations can also be applied to the directional KDE $\hat f_h$.
	For any point $\bm{x} \in \Omega_q$ and unit vector $\bm{v} \in T_{\bm{x}}$, we define a geodesic curve $\alpha: (-\epsilon, \epsilon) \to \Omega_q$ with $\alpha(0)=\bm{x}$ and $\alpha'(0)=\bm{v}$. Then, a differential of $f$ at $\bm{x}\in \Omega_q$ is a linear map characterized by
	\begin{equation*}
%	\label{diff_Dir}
	df_{\bm{x}}(\bm{v}) = \frac{d}{dt} f(\alpha(t)) \Big|_{t=0} = \nabla f(\alpha(t))^T \alpha'(t) \Big|_{t=0} = \nabla f(\bm{x})^T \alpha'(0) = \nabla f(\bm{x})^T \bm{v}
	\end{equation*}
	for any given $\bm{v} \in T_{\bm{x}}$. Thus, by the definition of the Riemannian gradient in \eqref{grad_new}, 
	$$
	df_{\bm{x}}(\bm{v})  = {\bm v}^T \grad f({\bm x})  = \nabla f(\bm{x})^T \bm{v} = \Tang\left(\nabla f(\bm{x})\right)^T \bm{v},$$
	and we conclude that 
	\begin{equation}
	\grad f({\bm x}) \equiv \Tang\left(\nabla f(\bm{x}) \right) = \left(I_{q+1}-\bm{x}\bm{x}^T \right)\nabla f(\bm{x}),
	\label{tangent}
	\end{equation} 
	which is the tangent component of the total gradient $\nabla f(\bm{x})$. 
	That is, the Riemannian gradient is the same as the tangent component of the total gradient.
	In addition, we can define the radial component of the total gradient as
	\begin{equation}
	\Rad(\nabla f(\bm{x})) = \nabla f(\bm{x}) - \Tang\left(\nabla f(\bm{x}) \right)  = \bm{x}\bm{x}^T \nabla f(\bm{x}).
	\label{radial}
	\end{equation}
	See Figure~\ref{fig:grad_proj} for a graphical illustration. 
	
	In the same context, we use the fact that $\alpha''(0) = -\bm x$ for the geodesic curve $\alpha$ and deduce that for any unit vector $\bm v\in T_{\bm x}\subset\R^{q+1}$,
	\begin{align}
	\label{Hess_curve}
	\begin{split}
	\bm v^T\mathcal{H} f(\bm x) \bm v &= \frac{d^2}{dt^2} f(\alpha(t)) \Big|_{t=0}\\
	&= \frac{d}{dt} \left[\nabla f(\alpha(t))^T \alpha'(t) \right]\Big|_{t=0}\\
	&\left(= \left[\sum_{i=1}^{q+1} \sum_{j=1}^{q+1} \frac{\partial^2 }{\partial x_i \partial x_j} f(\alpha(t)) \cdot \alpha_i'(t)\alpha_j'(t) + \sum_{i=1}^{q+1} \frac{\partial}{\partial x_i} f(\alpha(t)) \cdot \alpha_i''(t) \right]\Bigg|_{t=0}
	\right)\\
	&= \alpha'(0)^T \nabla\nabla f(\alpha(0)) \alpha'(0) + \nabla f(\alpha(0))^T \alpha''(0)\\
	&= \bm{v}^T \nabla\nabla f(\bm{x}) \bm{v} + \nabla f(\bm{x})^T \alpha''(0)\\
	&= \bm{v}^T (\nabla\nabla f(\bm{x}) - \nabla f(\bm{x})^T \bm x I_{q+1}) \bm{v}.
	\end{split}
	\end{align}
	One may conjecture that $(\nabla\nabla f(\bm{x}) - \nabla f(\bm{x})^T \alpha''(0))$ is 
	the Riemannian Hessian matrix.
	However, it does not satisfy the projection condition in Equation \eqref{projection}.
	To this end, we select
	\begin{equation}
	\label{Hess_Dir}
	\mathcal{H} f(\bm x) = (I_{q+1} - \bm{x}\bm{x}^T) \left[\nabla\nabla f(\bm{x}) - \nabla f(\bm{x})^T \bm{x} I_{q+1} \right](I_{q+1} - \bm{x}\bm{x}^T).
	\end{equation}
	One can verify that the Hessian matrix in \eqref{Hess_Dir}
	satisfies both \eqref{hess_new} and \eqref{projection};
	thus, it characterizes the relationship between the Riemannian Hessian and total gradient operator. More importantly, the Hessian matrix in \eqref{Hess_Dir} is indeed the Riemannian Hessian on $\Omega_q$. Detailed definitions of Riemannian Hessians can be found in Section 2 and 4.2 of \cite{Extrinsic_Look_Riem_Manifold}.

	\section{Mean Shift Algorithm with Directional Data}
	\label{Sec:MS_Dir}
	
	In this section, we present a detailed derivation of the mean shift algorithm with directional data. Given the directional KDE $\hat{f}_h(\bm{x})$ in \eqref{Dir_KDE}, \cite{vMF_MS2010,MSC_Dir2014} introduced a Lagrangian multiplier to maximize $\hat{f}_h(\bm{x})$ under the constraint $\bm{x}^T\bm{x}=1$ and derived the directional mean shift algorithm. To make a better comparison with the standard mean shift algorithm with Euclidean data, we provide an alternative derivation. 
	
	Given a Euclidean KDE of the form $\hat{p}_n(\bm{x}) = \frac{c_{k,d}}{nh^d} \sum\limits_{i=1}^n k\left(\norm{\frac{\bm{x}-\bm{X}_i}{h}}_2^2 \right)$ with a differentiable kernel profile $k:[0,\infty) \to [0,\infty)$, its (total) gradient has the following decomposition:
	\begin{equation}
		\label{KDE_grad_Euclidean}
		\nabla \hat{p}_n(\bm{x}) = \underbrace{\frac{2c_{k,d}}{nh^{d+2}} \left[\sum_{i=1}^n g\left(\norm{\frac{\bm{x}-\bm{X}_i}{h}}_2^2 \right) \right]}_{\text{term 1}} \underbrace{\left[\frac{\sum_{i=1}^n \bm{X}_i g\left(\norm{\frac{\bm{x}-\bm{X}_i}{h}}_2^2 \right)}{\sum_{i=1}^n g\left(\norm{\frac{\bm{x}-\bm{X}_i}{h}}_2^2 \right)} -\bm{x}\right]}_{\text{term 2}},
	\end{equation}
	where $g(x)=-k'(x)$ is the derivative of the selected kernel profile. As noted by \cite{MS2002}, the first term is proportional to the density estimate at $\bm{x}$ with the ``kernel'' $G(\bm{x})=c_{g,d}\cdot g(\norm{\bm{x}}_2^2)$, and the second term is the so-called \emph{mean shift} vector, which points toward the direction of maximum increase in the density estimator $\hat{p}_n$. 
	Thus, the standard mean shift algorithm with Euclidean data translates each query point according to the corresponding mean shift vector, which leads to a converging path to a local mode of $\hat{p}_n$ under some conditions \citep{MS2007_pf,MS2015_Gaussian,Ery2016}. 
	
	The key insight in our derivation of the directional mean shift algorithm is the following alternative representation of the directional KDE as:
	\begin{equation}
	\label{Dir_KDE2}
	\tilde{f}_h(\bm{x}) = \frac{c_{h,q}(L)}{n} \sum_{i=1}^n L\left(\frac{1}{2} \norm{\frac{\bm{x} -\bm{X}_i}{h}}_2^2 \right),
	\end{equation}
	given a directional random sample $\bm{X}_1,...,\bm{X}_n \in \Omega_q$. Recall that the original directional KDE in \eqref{Dir_KDE} is $\hat{f}_h(\bm{x}) = \frac{c_{h,q}(L)}{n} \sum\limits_{i=1}^n L\left(\frac{1-\bm{x}^T \bm{X}_i}{h^2} \right). $
	Both $\hat f_h$ and $\tilde f_h$ can be defined on any point in $\R^{q+1}\backslash \{\bm 0\}$.
	Although $\hat f_h(\bm{x}) \neq \tilde f_h(\bm{x})$ for $\bm{x}\notin \Omega_q$, their function values are identical on the sphere; that is, 
	\begin{equation}
	\hat f_h(\bm{x}) = \tilde f_h(\bm{x}),\quad \forall \bm{x}\in \Omega_q
	\label{equiv}
	\end{equation}	
	due to the fact that 
	$\frac{1}{2}\norm{\bm{x}-\bm{X}_i}_2^2 = 1-\bm{x}^T\bm{X}_i$ for any $\bm{x}\in \Omega_q$.

	Since the two directional KDEs are the same on $\Omega_q$,
	either of them can be used to express our density estimator. 
	The power of the expression $\tilde{f}_h$ is that its total gradient has a similar decomposition as the total gradient of the Euclidean KDE (cf. \eqref{KDE_grad_Euclidean}):
	\begin{align}
	\label{Dir_KDE_grad}
	\begin{split}
	\nabla \tilde{f}_h(\bm{x}) &= \frac{c_{h,q}(L)}{nh^2} \sum_{i=1}^n (\bm{x} -\bm{X}_i) \cdot L'\left(\frac{1}{2} \norm{\frac{\bm{x} -\bm{X}_i}{h}}_2^2 \right) \\
	&= \underbrace{\frac{c_{h,q}(L)}{nh^2} \left[\sum_{i=1}^n -L'\left(\frac{1}{2} \norm{\frac{\bm{x} -\bm{X}_i}{h}}_2^2 \right) \right]}_{\text{term 1}} \cdot \underbrace{\left[\frac{\sum_{i=1}^n \bm{X}_i \cdot L'\left(\frac{1}{2} \norm{\frac{\bm{x} -\bm{X}_i}{h}}_2^2 \right)}{\sum_{i=1}^n L'\left(\frac{1}{2} \norm{\frac{\bm{x} -\bm{X}_i}{h}}_2^2 \right)} -\bm{x}\right]}_{\text{term 2}}.
	\end{split}
	\end{align}

	Similar to the density gradient estimator with Euclidean data (cf. Equation~\eqref{KDE_grad_Euclidean}), the first term of the product in (\ref{Dir_KDE_grad}) can be viewed as a proportional form of the directional density estimate at $\bm{x}$ with ``kernel'' $G(r) = -L'(r)$:
	\begin{align}
	\label{prod1}
	\begin{split}
	\tilde{f}_{h,G}(\bm{x}) &= \frac{c_{h,q}(G)}{n} \sum_{i=1}^n -L'\left(\frac{1}{2} \norm{\frac{\bm{x} -\bm{X}_i}{h}}_2^2 \right) = \frac{c_{h,q}(G)}{n} \sum_{i=1}^n -L'\left(\frac{1-\bm{x}^T\bm{X}_i}{h^2} \right)
	\end{split}
	\end{align}
	given that $-L'(r)$ is non-negative on $[0,\infty)$. Some commonly used kernel functions, such as the von-Mises kernel $L(r)=e^{-r}$, easily satisfy this condition. The second term of the product in (\ref{Dir_KDE_grad}) is indeed the \emph{directional mean shift} vector
	\begin{equation}
	\label{mean_shift}
	\Xi_h(\bm{x}) =\frac{\sum_{i=1}^n \bm{X}_i  L'\left(\frac{1}{2} \norm{\frac{\bm{x} -\bm{X}_i}{h}}_2^2 \right)}{\sum_{i=1}^n L'\left(\frac{1}{2} \norm{\frac{\bm{x} -\bm{X}_i}{h}}_2^2 \right)} -\bm{x}= \frac{\sum_{i=1}^n \bm{X}_i L'\left(\frac{1-\bm{x}^T\bm{X}_i}{h^2} \right)}{\sum_{i=1}^n L'\left(\frac{1-\bm{x}^T\bm{X}_i}{h^2} \right)} -\bm{x},
	\end{equation}
	which is the difference between a weighted sample mean with weights $\frac{L'\left(\frac{1-\bm{x}^T \bm{X}_i}{h^2} \right)}{\sum\limits_{i=1}^n  L'\left(\frac{1-\bm{x}^T \bm{X}_i}{h^2} \right)}$, $i=1,...,n$, and $\bm{x}$, the current query point of the directional density estimation. It is worth mentioning that these weights are strictly positive when the von-Mises kernel $L(r)=e^{-r}$ is applied. From Equations (\ref{prod1}) and (\ref{mean_shift}), the total gradient estimator at $\bm{x}$ becomes
	$$\nabla \tilde{f}_h(\bm{x}) = \frac{c_{h,q}(L)}{c_{h,q}(G) h^2} \cdot \tilde{f}_{h,G}(\bm{x}) \cdot \Xi_h(\bm{x}),$$
	yielding
	$$\Xi_h(\bm{x}) = \frac{c_{h,q}(G) h^2}{c_{h,q}(L)} \cdot \frac{\nabla \tilde{f}_h(\bm{x})}{\tilde{f}_{h,G}(\bm{x})}.$$
	As is illustrated in \eqref{tangent}, the total gradient of the directional KDE at $\bm{x}$, $\nabla \tilde{f}_h(\bm{x})$, becomes the Riemannian gradient of $\tilde{f}_h(\bm{x})=\hat{f}_h(\bm{x})$ on $\Omega_q$ after being projected onto the tangent space $T_{\bm{x}}$. This suggests that the directional mean shift vector $\Xi_h(\bm{x})$, which is parallel to the total gradient of $\tilde{f}_h$ at $\bm{x}$, points in the direction of maximum increase in the estimated density $\tilde{f}_h$ after being projected onto the tangent space $T_{\bm{x}}$. 
	However, due to the manifold structure of $\Omega_q$, translating a point $\bm{x} \in \Omega_q$ in the mean shift direction $\Xi_h(\bm{x})$ deviates the point from $\Omega_q$. 
	We thus project the translated point $\bm x+\Xi_h(\bm{x})$ onto $\Omega_q$ by a simple standardization: $\bm x+\Xi_h(\bm{x}) \mapsto \frac{\bm x+\Xi_h(\bm{x})}{\norm{\bm x+\Xi_h(\bm{x})}_2}$.
	In summary, starting at point $\bm x$, the directional mean shift algorithm moves this point to a new location $\frac{\bm x+\Xi_h(\bm{x})}{\norm{\bm x+\Xi_h(\bm{x})}_2}$.
	This movement creates a path leading to a local mode of the estimated directional density under suitable conditions (Theorems \ref{MS_asc} and \ref{MS_conv}).
	
	\begin{algorithm}[t]
		\caption{Mean Shift Algorithm with Directional Data}
		\label{Algo:MS}
		\begin{algorithmic}
			\State \textbf{Input}: 
			\begin{itemize}
				\item Directional data sample $\bm{X}_1,...,\bm{X}_n \sim f(\bm{x})$ on $\Omega_q$.
				\item The smoothing bandwidth $h$.
				\item An initial point $\hat{\bm{y}}_0 \in \Omega_q$ and the precision threshold $\epsilon >0$.
			\end{itemize} 
			\While {$1- \hat{\bm{y}}_{s+1}^T \hat{\bm{y}}_s > \epsilon$}
			\State 
			\begin{equation}
			\hat{\bm{y}}_{s+1} = -\frac{\sum_{i=1}^n \bm{X}_i L'\left(\frac{1-\hat{\bm{y}}_s^T\bm{X}_i}{h^2} \right) }{\norm{\sum_{i=1}^n \bm{X}_i L'\left(\frac{1-\hat{\bm{y}}_s^T\bm{X}_i}{h^2} \right)}_2}
			\label{fix_point_eq}
			\end{equation}
			\EndWhile
			\State \textbf{Output}: A candidate local mode of directional KDE, $\hat{\bm{y}}_s$.
		\end{algorithmic}
	\end{algorithm}
	
	We can encapsulate the directional mean shift algorithm into a single fixed-point equation. Let $\left\{\hat{\bm{y}}_s\right\}_{s=0}^{\infty} \subset \Omega_q$ denote the path of successive points defined by the directional mean shift algorithm, where $\hat{\bm{y}}_0$ is the initial point of the iteration. Translating the query point $\hat{\bm{y}}_s$ by the directional mean shift vector \eqref{mean_shift} at step $s$ leads to
	$$\Xi_h\left(\hat{\bm{y}}_s \right) + \hat{\bm{y}}_s = \frac{\sum_{i=1}^n \bm{X}_i L'\left(\frac{1-\hat{\bm{y}}_s^T\bm{X}_i}{h^2} \right)}{\sum_{i=1}^n L'\left(\frac{1-\hat{\bm{y}}_s^T\bm{X}_i}{h^2} \right)}.$$	
	When $L(r)$ is decreasing, $L'(r)$ is non-positive on $[0,\infty)$ and
	$$\norm{\Xi_h\left(\hat{\bm{y}}_s \right) + \hat{\bm{y}}_s}_2 = \frac{\norm{\sum_{i=1}^n \bm{X}_i L'\left(\frac{1-\hat{\bm{y}}_s^T\bm{X}_i}{h^2} \right) }_2}{\left|\sum_{i=1}^n L'\left(\frac{1-\hat{\bm{y}}_s^T\bm{X}_i}{h^2} \right) \right|} = -\frac{\norm{\sum_{i=1}^n \bm{X}_i L'\left(\frac{1-\hat{\bm{y}}_s^T\bm{X}_i}{h^2} \right) }_2}{\sum_{i=1}^n L'\left(\frac{1-\hat{\bm{y}}_s^T\bm{X}_i}{h^2} \right)}$$
	given that $\sum\limits_{i=1}^n L'\left(\frac{1-\bm{y}_s^T\bm{X}_i}{h^2} \right) \neq 0$. (Here $L'(r)$ can be replaced by subgradients at non-differentiable points of $L$. See also Remark~\ref{Diff_Relax}.) Again, many commonly used kernel functions, such as the von-Mises kernel $L(r)=e^{-r}$, have nonzero derivatives on $[0,\infty)$ and satisfy this mild condition. Therefore,
	\begin{equation*}
	% \label{fix_point}
	\hat{\bm{y}}_{s+1} = \frac{\Xi_h\left(\hat{\bm{y}}_s \right) + \hat{\bm{y}}_s}{\norm{\Xi_h\left(\hat{\bm{y}}_s \right) + \hat{\bm{y}}_s}_2} = -\frac{\sum_{i=1}^n \bm{X}_i L'\left(\frac{1-\hat{\bm{y}}_s^T\bm{X}_i}{h^2} \right) }{\norm{\sum_{i=1}^n \bm{X}_i L'\left(\frac{1-\hat{\bm{y}}_s^T\bm{X}_i}{h^2} \right)}_2}
	\end{equation*}
	is the resulting fixed-point equation for $s=0,1,...$, whose right-hand side is a standardized weighted sample mean at $\hat{\bm{y}}_s$ computed with ``kernel'' $G(r)=-L'(r)$. The entire mean shift algorithm with directional data is summarized in Algorithm \ref{Algo:MS} (see also Figure~\ref{fig:MS_One_Step} for a graphical illustration).
		
	\begin{figure}
		\centering
		\includegraphics[width=0.7\linewidth]{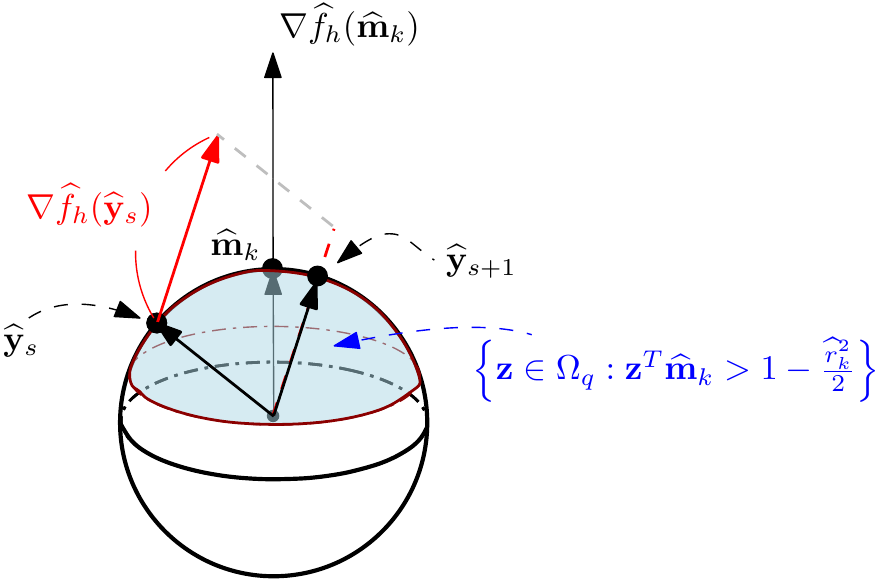}
		\caption{Illustration of one-step iteration of Algorithm \ref{Algo:MS}}
		\label{fig:MS_One_Step}
	\end{figure}
	
	Analogous to the mean shift algorithm with Euclidean data, Algorithm \ref{Algo:MS} can be leveraged for mode seeking and clustering with directional data.
	We derive statistical and computational learning theory for mode seeking in Sections \ref{Sec:grad_Hess_consist} and \ref{Sec:Algo_conv}. 
	For clustering, we demonstrate with both simulated and real-world data sets that the algorithm can be used to cluster directional data in Section~\ref{Sec:Experiments}.
	It should also be noted that the directional mean shift algorithm can be viewed as a gradient ascent method on $\Omega_q$ with an adaptive step size; see Section~\ref{sec:linear} for details.
	
	More importantly, similar to the standard mean shift algorithm with Euclidean data, the directional mean shift algorithm has several advantages over a regular gradient ascent method.
	First, the directional mean shift algorithm requires no tuning of the step size parameter, yet exhibits mathematical simplicity when it is written as the fixed-point iteration \eqref{fix_point_eq}. Second, the algorithm does not need to estimate the normalizing constant $c_{h,q}(L)$ of the directional KDE in its application. Specifically, in order to identify local modes of the directional KDE using our algorithm, it is only necessary to specify the directional kernel $L$ up to a constant. This avoids additional computational cost in estimating the normalizing constant $c_{h,q}(L)$ for the kernel, because the constant $c_{h,q}(L)$ often involves complicated functions for high dimensional directional data. For instance, estimating the normalizing constant of the von Mises kernel involves an approximation of a modified Bessel function of the first kind, though several efficient algorithms have been developed; see, for instance, \cite{Sra2011}.

	\section{Statistical Learning Theory of Directional KDE and its Derivatives}
	\label{Sec:grad_Hess_consist}

	Because the (directional) mean shift algorithm is inspired by a gradient ascent method, we study the gradient and Hessian systems of the two estimators $\hat f_h$ and $\tilde f_h$.

	\subsection{Gradient and Hessian of Directional KDEs}
	\label{Sec:grad_Hess_Dir}
 	
	We have demonstrated that it is valid to deduce two mathematically equivalent directional KDEs (\ref{Dir_KDE}) and (\ref{Dir_KDE2}) for estimating the true directional density $f$. 
	Somewhat surprisingly, the corresponding total gradients are different in general. The total gradient of $\tilde f_h$ is 
	\begin{equation}
	\label{Dir_KDE_grad1}
	\nabla \tilde{f}_h(\bm{x}) = \frac{c_{h,q}(L)}{nh^2} \sum_{i=1}^n (\bm{x} -\bm{X}_i) \cdot L'\left(\frac{1}{2} \norm{\frac{\bm{x} -\bm{X}_i}{h}}_2^2 \right),
	\end{equation}
	while the total gradient of $\hat{f}_h$ is
	\begin{equation}
	\label{Dir_KDE_grad2}
	\nabla \hat{f}_h(\bm{x}) = -\frac{c_{h,q}(L)}{nh^2} \sum\limits_{i=1}^n  \bm{X}_i  L'\left(\frac{1-\bm{x}^T\bm{X}_i}{h^2} \right).
	\end{equation}
	Although the total gradients $\nabla \tilde{f}_h$ and $\nabla \hat{f}_h$ have different values even on $\Omega_q$, they both play a vital role in the directional mean shift algorithm (Algorithm~\ref{Algo:MS}). On the one hand, we have argued in Section \ref{Sec:MS_Dir} that $\nabla \tilde{f}_h(\bm{x})$ has a similar decomposition as the total gradient of the Euclidean KDE, and derived Algorithm~\ref{Algo:MS} based on $\nabla \tilde{f}_h(\bm{x})$. On the other hand, given the form of $\nabla \hat{f}_h(\bm{x})$ in \eqref{Dir_KDE_grad2}, the fixed-point equation \eqref{fix_point_eq} in Algorithm~\ref{Algo:MS} can be written as
	\begin{equation}
	\label{fix_point_grad}
	\hat{\bm{y}}_{s+1} = \frac{\nabla \hat{f}_h(\hat{\bm{y}}_s)}{\norm{\nabla \hat{f}_h(\hat{\bm{y}}_s)}_2}.
	\end{equation}
	As argued in Section \ref{Diff_Geo_Review} and \eqref{radial}, any total gradient at $\bm{x} \in \Omega_q$ can be decomposed into radial and tangent components. Therefore, the total gradient $\nabla \tilde{f}_h(\bm{x})$ is decomposed as
	\begin{align*}
	\nabla \tilde{f}_h(\bm{x}) &= \bm{x}\bm{x}^T \nabla \tilde{f}_h(\bm{x}) + \left(I_{q+1} - \bm{x}\bm{x}^T \right)\nabla \tilde{f}_h(\bm{x})\\
	&= \frac{c_{h,q}(L)}{nh^2} \sum_{i=1}^n \bm{x} \left(1 -\bm{x}^T\bm{X}_i \right) L'\left(\frac{1 -\bm{x}^T \bm{X}_i}{h^2} \right)\\
	&\quad + \frac{c_{h,q}(L)}{nh^2} \sum_{i=1}^n \left(\bm{x} \cdot \bm{x}^T \bm{X}_i  - \bm{X}_i \right) L'\left(\frac{1-\bm{x}^T\bm{X}_i}{h^2} \right)\\
	&\equiv \Rad\left(\nabla \tilde{f}_h(\bm{x}) \right) + \Tang\left(\nabla \tilde{f}_h(\bm{x}) \right),
	\end{align*}
	where $\Rad$ and $\Tang$ are the radial and tangent components of the total gradient, as in \eqref{radial} and \eqref{tangent}. Similarly, we decompose $\nabla \hat{f}_h(\bm{x})$ as 
	\begin{align*}
	\nabla \hat{f}_h(\bm{x}) &= \bm{x}\bm{x}^T \nabla \hat{f}_h(\bm{x}) + \left(I_{q+1} -\bm{x}\bm{x}^T \right) \nabla \hat{f}_h(\bm{x})\\
	&= -\frac{c_{h,q}(L)}{nh^2} \sum_{i=1}^n \bm{x}\bm{x}^T\bm{X}_i \cdot L'\left(\frac{1 -\bm{x}^T \bm{X}_i}{h^2} \right) \\
	&\quad + \frac{c_{h,q}(L)}{nh^2} \sum_{i=1}^n \left(\bm{x} \cdot \bm{x}^T \bm{X}_i  - \bm{X}_i \right) \cdot L'\left(\frac{1-\bm{x}^T\bm{X}_i}{h^2} \right)\\
	&\equiv \Rad\left(\nabla \hat{f}_h(\bm{x}) \right) + \Tang\left(\nabla \hat{f}_h(\bm{x}) \right).
	\end{align*}
	Therefore, the difference between the two total gradients $\nabla \tilde{f}_h(\bm{x})$ and $\nabla \hat{f}_h(\bm{x})$ is
	\begin{equation}
	\label{tot_grad_diff}
	\nabla \tilde{f}_h(\bm{x}) - \nabla \hat{f}_h(\bm{x}) = \frac{c_{h,q}(L)}{nh^2} \sum_{i=1}^n  L'\left(\frac{1 -\bm{x}^T \bm{X}_i}{h^2} \right) \cdot \bm{x},
	\end{equation}
	which is parallel to the radial direction $\bm{x}$. This implies that given kernel $L$, the Riemannian gradients of the two estimators are the same, that is,
	\begin{align}
	\label{tang_grad}
	\begin{split}
	\grad \hat{f}_h(\bm{x}) \equiv \Tang\left(\nabla \hat{f}_h(\bm{x}) \right) &= \nabla \hat{f}_h(\bm{x}) - \left[\bm{x}^T \nabla \hat{f}_h(\bm{x}) \right]\cdot \bm{x}\\
	&= \frac{c_{h,q}(L)}{nh^2} \sum_{i=1}^n \left(\bm{x}^T \bm{X}_i \cdot \bm{x} - \bm{X}_i \right) \cdot L'\left(\frac{1-\bm{x}^T\bm{X}_i}{h^2} \right)\\
	&=\grad \tilde{f}_h(\bm{x}) \equiv \Tang\left(\nabla \tilde{f}_h(\bm{x}) \right).
	\end{split}
	\end{align}
	Later, we demonstrate in Theorems \ref{pw_conv_tang} and \ref{unif_conv_tang} that the Riemannian gradients of $\hat f_h$ and $\tilde f_h$ are consistent estimators of the Riemannian gradient of the underlying density $f$ that generates directional data. One can also deduce the same fixed-point equation \eqref{fix_point_eq} (or equivalently \eqref{fix_point_grad}) from the Riemannian/tangent gradient estimator $\grad \hat{f}_h(\bm{x})\equiv \Tang\left(\nabla \hat{f}_h(\bm{x}) \right)$, although the assumption on the directional estimated density $\hat{f}_h$ is stricter. See Appendix~\ref{Appendix:Tang_MS_DR} for detailed derivations.
	
	Having demonstrated that the Riemannian gradients of $\tilde f_h$ and $\hat f_h$ are identical, we now study the Riemannian Hessians of $\tilde f_h$ and $\hat f_h$.
	By \eqref{Hess_Dir}, the Riemannian Hessian of $\hat f_h$ is associated with the total gradient operator $\nabla$ via
	\begin{equation}
	\label{Hess_est}
	\mathcal{H} \hat f_h(\bm x) = (I_{q+1} - \bm{x}\bm{x}^T) \left[\nabla\nabla \hat f_h(\bm{x}) - \nabla \hat f_h(\bm{x})^T \bm{x} I_{q+1} \right](I_{q+1} - \bm{x}\bm{x}^T)
	\end{equation}
	and similarly for $\mathcal{H} \tilde f_h(\bm x)$. The following lemma shows that when a directional kernel $L$ is smooth, the two Riemannian Hessians are identical.
	
	\begin{lemma}
	\label{lem:Hessian}
	Assume that kernel $L$ is twice continuously differentiable.
	Then,
	$$\mathcal{H}\tilde f_h(\bm x)=\mathcal{H}\hat f_h(\bm x)$$
	for any point $\bm{x}\in\Omega_q$.
	\end{lemma}

	The proof of Lemma~\ref{lem:Hessian} can be found in Appendix~\ref{Appendix:lem1_pf}. As a result, we can compute the Riemannian Hessian estimator at point $\bm{x}\in \Omega_q$ 
	based on either $ \tilde{f}_h(\bm{x})$ or $ \hat{f}_h(\bm{x})$, which will produce the same expression. 
	Later, in Theorems \ref{pw_conv_tang} and \ref{unif_conv_tang}, we demonstrate that $\mathcal{H} \hat{f}_h(\bm x)$ is a (uniformly) consistent estimator of $\mathcal{H} f(\bm x)$ defined in (\ref{Hess_Dir}).

	\subsection{Assumptions}
	\label{Sec:Consist_Assump}

	To apply the total gradient operator $\nabla$ to a directional density $f$ that generates data, we extend it from $\Omega_q$ to $\mathbb{R}^{q+1} \setminus\{\bm{0}\}$ by defining $f(\bm{x}) \equiv f\left(\frac{\bm{x}}{||\bm{x}||_2} \right)$ for all $\bm{x}\in \mathbb{R}^{q+1} \setminus \{\bm{0} \}$. In this extension, we assume that the total gradient $\nabla f(\bm{x}) = \left(\frac{\partial f(\bm{x})}{\partial x_1},..., \frac{\partial f(\bm{x})}{\partial x_{q+1}} \right)^T$ and total Hessian matrix $\nabla\nabla f(\bm{x}) = \left(\frac{\partial^2 f(\bm{x})}{\partial x_i \partial x_j} \right)_{1\leq i,j \leq q+1}$ in $\mathbb{R}^{q+1}$ exist, and are continuous on $\mathbb{R}^{q+1} \setminus \{\bm{0} \}$ and square integrable on $\Omega_q$. 
	This extension has also been used by \cite{Zhao2001,Dir_Linear2013,Exact_Risk_bw2013}. 
	Note that the Riemannian gradient and Hessian are invariant under this extension.
	
	To establish the consistency results of gradient and Hessian estimators (cf. \eqref{tang_grad} and \eqref{Hess_est} or \eqref{Hess_KDE} in its explicit form), we consider the following assumptions.
	\begin{itemize}
		\item {\bf (D1)} Assume that the extended density function $f$ is at least three times continuously differentiable on $\mathbb{R}^{q+1}\setminus \left\{\bm{0}\right\}$ and that its derivatives are square integrable on $\Omega_q$.
		\item {\bf (D2)} Assume that $L: [0,\infty) \to [0,\infty)$ is a bounded and Riemann integrable function such that 
		$$0< \int_0^{\infty} L^k(r) r^{\frac{q}{2}-1} dr < \infty$$
		for all $q\geq 1$ and $k=1,2$.
		\item {\bf (D2')} Under (D2), we further assume that $L$ is a twice continuously differentiable function on $(-\delta_L,\infty) \subset \mathbb{R}$ for some constant $\delta_L>0$ such that 
		$$0< \int_0^{\infty} L'(r)^k r^{\frac{q}{2}-1} dr < \infty, \quad 0< \int_0^{\infty} L''(r)^k r^{\frac{q}{2}-1} dr < \infty$$
		for all $q\geq 1$ and $k=1,2$.
%		and 
%		\begin{equation}
%		\label{kernel_rate}
%		\int_{\epsilon^{-1}}^{\infty} \left|L'(r) \right| r^{\frac{q}{2}} dr = O(\epsilon)
%		\end{equation}
%		as $\epsilon \to 0$ for any $\epsilon >0$.
	\end{itemize}
	Here, conditions (D1) and (D2) are required for the consistency of the directional KDE \citep{KDE_Sphe1987,KLEMELA2000,Zhao2001,Dir_Linear2013,Exact_Risk_bw2013}. The stronger condition (D2') is imposed for the consistency of Riemannian gradient estimator $\grad \hat{f}_h(\bm{x}) \equiv \Tang\left(\nabla \hat{f}_h(\bm{x}) \right)$ and Hessian estimator $\mathcal{H} \hat{f}_h(\bm x)$. 
	The differentiability condition in (D2') can be relaxed so that $L$, after being smoothly extrapolated from $[0,\infty)$ to $(-\delta_L,\infty)$ for some constant $\delta_L >0$, is (twice) continuously differentiable except for a set of points with Lebesgue measure $0$ on $(-\delta_L,\infty)$.
	One can justify via integration by parts that many commonly used kernels, such as the von-Mises kernel $L(r)=e^{-r}$ or compactly supported kernels, satisfy condition (D2').
	
	Under conditions (D1) and (D2), the pointwise convergence rate of $\hat f_h$ is 
	$$\hat{f}_h(\bm{x}) - f(\bm{x}) = O(h^2) + O_P\left(\sqrt{\frac{1}{nh^q}} \right);$$
	see, for instance, \cite{KDE_Sphe1987, Zhao2001,Exact_Risk_bw2013, Dir_Linear2013}.
	Moreover, \cite{KDE_direct1988} used a piecewise constant kernel function to approximate the given kernel $L$ and derived the uniform convergence rate as
	\begin{equation}
	\label{Dir_KDE_unif_conv}
	\|\hat{f}_h- f\|_{\infty}= \sup_{\bm{x}\in \Omega_q} \left|\hat{f}_h(\bm{x}) -f(\bm{x}) \right| = O(h^2) + O_P\left(\sqrt{\frac{\log n}{nh^q}} \right).
	\end{equation} 
	One can also prove the uniform consistency of the directional KDE by slightly modifying the technique in \cite{Gine2002} and \cite{Einmahl2005} for the consistency results of the usual Euclidean KDE. 
	We will leverage such technique in our proof for the uniform convergence rates of the Riemannian gradient and Hessian estimators.
	
	\subsection{Pointwise Consistency}
	
	Our derivations of the pointwise convergence rates of the (Riemannian) gradient and Hessian estimators of the directional KDE $\hat{f}_h$ are analogous to the arguments for the usual Euclidean KDE \citep{Silverman1986,Scott2015}, which rely on the Taylor's expansion. The difference in the directional KDE case is that the integrals are taken over the Lebesgue measure $\omega_q$ on $\Omega_q$ when we compute the expectations $\mathbb{E}\left[\grad \hat{f}_h(\bm{x}) \right]$ and $\mathbb{E}\left[\mathcal{H} \hat{f}_h(\bm x) \right]$. The key argument for evaluating directional integrals is the following change-of-variable formula
	$$\omega_q(d\bm{x}) = (1-t^2)^{\frac{q}{2}-1} dt\, \omega_{q-1}(d\bm{\xi}),$$
	where $t=\bm{x}^T\bm{y}$ for a fixed point $\bm{y}\in \Omega_q$ and $\bm{\xi} \in \Omega_q$ is a unit vector orthogonal to $\bm{y}$. The formula is proved in Lemma 2 of \cite{Dir_Linear2013} and on pages 91-93 in \cite{Sphe_Harm} in two different ways. The surface area of $\Omega_q$ in (\ref{surf_area}) easily follows from this formula. 
	With this formula, we have the following convergence results.	
	\begin{theorem}
		\label{pw_conv_tang}
		Assume conditions (D1) and (D2').
		For any fixed $\bm{x}\in \Omega_q$, we have
		$$\grad \hat{f}_h(\bm{x})- \grad f(\bm{x}) = O(h^2) + O_P\left(\sqrt{\frac{1}{nh^{q+2}}} \right)$$
		as $h\to 0$ and $nh^{q+2} \to \infty$.\\
		Under the same condition, for any fixed $\bm{x}\in \Omega_q$, we have
		$$\mathcal{H} \hat{f}_h(\bm x) - \mathcal{H} f(\bm x) = O(h^2) + O_P\left(\sqrt{\frac{1}{nh^{q+4}}} \right)$$
		as $h\to 0$ and $nh^{q+4} \to \infty$.
	\end{theorem}

	The proof of Theorem~\ref{pw_conv_tang} is lengthy and deferred to Appendix~\ref{Appendix:Thm2_pf}. Theorem~\ref{pw_conv_tang} demonstrates that the Riemannian gradient of a directional KDE is a consistent estimator of the Riemannian gradient of the directional density that generates data.
	A similar result holds for the Riemannian Hessian. 
	% Without the asymptotic rate \eqref{kernel_rate} in condition (D2'), the convergence rate of the bias terms in Theorem~\ref{pw_conv_tang} can be slower than $O(h^2)$.
	It cannot be claimed that the total gradients $\nabla \hat f_h$ or $\nabla \tilde f_h$ converge to $\nabla f$ since the radial component of $f$ depends on how $f$ is extended to points outside $\Omega_q$. Lemma~\ref{Rad_grad} below and the proof of Theorem~\ref{pw_conv_tang} demonstrate that the limiting behaviors of $\Rad\left(\nabla \hat{f}_h(\bm{x}) \right)$ and $\Rad\left(\nabla \tilde{f}_h(\bm{x}) \right)$ are different; the former one is of the order $O\left(h^{-2} \right) + O_P\left(\sqrt{\frac{1}{nh^{q+4}}} \right)$ while the latter one is of the order $O(1) + O_P\left(\sqrt{\frac{1}{nh^{q+2}}} \right)$.
	Note that 
	\cite{KLEMELA2000} also derived similar convergence rates of the derivatives of a directional KDE, while the definitions of directional KDE and its derivatives in \cite{KLEMELA2000} are different from ours and the results are more complex.
	
	\begin{remark}
		Under some smoothness conditions \citep{Asymp_deri_KDE2011}, the pointwise convergence rates of gradient and Hessian estimators defined by the usual Euclidean KDE are
		$$O(h^2) + O_P\left(\frac{1}{nh^{d+2}} \right) \quad \text{ and } \quad O(h^2) + O_P\left(\frac{1}{nh^{d+4}} \right),$$
		where $d$ represents the dimension of the Euclidean data. Therefore, our pointwise consistency results for the Riemannian gradient and Hessian of the directional KDE in Theorem~\ref{pw_conv_tang} align with the pointwise convergence rates of the usual Euclidean KDE, in the sense that the dimension $d$ is replaced by the (intrinsic) manifold dimension $q$ of directional data.
	\end{remark}
	
	\subsection{Uniform Consistency}
	
	We now strengthen the convergence results in Theorem \ref{pw_conv_tang} to uniform convergence rates with the assumptions and techniques developed by \cite{Gine2002} and \cite{Einmahl2005}.
	
	Let $[\tau]=(\tau_1,...,\tau_{q+1})$ be a multi-index (that is, $\tau_1,...,\tau_{q+1}$ are non-negative integers and $|[\tau]|=\sum\limits_{j=1}^{q+1} \tau_j$). Define $D^{[\tau]} = \frac{\partial^{\tau_1}}{\partial x_1^{\tau_1}} \cdots \frac{\partial^{\tau_{q+1}}}{\partial x_1^{\tau_{q+1}}}$ as the $|[\tau]|$-th order partial derivative operator. 
	Given the directional KDE in \eqref{Dir_KDE2}, we define the following function class of the kernel function $L$ and its partial derivatives as
	$$\mathcal{K} = \left\{ \bm{u}\mapsto K\left(\frac{\bm{z}-\bm{u}}{h} \right): \bm{u}, \bm{z}\in \Omega_q, h>0, K(\bm{x}) = D^{[\tau]} L\left(\frac{1}{2}||\bm{x}||_2^2 \right), |[\tau]|=0,1,2 \right\}.$$
	Under condition (D2'), $\mathcal{K}$ is a collection of bounded measurable functions on $\Omega_q$. To guarantee the uniform consistency of the directional KDE itself as well as its (Riemannian) gradient and Hessian, we assume the following:
	\begin{itemize}
		\item {\bf (K1)} $\mathcal{K}$ is a bounded VC (subgraph) class of measurable functions on $\Omega_q$, that is, there exist constants $A,\vartheta >0$ such that for any $0 < \epsilon <1$,
		$$\sup_Q N\left(\mathcal{K}, L_2(Q), \epsilon ||F||_{L_2(Q)} \right) \leq \left(\frac{A}{\epsilon} \right)^{\vartheta},$$
		where $N(T,d_T,\epsilon)$ is the $\epsilon$-covering number of the pseudometric space $(T,d_T)$, $Q$ is any probability measure on $\Omega_q$, and $F$ is an envelope function of $\mathcal{K}$. The constants $A$ and $\vartheta$ are usually called the VC (Vapnik-Chervonenkis) characteristics of $\mathcal{K}$ and the norm $||F||_{L_2(Q)}$ is defined as $\left[\int_{\Omega_q} |F(\bm{x})|^2 dQ(\bm{x}) \right]^{\frac{1}{2}}$.
	\end{itemize}
	
	Given the differentiability of kernel $L$ guaranteed by (D2'), we can take $F$ as a constant envelope function 
	$$C_{\mathcal{K}} = \sup_{\bm{x}\in \mathbb{R}^{q+1}, |[\tau]|=0,1,2} \left|D^{[\tau]} L\left(\frac{1}{2}||\bm{x}||_2^2 \right) \right|$$
	when it is finite.
	Condition (K1) is not stringent in practice and can be satisfied by many kernel functions, such as the von-Mises kernel $L(r)=e^{-r}$ and many compactly supported kernels on $[0,\infty)$. For these kernel options, the resulting function class $\mathcal{K}$ comprises only functions of the form $\bm{u} \mapsto \Phi\left(|\bm{A}\bm{u}+\bm{b}| \right)$, where $\Phi$ is a real-valued function of bounded variation on $[0,\infty)$, $\bm{A}$ ranges over matrices in $\mathbb{R}^{(q+1)\times (q+1)}$, and $\bm{b}$ ranges over $\mathbb{R}^{q+1}$. Thus, $\mathcal{K}$ is of VC (subgraph) class by Lemma 22 in \cite{Nolan1987}.
	
	Under conditions (D1), (D2'), and (K1), the uniform consistency results of the directional KDE (restated) as well as its Riemannian gradient and Hessian estimators are summarized as the following theorem, whose proof can be founded in Appendix~\ref{Appendix:Thm4_pf}.
	
	\begin{theorem}
		\label{unif_conv_tang}
		Assume (D1), (D2'), and (K1). The uniform convergence rate of $\hat{f}_h$ is given by
		$$\sup_{\bm x \in\Omega_q}|\hat{f}_h(\bm x) -f(\bm x)| = O(h^2) + O_P\left(\sqrt{\frac{|\log h|}{nh^q}} \right)$$
		as $h\to 0$ and $\frac{nh^q}{|\log h|} \to \infty$. \\
		Furthermore, the uniform convergence rate of $\grad \hat{f}_h(\bm{x})$ on $\Omega_q$ is 
		$$\sup_{\bm x\in\Omega_q}\norm{\grad \hat{f}_h (\bm x)- \grad f(\bm x) }_{\max} = O(h^2) + O_P\left(\sqrt{\frac{|\log h|}{nh^{q+2}}} \right),
		$$
		as $h\to 0$ and $\frac{nh^{q+2}}{|\log h|} \to \infty$.
		Finally, the uniform convergence rate of $\mathcal{H} \hat{f}_h(\bm x)$ on $\Omega_q$ is
		$$
		\sup_{\bm{x}\in \Omega_q}\norm{\mathcal{H} \hat{f}_h(\bm x) - \mathcal{H} f(\bm x)}_{\max} 
		= O(h^2) + O_P\left(\sqrt{\frac{|\log h|}{nh^{q+4}}} \right),
		$$
		as $h\to 0$ and $\frac{nh^{q+4}}{|\log h|} \to \infty$, where $\norm{\cdot}_{\max} $ is the elementwise maximum norm for a vector in $\mathbb{R}^{q+1}$ or a matrix in $\mathbb{R}^{(q+1)\times (q+1)}$.
	\end{theorem}
	
	\begin{remark}
	Theorem~\ref{unif_conv_tang} can also be generalized to higher-order derivatives.
	All that is necessary is to modify the assumptions (D2') and (K1) to higher-order derivatives (projected on the tangent direction) as well as strengthen the differentiable assumptions on $f$ in (D1). 
	The elementwise maximum norm between the derivative estimator and the true quantity will embrace the rate
	$$
	O(h^2) + O_P\left(\sqrt{\frac{|\log h|}{nh^{q+2m}}}\right),
	$$
	where $m$ is the highest order of derivatives desired.
	\end{remark}
	
	\subsection{Mode Consistency}
	\label{Sec:Mode_Const}
	
	Consistency of estimating local modes has been established for the usual Euclidean KDE by \cite{Mode_clu2016}, where the authors demonstrated that with probability tending to 1, the number of estimated local modes is the same as the number of true local modes under appropriate assumptions. Moreover, the convergence rate of the Hausdorff distance (a common distance between two sets) between the collection of local modes and its estimator is elucidated. Here, we reproduce the consistency of estimating local modes of a directional density $f$ supported on $\Omega_q$ by the local modes of the directional KDE $\hat{f}_h$.

	 Given two sets $A,B \subset \Omega_q$, their Hausdorff distance is
	\begin{equation}
	\label{Haus_def}
	\Haus(A,B) = \inf\left\{r>0: A\subset B \oplus r, B \subset A \oplus r \right\},
	\end{equation}
	where $A\oplus r =\left\{\bm{y}\in \Omega_q: \inf_{\bm{x}\in A} \norm{\bm{x}-\bm{y}}_2 \leq r \right\} = \left\{\bm{y}\in \Omega_q: \sup_{\bm{x}\in A} \bm{x}^T\bm{y} \geq 1-\frac{r^2}{2} \right\}$. The equality follows from the fact that $\norm{\bm{x}}_2^2=1$ for any $\bm{x}\in \Omega_q$.

	Let $C_3$ be the upper bound for the partial derivatives of the directional density $f$ on the compact manifold $\Omega_q$ up to the third order.  
	Such constant exists under condition (D1).
	Let $\hat{\mathcal{M}}_n = \left\{\hat{\bm{m}}_1,...,\hat{\bm{m}}_{\hat{K}_n} \right\}$ be the collection of local modes of $\hat{f}_h$ and $\mathcal{M}=\left\{\bm{m}_1,...,\bm{m}_K \right\}$ be the collection of local modes of $f$. Here, $\hat{K}_n$ is the number of estimated local modes and $K$ is the number of true local modes. We consider the following assumptions.
	\begin{itemize}
		\item {\bf (M1)} There exists $\lambda_* >0$ such that 
		$$0 < \lambda_* \leq |\lambda_1(\bm{m}_j)|, \quad \text{ for all } j=1,...,K,$$
		where $0>\lambda_1(\bm{x}) \geq \cdots \geq \lambda_q(\bm{x})$ are the $q$ smallest (negatively-largest) eigenvalues of the Riemannian Hessian $\mathcal{H} f(\bm x)$.
		\item {\bf (M2)} There exist $\Theta_1,\rho_* >0$ such that 
		$$\left\{\bm{x}\in \Omega_q: \norm{\Tang(\nabla f(\bm{x}))}_{\max} \leq \Theta_1, \lambda_1(\bm{x}) \leq -\frac{\lambda_*}{2} <0 \right\} \subset \mathcal{M} \oplus \rho_*,$$
		where $\lambda_*$ is defined in (M1) and $0<\rho_* < \min\left\{\sqrt{2-2\cos\left(\frac{3\lambda_*}{2C_3}\right)}, 2\right\}$.
	\end{itemize}
	
	Condition (M1) is imposed so that every local mode of $f$ is isolated from other critical points; see Lemma 3.2 in \cite{Morse_Homology2004}. The condition also guarantees that the number of local modes of $f$ supported on the compact manifold $\Omega_q$ is finite. As noted by \cite{Mode_clu2016}, condition (M1) always holds when $f$ is a Morse function on $\Omega_q$. The second condition (M2) regularizes the behavior of $f$ so that points with near 0 (Riemannian) gradients and negative eigenvalues of $\mathcal{H} f(\bm x)$ within the tangent space $T_{\bm{x}}$ must be close to local modes. See the paper by \cite{Mode_clu2016} for detailed discussion. The constant $\sqrt{2-2\cos\left(\frac{3\lambda_*}{2C_3}\right)}$ is selected so that the great-circle distance from $\bm{m}_k$ to the boundary of $\bm{m}_k \oplus \rho_*$, that is, $\arccos(\bm{m}_k^T \bm{x})$ with $\bm{x}\in \partial S_k$, is less than $\frac{3\lambda_*}{2C_3}$ for any $\bm{m}_k \in \mathcal{M}$, where $S_k=\left\{\bm{x}\in \Omega_q: \norm{\bm{x}-\bm{m}_k}_2 \leq \rho_* \right\}$ and $\partial S_k = \{\bm{x}\in \Omega_q: \norm{\bm{x}-\bm{m}_k}_2 = \rho_* \}$.
	
	It should be emphasized that condition (M1) is a weak condition that can be satisfied by the local modes of common directional densities. We take the von-Mises-Fisher density as an example. With the formula \eqref{vMF_density}, we naturally extend $f_{\text{vMF}}$ to $\mathbb{R}^{q+1}$ and deduce that
	$$\nabla f_{\text{vMF}}(\bm{x}) = \nu \bm{\mu} C_q(\nu) \cdot \exp\left(\nu \bm{\mu}^T \bm{x} \right) \quad \text{ and } \quad \nabla\nabla f_{\text{vMF}}(\bm{x}) = \nu^2 \bm{\mu} \bm{\mu}^T C_q(\nu) \cdot \exp\left(\nu \bm{\mu}^T \bm{x} \right),$$
	which in turn indicates that at the mode $\bm{\mu} \in \Omega_q$,
	\begin{align*}
		\mathcal{H} f_{\text{vMF}}(\bm{\mu}) &= \left(I_{q+1} -\bm{\mu} \bm{\mu}^T \right) \nabla\nabla f_{\text{vMF}}(\bm{\mu}) \left(I_{q+1} -\bm{\mu} \bm{\mu}^T \right) -\bm{\mu}^T \nabla f_{\text{vMF}}(\bm{\mu}) \left(I_{q+1} -\bm{\mu} \bm{\mu}^T \right)\\
		&= -\nu C_q(\nu) \cdot e^{\nu} \left(I_{q+1} -\bm{\mu} \bm{\mu}^T \right).
	\end{align*}
	By Brauer's theorem (Example 1.2.8 in \citealt{HJ2012}), we conclude that the eigenvalues of $\mathcal{H} f_{\text{vMF}}(\bm{\mu})$ are 0 with (algebraic) multiplicity 1, which is associated with the eigenvector $\bm{\mu}$, and $-\nu C_q(\nu) \cdot e^{\nu}$ with multiplicity $q$, which are associated with the eigenvectors in $T_{\bm{\mu}}$.
		
	Given these assumptions, the mode consistency of the directional KDE is as follows.
	
	\begin{theorem}
		\label{Mode_cons}
		Assume (D1), (D2'), (K1), and (M1-2). For any $\delta \in (0,1)$, when $h$ is sufficiently small and $n$ is sufficiently large,
		\begin{enumerate}[label=(\alph*)]
			\item there must be at least one estimated local mode $\hat{\bm{m}}_k$ within $S_k = \bm{m}_k \oplus \rho_*$ for every $\bm{m}_k \in \mathcal{M}$, and
			\item the collection of estimated modes satisfies $\hat{\mathcal{M}}_n \subset \mathcal{M} \oplus \rho_*$ and there is a unique estimated local mode $\hat{\bm{m}}_k$ within $S_k=\bm{m}_k\oplus \rho_*$
		\end{enumerate}
		with probability at least $1-\delta$. In total, when $h$ is sufficiently small and $n$ is sufficiently large, there exist some constants $A_3, B_3 >0$ such that
		$$\mathbb{P}\left(\hat{K}_n \neq K \right) \leq B_3 e^{-A_3nh^{q+4}}.$$
		\begin{enumerate}[label=(c)]
			\item The Hausdorff distance between the collection of local modes and its estimator satisfies $$\Haus\left(\mathcal{M},\hat{\mathcal{M}}_n \right) = O(h^2) + O_P\left(\sqrt{\frac{1}{nh^{q+2}}} \right),$$
			as $h\to 0$ and $nh^{q+2} \to \infty$.
		\end{enumerate}
	\end{theorem}

	The proof of Theorem~\ref{Mode_cons} is in Appendix~\ref{Appendix:Thm6_pf}. It states that 
	asymptotically, the set of estimated local modes are close to the set of true local modes and there exists a $1-1$ mapping between pairs of estimated and true local modes. 
	Thus, the local modes of the directional KDE are good estimators of the local modes of the population directional density.

	\begin{remark}
		Unlike the statement of Theorem 1 by \cite{Mode_clu2016}, the radius $\rho_*$ in (M2) for $\mathcal{M}$ to contain $\hat{\mathcal{M}}_n$ can be selected to be independent of the dimension of the data. The reason lies in the fact that the proof of statement (a) in Theorem \ref{Mode_cons} performs a Taylor's expansion to the third order and leverages the constant upper bound for the third-order partial derivatives. The same technique can be used to improve the original proof in Theorem 1 of \cite{Mode_clu2016} to obtain a dimension-free radius for mode consistency.
	\end{remark}

	\section{Computational Learning Theory of Directional Mean Shift Algorithm}
	\label{Sec:Algo_conv}
	
	In this section, we study the algorithmic convergence of Algorithm~\ref{Algo:MS}.
	We start with the ascending property and convergence of Algorithm~\ref{Algo:MS}, and then prove the linear convergence of gradient ascent algorithms on the sphere $\Omega_q$. By shrinking the bandwidth parameter, the adaptive step size of Algorithm~\ref{Algo:MS} as a gradient ascent iteration on $\Omega_q$ can be sufficiently small so that the algorithm converges linearly to the estimated local modes around their neighborhoods. Finally, we discuss on the computational complexity of Algorithm~\ref{Algo:MS}.

	\subsection{Ascending Property and Convergence of Algorithm \ref{Algo:MS}}
	
	Let $\left\{\hat{\bm{y}}_s\right\}_{s=0}^{\infty}$ be the path of successive points generated by Algorithm~\ref{Algo:MS}. 
	The corresponding sequence of directional density estimates is given by
	$$\hat{f}_h(\hat{\bm{y}}_s) = \frac{c_{h,q}(L)}{n} \sum_{i=1}^n L\left(\frac{1-\hat{\bm{y}}_s^T \bm{X}_i}{h^2} \right) \quad \text{ for } s=0,1,\dots.$$
	
	\begin{theorem}[Ascending Property]
		\label{MS_asc}
		If kernel $L:[0,\infty) \to [0,\infty)$ is monotonically decreasing, differentiable, and convex with $L(0)<\infty$, then the sequence $\left\{\hat{f}_h(\hat{\bm{y}}_s) \right\}_{s=0}^{\infty}$ is monotonically increasing and thus converges.
	\end{theorem}
	
	At a high level, the proof of Theorem~\ref{MS_asc} follows from the inequality
	\begin{equation}
	L(x_2) -L(x_1) \geq L'(x_1) \cdot (x_2-x_1),
	\label{MS_ineq}
	\end{equation}
	which is guaranteed by the convexity and differentiability of the kernel function $L$; see Appendix~\ref{Appendix:Thm8_10_11_pf} for details.

	\begin{remark}
		\label{Diff_Relax}
		Note that the differentiability of kernel $L$ in Theorem \ref{MS_asc} can be relaxed. 
		The monotonicity and convexity of $L$ already imply that $L$ is differentiable except for a countable set of points 
		$\mathcal{N}$ (see Section 6.2 and 6.6 in \citealt{Real_Analysis}). Moreover, the left and right derivatives of $L$ on $\mathcal{N}$ exist and are finite.
	    Therefore, for any $x_1 \in \mathcal{N}$, we can replace $L'(x_1)$ in \eqref{MS_ineq} by any subgradient $g_{x_1}$ without impacting other parts of the inequality. Furthermore, as the left or right derivatives of the convex function $L$ are non-decreasing, any subgradient $g_{x_1}$ at point $x_1$ satisfies $L'(x_1^-) \leq g_{x_1} \leq L'(x_1^+)$; thus, \eqref{MS_ineq} holds.
	\end{remark}
	
	The ascending property of $\left\{\hat{f}_h(\hat{\bm{y}}_s) \right\}_{s=0}^{\infty}$ under the directional mean shift algorithm is not sufficient to guarantee the convergence of its iterative sequence $\{\hat{\bm{y}}_s\}_{s=0}^{\infty}$. To derive the convergence of $\{\hat{\bm{y}}_s\}_{s=0}^{\infty}$, we make the following assumptions on the directional KDE $\hat{f}_h$.
	\begin{itemize}
		\item {\bf (C1)} The number of local modes of $\hat{f}_h$ on $\Omega_q$ is finite, and the modes are isolated from other critical points.
		\item {\bf (C2)} Given the current values of $n$ and $h>0$, we assume that $\hat{\bm{m}}_k^T \nabla \hat{f}_h(\hat{\bm{m}}_k) \neq 0$ for all $\hat{\bm{m}}_k \in \hat{\mathcal{M}}_n$, that is, $\sum\limits_{i=1}^n \hat{\bm{m}}_k^T \bm{X}_i \cdot L'\left(\frac{1-\hat{\bm{m}}_k^T \bm{X}_i}{h^2} \right) \neq 0$.
	\end{itemize}
	
	Condition (C1) is a weak condition when the uniform consistency (Theorem~\ref{unif_conv_tang}) and mode consistency (Theorem~\ref{Mode_cons}) are established. In reality, condition (C1) is implied by conditions (D1) and (M1-2) on $f$ as well as (D2') and (K1) on the kernel function $L$ with a probability tending to $1$ as the sample size increases and the bandwidth parameter decreases accordingly.
	
	Condition (C2) may look strange at first glance; however, it is a reasonable and common assumption. In practice, it will be valid with those commonly chosen kernel functions, a reasonable sample size $n$, and a properly tuned bandwidth parameter $h>0$. More importantly, because the directional density $f$ is always positive around its local mode, the following lemma demonstrates that condition (C2) holds with probability tending to 1 as the sample size increases to infinity and the bandwidth parameter tends to 0 accordingly.
	
	\begin{lemma}
	\label{Rad_grad}
	Assume conditions (D1) and (D2'). For any fixed $\bm{x} \in \Omega_q$, we have
	$$h^2 \cdot \Rad\left(\nabla \hat{f}_h(\bm{x}) \right) \asymp h^2 \cdot \nabla\hat{f}_h(\bm{x})  = \bm{x} f(\bm{x}) C_{L,q} + o\left(1 \right) + O_P\left(\sqrt{\frac{1}{nh^q}} \right)$$
	as $nh^q \to \infty$ and $h\to 0$, where $C_{L,q}=-\frac{\int_0^{\infty} L'(r) r^{\frac{q}{2}-1} dr}{\int_0^{\infty} L(r) r^{\frac{q}{2}-1} dr} > 0$ is a constant depending only on kernel $L$ and dimension $q$ and ``$\asymp$'' stands for an asymptotic equivalence.
	\end{lemma}

	The proof of Lemma~\ref{Rad_grad} can be found in Appendix~\ref{Appendix:Thm8_10_11_pf}. 
	With Lemma~\ref{Rad_grad}, we know that while the tangent component of $\nabla \hat{f}_h$ at each local mode is $0$,
	its radial component is diverging; thus, condition (C2) holds asymptotically. 
	This is not a surprising result, because observations in a directional data sample are supported on the sphere and the directional KDE $\hat{f}_h$ would thus decrease rapidly when moving away from the sphere.
	In addition, the limiting behavior of $\nabla \hat{f}_h$ determines the adaptive step size of the directional mean shift algorithm when it approaches the estimated local modes (see Section~\ref{sec:linear} for details). A similar asymptotic behavior of the step size of the mean shift algorithm in the Euclidean setting has been noticed by \cite{MS1995} and restated by \cite{Ery2016}. 
	
	We now state the convergence of Algorithm \ref{Algo:MS} under conditions (C1) and (C2).
	
	\begin{theorem}
		\label{MS_conv}
		Assume (C1) and (C2) and the conditions on kernel $L$ in Theorem \ref{MS_asc}. We further assume that $L$ is continuously differentiable. Then, for each local mode $\hat{\bm{m}}_k \in \hat{\mathcal{M}}_n$, there exists a $\hat{r}_k >0$ such that the sequence $\{\hat{\bm{y}}_s\}_{s=0}^{\infty}$ converges to $\hat{\bm{m}}_k$ whenever the initial point $\hat{\bm{y}}_0 \in \Omega_q$ satisfies $\norm{\hat{\bm{y}}_0 -\hat{\bm{m}}_k}_2 \leq \hat{r}_k$. Moreover, under conditions (D1) and (D2'), there exists a fixed constant $r^* >0$ such that $\mathbb{P}(\hat{r}_k \geq r^*) \to 1$ as $h\to 0$ and $nh^q \to \infty$.
	\end{theorem}
	
	The proof of Theorem~\ref{MS_conv} is in Appendix~\ref{Appendix:Thm8_10_11_pf}. The theorem implies that when we initialize the directional mean shift algorithm (Algorithm~\ref{Algo:MS}) sufficiently close to an estimated local mode, it will converge to this mode.

	\subsection{Linear Convergence of Gradient Ascent Algorithms on $\Omega_q$}	\label{sec:linear}
	
	We now discuss the linear convergence of gradient ascent algorithms on $\Omega_q$. 
	Because the sphere $\Omega_q$ is not a conventional Euclidean space but a Riemannian manifold, 
	the definition of a gradient ascent update is more complex. We first provide a brief introduction to some useful concepts from differential geometry.
	The interested readers can consult Appendix \ref{sec::GH} for additional details.
	
	An \emph{exponential map} at $\bm x\in \Omega_q$ is a mapping $\Exp_{\bm x}: T_{\bm x} \to \Omega_q$ such that a vector $\bm v\in T_{\bm x}$ is mapped to point $\bm y:=\Exp_{\bm x}(\bm v) \in \Omega_q$ with $\gamma(0) =\bm x, \gamma(1)=\bm y$ and $\gamma'(0)=\bm v$, where $\gamma: [0,1] \to \Omega_q$ is a geodesic. 
	An intuitive way of thinking of the exponential map evaluated at $\bm{v}$ on the sphere $\Omega_q$ is that starting at point $\bm x$, we identify another point $\bm y$ on $\Omega_q$ along the great circle in the direction of $\bm{v}$ so that the geodesic distance between $\bm x$ and $\bm y$ is $\norm{\bm{v}}_2$.
	
	The inverse of the exponential map is a mapping $\Exp_{\bm x}^{-1}: U\subset \Omega_q \to T_{\bm x}$ such that $\Exp_{\bm x}^{-1}(\bm y)$ represents the vector in $T_{\bm x}$ starting at $\bm x$, pointing to $\bm y$, and with its length equal to the geodesic distance between $\bm x$ and $\bm y$. 
	$\Exp_{\bm x}^{-1}$ is sometimes called the logarithmic map. 
	
	On $\Omega_q$, the notion of \emph{parallel transport} provides a sensible way to transport a vector along a geodesic \citep{Geo_Convex_Op2016}. Intuitively, a tangent vector $\bm v\in T_{\bm x}$ at $\bm x\in \Omega_q$ of a geodesic $\gamma$ is still a tangent vector $\Gamma_{\bm x}^{\bm y}(\bm v) \in T_{\bm y}$ of $\gamma$ after being transported to point $\bm y$ along $\gamma$. Furthermore, parallel transport preserves inner products, that is, $\langle {\bm u},{\bm v} \rangle_{\bm x} =\langle \Gamma_{\bm x}^{\bm y}(\bm u), \Gamma_{\bm x}^{\bm y}(\bm v) \rangle_{\bm y}$.
	The above concepts can be defined on a general Riemannian manifold; however, for our purposes, it suffices to focus on the case of $\Omega_q$.

	Adopting the notation of \cite{Geo_Convex_Op2016}, 
	a gradient ascent algorithm applied to an objective function $f$ on $\Omega_q$ (a Riemannian manifold) is written as
	\begin{equation}
	\label{grad_ascent_Manifold}
	\bm{y}_{s+1} = \Exp_{\bm{y}_s}\left(\eta \cdot \grad f(\bm{y}_s) \right).
	\end{equation}
	Recall that given a sequence $\left\{\bm{y}_s \right\}_{s=0}^{\infty}$ converging to $\bm{m}_k \in \mathcal{M}$, the convergence is said to be linear if there exists a positive constant $\Upsilon <1$ (rate of convergence) such that $\norm{\bm{y}_{s+1} -\bm{m}_k} \leq \Upsilon \norm{\bm{y}_s -\bm{m}_k}$ when $s$ is sufficiently large \citep{Boyd2004}. In our context, the norm $\norm{\cdot}$ refers to the geodesic (or great-circle) distance $d(\bm{x},\bm{y})=\norm{\Exp_{\bm x}^{-1}(\bm{y})}_2$ between two points $\bm{x},\bm{y}\in \Omega_q$. 
	An equivalent statement of linear convergence is that the algorithm takes  $O(\log(1/\epsilon))$ iterations to converge to an $\epsilon$-error of $\hat{\bm{m}}_k$. 
	
	Here, we first prove the linear convergence results for the gradient ascent algorithm with $f$ and $\hat f_h$ on $\Omega_q$ under a feasible range of step size $\eta$.
	We then demonstrate that the directional mean shift algorithm is an exemplification of the gradient ascent algorithm on $\Omega_q$ with an adaptive step size, and that its step size eventually falls into the feasible range with a properly tuned bandwidth parameter. Using the notation in \cite{Geo_Convex_Op2016}, we let $\zeta(1, c) \equiv \frac{c}{\tanh(c)}$. One can show by differentiating $\zeta(1, c)$ that $\zeta(1, c)$ is strictly increasing and $\zeta(1, c) > 1$ for any $c>0$.
	
	\begin{theorem}
		\label{Linear_Conv_GA}
Assume (D1) and (M1).
\begin{enumerate}[label=(\alph*)]
	\item \textbf{Linear convergence of gradient ascent with $f$}: Given a convergence radius $r_0$ with $0< r_0 \leq \sqrt{2-2\cos\left[\frac{3\lambda_*}{2(q+1)^{\frac{3}{2}}C_3} \right]}$, the iterative sequence $\left\{\bm{y}_s\right\}_{s=0}^{\infty}$ defined by the population-level gradient ascent algorithm \eqref{grad_ascent_Manifold} satisfies
	$$d(\bm{y}_s, \bm{m}_k) \leq \Upsilon^s \cdot d(\bm{y}_0, \bm{m}_k) \quad \text{ with } \quad \Upsilon = \sqrt{1-\frac{\eta\lambda_*}{2}},$$
	whenever $\eta \leq \min\left\{\frac{2}{\lambda_*}, \frac{1}{(q+1)C_3\zeta(1,r_0)} \right\}$ and the initial point $\bm{y}_0 \in \left\{\bm{z}\in \Omega_q: \norm{\bm{z}-\bm{m}_k}_2 \leq r_0 \right\}$ for some $\bm{m}_k \in \mathcal{M}$.
	We recall from Section~\ref{Sec:Mode_Const} that $C_3$ is an upper bound for the derivatives of the directional density $f$ up to the third order, $\lambda_*>0$ is defined in (M1), and $\mathcal{M}$ is the set of local modes of the directional density $f$.
\end{enumerate}
We further assume (D2') and (K1) in the sequel.
\begin{enumerate}[label=(b)]
	\item \textbf{Linear convergence of gradient ascent with $\hat{f}_h$}: Let the sample-based gradient ascent update on $\Omega_q$ be $\hat{\bm{y}}_{s+1} = \Exp_{\bm{y}_s}\left(\eta\cdot \grad \hat{f}_h(\hat{\bm{y}}_s) \right)$. With the same choice of the convergence radius $r_0>0$ and $\Upsilon=\sqrt{1-\frac{\eta\lambda_*}{2}}$ as in (a), if $h\to 0$ and $\frac{nh^{q+2}}{|\log h|} \to \infty$, then for any $\delta \in (0,1)$,
	$$d\left(\hat{\bm{y}}_s,\bm{m}_k \right) \leq \Upsilon^s \cdot d\left(\hat{\bm{y}}_0,\bm{m}_k \right) + O(h^2) + O_P\left(\sqrt{\frac{|\log h|}{nh^{q+2}}} \right)$$
	with probability at least $1-\delta$, whenever $\eta \leq \min\left\{\frac{2}{\lambda_*}, \frac{1}{(q+1)C_3\cdot\zeta(1,r_0)} \right\}$ and the initial point $\hat{\bm{y}}_0 \in \left\{\bm{z}\in \Omega_q: \norm{\bm{z}-\bm{m}_k}_2 \leq r_0 \right\}$ for some $\bm{m}_k \in \mathcal{M}$.
\end{enumerate}
	\end{theorem}
	
	The proof of Theorem~\ref{Linear_Conv_GA} is in Appendix~\ref{Appendix:Thm12_pf}. As shown in (a) of Theorem~\ref{Linear_Conv_GA}, the linear convergence radius of gradient ascent algorithm \eqref{grad_ascent_Manifold} on $\Omega_q$ generally depends on the lower bound $\lambda_*$ on absolute eigenvalues of the Riemannian Hessian $\mathcal{H} f(\bm{x})$ (within the tangent space $T_{\bm{x}}$), the upper bound $C_3$ for the (partial) derivatives of $f$ up to the third order, and manifold dimension $q$. 
	
	In practice, we are more interested in the algorithmic convergence rate of sample-based gradient ascent algorithms with directional KDEs to the estimated local modes $\hat{M}_n$. As indicated by Theorem \ref{unif_conv_tang}, the Hessian matrices of $\hat{f}_h$ at its local modes have only negative eigenvalues within the corresponding tangent spaces given (M1), sufficiently small $h$, and sufficiently large $\frac{nh^{q+4}}{|\log h|}$. In reality, unless the data configuration is highly abnormal, the local modes of directional KDEs are non-degenerate and $\hat{f}_h$ is geodesically strongly concave (see Appendix \ref{sec::GH} for a precise definition) around small neighborhoods of the estimated local modes. Together with an application of smooth kernels, says the von Mises kernel, $\hat{f}_h$ is $\beta$-smooth on $\Omega_q$ and, consequently, a sample-based gradient ascent algorithm with the directional KDE $\hat{f}_h$ converges linearly to the estimated local modes around their small neighborhoods, given a proper step size.
	
	With respect to the directional mean shift algorithm, we recall from the fixed-point equation \eqref{fix_point_grad} that the geodesic distance between $\hat{\bm{y}}_{s+1}$ and $\hat{\bm{y}}_s$ (one-step iteration) is
	$$\arccos\left(\frac{\nabla \hat{f}_h(\hat{\bm{y}}_s)^T \hat{\bm{y}}_s}{\norm{\nabla \hat{f}_h(\hat{\bm{y}}_s)}_2} \right).$$
	To derive the adaptive step size $\hat{\eta}_s$ of the directional mean shift algorithm as a sample-based gradient ascent iteration $\hat{\bm{y}}_{s+1} = \Exp_{\hat{\bm{y}}_s}\left(\hat{\eta}_s\cdot \grad \hat{f}_h(\hat{\bm{y}}_s) \right)$ on $\Omega_q$, we notice the following geodesic distance equation:
	$$\norm{\hat{\eta}_s \cdot \grad \hat{f}_h(\hat{\bm{y}}_s)}_2 = \arccos\left(\frac{\nabla \hat{f}_h(\hat{\bm{y}}_s)^T \hat{\bm{y}}_s}{\norm{\nabla \hat{f}_h(\hat{\bm{y}}_s)}_2} \right).$$
	This shows that the directional mean shift algorithm is a gradient ascent method on $\Omega_q$ with an adaptive step size
	$$\hat{\eta}_s = \arccos\left(\frac{\nabla \hat{f}_h(\hat{\bm{y}}_s)^T \hat{\bm{y}}_s}{\norm{\nabla \hat{f}_h(\hat{\bm{y}}_s)}_2} \right) \cdot \frac{1}{\norm{\grad \hat{f}_h(\hat{\bm{y}}_s)}_2}$$
	for $s=0,1,...$. We denote the angle between the total gradient estimator $\nabla \hat{f}_h(\hat{\bm{y}}_s)$ and $\hat{\bm{y}}_s$ by $\hat{\theta}_s$. By the orthogonality of $\hat{\bm{y}}_s$ and $\grad \hat{f}_h(\hat{\bm{y}}_s)\equiv \Tang\left(\nabla \hat{f}_h(\hat{\bm{y}}_s) \right)$, the expression for the adaptive step size $\hat{\eta}_s$ becomes
	$$\hat{\eta}_s = \frac{\hat{\theta}_s}{\left(\sin \hat{\theta}_s \right) \cdot \norm{\nabla \hat{f}_h(\hat{\bm{y}}_s)}_2}$$
	for $s=0,1,...$. As the directional mean shift algorithm approaches a local mode of $\hat{f}_h$, $\hat{\theta}_s$ tends to 0 and $\frac{\hat{\theta}_s}{\sin \hat{\theta}_s}$ is approximately equal to 1. Thus, the step size $\hat{\eta}_s$ is essentially controlled by the (Euclidean) norm of the total gradient estimator, that is, $\norm{\nabla \hat{f}_h(\hat{\bm{y}}_s)}_2$. The larger the norm of $\nabla \hat{f}_h(\hat{\bm{y}}_s)$ at step $s$, the shorter the step size of Algorithm~\ref{Algo:MS}. Because the tangent component of $\nabla \hat{f}_h(\hat{\bm{y}}_s)$ is small around the estimated local modes, its radial component $\Rad \left(\nabla \hat{f}_h(\hat{\bm{y}}_s) \right)$ dominates the norm $\norm{\nabla \hat{f}_h(\hat{\bm{y}}_s)}_2$.
	Lemma~\ref{Rad_grad} suggests that $\norm{\nabla \hat{f}_h(\hat{\bm{y}}_s)}_2$ can be sufficiently large as the sample size increases to infinity and the bandwidth parameter decreases to 0 accordingly; therefore, one can always select a small bandwidth parameter $h$ such that the step size $\hat{\eta}_s$ lies within the feasible range for linear convergence. Algorithm~\ref{Algo:MS} thus converges (at least) linearly around the local modes of the directional KDE $\hat{f}_h$.
	
	\subsection{Computational Complexity}
	\label{Sec:Com_Complexity}
	
	Given a fixed data set $\{\bm{X}_1,...,\bm{X}_n\} \subset \Omega_q$, the time complexity of Algorithm \ref{Algo:MS} is $O(n\times q)$ for one iteration of the algorithm on a single query point. When Algorithm~\ref{Algo:MS} is applied to the entire data set as the set of query points, each iteration exhibits $O(n^2\times q)$ time complexity. Assuming that the algorithm converges linearly, the total time complexity for reaching an $\epsilon$-error is $O\left(n^2\times q\times \log(1/\epsilon) \right)$. The space complexity of mode clustering with Algorithm~\ref{Algo:MS} is, in general, $O(n\times q)$ if mode clustering is performed on the entire data set to estimate the directional density and only the current set of iteration points are stored in memory. Algorithm \ref{Algo:MS} inevitably faces a computational bottleneck or even inferior performance when the (intrinsic) dimension $q$ is large. This drawback of the algorithm results not only from its time and space complexity, but also from its original dependency on nonparametric density estimation, which is known to suffer from \emph{the curse of dimensionality}.
	
	\section{Experiments}
	\label{Sec:Experiments}
	
	In this section, we present our experimental results of the directional mean shift algorithm on both simulated and real-world data sets. Unless stated otherwise, we use the von Mises kernel $L(r)=e^{-r}$ in the directional KDE \eqref{Dir_KDE} to estimate the directional densities and their derivatives. Given the data sample $\{\bm{X}_1,...,\bm{X}_n\}$, the default bandwidth parameter is selected via the rule of thumb in Proposition 2 in \cite{Exact_Risk_bw2013}:  
	\begin{equation}
	\label{bw_ROT}
	h_{\text{ROT}} = \left[\frac{4\pi^{\frac{1}{2}} \mathcal{I}_{\frac{q-1}{2}}(\hat{\nu})^2}{\hat{\nu}^{\frac{q+1}{2}}\left[2 q\cdot\mathcal{I}_{\frac{q+1}{2}}(2\hat{\nu}) + (q+2)\hat{\nu} \cdot \mathcal{I}_{\frac{q+3}{2}}(2\hat{\nu}) \right]n} \right]^{\frac{1}{q+4}}
	\end{equation}
	for $q\geq 1$, which is the optimal bandwidth for the directional KDE that minimizes the asymptotic mean integrated squared error when the underlying density is a von Mises-Fisher density and the von Mises kernel is applied. The estimated concentration parameter $\hat{\nu}$ is given by (4.4) in \cite{spherical_EM} as
	$$\hat{\nu} = \frac{\bar{R}(q+1-\bar{R}^2)}{1-\bar{R}^2},$$
	where $\bar{R} = \frac{\norm{\sum_{i=1}^n \bm{X}_i}_2}{n}$ (see also \cite{Sra2011} for a detailed discussion and experiments on the numerical approximation of the concentration parameter for von Mises-Fisher distributions). We also perform mode clustering \citep{Mode_clu2016} (sometimes called mean shift clustering in \citealt{MS1975, MS1995}) on the original data sets in our simulation studies, in which data points are assigned to the same cluster if they converge to the same (estimated) local mode. When such procedure is carried out on the entire data space, it partitions the space into different regions called \emph{basins of attraction} of the (directional) KDE. 
	% Different regions of the data space in which points converge to the same estimated local mode are called \emph{basins of attraction} of the directional KDE. 
	As the true density component from which a data point is generated is known a priori in our simulation studies (i.e., we know the label of each observation), we also provide the misclassification rates or confusion matrices of mode clustering with the directional mean shift algorithm, though one should be aware that mode clustering, by its nature, embraces non-overlapping basins of attraction \citep{Morse_Homology2004,chacon2015population}. Figures~\ref{fig:Two_d_ThreeM}, \ref{fig:Mars_data}, and \ref{fig:Earthquake} in this section as well as Figures~\ref{fig:Two_d_OneM} and \ref{fig:Mode_clu_Mars} in Appendix~\ref{Appendix:Ad_Experiments} are plotted via the Matplotlib Basemap Toolkit (\url{https://matplotlib.org/basemap/}).
	
	\subsection{Simulation Studies}
	
	\subsubsection{Circular Case}
	
	To evaluate the effectiveness of Algorithm~\ref{Algo:MS}, we first randomly generate 60 data points from a circular density
	$$f_1(x)=\frac{e^{-|x|}}{4(1-e^{-\pi})}\cdot \mathbbm{1}_{[-\pi,\pi]}(x)+\frac{1}{4\pi \mathcal{I}_0(6)}\exp\left[6 \cos\left(x-\frac{\pi}{2} \right) \right],$$
	which is a mixture of a Laplace density with mean 0 and scale 1 truncated to $[-\pi,\pi]$ and a von Mises density with mean $\frac{\pi}{2}$ and concentration parameter $\nu=6$. The von Mises(-Fisher) distributed samples are generated via rejection sampling with the uniform distribution as the proposal density. The true local modes are 0 and $\frac{\pi}{2}$ in terms of angular representations or $(0,0)$ and $(0,1)$ in Cartesian coordinates. The directional KDE on the simulated data and directional mean shift iterations are presented in Figure~\ref{fig:One_d_MS}. The bandwidth parameter here is selected as $h=0.3 < h_{\text{ROT}} \approx0.4181$ because the aforementioned rule of thumb $h_{\text{ROT}}$ tends to be oversmoothing when the underlying density is not von Mises distributed. In addition, the tolerance level for terminating the algorithm is set to $\epsilon=10^{-7}$. 
	
	\begin{figure}
		\captionsetup[subfigure]{justification=centering}
		\centering
		\begin{subfigure}[t]{.32\textwidth}
			\centering
			\includegraphics[width=1\linewidth]{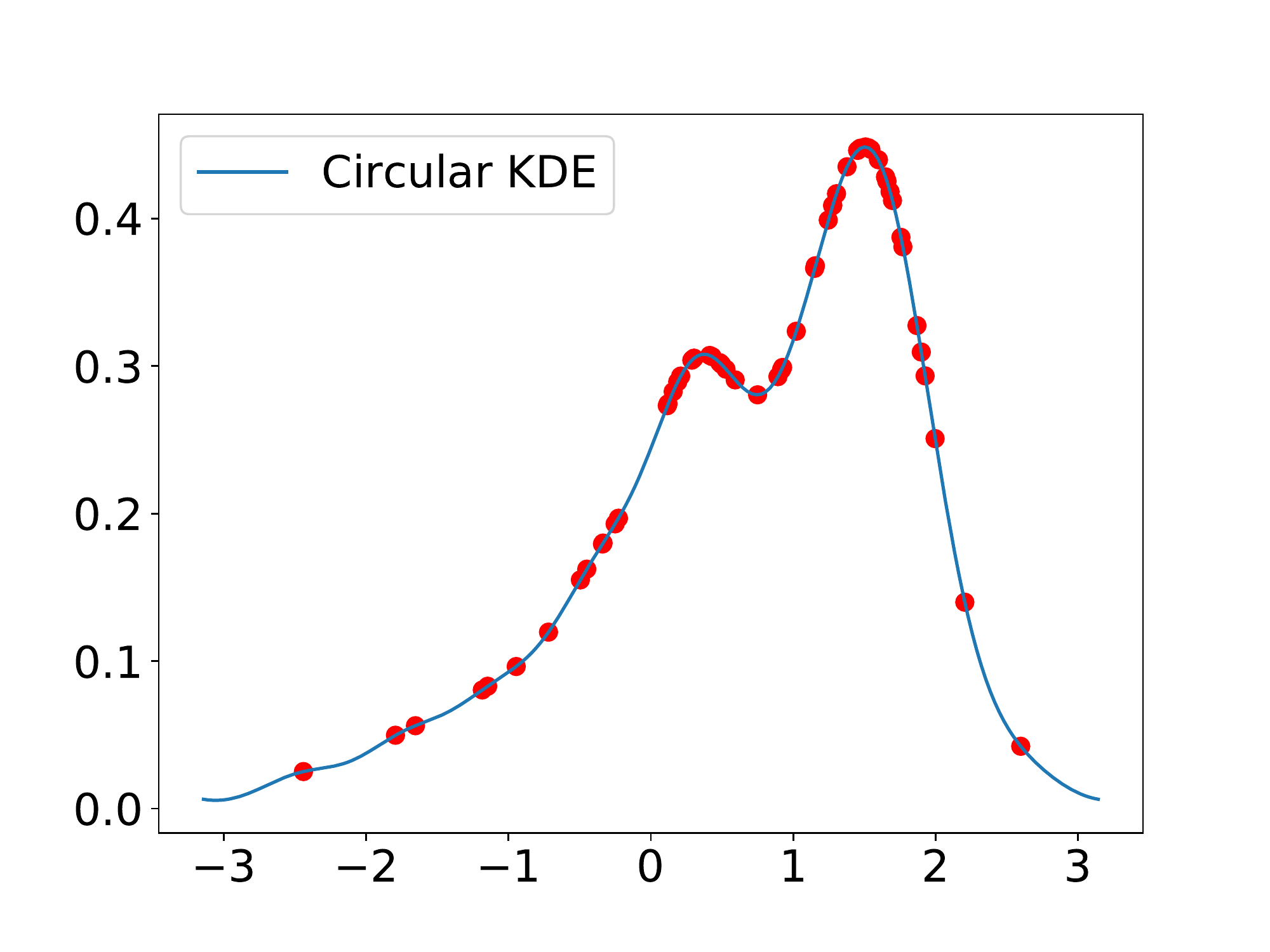}
			\caption{Step 0}
		\end{subfigure}
		\hfil
		\begin{subfigure}[t]{.32\textwidth}
			\centering
			\includegraphics[width=1\linewidth]{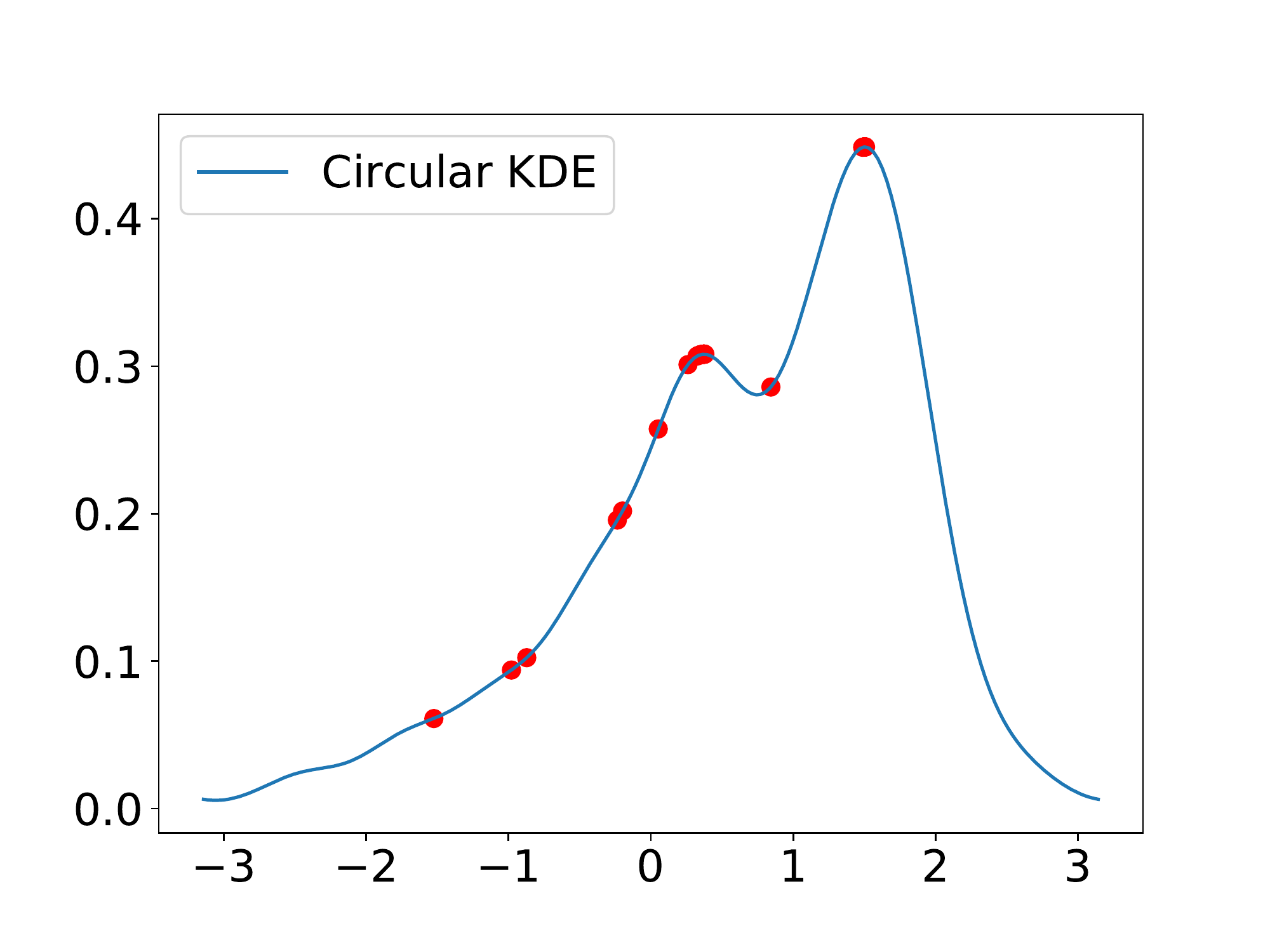}
			\caption{Step 11}
		\end{subfigure}%
		\hfil
		\begin{subfigure}[t]{.32\textwidth}
			\centering
			\includegraphics[width=1\linewidth]{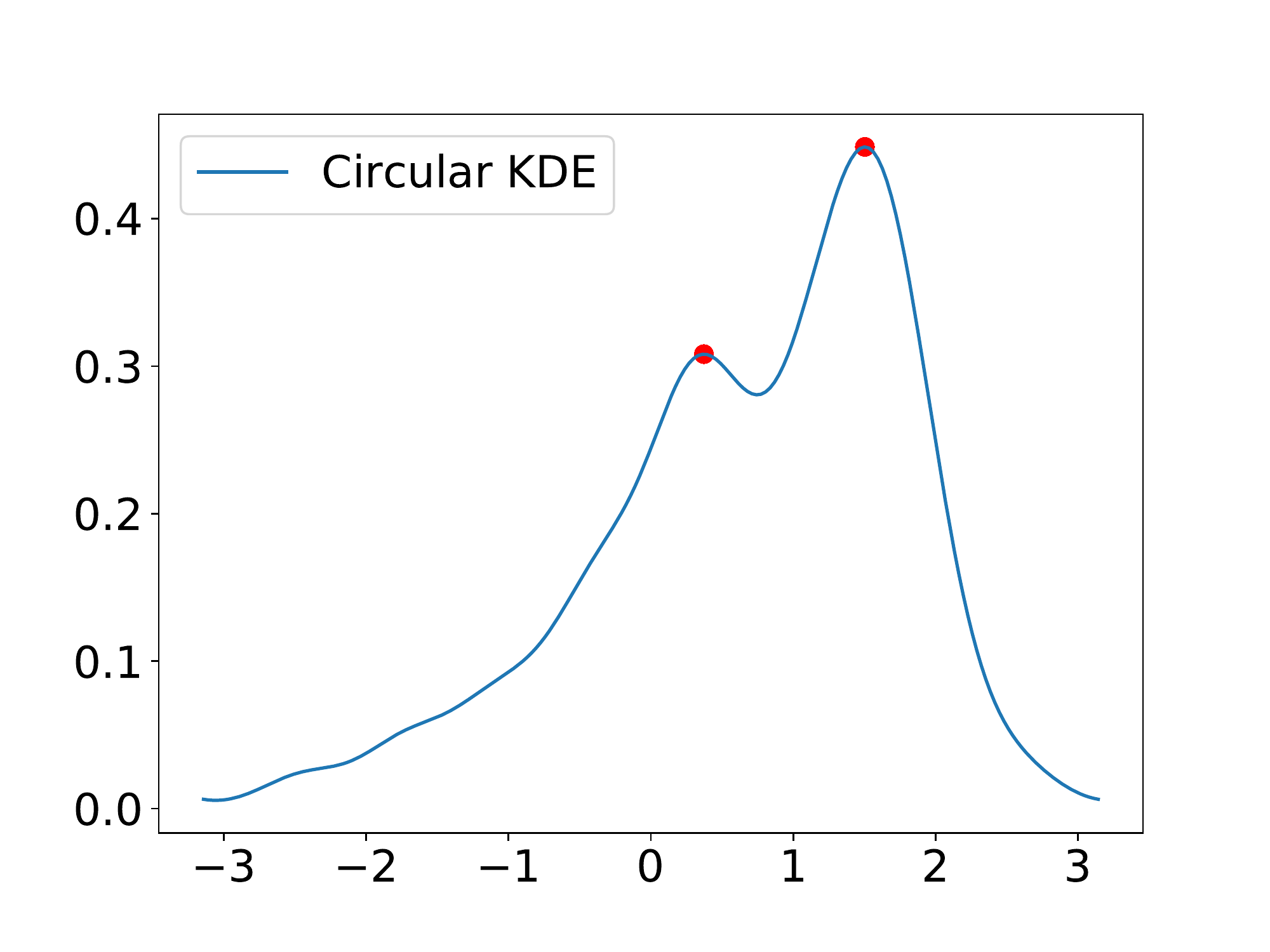}
			\caption{Step 45 (converged)}
		\end{subfigure}%
		
		\begin{subfigure}{.32\textwidth}
			\centering
			\includegraphics[width=1\linewidth]{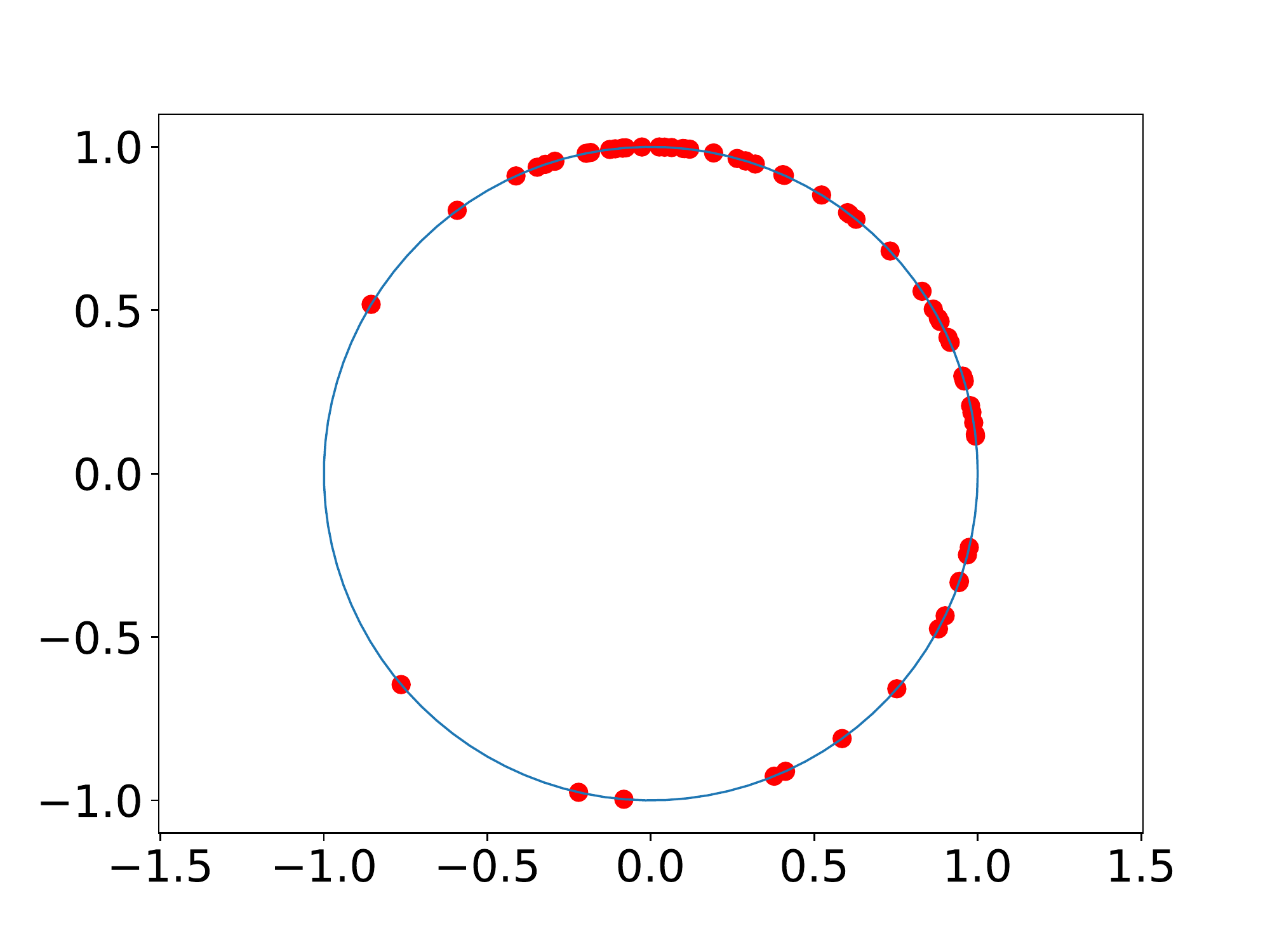}
			\caption{Step 0}
		\end{subfigure}
		\begin{subfigure}{.32\textwidth}
			\centering
			\includegraphics[width=1\linewidth]{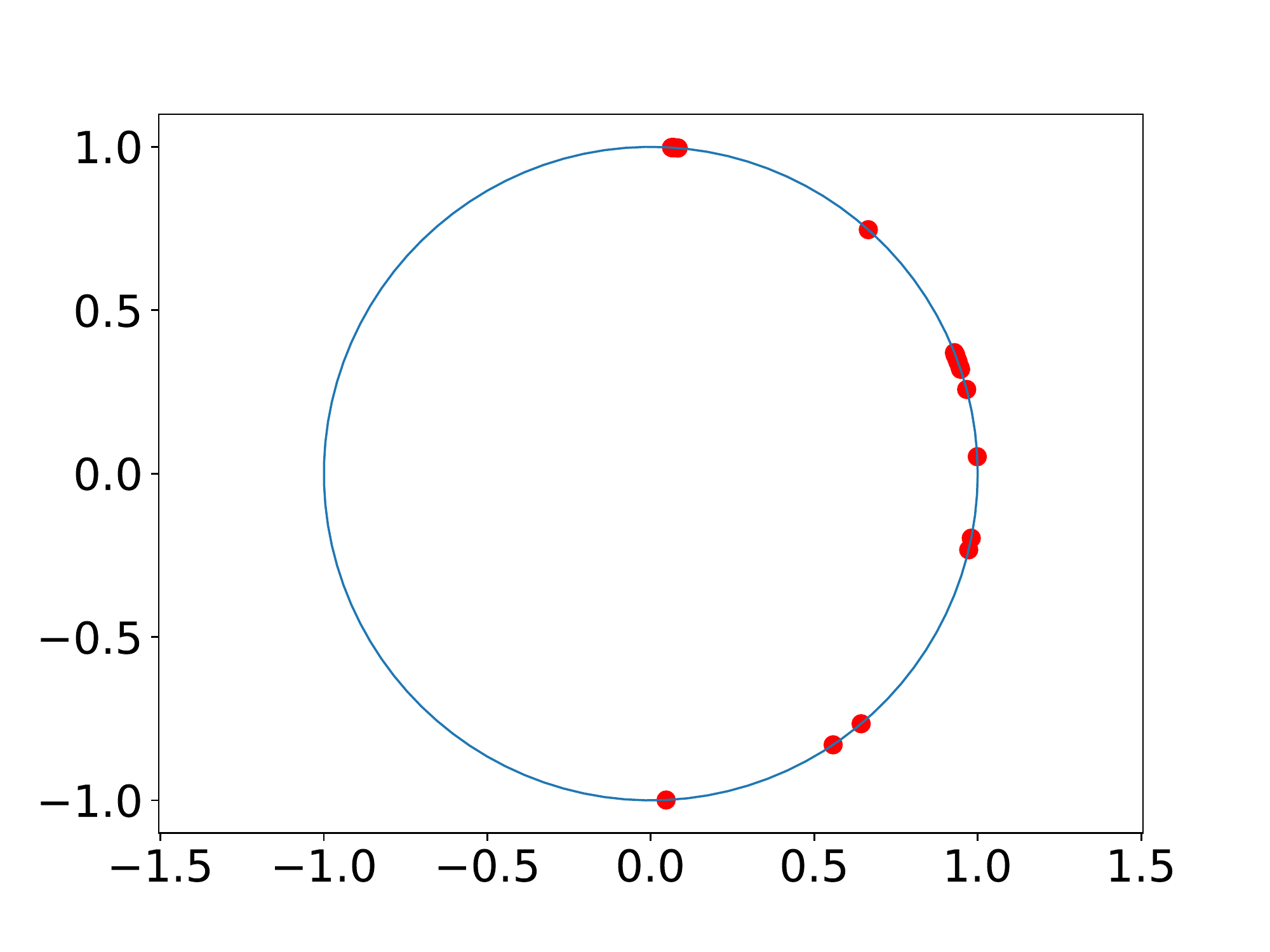}
			\caption{Step 11}
		\end{subfigure}
		\hfil
		\begin{subfigure}{.32\textwidth}
			\centering
			\includegraphics[width=1\linewidth]{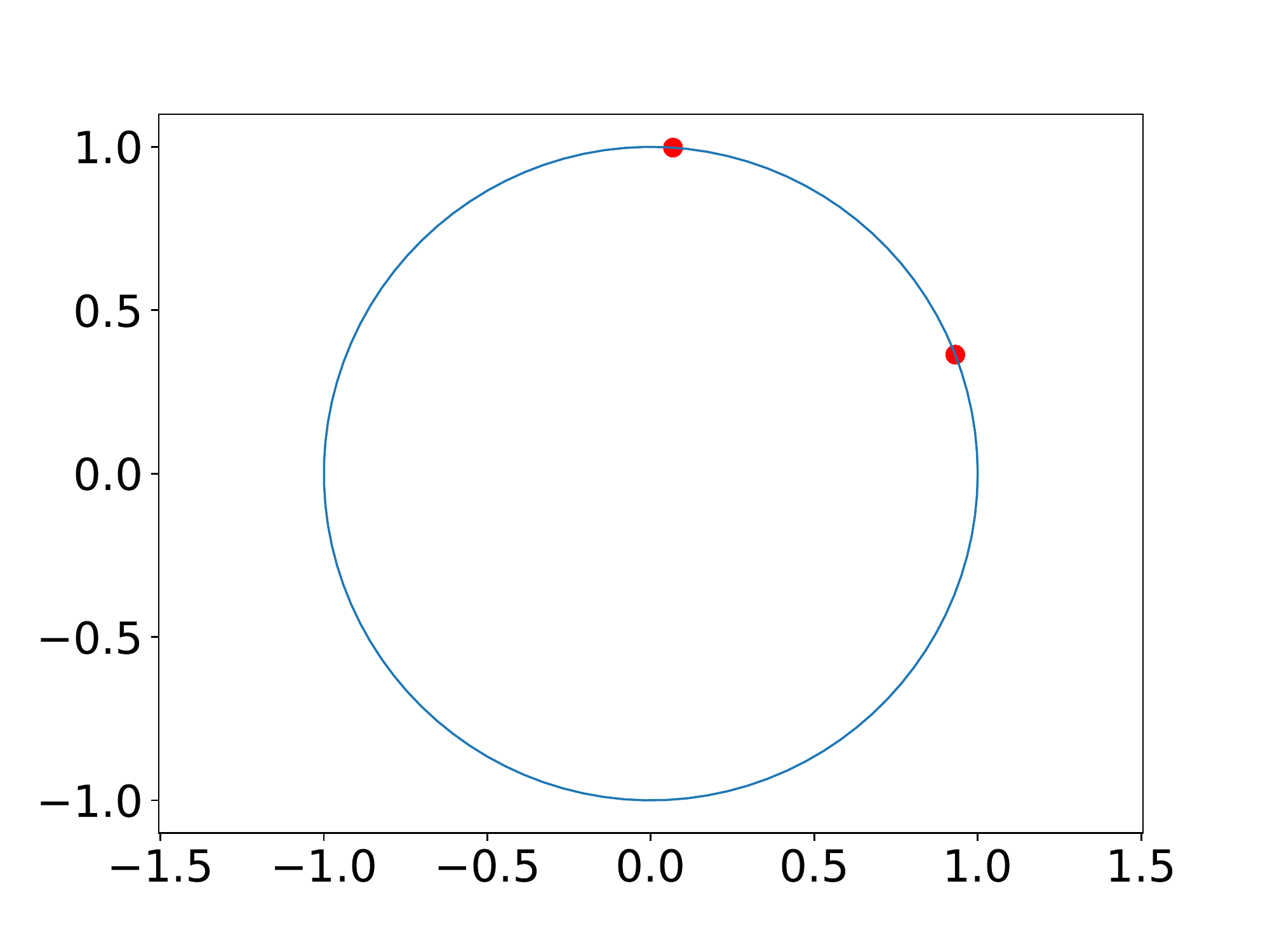}
			\caption{Step 45 (converged)}
		\end{subfigure}
		\hfil
		\begin{center}
			\begin{subfigure}[t]{.49\textwidth}
				\centering
				\includegraphics[width=1\linewidth]{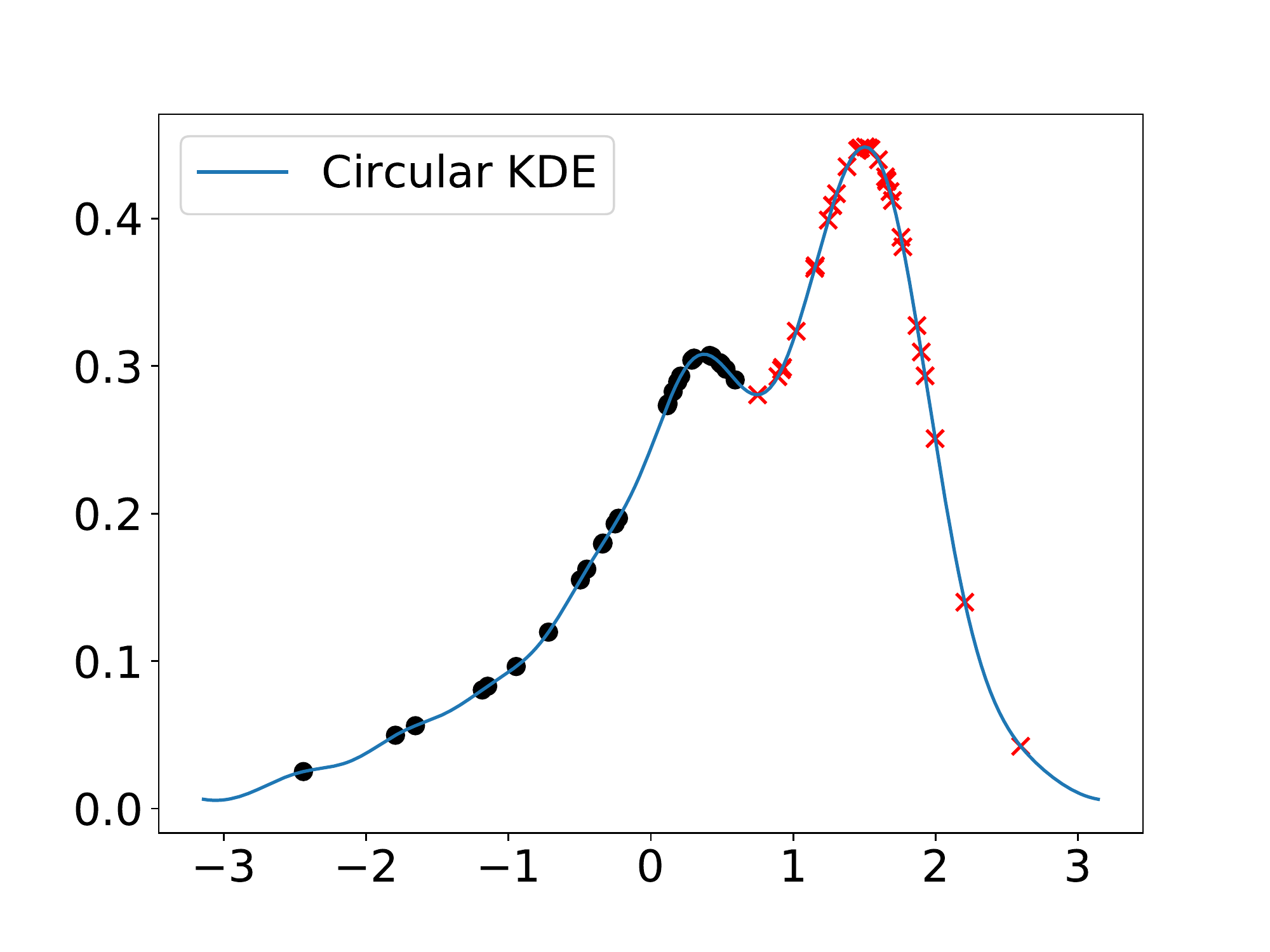}
				\caption{Mode clustering (viewed on function values)}
			\end{subfigure}
			\begin{subfigure}[t]{.49\textwidth}
				\centering
				\includegraphics[width=1\linewidth]{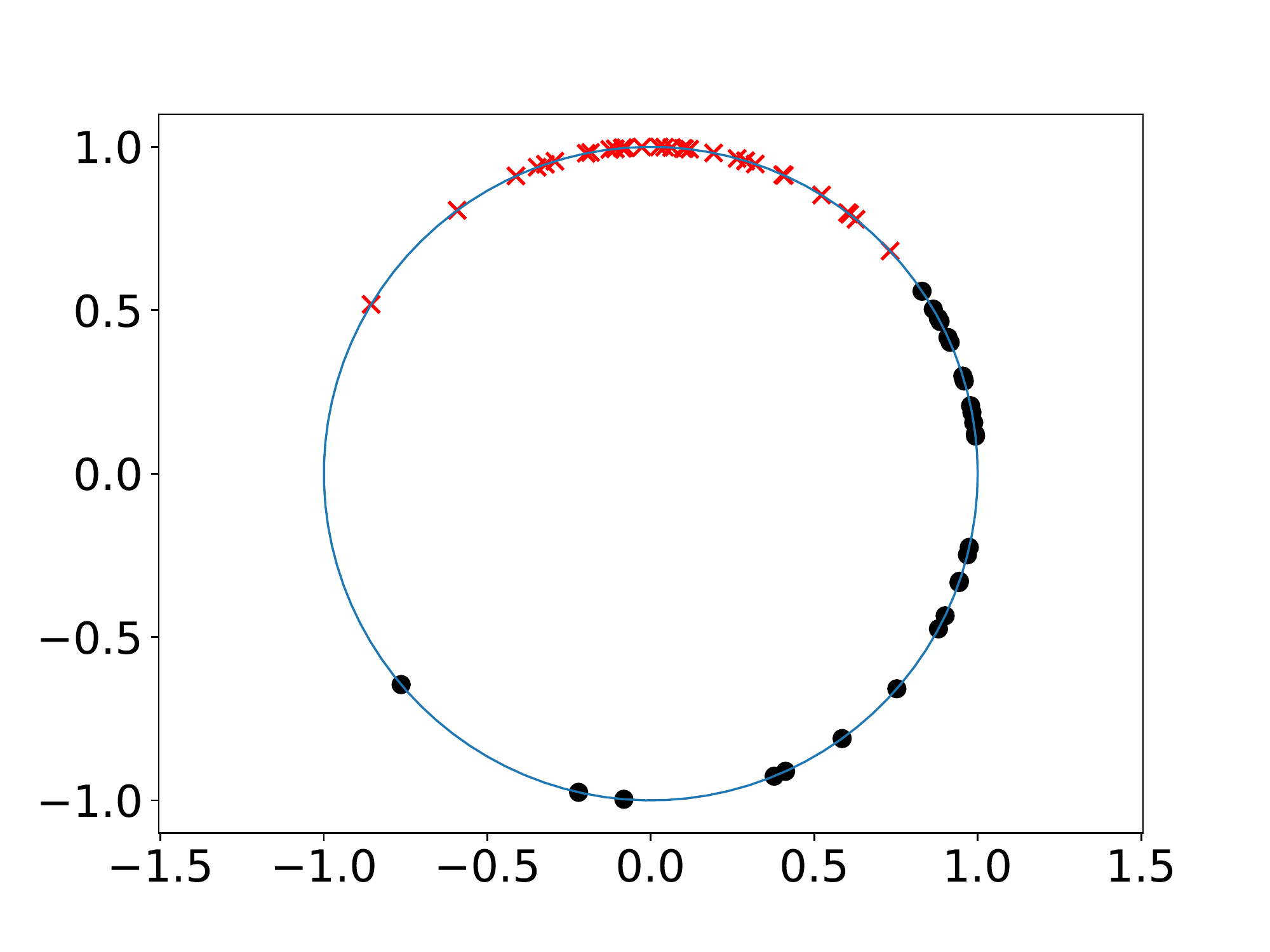}
				\caption{Mode clustering (viewed on $\Omega_1$)}
			\end{subfigure}
		\end{center}
		\caption{Directional mean shift algorithm performed on simulated data on $\Omega_1$. 
		{\bf Panel (a)-(c):} Outcomes under different iterations of the algorithm.
		{\bf Panel (d)-(f):} Corresponding locations of points in panels (a)-(c) on a unit circle.
		{\bf Panel (g) and (h):} Visualization of the affiliations of data points after mode clustering.}
		\label{fig:One_d_MS}
	\end{figure}
	
	Figure~\ref{fig:One_d_MS} empirically demonstrates the validity of Algorithm~\ref{Algo:MS} on the unit circle $\Omega_1$, in which all the simulated data points converge to the local modes of the circular density estimator. In addition, the misclassification rate in this simulation study is 0.1.

	\subsubsection{Spherical Case}	\label{sec:spherical}
	
	We simulate 1000 data points from the following density
	$$f_3(\bm{x}) = 0.3\cdot f_{\text{vMF}}(\bm{x};\bm{\mu}_1,\nu_1)+ 0.3\cdot f_{\text{vMF}}(\bm{x};\bm{\mu}_2,\nu_2) + 0.4\cdot f_{\text{vMF}}(\bm{x};\bm{\mu}_3,\nu_3)$$
	with $\bm{\mu}_1 \approx (-0.35, -0.61,-0.71)$, $\bm{\mu}_2 \approx (0.5,0,0.87)$, $\bm{\mu}_3=(-0.87,0.5,0)$ (or $[-120^{\circ},-45^{\circ}]$, $[0^{\circ},60^{\circ}]$, $[150^{\circ},0^{\circ}]$ in their precise spherical [longitude, latitude] coordinates), and $\nu_1=\nu_2=8$, $\nu_3=5$. The bandwidth parameter is selected using \eqref{bw_ROT}, and the tolerance level for terminating the algorithm is again set to $\epsilon=10^{-7}$. The results are presented in Figure \ref{fig:Two_d_ThreeM}. 
	
	\begin{figure}
		\captionsetup[subfigure]{justification=centering}
		\centering
		\begin{subfigure}[t]{.32\textwidth}
			\centering
			\includegraphics[width=1\linewidth,height=0.8\linewidth]{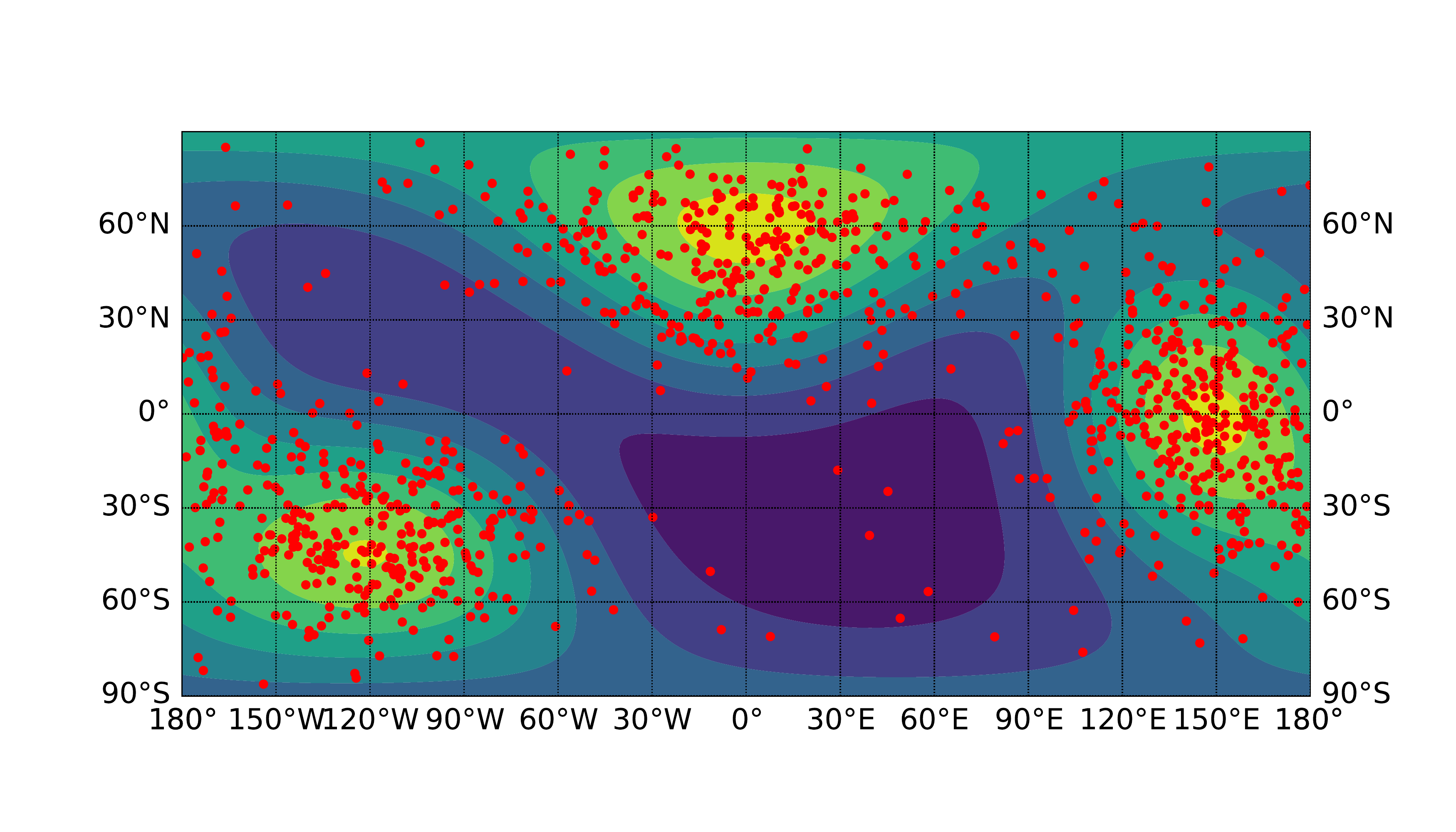}
			\caption{Step 0}
		\end{subfigure}
		\hfil
		\begin{subfigure}[t]{.32\textwidth}
			\centering
			\includegraphics[width=1\linewidth,height=0.8\linewidth]{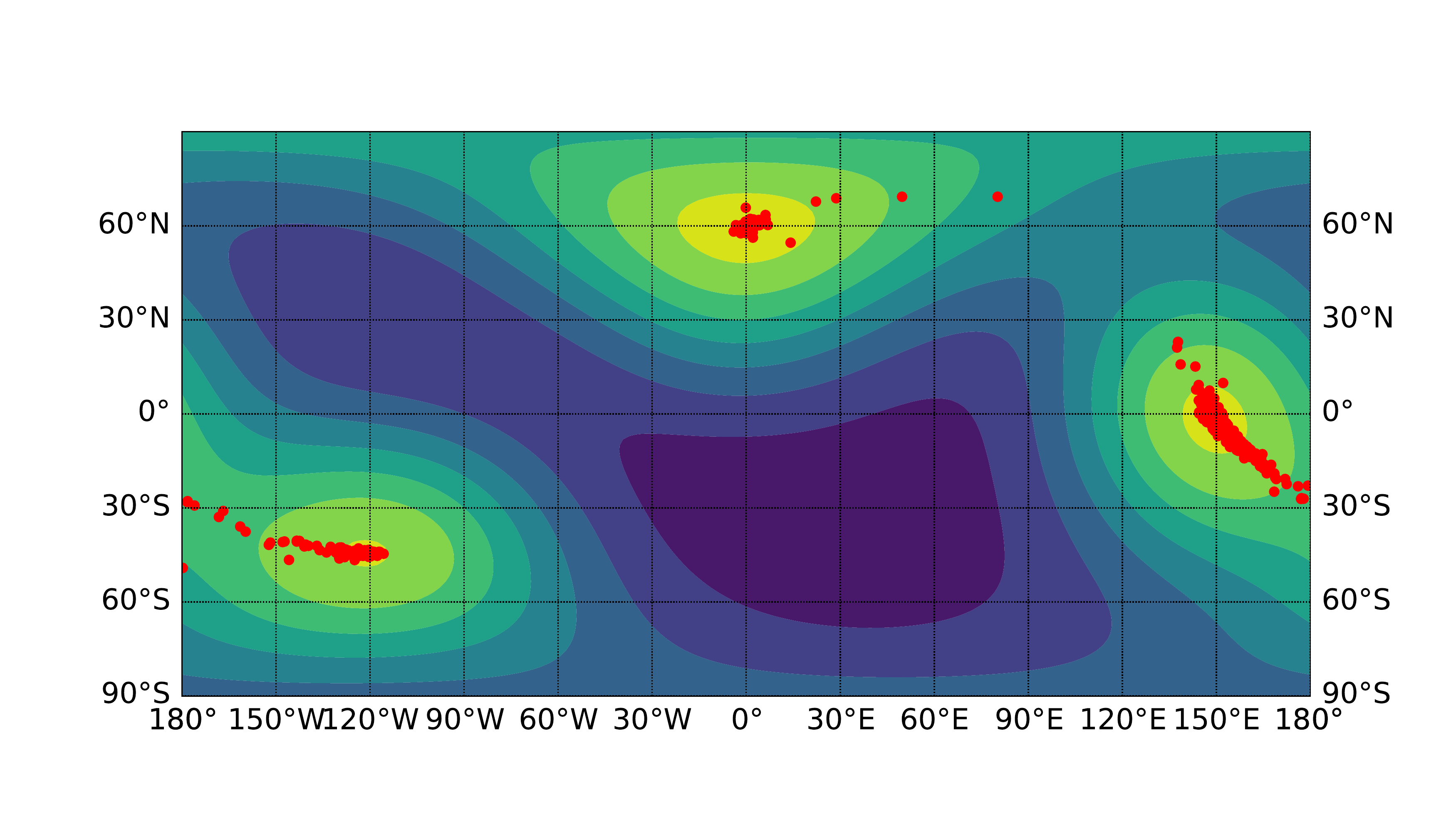}
			\caption{Step 5}
		\end{subfigure}%
		\hfil
		\begin{subfigure}[t]{.32\textwidth}
			\centering
			\includegraphics[width=1\linewidth,height=0.8\linewidth]{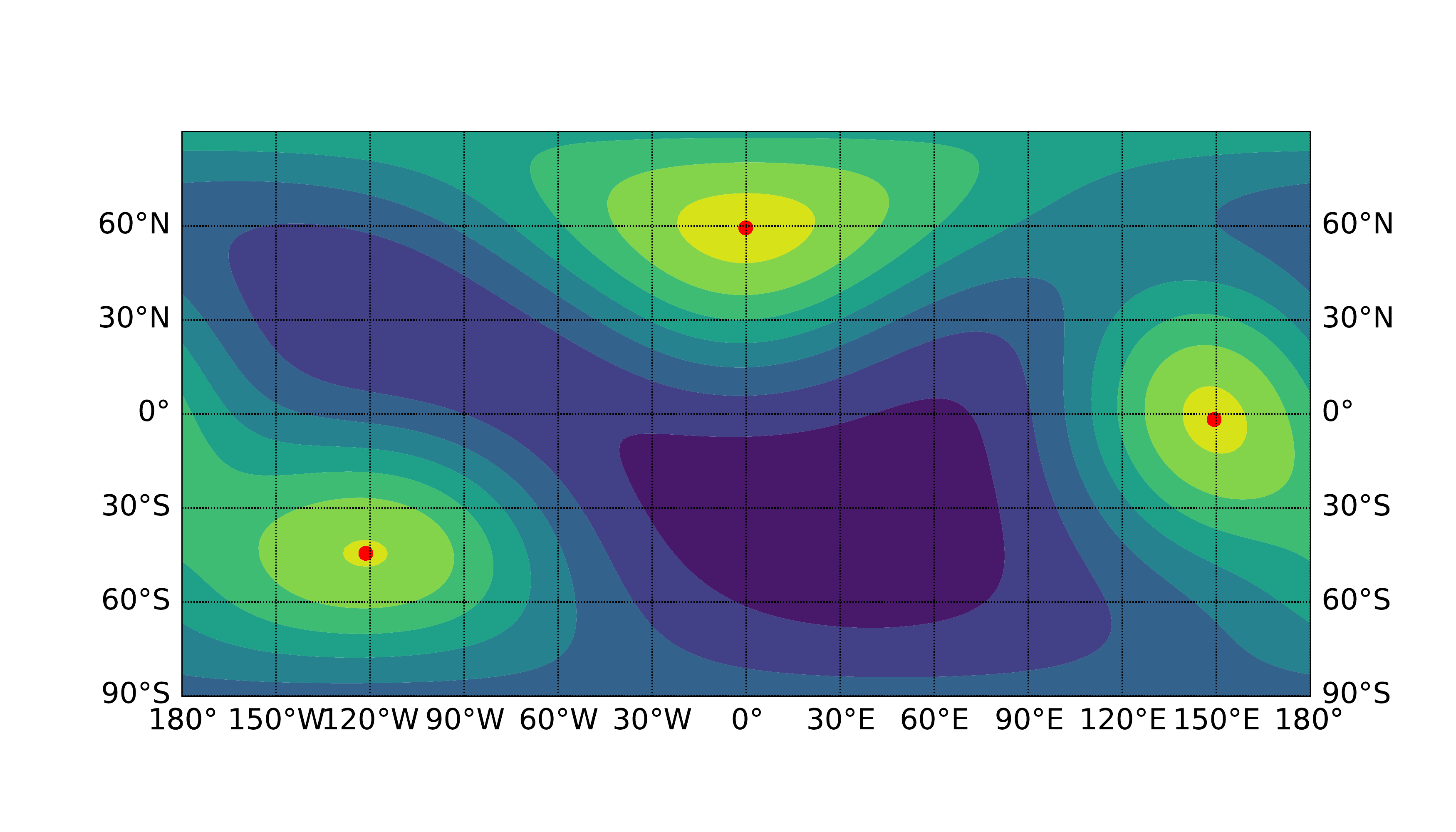}
			\caption{Step 28 (converged)}
		\end{subfigure}%
		
		\begin{subfigure}{.32\textwidth}
			\centering
			\includegraphics[width=1\linewidth]{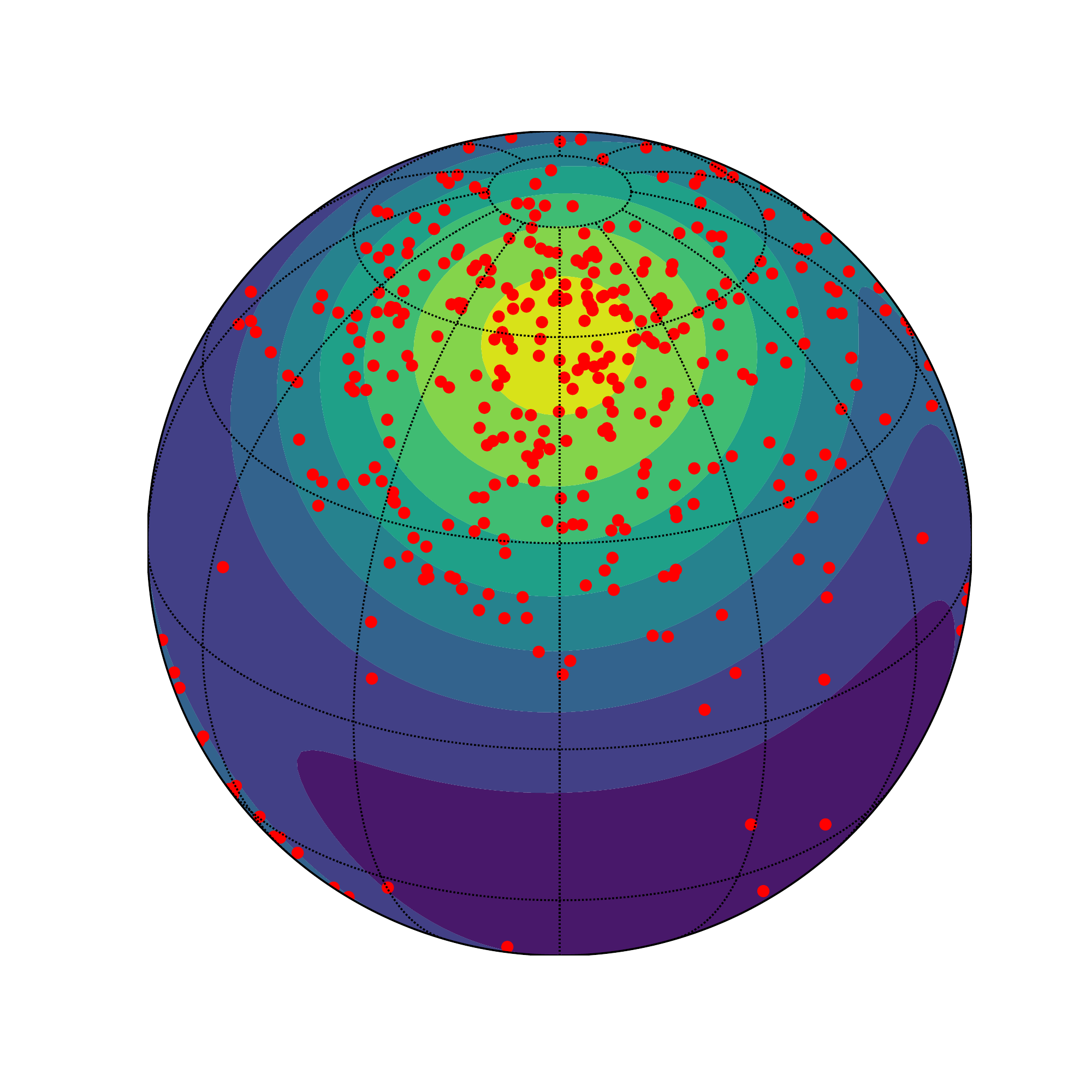}
			\caption{Step 0}
		\end{subfigure}
		\begin{subfigure}{.32\textwidth}
			\centering
			\includegraphics[width=1\linewidth]{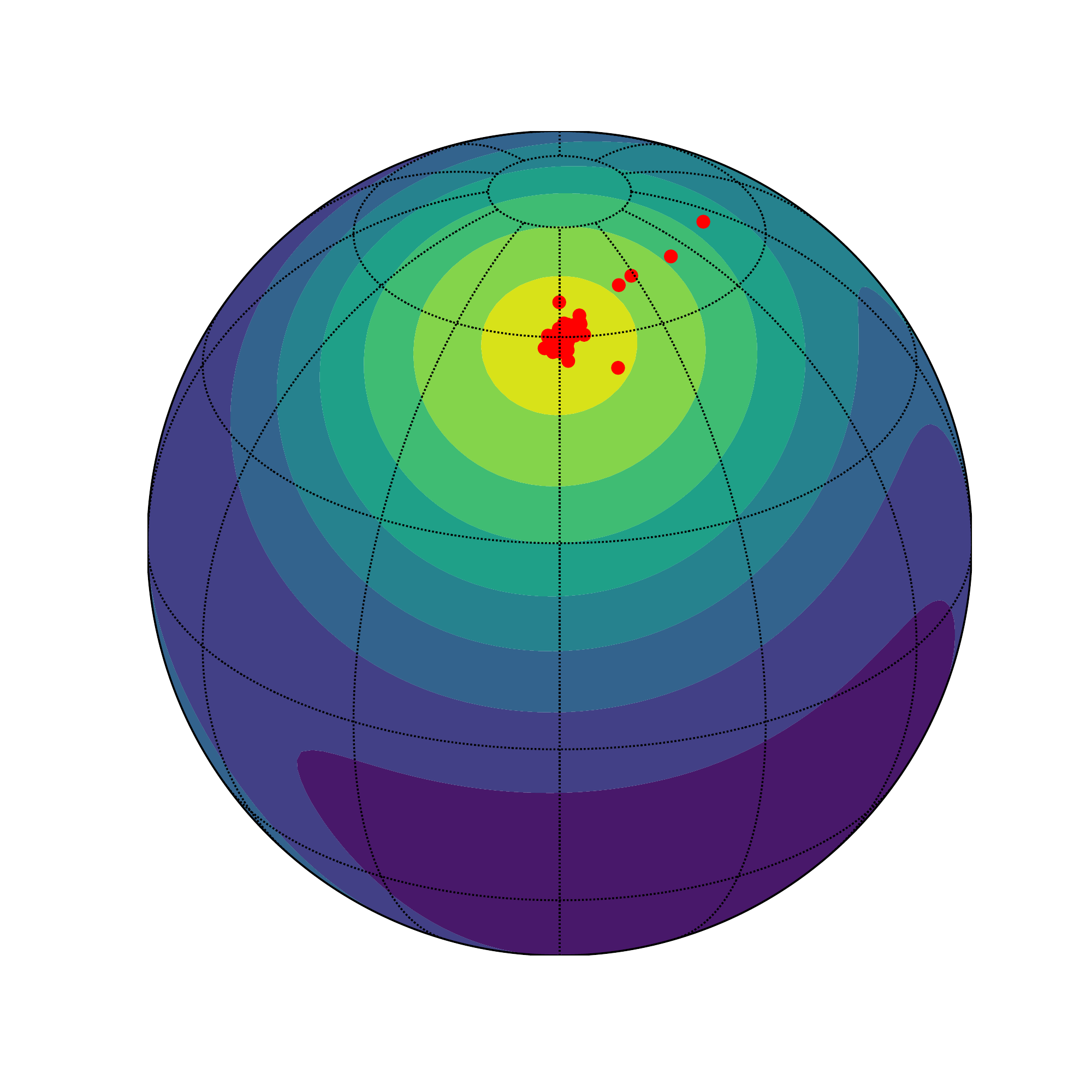}
			\caption{Step 5}
		\end{subfigure}
		\hfil
		\begin{subfigure}{.32\textwidth}
			\centering
			\includegraphics[width=1\linewidth]{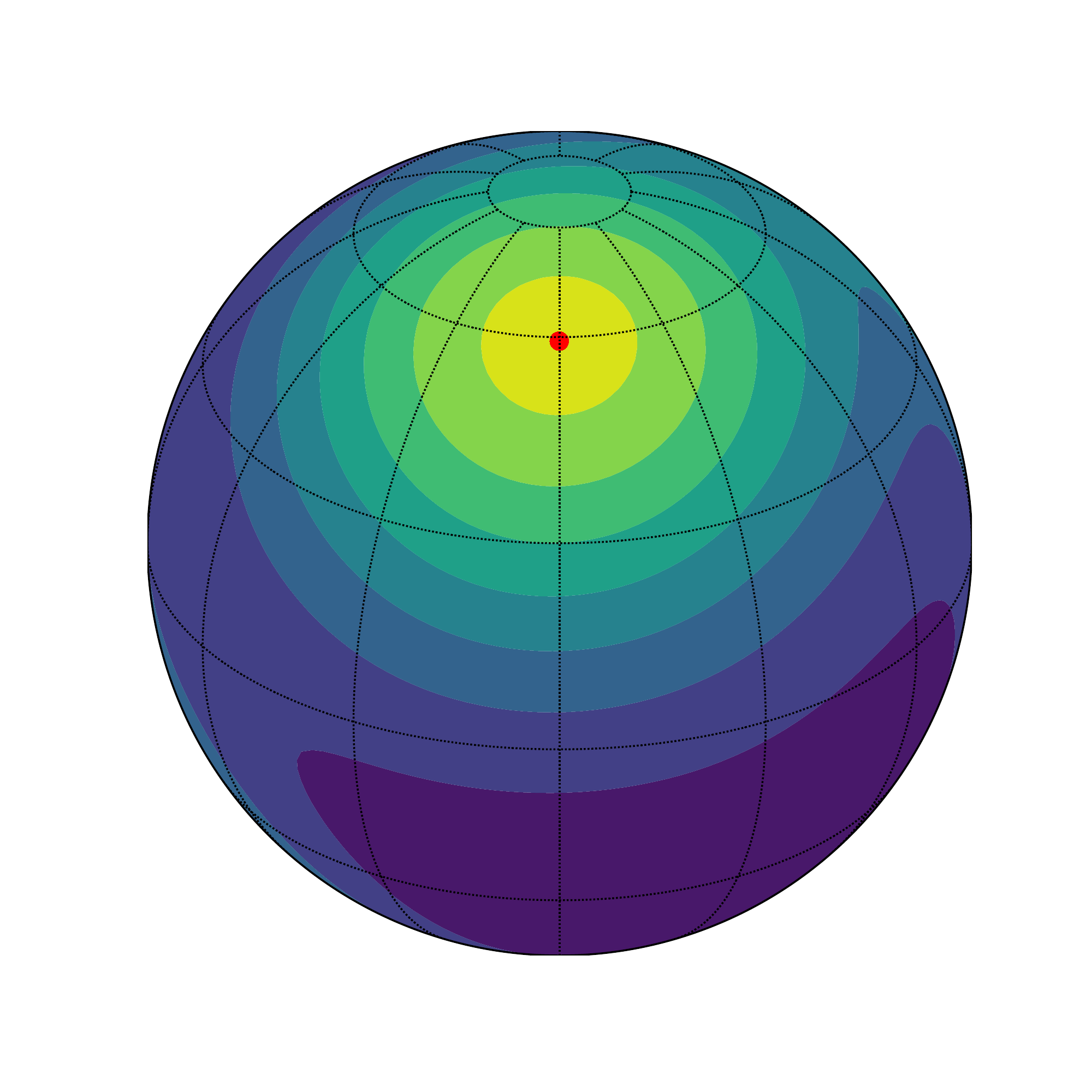}
			\caption{Step 28 (converged)}
		\end{subfigure}
		\hfil
		\begin{center}
			\begin{subfigure}[t]{.8\textwidth}
				\centering
				\includegraphics[width=1\linewidth]{Figures/MS_TripMode_MC_affi.pdf}
				\caption{Mode clustering (Hammer projection view)}
			\end{subfigure}
		\end{center}
		\caption{Directional mean shift algorithm performed on simulated data with three local modes on $\Omega_2$. 
		{\bf Panel (a)-(c):} Outcomes under different iterations of the algorithm displayed in a cylindrical equidistant view.
		{\bf Panel (d)-(f):} Corresponding locations of points in panels (a)-(c) in an orthographic view.
		Note: two local modes are at the back of the sphere; thus, we cannot directly see them. 
		{\bf Panel (g):} Clustering result under the Hammer projection (page 160 in \citealt{Snyder1989album}).}
		\label{fig:Two_d_ThreeM}
	\end{figure}
	
	In Figure~\ref{fig:Two_d_ThreeM}, all simulated data points converge to the local modes of the estimated directional density under the application of Algorithm~\ref{Algo:MS}; therefore, all the original data points are clustered according to where they converge. The confusion matrix in this simulation study is $\begin{bmatrix}
	\bm{278} & 0 & 9\\
	 0 & \bm{323} & 1\\
	 20 & 8 & \bm{361}\\
	\end{bmatrix}$ and the misclassification rate is thus 0.038. 
%	Note that since this is a simulated data, we know the exact label of each observation, so we can compute the misclassification rate.
	Moreover, the total number of iterative steps is much lower than the case with a single mode in Appendix~\ref{Appendix:Two_d_OneM} (Figure~\ref{fig:Two_d_OneM}). We also observe that most of data points already converge to the local modes of the directional KDE after a few initial steps, while most of the subsequent iterations handle those points that are geodesically far away from an estimated local mode and have small estimated (tangent/Riemannian) gradients on their iterative paths. 
	
	\subsubsection{$q$-Directional Case with $q>2$}
	
	Our algorithmic formulation of the directional mean shift algorithm (Algorithm~\ref{Algo:MS}) and its associated learning theory are valid on any general (intrinsic) dimension $q$ of $\Omega_q$. For this reason, we are also interested in how the algorithm behaves as the dimension $q$ of directional data increases. We randomly simulate 1000 data points from each of the following densities repeatedly,
	$$\sum\limits_{i=1}^4 0.25 \cdot f_{\text{vMF}}(\bm{x};\bm{\mu}_{i,q},\nu')$$
	with $\bm{\mu}_{i,q} = \bm{e}_{i,q+1}\in \Omega_q \subset \mathbb{R}^{q+1}$ for $q=3,4,...,12$ and $i=1,...,4$, where the concentration parameter $\nu'=10$ and the mixture weight of each density component are constant. Here, $\left\{\bm{e}_{i,q+1}\right\}_{i=1}^{q+1} \subset \Omega_q$ is the standard basis of the ambient Euclidean space $\mathbb{R}^{q+1}$. For each value of dimension $q$, we repeat the data simulation process 20 times and compute the average misclassification rate of mode clustering with Algorithm~\ref{Algo:MS} on each simulated data set accordingly. Figure~\ref{fig:boxplot} shows the boxplots of misclassification rates.
%	Note that since this is a simulated data, we know the exact label of each observation so we can compute the misclassification rate.
	
	\begin{figure}
		\centering
		\includegraphics[width=0.8\linewidth]{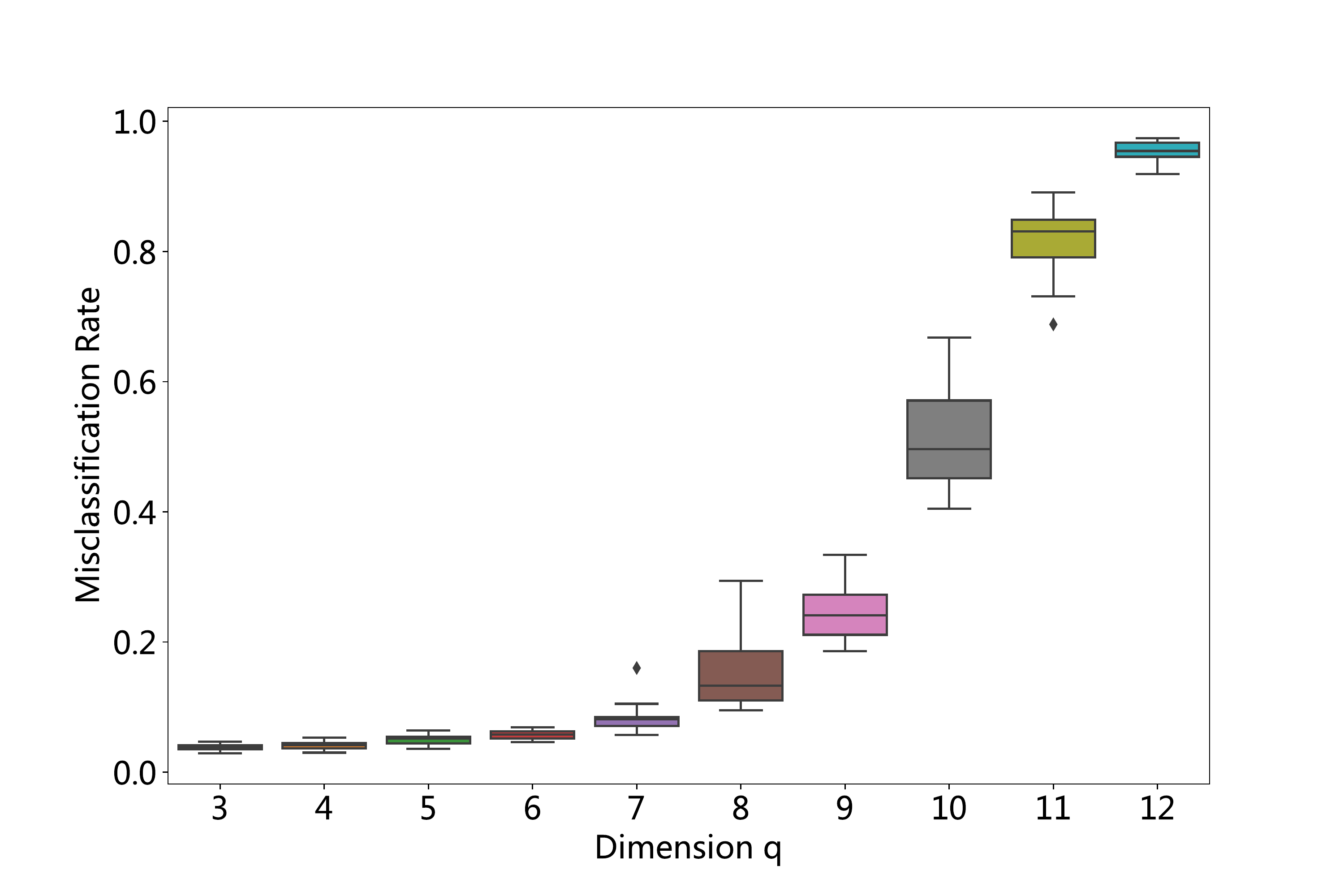}
		\caption{Boxplots of misclassification rates of mode clustering under different values of dimension $q$}
		\label{fig:boxplot}
	\end{figure}

    As the dimension $q$ of directional data becomes larger, the misclassification rates of mode clustering with Algorithm~\ref{Algo:MS} also gradually increase to 1 (the worst case), which in turn implies that the ability of Algorithm~\ref{Algo:MS} to identify the density component from which a data point is simulated tends to deteriorate with respect to the dimension. Such inferior performances of the directional mean shift algorithm on higher-dimensional hyperspheres are not surprising because (i) we do not fine-tune the bandwidth parameter (but simply apply the rule of thumb \eqref{bw_ROT}), and (ii) the algorithm is subject to the curse of dimensionality (see also Section~\ref{Sec:Com_Complexity}). However, since directional data in real-world applications mostly lie on a (hyper)sphere with dimension $q\leq 3$, Algorithm~\ref{Algo:MS} is effective in practice, as we will demonstrate in Section~\ref{Sec:real_world_app}.

	\subsection{Real-World Applications}
	\label{Sec:real_world_app}
	
	We illustrate the practical relevance of the directional mean shift algorithm (Algorithm~\ref{Algo:MS}) with two applications in astronomy and seismology.
	
	\subsubsection{Craters on Mars}
	
	The distribution and cluster configuration of craters on Mars shed light on the planetary subsurface structure (water or ice), relative surfaces ages, resurfacing history, and past geologic processes \citep{Lakes_On_Mars,BARLOW2015}. \cite{Unif_test_hypersphere2020} conducted three different statistical tests (Cramer-von Mises, Rothman, and Anderson-Darling-like tests) on Martian crater data to statistically validate the non-uniformity of the crater distribution on Mars. We apply the directional KDE \eqref{Dir_KDE} together with the directional mean shift algorithm to further estimate the density of craters and determine crater clusters on the surface of Mars. Martian crater data are publicly available on the Gazetteer of Planetary Nomenclature database (\url{https://planetarynames.wr.usgs.gov/AdvancedSearch}) of the International Astronomical Union (IUA). The positions of craters are recorded in areocentric coordinates (the planetocentric coordinates on Mars) so that the areocentric longitudes range from $0^{\circ}$ to $360^{\circ}$ and areocentric latitudes range from $-90^{\circ}$ to $90^{\circ}$. As craters with areocentric longitudes greater than $180^{\circ}$ are on the western hemisphere of Mars, we transform their longitudes back to the interval $(-180^{\circ},0^{\circ})$. (Note that $360^{\circ}$ in longitude corresponds to $0^{\circ}$ west/east after transformation.) In addition, we remove those small craters whose diameters are less than 5 kilometers from the crater data, as their presence may provide spurious information. After trimming, the data set contains 1653 craters. The bandwidth parameter is selected using \eqref{bw_ROT}, which becomes $h_{\text{ROT}} \approx 0.338$ for the trimmed data set. The estimated distribution of craters on Mars and clustering results are presented in Figure~\ref{fig:Mars_data}.
	
	\begin{figure}
		\captionsetup[subfigure]{justification=centering}
		\centering
		\begin{subfigure}[t]{.32\textwidth}
			\centering
			\includegraphics[width=1\linewidth,height=0.8\linewidth]{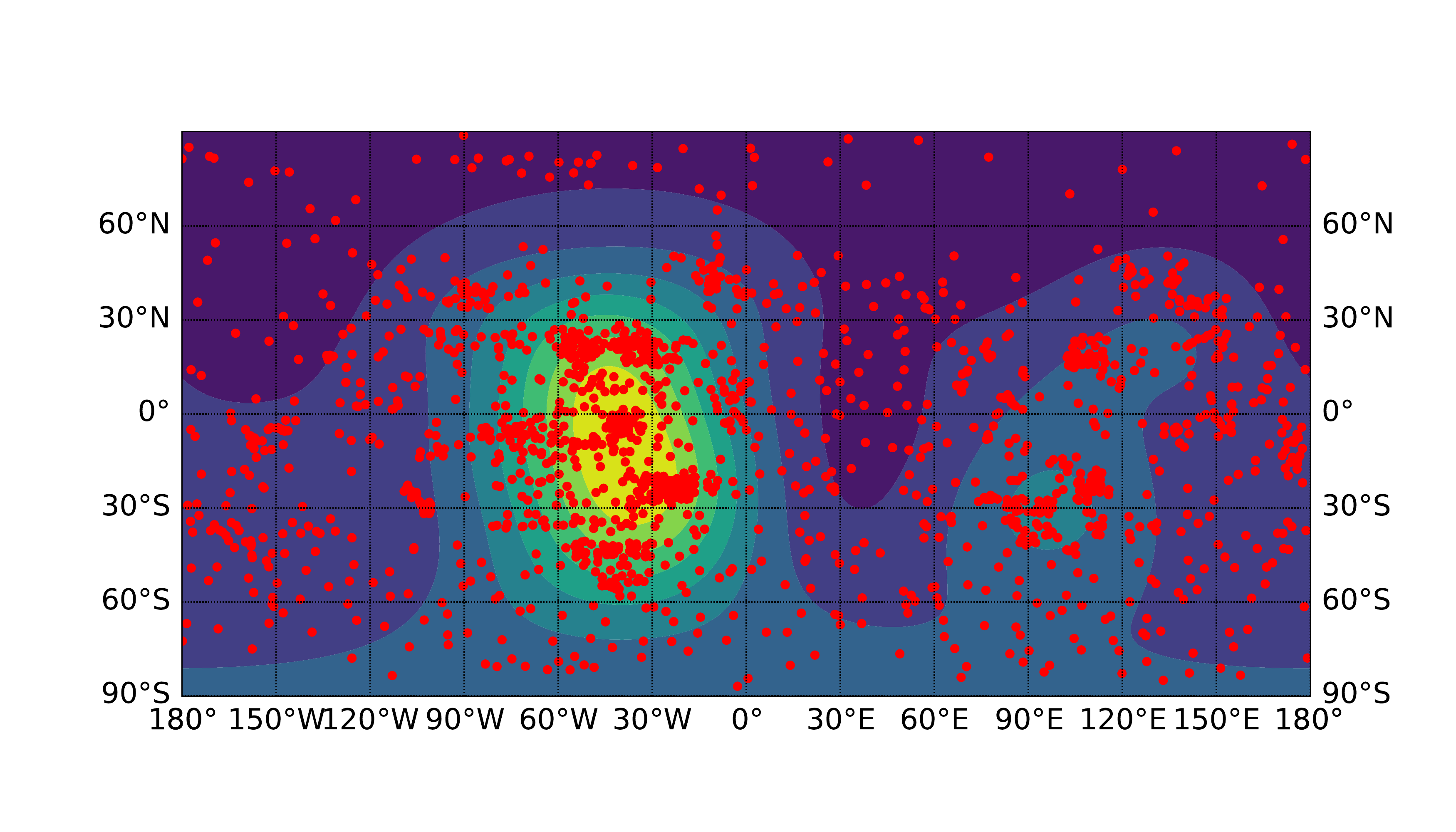}
			\caption{Step 0}
		\end{subfigure}
		\hfil
		\begin{subfigure}[t]{.32\textwidth}
			\centering
			\includegraphics[width=1\linewidth,height=0.8\linewidth]{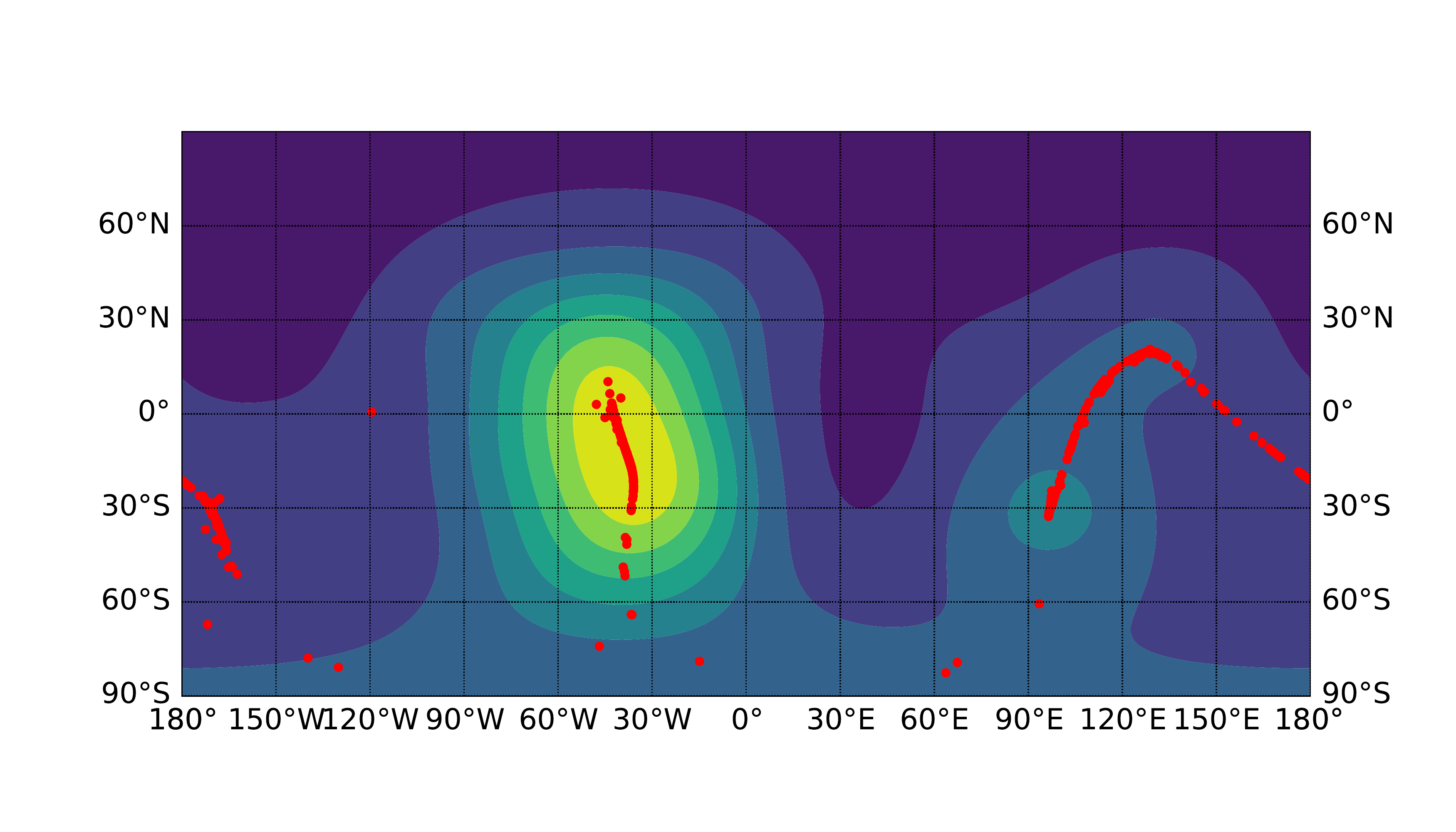}
			\caption{Step 17}
		\end{subfigure}%
		\hfil
		\begin{subfigure}[t]{.32\textwidth}
			\centering
			\includegraphics[width=1\linewidth,height=0.8\linewidth]{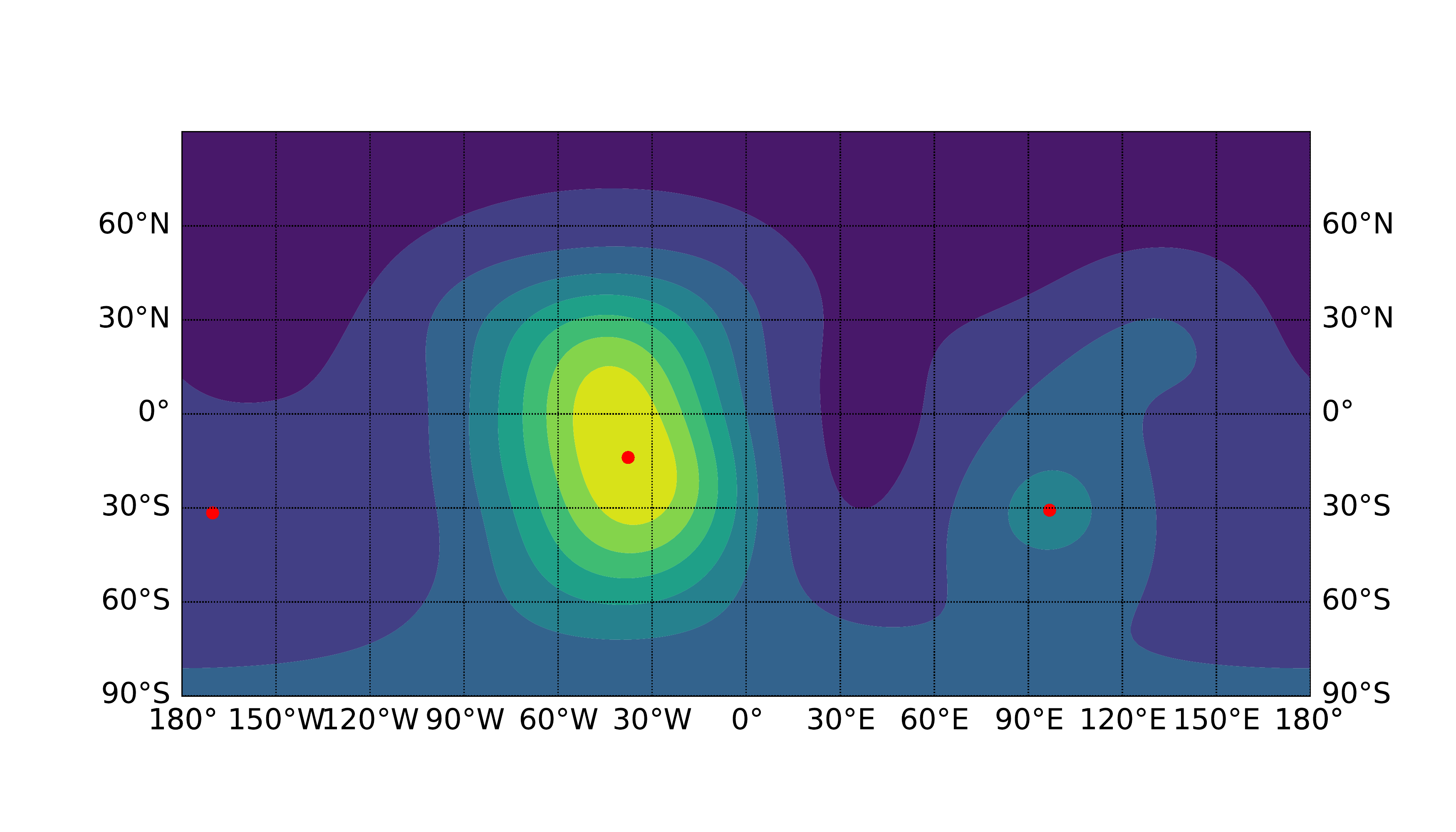}
			\caption{Step 178 (converged)}
		\end{subfigure}%
		
		\begin{subfigure}{.32\textwidth}
			\centering
			\includegraphics[width=1\linewidth]{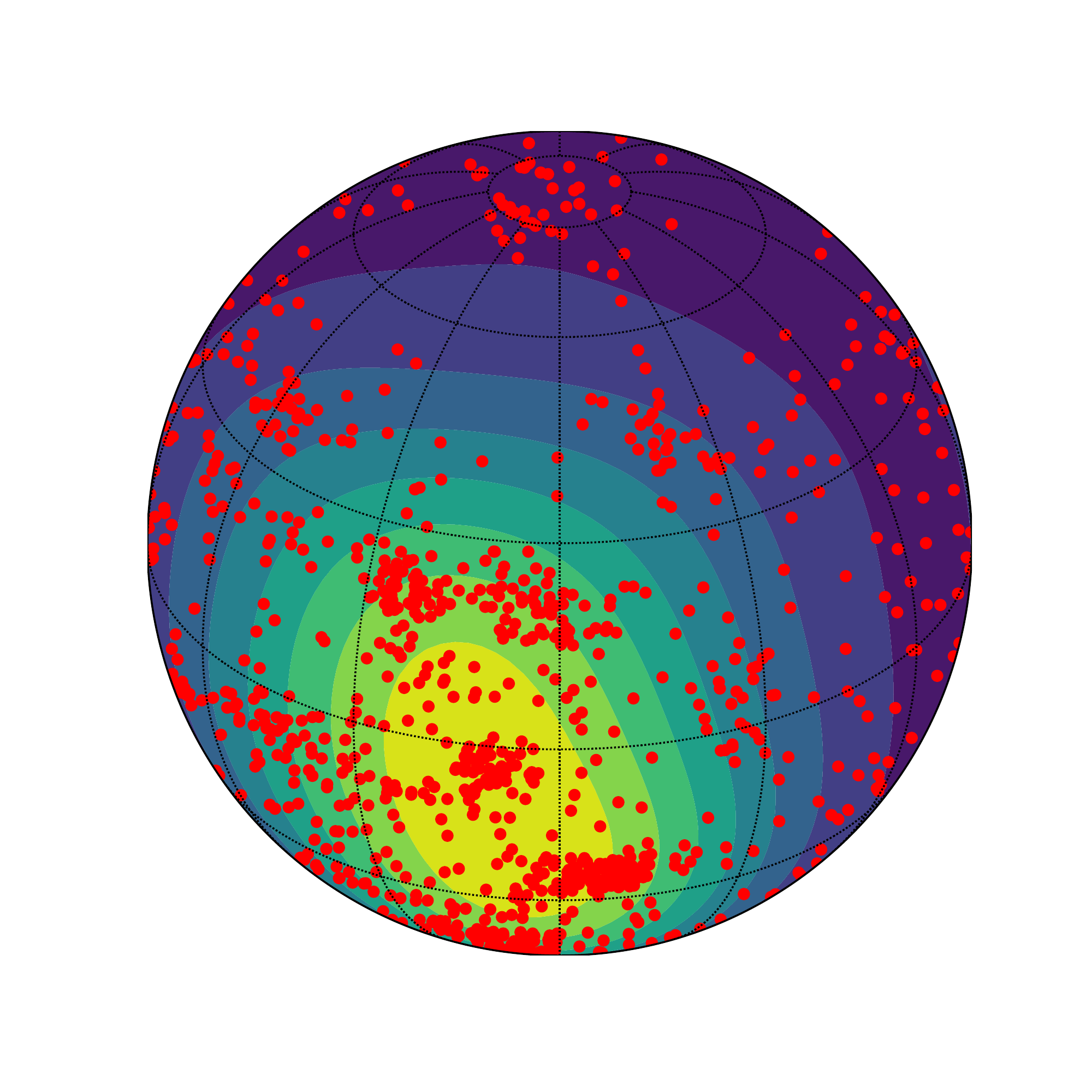}
			\caption{Step 0}
		\end{subfigure}
		\begin{subfigure}{.32\textwidth}
			\centering
			\includegraphics[width=1\linewidth]{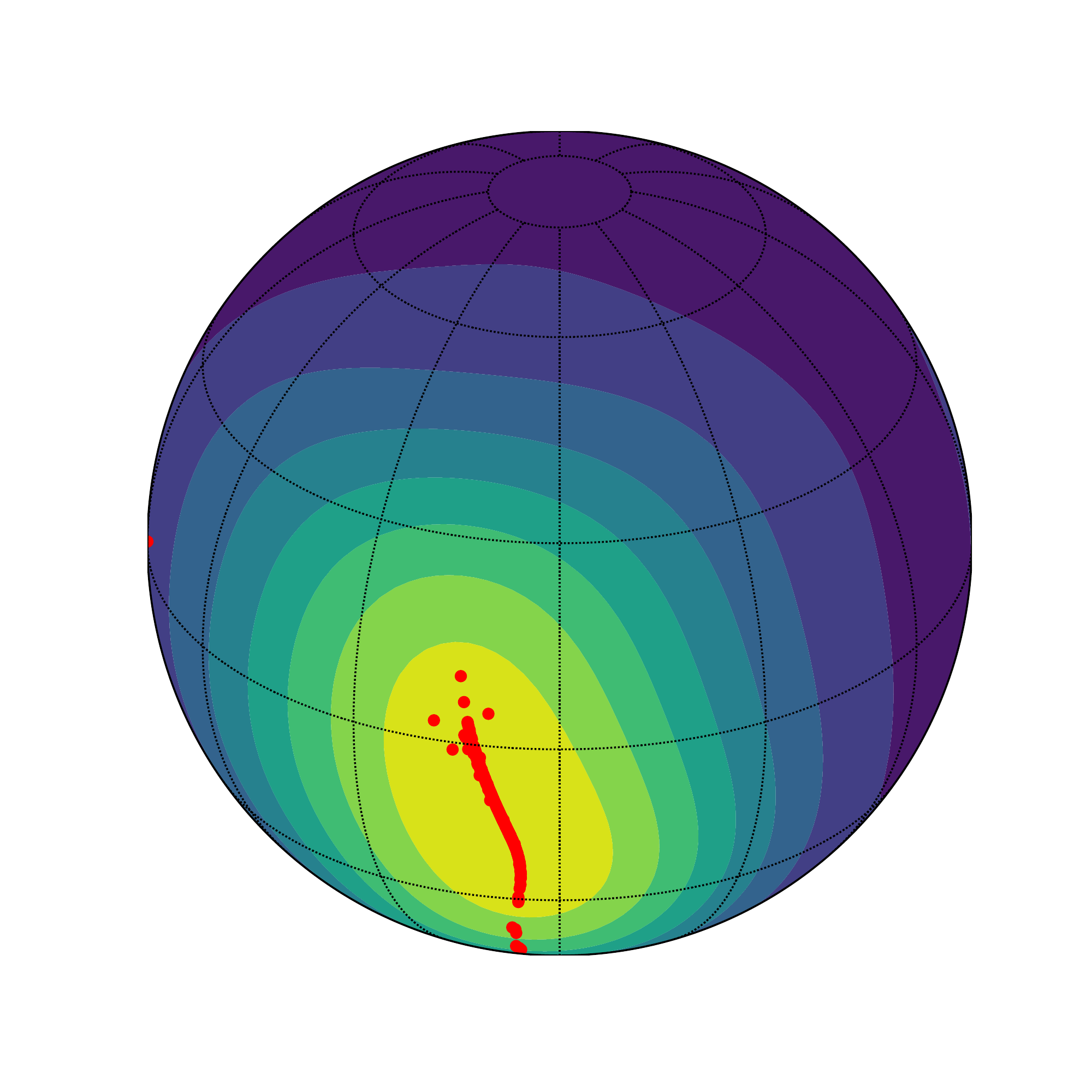}
			\caption{Step 17}
		\end{subfigure}
		\hfil
		\begin{subfigure}{.32\textwidth}
			\centering
			\includegraphics[width=1\linewidth]{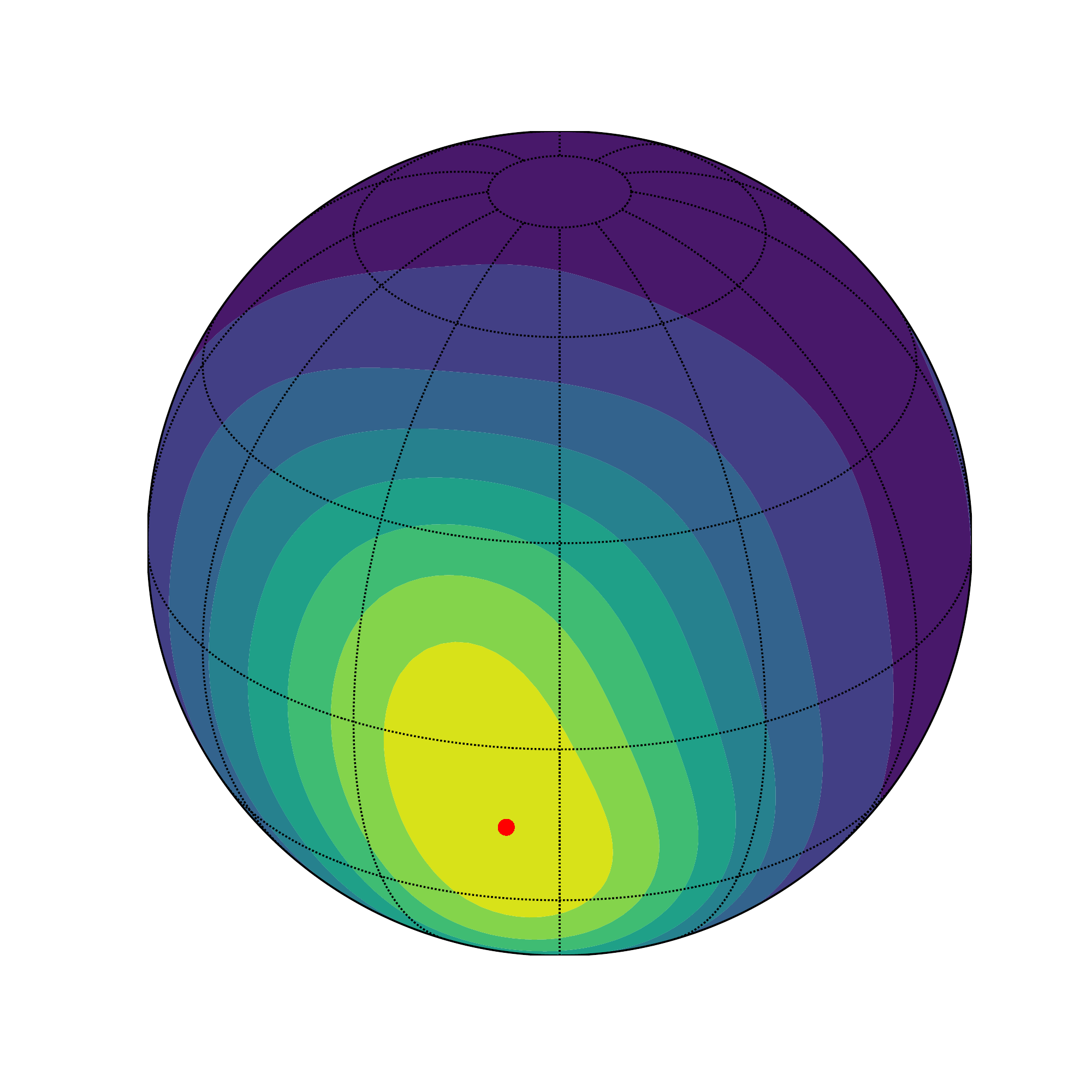}
			\caption{Step 178 (converged)}
		\end{subfigure}
		\hfil
		\begin{center}
			\begin{subfigure}[t]{.8\textwidth}
				\centering
				\includegraphics[width=1\linewidth]{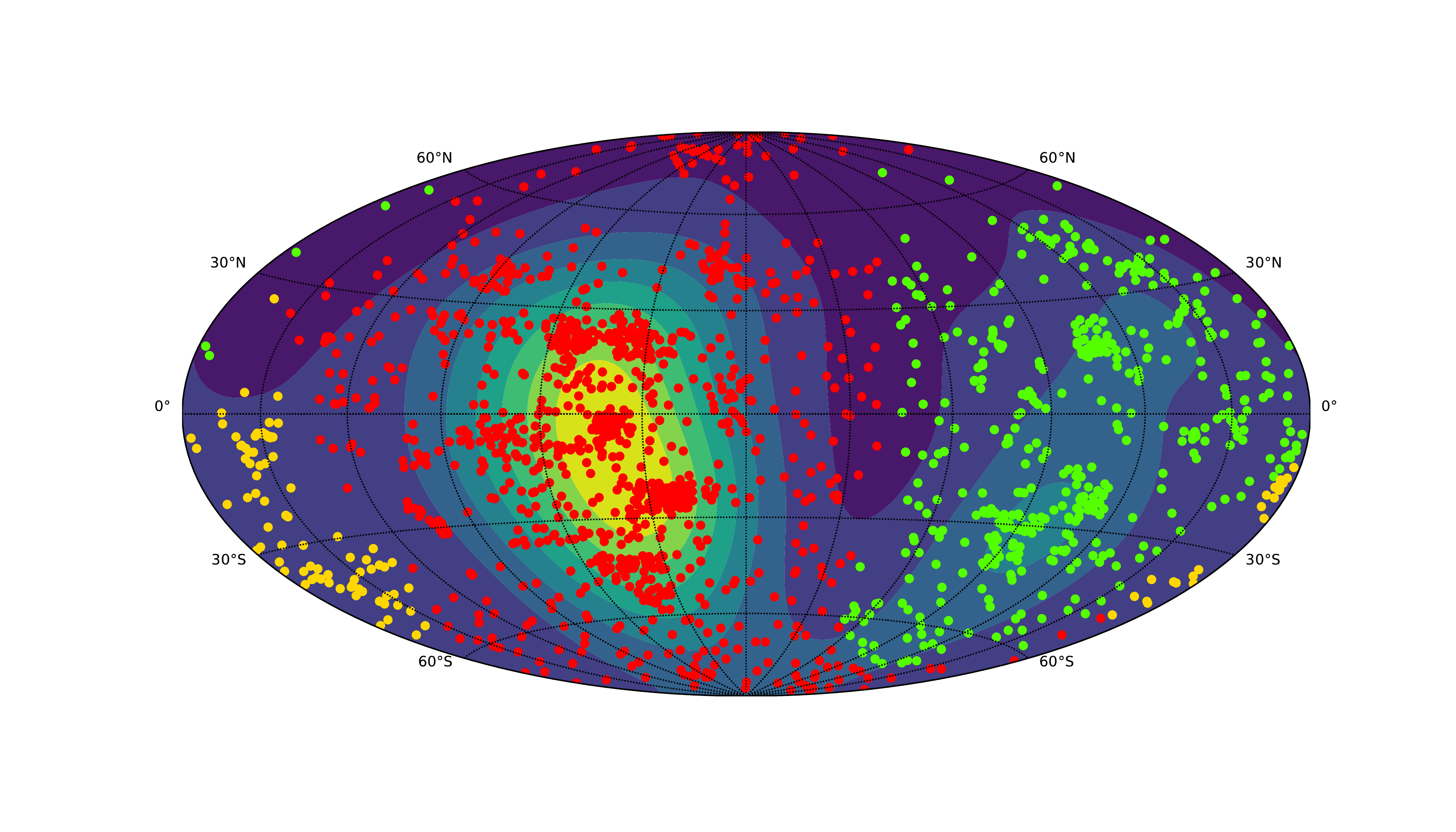}
				\caption{Mode clustering (Hammer projection view)}
			\end{subfigure}
		\end{center}
		\caption{Directional mean shift algorithm performed on Martian crater data. 
		The analysis is displayed in a similar way to Figure~\ref{fig:Two_d_ThreeM}. {\bf Panel (a)-(c):} Outcomes under different iterations of the algorithm displayed in a cylindrical equidistant view.
		{\bf Panel (d)-(f):} Corresponding locations of points in panels (a)-(c) in an orthographic view.
		{\bf Panel (g):} Clustering result under the Hammer projection. }
		\label{fig:Mars_data}
	\end{figure}
	
	As illustrated in Figure~\ref{fig:Mars_data}, the directional mean shift algorithm is capable of recovering the local modes of the estimated Martian crater density. Because we do not properly tune the bandwidth parameter, there is a spurious local mode around $(180^{\circ} W, 30^{\circ} S)$. Nevertheless, the mode clustering based on Algorithm~\ref{Algo:MS} succeeds in capturing two major crater clusters (or basins of attraction) on Mars, in which one cluster is densely cratered while the other is lightly catered. This finding aligns with prior research on the Martian crater distribution, stating that Mars can be divided into two general classes of terrain \citep{SODERBLOM1974239}. In Appendix~\ref{Appendix:Mars_Data}, 
	we perform mode clustering with various smoothing bandwidths to illustrate multi-scale structures in the data.

	\subsubsection{Earthquakes on Earth}
	
	Earthquakes on Earth tend to occur more frequently in some regions than others. We again leverage the directional KDE \eqref{Dir_KDE} as well as the directional mean shift algorithm to analyze earthquakes with magnitudes of 2.5+ occurring between 2020-08-21 00:00:00 UTC and 2020-09-21 23:59:59 UTC around the world. The earthquake data can be obtained from the Earthquake Catalog (\url{https://earthquake.usgs.gov/earthquakes/search/}) of the United States Geological Survey. The data set contains 1666 earthquakes worldwide for this one-month period. We use the default bandwidth estimator \eqref{bw_ROT}, which yields $h_{\text{ROT}} \approx 0.245$ on the earthquake data set, and set the tolerance level to $\epsilon=10^{-7}$ throughout the analysis.

	Figure~\ref{fig:Earthquake} displays the results.	
	There are seven local modes recovered by the directional mean shift algorithm, and they are located near (from left to right and top to bottom in Panel (g) of Figure~\ref{fig:Earthquake}) the Gulf of Alaska, the west side of the Rocky Mountain in Nevada, the Caribbean Sea, the west side of the Andes mountains in Chile, the Middle East, Indonesia, and Fiji. These regions are well-known active seismic areas along subduction zones, and our clustering of earthquake data elegantly partitions earthquakes into these regions without any prior knowledge from seismology.
	
	\begin{figure}
		\captionsetup[subfigure]{justification=centering}
		\centering
		\begin{subfigure}[t]{.32\textwidth}
			\centering
			\includegraphics[width=1\linewidth,height=0.8\linewidth]{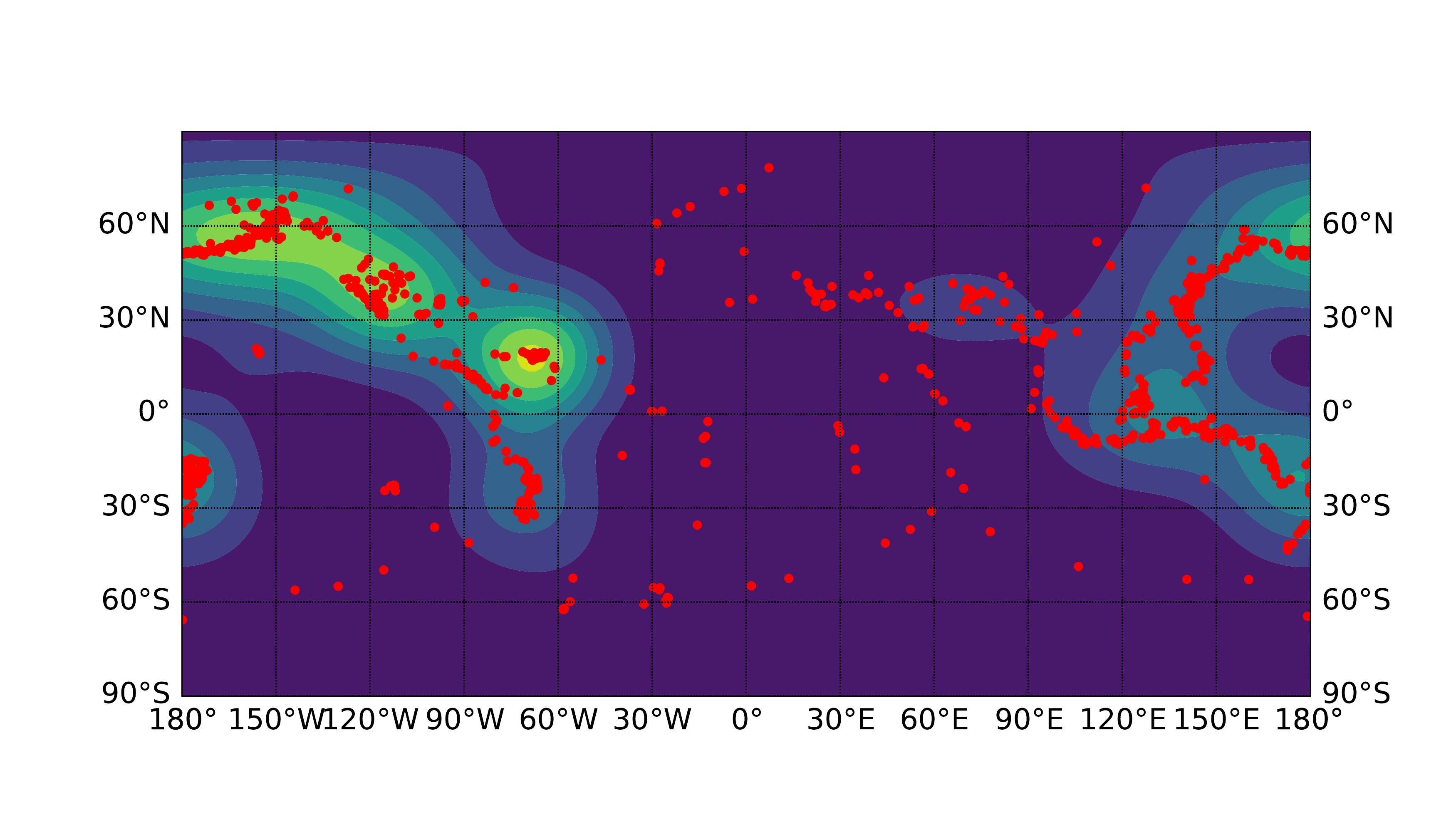}
			\caption{Step 0}
		\end{subfigure}
		\hfil
		\begin{subfigure}[t]{.32\textwidth}
			\centering
			\includegraphics[width=1\linewidth,height=0.8\linewidth]{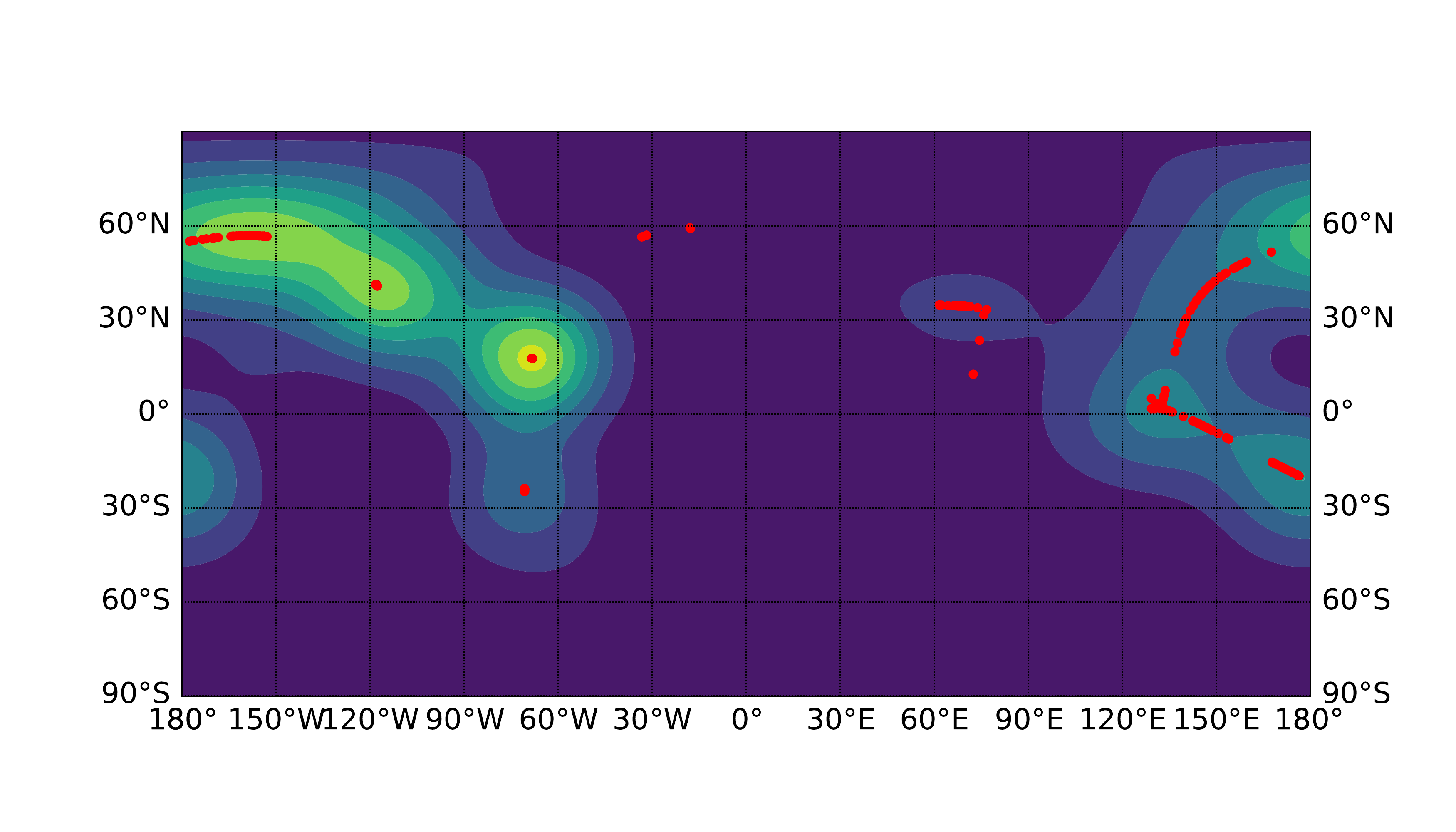}
			\caption{Step 10}
		\end{subfigure}%
		\hfil
		\begin{subfigure}[t]{.32\textwidth}
			\centering
			\includegraphics[width=1\linewidth,height=0.8\linewidth]{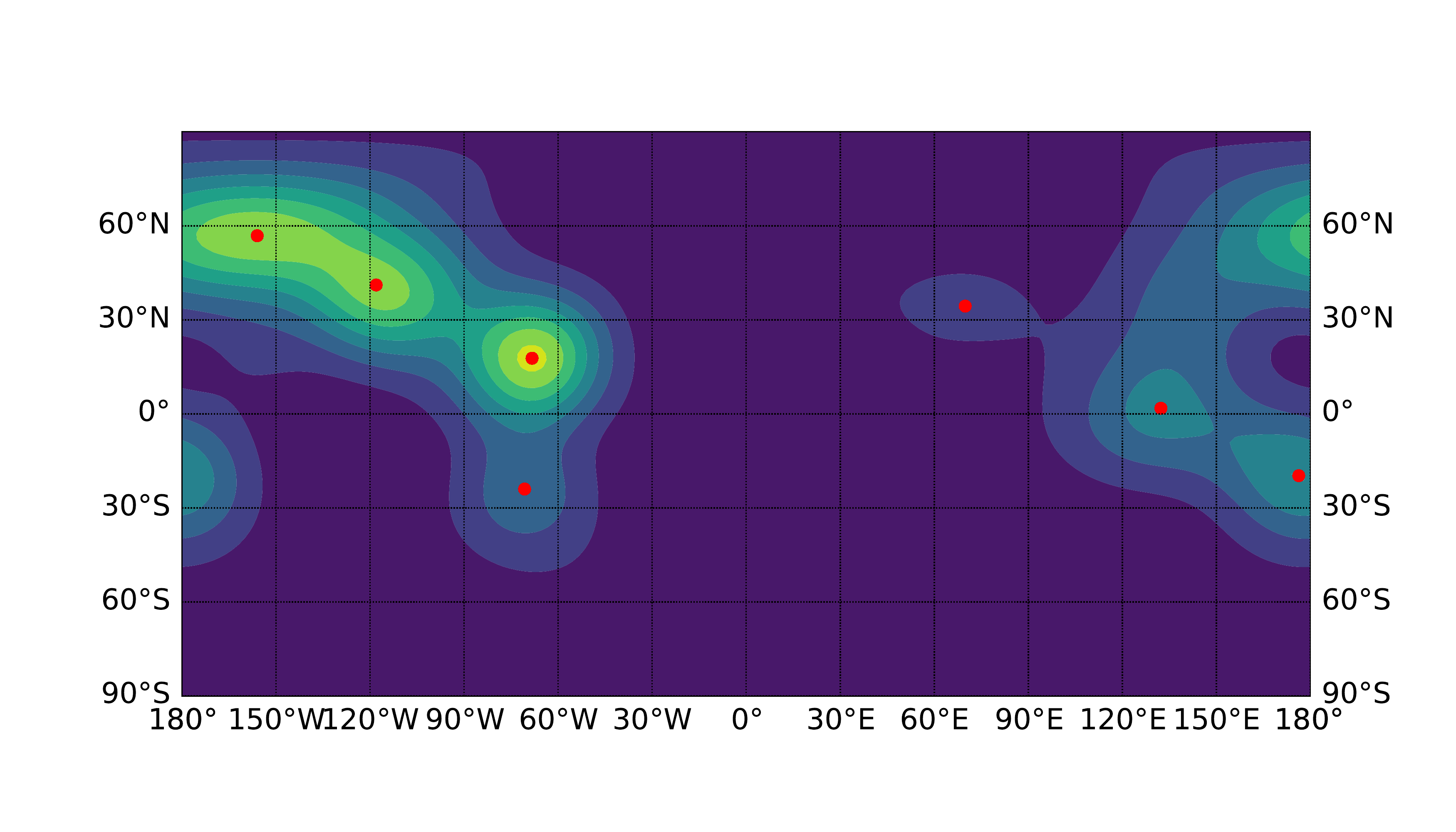}
			\caption{Step 82 (converged)}
		\end{subfigure}%
		
		\begin{subfigure}{.32\textwidth}
			\centering
			\includegraphics[width=1\linewidth]{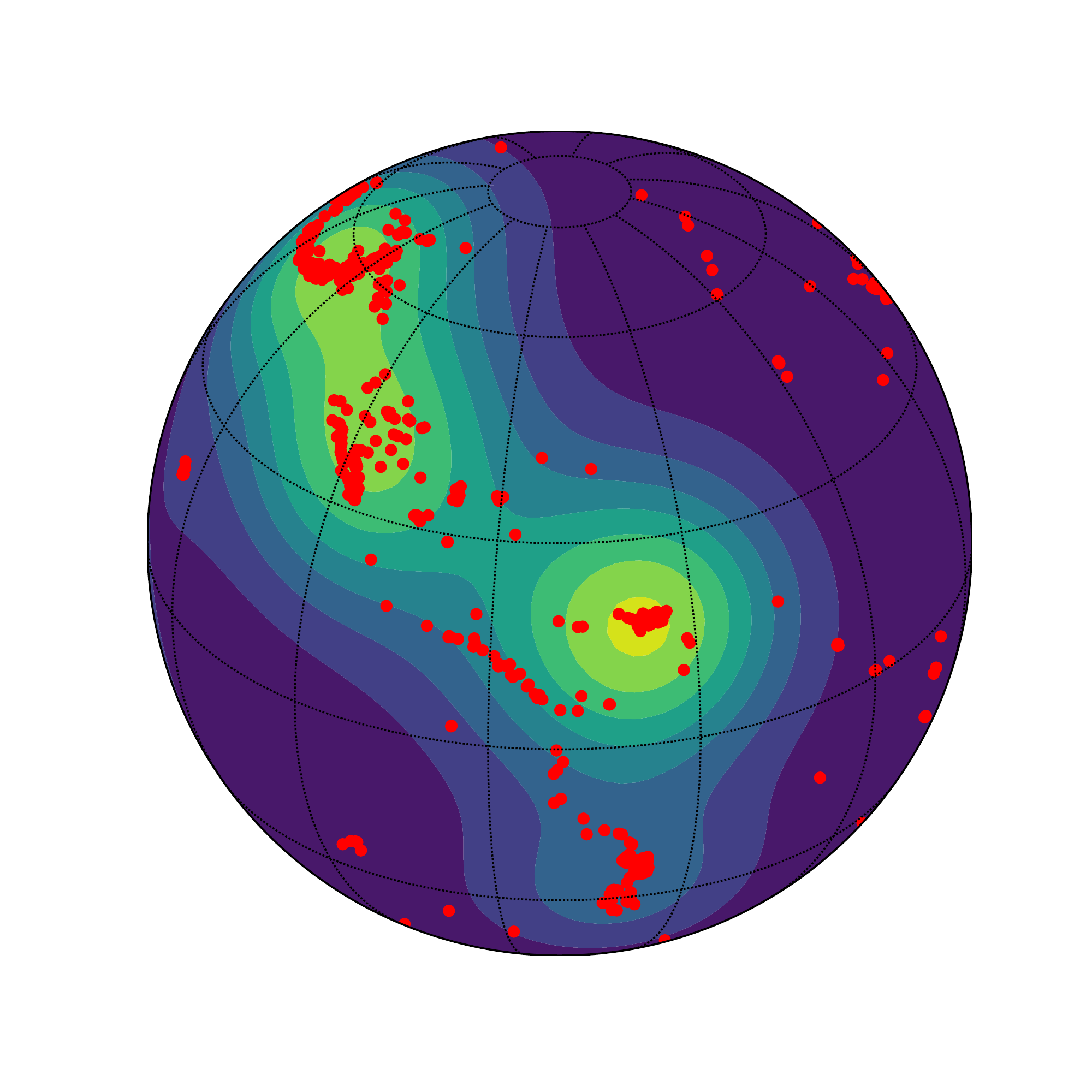}
			\caption{Step 0}
		\end{subfigure}
		\begin{subfigure}{.32\textwidth}
			\centering
			\includegraphics[width=1\linewidth]{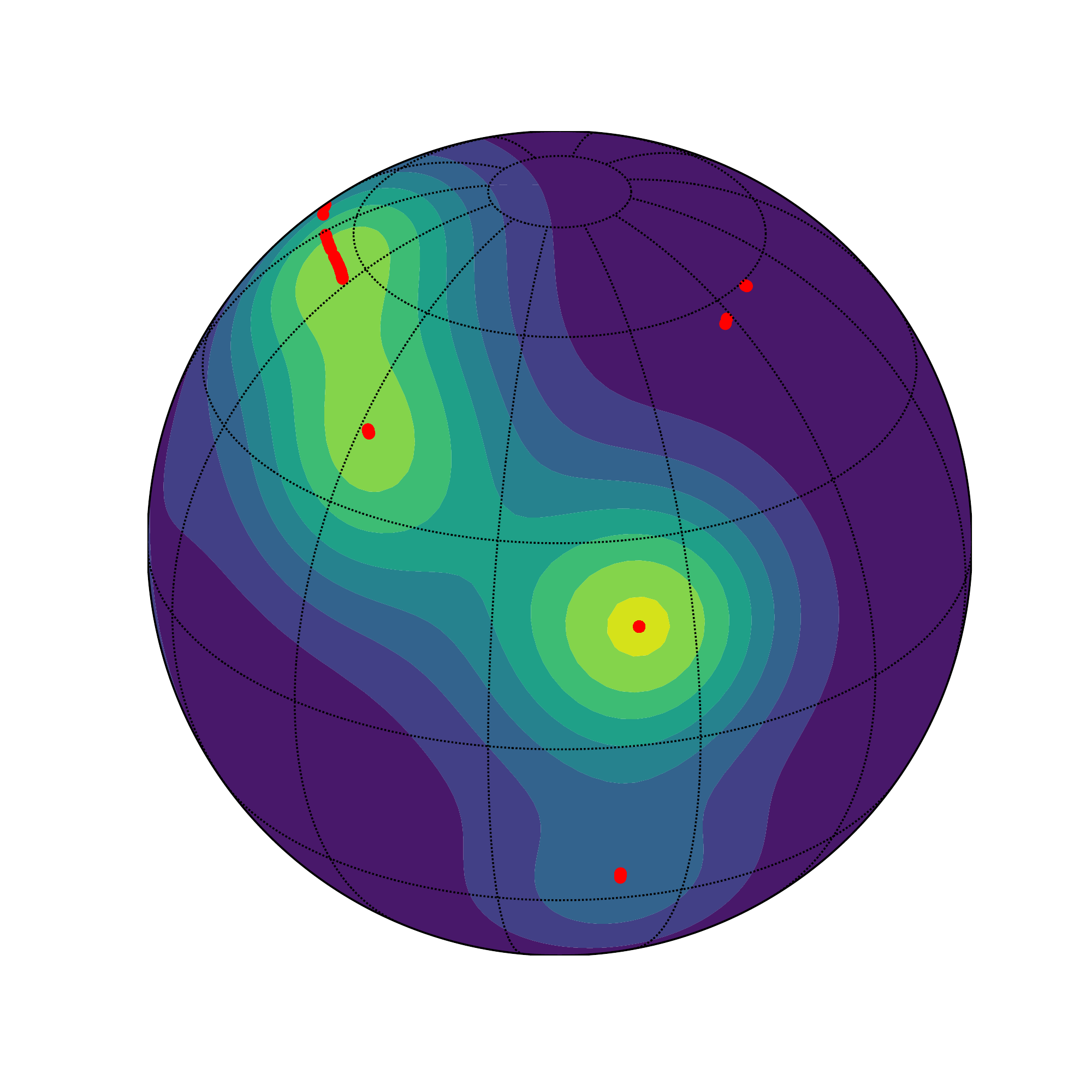}
			\caption{Step 10}
		\end{subfigure}
		\hfil
		\begin{subfigure}{.32\textwidth}
			\centering
			\includegraphics[width=1\linewidth]{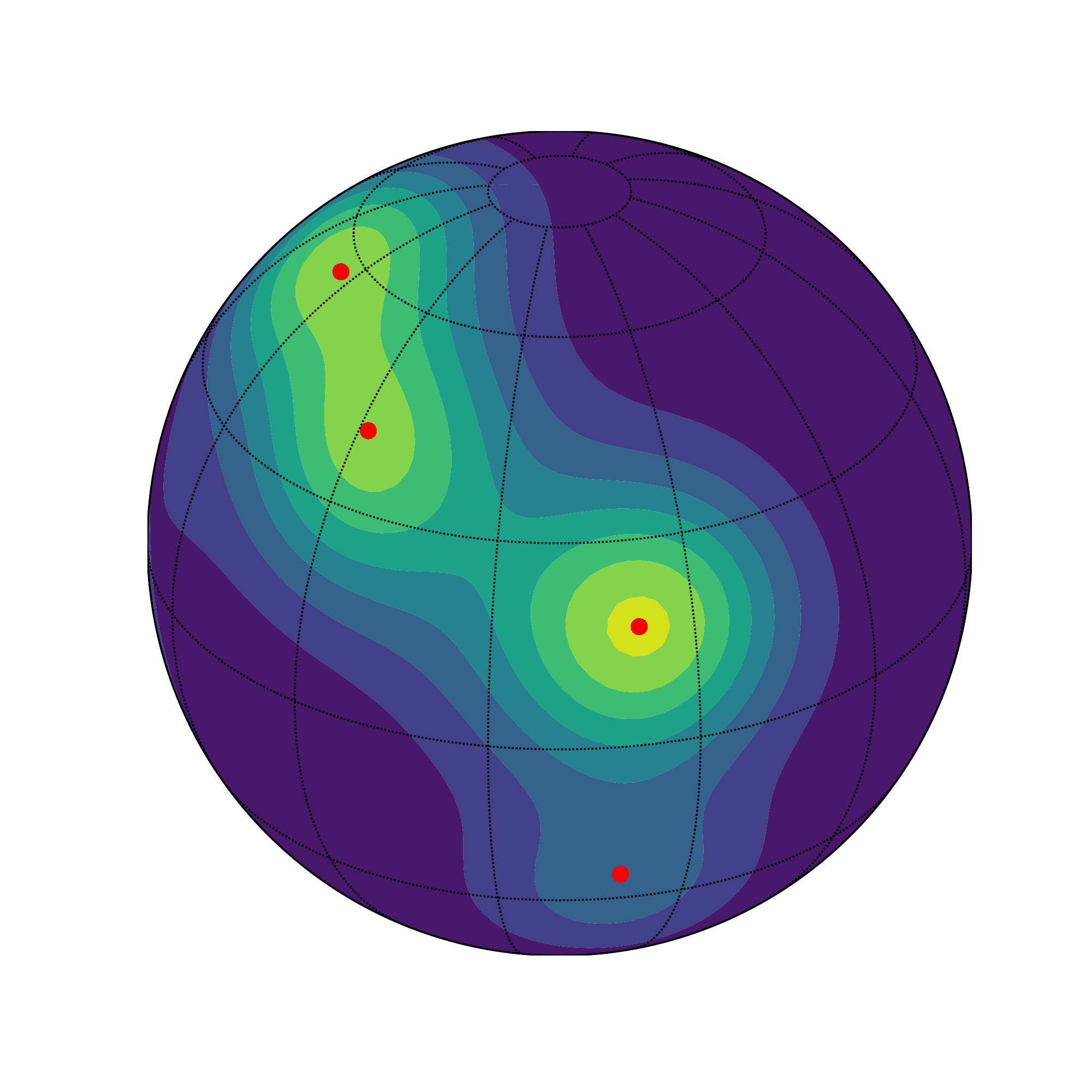}
			\caption{Step 82 (converged)}
		\end{subfigure}
		\hfil
		\begin{center}
			\begin{subfigure}[t]{.49\textwidth}
				\centering
				\includegraphics[width=1\linewidth,height=0.85\linewidth]{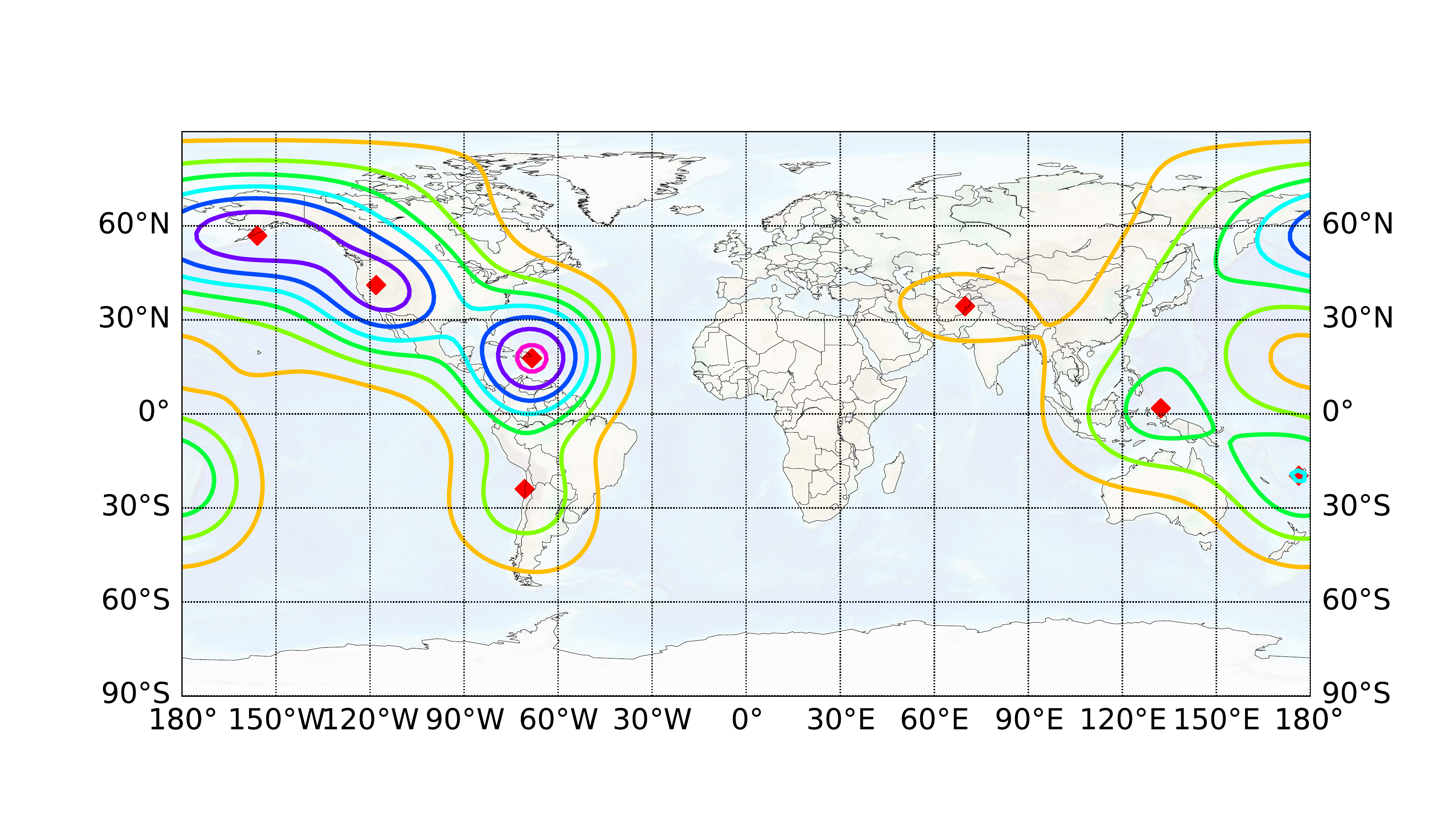}
				\caption{Local modes and contour lines on the world map}
			\end{subfigure}
			\begin{subfigure}[t]{.49\textwidth}
				\centering
				\includegraphics[width=1\linewidth,height=0.85\linewidth]{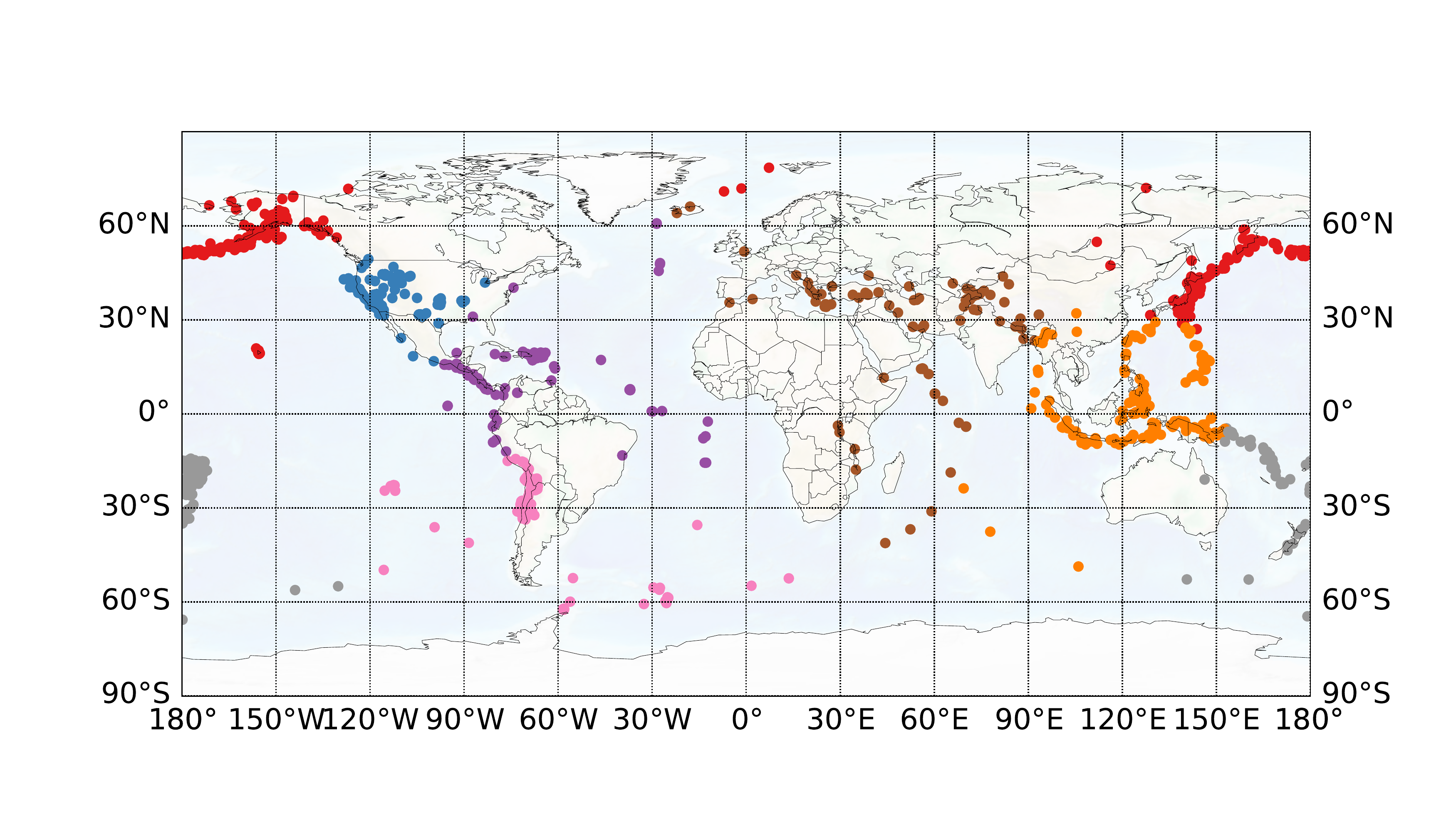}
				\caption{Mode clustering on the world map}
			\end{subfigure}
		\end{center}
		\caption{Directional mean shift algorithm performed on earthquake data for a one-month period. 
		The first two rows display the analysis similar to Figure~\ref{fig:Two_d_ThreeM}.
		{\bf Panel (a)-(c):} Outcomes under different iterations of the algorithm displayed in a cylindrical equidistant view.
		{\bf Panel (d)-(f):} Corresponding locations of points in panels (a)-(c) in an orthographic view.
		{\bf Panel (g):} Contour plots of estimated density. {\bf Panel (h):} Clustering result using the directional mean shift algorithm.}
		\label{fig:Earthquake}
	\end{figure}
	
	\section{Conclusion}
	
	In this paper, we generalize the standard mean shift algorithm to directional data based on a total gradient (or differential) of the directional KDE and formulate it as a fixed-point iteration. 
	We derive the explicit forms of the (Riemannian) gradient and Hessian estimators from a general directional KDE and establish pointwise and uniform rates of convergence for the two derivative estimators. With these powerful uniform consistency results, we demonstrate that the collection of estimated local modes obtained by the directional mean shift algorithm is a statistically consistent estimator of the set of true local modes under some mild regularity conditions.
	Additionally, the ascending property and convergence of the proposed algorithm are proved.
	Finally, given a proper bandwidth parameter (or step size for other general gradient ascent algorithms on $\Omega_q$), we argue that the directional mean shift algorithm (or other general gradient ascent algorithms on $\Omega_q$) converge(s) linearly to the (estimated) local modes within their small neighborhoods regardless of whether a population-level or sample-based version of gradient ascent is applied.
	
    Possible future extensions of our work are as follows.
	\begin{itemize}
		\item \textbf{Bandwidth Selection}. Current studies on bandwidth selectors for directional kernel smoothing settings primarily optimize the directional KDE itself. Research on bandwidth selection for derivatives of the directional KDE, especially gradient and Hessian estimators, has lagged behind. A well-designed bandwidth selector for the first-order derivatives of the directional KDE can further improve the algorithmic convergence rate of our algorithm in real-world applications. There are at least two common approaches to perform such bandwidth selection. The first is to calculate the explicit forms of dominating constants in the bias and stochastic variation terms when we derive pointwise convergence rates of the (Riemannian) gradient and Hessian estimators in Theorem~\ref{pw_conv_tang}. Then, under some assumptions on the underlying directional distribution, such as the von Mises Fisher distribution, a directional analogue to the rule of thumb of \cite{Silverman1986} for gradient and Hessian estimators can be explicated, although the calculations may be heavy. Another approach is to rely on data-adaptive methods, such as cross-validation \citep{KDE_Sphe1987} and bootstrap \citep{KDE_torus2011,Nonp_Dir_HDR2020}, which should be suitable for estimating the derivatives of the directional KDE. In addition, a bandwidth selector that is locally adaptive to the distribution of directional data is of great significance when the dimension is high.
		
		\item \textbf{Accelerated Directional Mean Shift}. Another future direction is to accelerate the current directional mean shift algorithm when the sample size is large, as the number of iterations for convergence is over 150 in one of our real-world applications. There are several feasible approaches mentioned in Section~\ref{Sec:Intro} for the Euclidean mean shift algorithm. One of the most notable methods is the blurring procedure \citep{Fast_GBMS2006,GBMS2008}, in which the (Gaussian) mean shift algorithm is performed with a crucial modification that successively updates the data set for density estimation after each mean shift iteration. It has been demonstrated that the blurring procedure improves the convergence rate of the (Gaussian) mean shift algorithm to be cubic or even higher order with Gaussian clusters and an appropriate step size. 
		We present preliminary results of introducing blurring procedures into the directional mean shift algorithm with the von-Mises kernel in Appendix~\ref{Appendix:BMS}, where the blurring procedures are able to substantially reduce the total number of iterations. However, in addition to those valid estimated local modes identified by the original directional mean shift algorithm, the blurring version also recovers some spurious local mode estimates (see Table~\ref{table:BMS} in Appendix~\ref{Appendix:BMS} for additional details). Because the current stopping criterion applied in the blurring directional mean shift algorithm is adopted from the criterion for Gaussian blurring mean shift \citep{Fast_GBMS2006}, we plan to develop an improved stopping criterion for the blurring directional mean shift algorithm and investigate its rate of convergence.
		
		\item \textbf{Connections to the EM Algorithm}. As pointed out by \cite{MS_EM2007}, the Gaussian mean shift algorithm for Euclidean data is an EM algorithm, while the mean shift algorithm with a non-Gaussian kernel is a generalized EM algorithm. It is unclear whether the directional mean shift algorithm with the von Mises kernel is also an EM algorithm on a mixture of von Mises-Fisher distributions on $\Omega_q$ \citep{spherical_EM} or even a generalized EM algorithm when other kernels are used in Algorithm~\ref{Algo:MS}. Bridging this connection can help establish the linear rate of convergence for the algorithm from a different angle. 
	\end{itemize}
	
%\pagebreak
	
	\acks{We thank the anonymous reviewers and AE for their constructive comments that improved the quality of this paper. We also thank members in the Geometric Data Analysis Reading Group at the University of Washington for their helpful comments. YC is supported by NSF DMS - 1810960 and DMS - 1952781, NIH U01 - AG0169761.}
	
%	\vskip 0.2in
%	\bibliography{Bibliography}

	\newpage
	\appendix
	\begin{appendices}
		
		\section{Additional Experimental Results}
		\label{Appendix:Ad_Experiments}
		
		\subsection{Spherical Case with One Mode (Simulation Study)}
		\label{Appendix:Two_d_OneM}
		
		We simulate 1000 data points from the density $f_2(\bm{x}) = f_{\text{vMF}}(\bm{x};\bm{\mu},\nu)$ with $\bm{\mu}=(1,0,0)$ and $\nu=5$. The bandwidth parameter is selected using \eqref{bw_ROT} and the tolerance level for terminating the algorithm is set to $\epsilon=10^{-7}$. As presented in Figure~\ref{fig:Two_d_OneM}, all the simulated data points converge to the mode of the estimated directional density except for a small portion of outliers. In addition, the misclassification rate in this example is 0.011, because there are some spurious local modes in the low density region. The total number of iterative steps in this case is greater than the case with three local modes in Figure~\ref{fig:Two_d_ThreeM}.
		
		\begin{figure}
			\captionsetup[subfigure]{justification=centering}
			\centering
			\begin{subfigure}[t]{.32\textwidth}
				\centering
				\includegraphics[width=1\linewidth,height=0.75\linewidth]{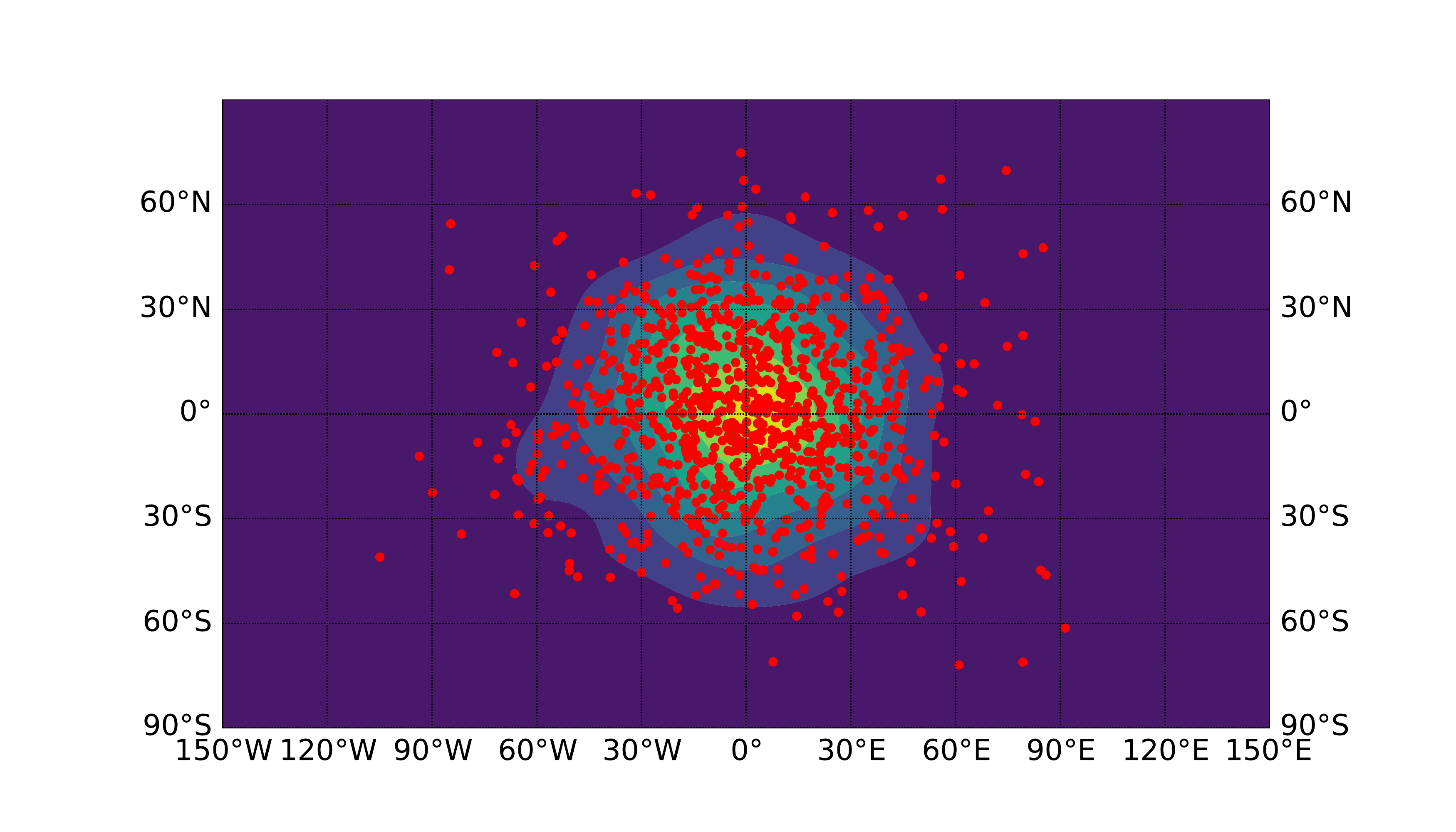}
				\caption{Step 0}
			\end{subfigure}
			\hfil
			\begin{subfigure}[t]{.32\textwidth}
				\centering
				\includegraphics[width=1\linewidth,height=0.75\linewidth]{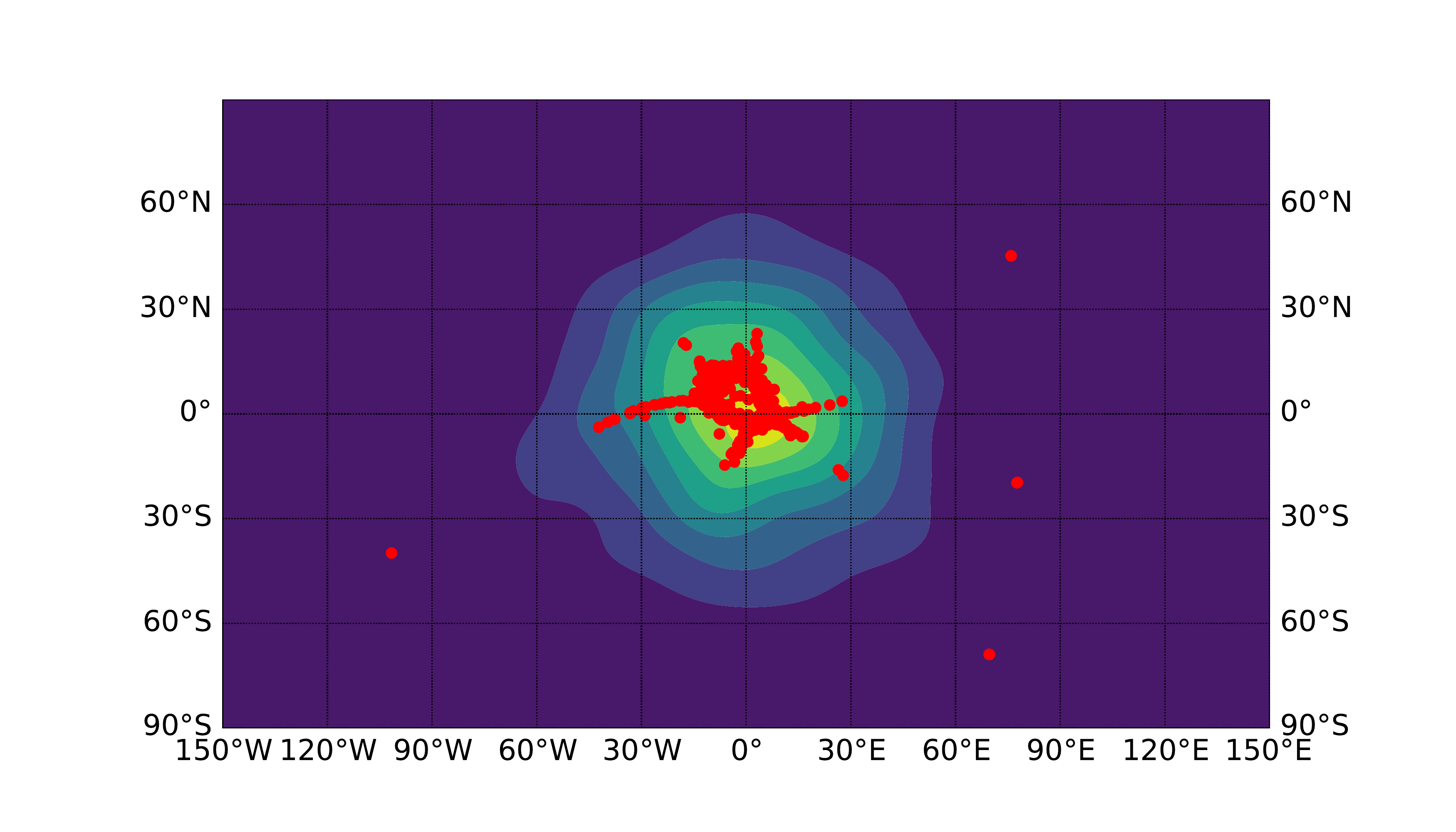}
				\caption{Step 19}
			\end{subfigure}%
			\hfil
			\begin{subfigure}[t]{.32\textwidth}
				\centering
				\includegraphics[width=1\linewidth,height=0.75\linewidth]{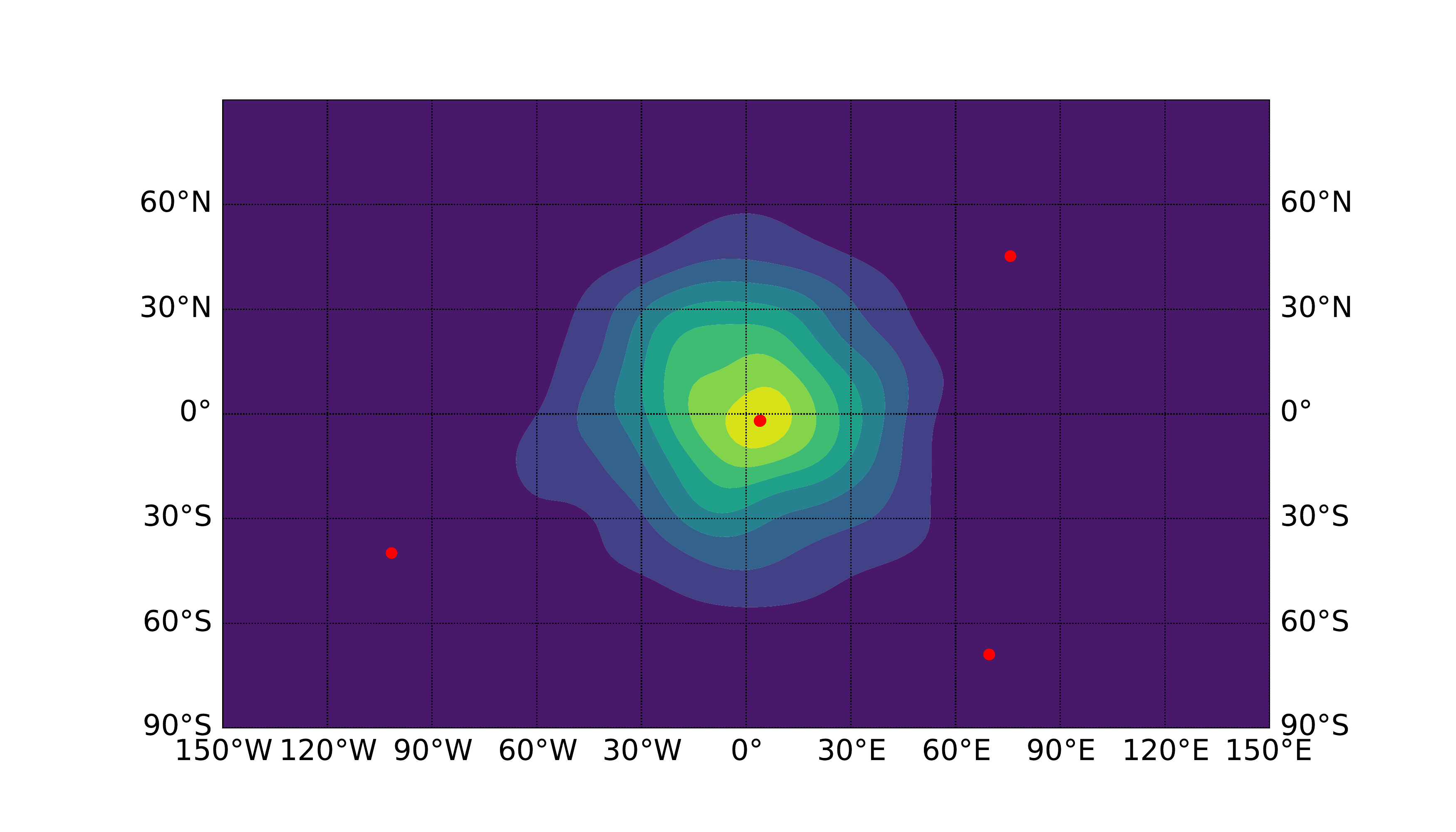}
				\caption{Step 79 (converged)}
			\end{subfigure}%
			
			\begin{subfigure}{.32\textwidth}
				\centering
				\includegraphics[width=1\linewidth]{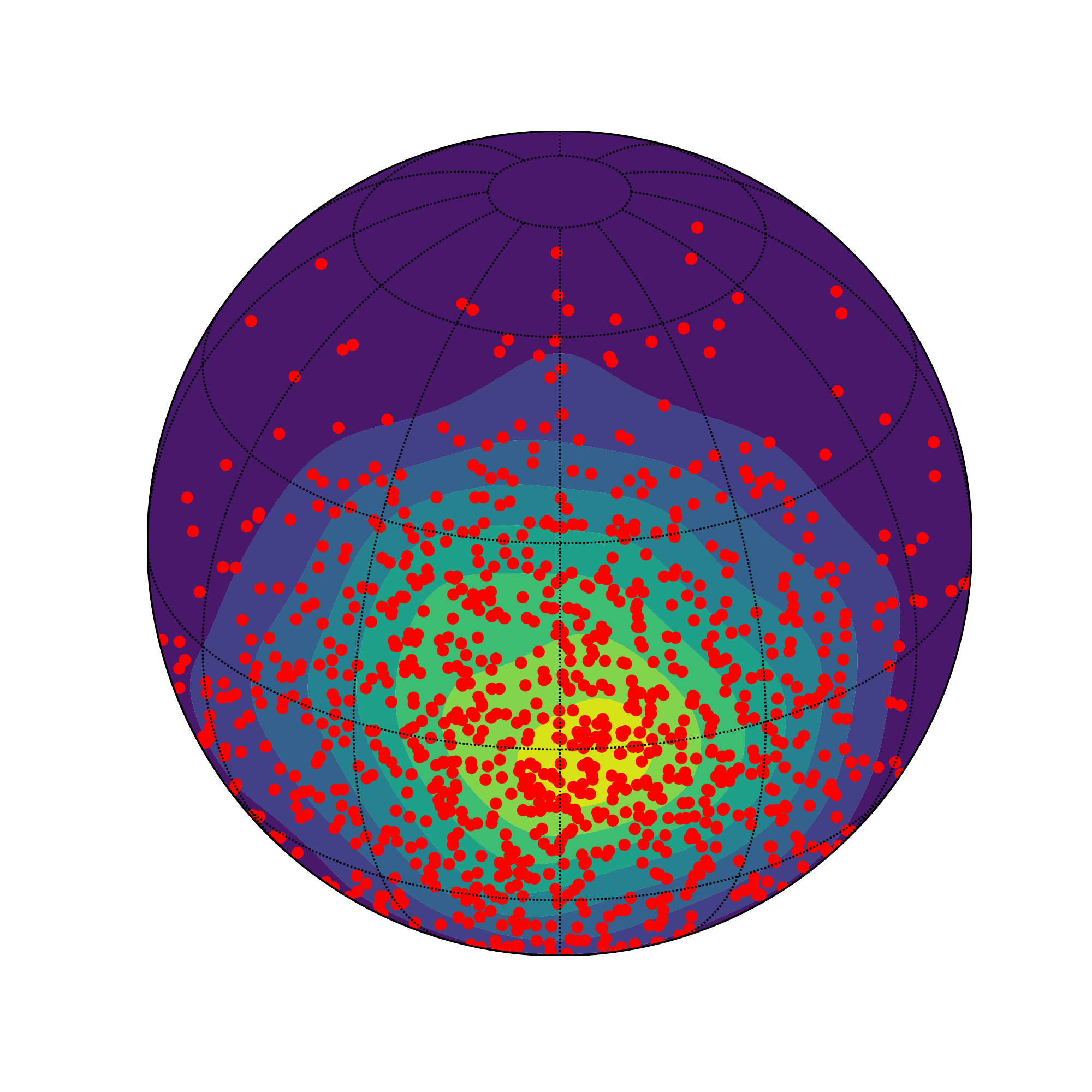}
				\caption{Step 0}
			\end{subfigure}
			\begin{subfigure}{.32\textwidth}
				\centering
				\includegraphics[width=1\linewidth]{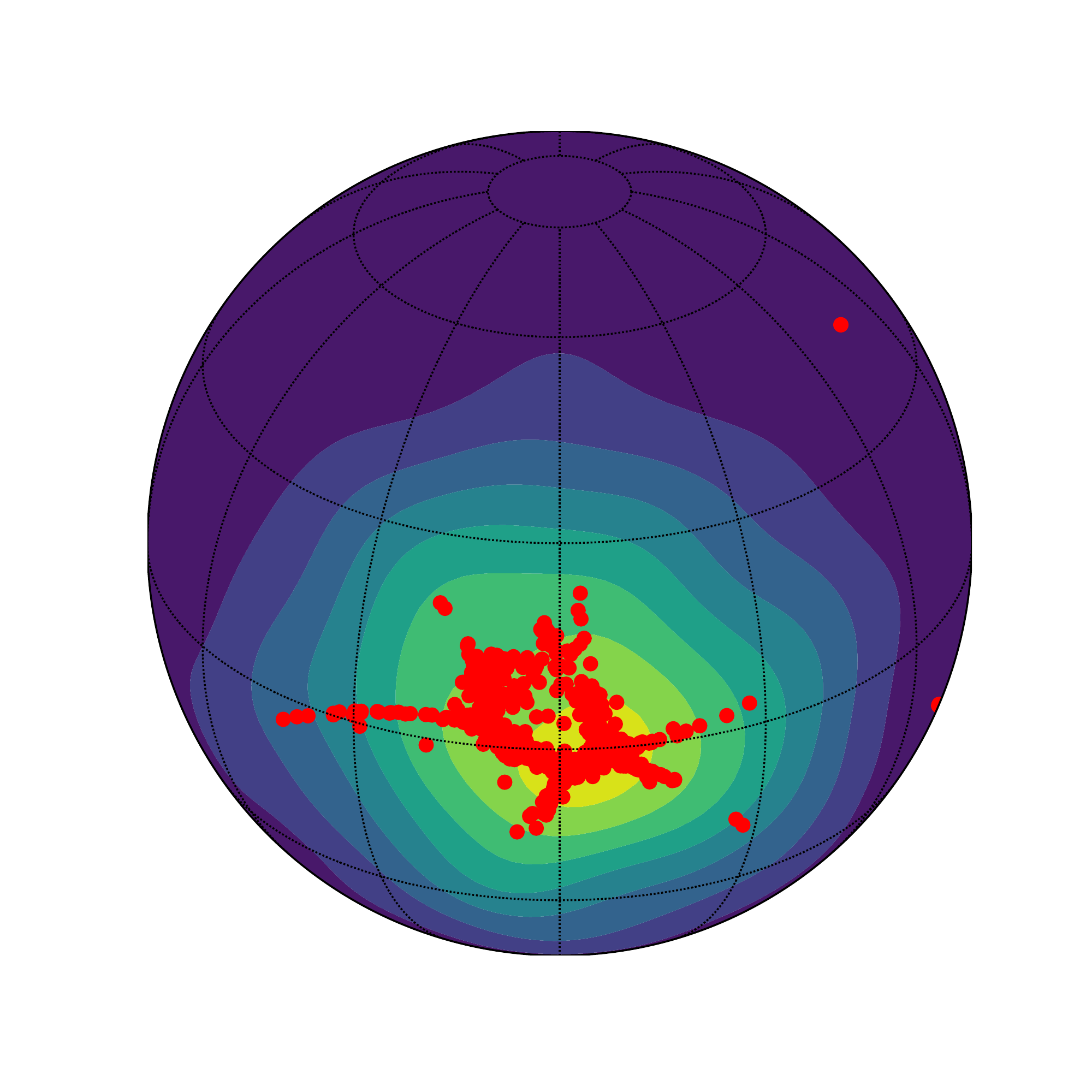}
				\caption{Step 19}
			\end{subfigure}
			\hfil
			\begin{subfigure}{.32\textwidth}
				\centering
				\includegraphics[width=1\linewidth]{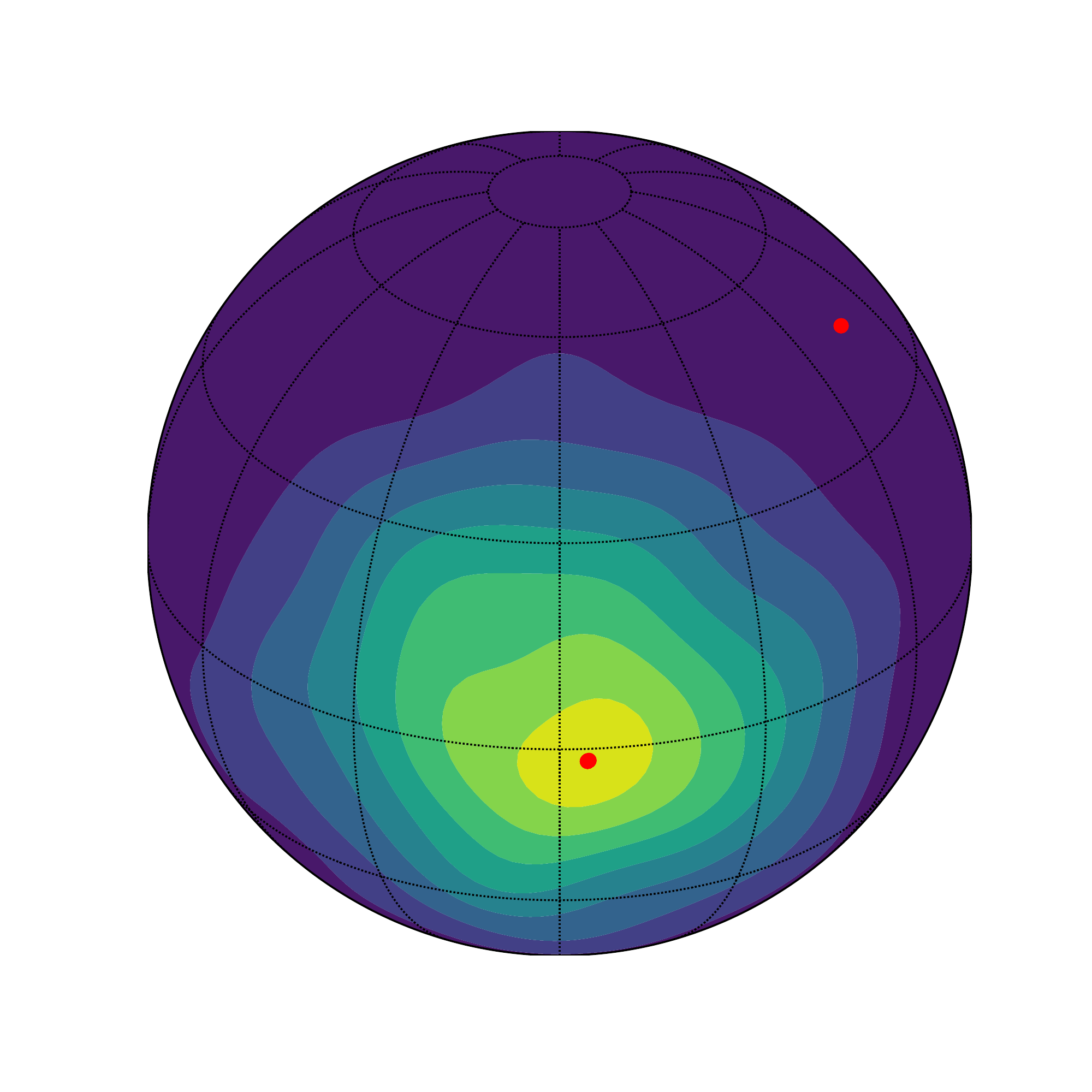}
				\caption{Step 79 (converged)}
			\end{subfigure}
			\hfil
			\begin{center}
				\begin{subfigure}[t]{.8\textwidth}
					\centering
					\includegraphics[width=1\linewidth]{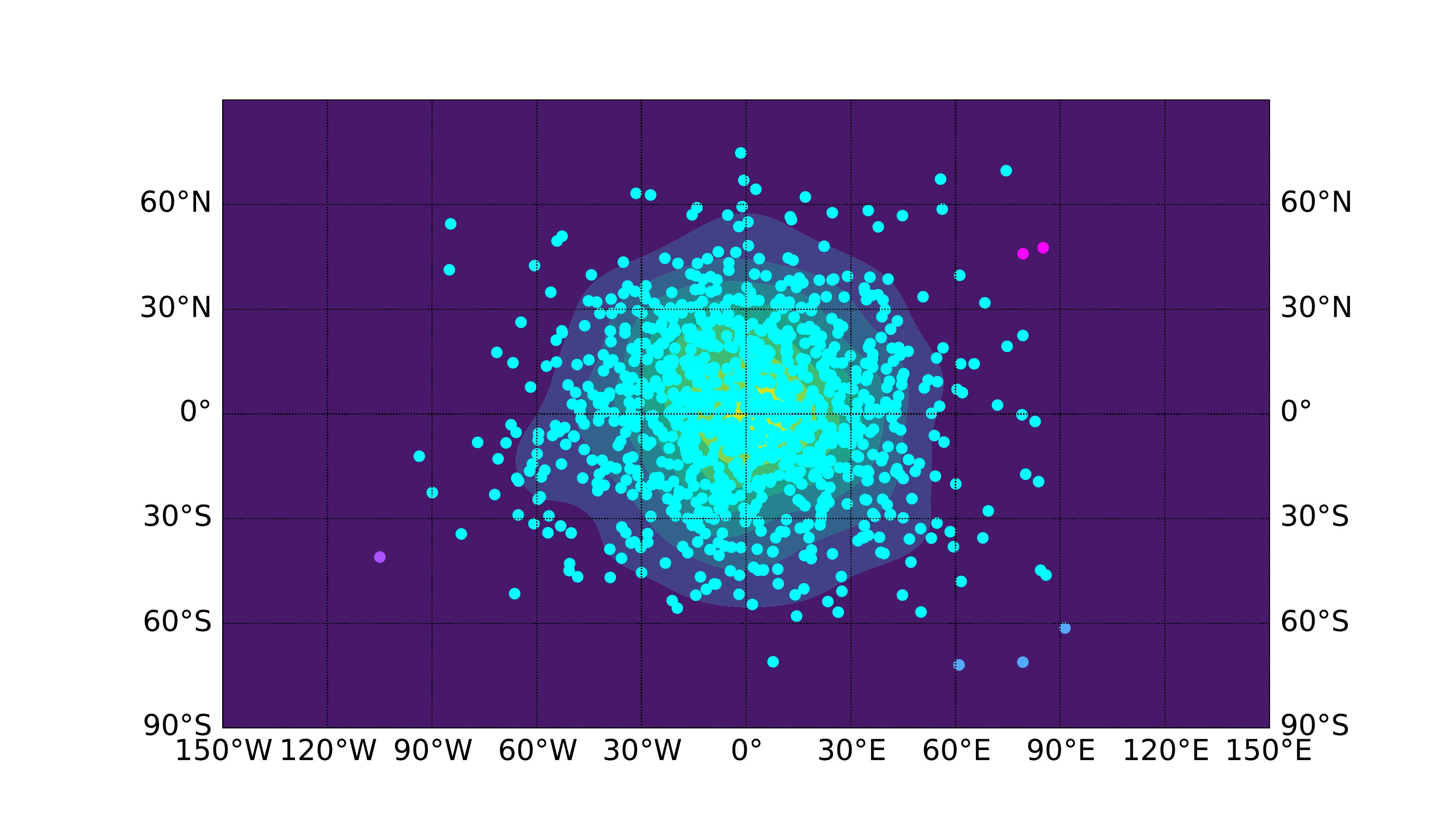}
					\caption{Mode clustering (Cylindrical equidistant view)}
				\end{subfigure}
			\end{center}
			\caption{Directional mean shift algorithm performed on simulated data with one mode on $\Omega_2$. 
		The analysis is displayed similar to Figure~\ref{fig:Two_d_ThreeM}.
		{\bf Panel (a)-(c):} Outcomes under different iterations of the algorithm displayed in a cylindrical equidistant view.
		{\bf Panel (d)-(f):} Corresponding locations of points in panels (a-c) in an orthographic view.
		{\bf Panel (g):} Clustering result in a cylindrical equidistant view.}			
			\label{fig:Two_d_OneM}
		\end{figure}
		
		\subsection{Additional Mode Clustering Results on the Martian Crater Data}
		\label{Appendix:Mars_Data}
		
		We varies the bandwidth parameter $h$ from 0.1 to 0.6 with a step size 0.05 when conducting mode clustering on the trimmed Martian crater data set. The tolerance level for stopping Algorithm~\ref{Algo:MS} is set to $\epsilon=10^{-7}$. The number of crater clusters on Mars yielded from Algorithm~\ref{Algo:MS} ranges from 37 when $h=0.1$ to 1 when $h=0.6$ in Figure~\ref{fig:Mode_clu_Mars}. Those small crater clusters yielded by the directional mean shift algorithm with a small bandwidth parameter are of practical significance, since it may give astronomers more insight into the planetary subsurface structure and geologic processes on Mars.
		
		\begin{figure}
			\captionsetup[subfigure]{justification=centering}
			\centering
			\begin{subfigure}[t]{.49\textwidth}
				\centering
				\includegraphics[width=1\linewidth,height=0.57\linewidth]{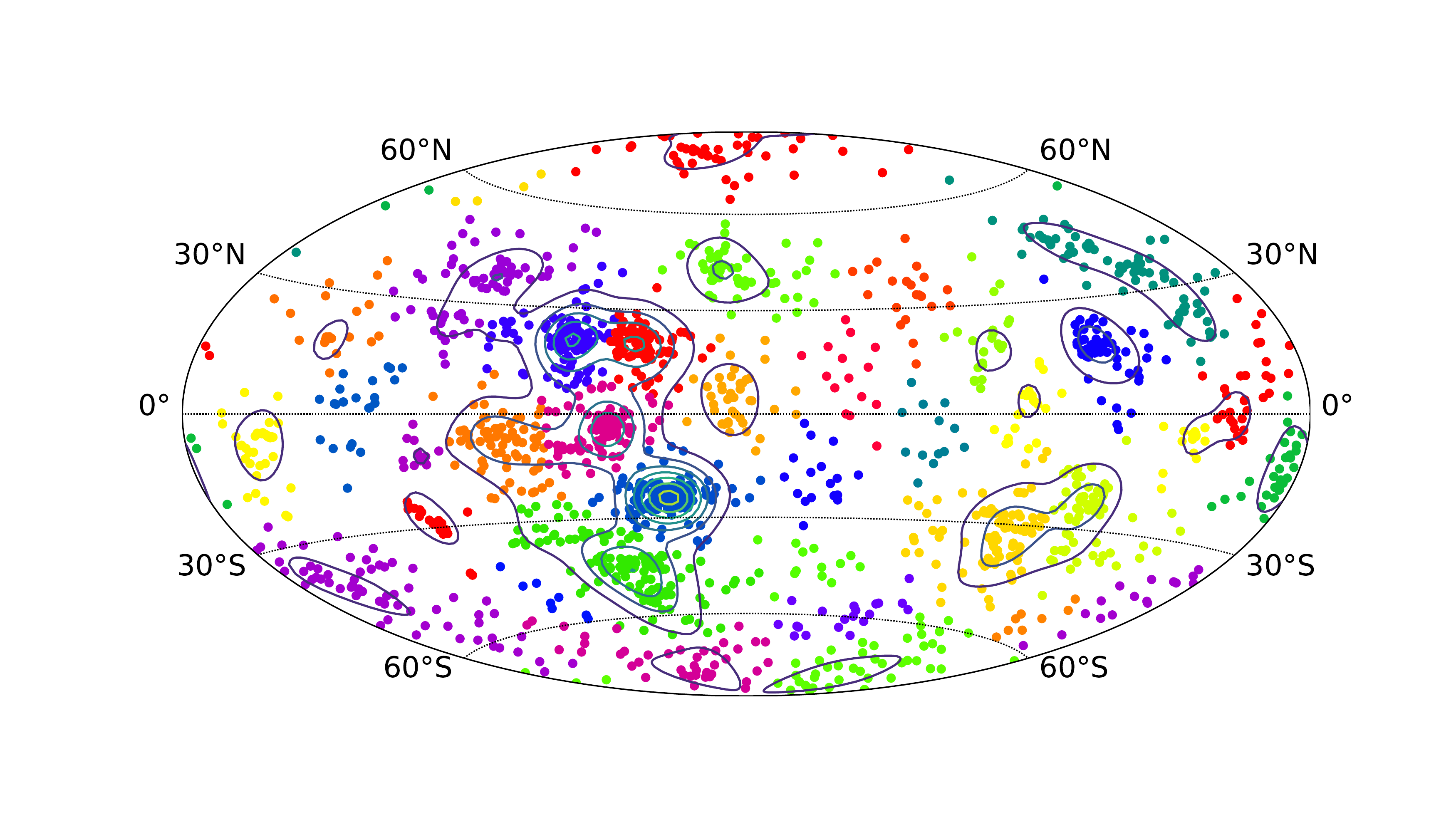}
				\caption{37 clusters (Step 137 (converged), $h=0.1$)}
			\end{subfigure}
			\hfil
			\begin{subfigure}[t]{.49\textwidth}
				\centering
				\includegraphics[width=1\linewidth,height=0.57\linewidth]{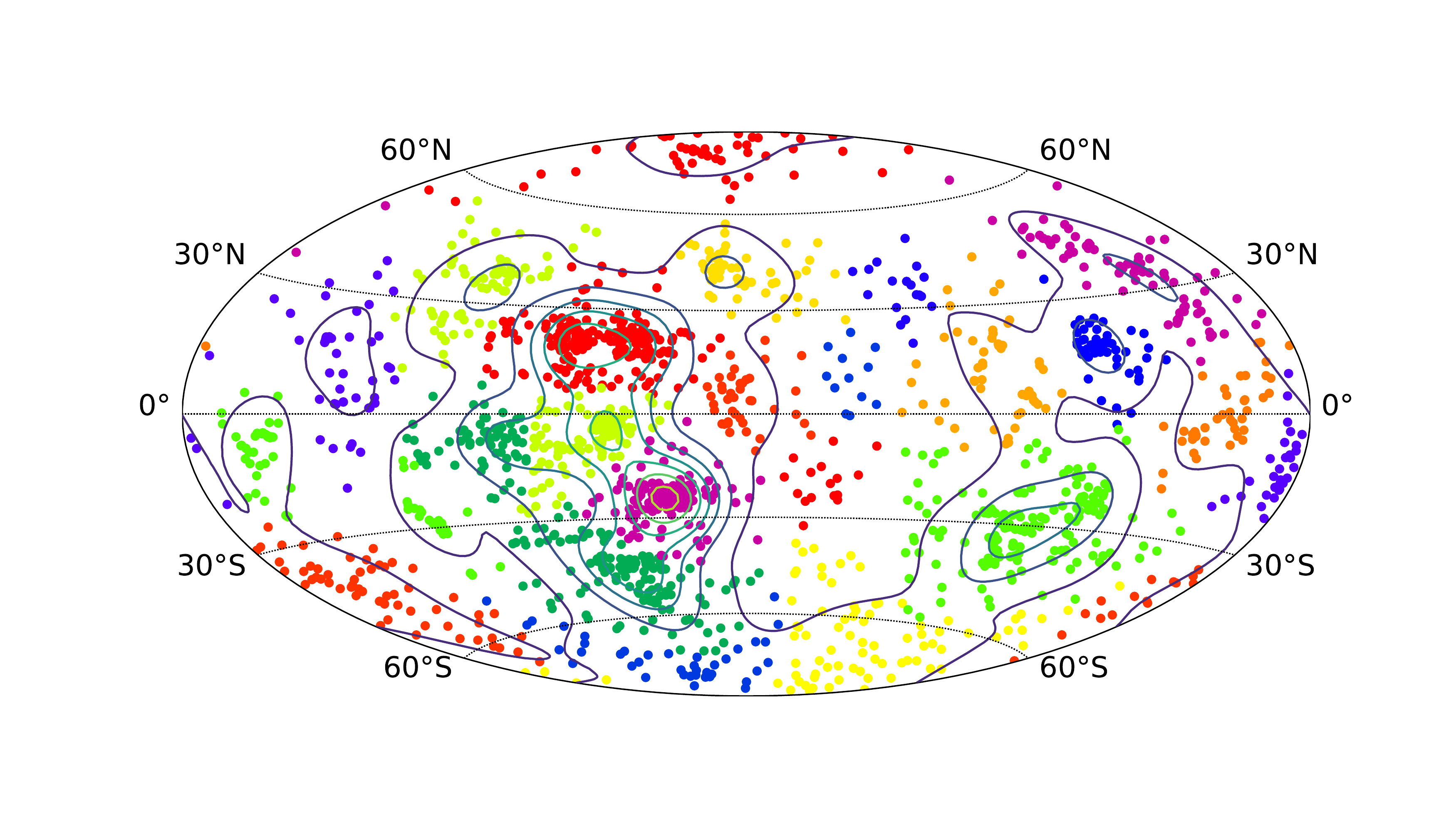}
				\caption{24 clusters (Step 62 (converged), $h=0.15$)}
			\end{subfigure}
			
			\begin{subfigure}[t]{.49\textwidth}
				\centering
				\includegraphics[width=1\linewidth,height=0.57\linewidth]{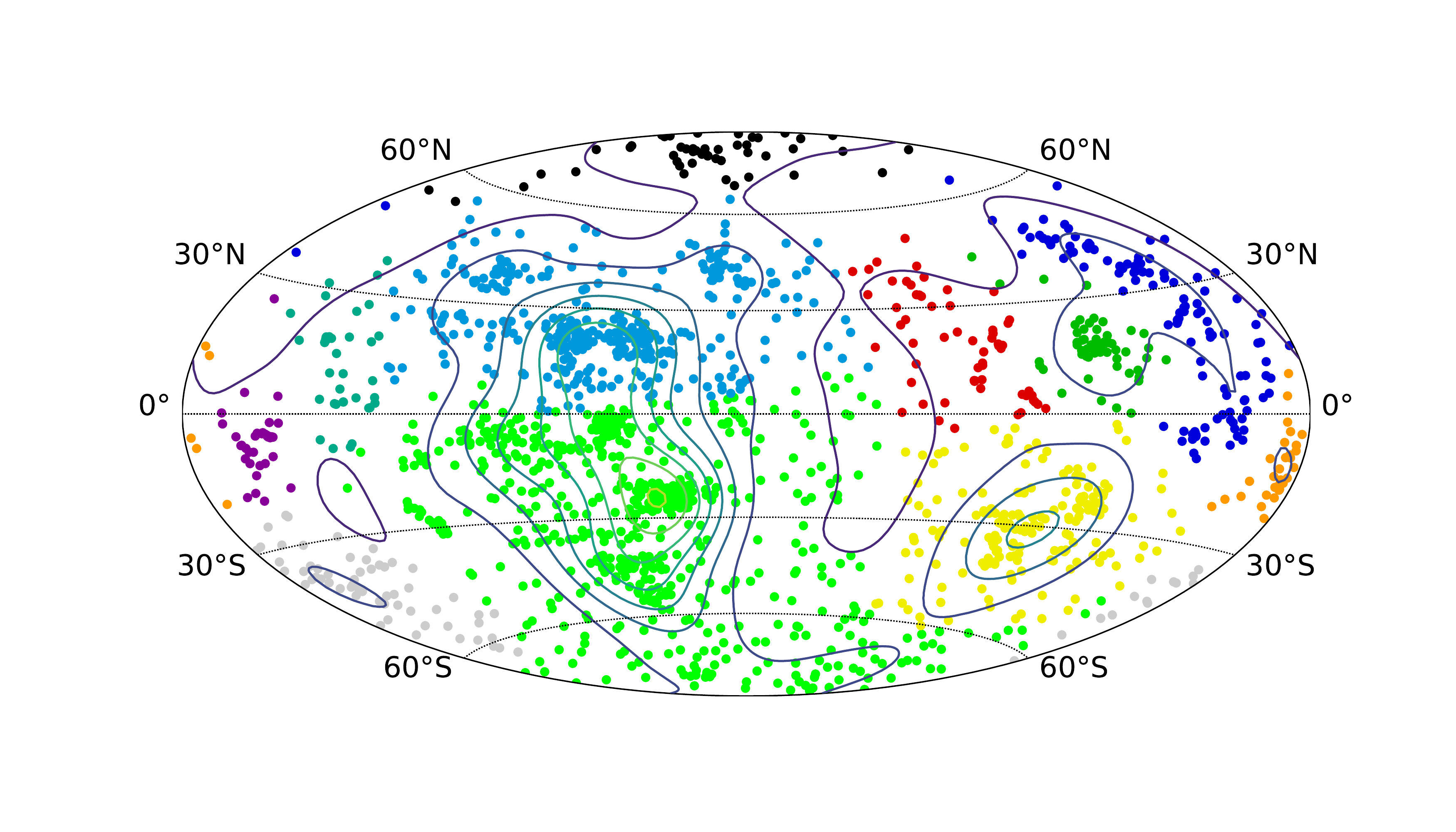}
				\caption{11 clusters (Step 135 (converged), $h=0.2$)}
			\end{subfigure}
			\hfil
			\begin{subfigure}[t]{.49\textwidth}
				\centering
				\includegraphics[width=1\linewidth,height=0.57\linewidth]{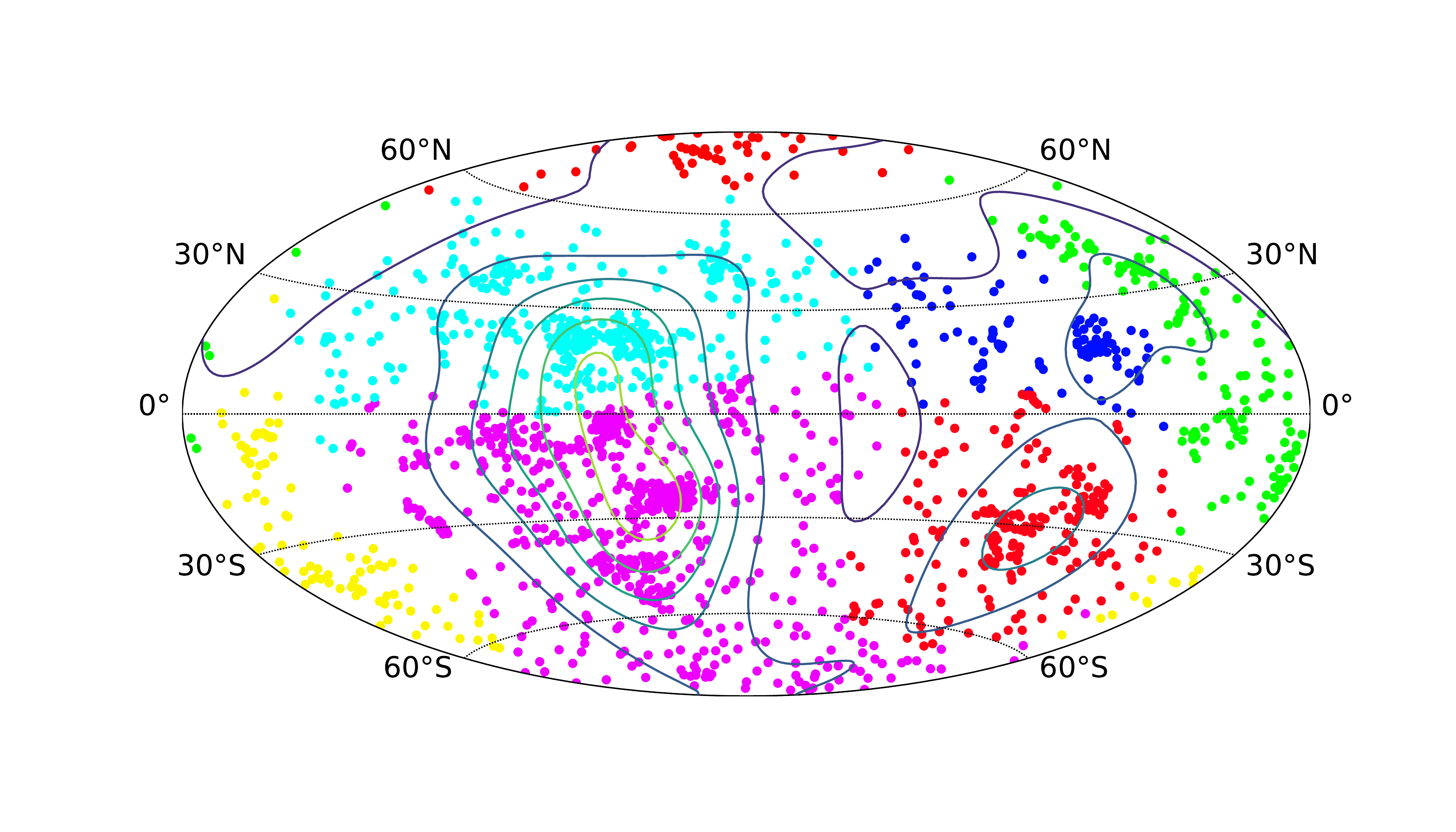}
				\caption{7 clusters (Step 234 (converged), $h=0.25$)}
			\end{subfigure}
			
			\begin{subfigure}[t]{.49\textwidth}
				\centering
				\includegraphics[width=1\linewidth,height=0.57\linewidth]{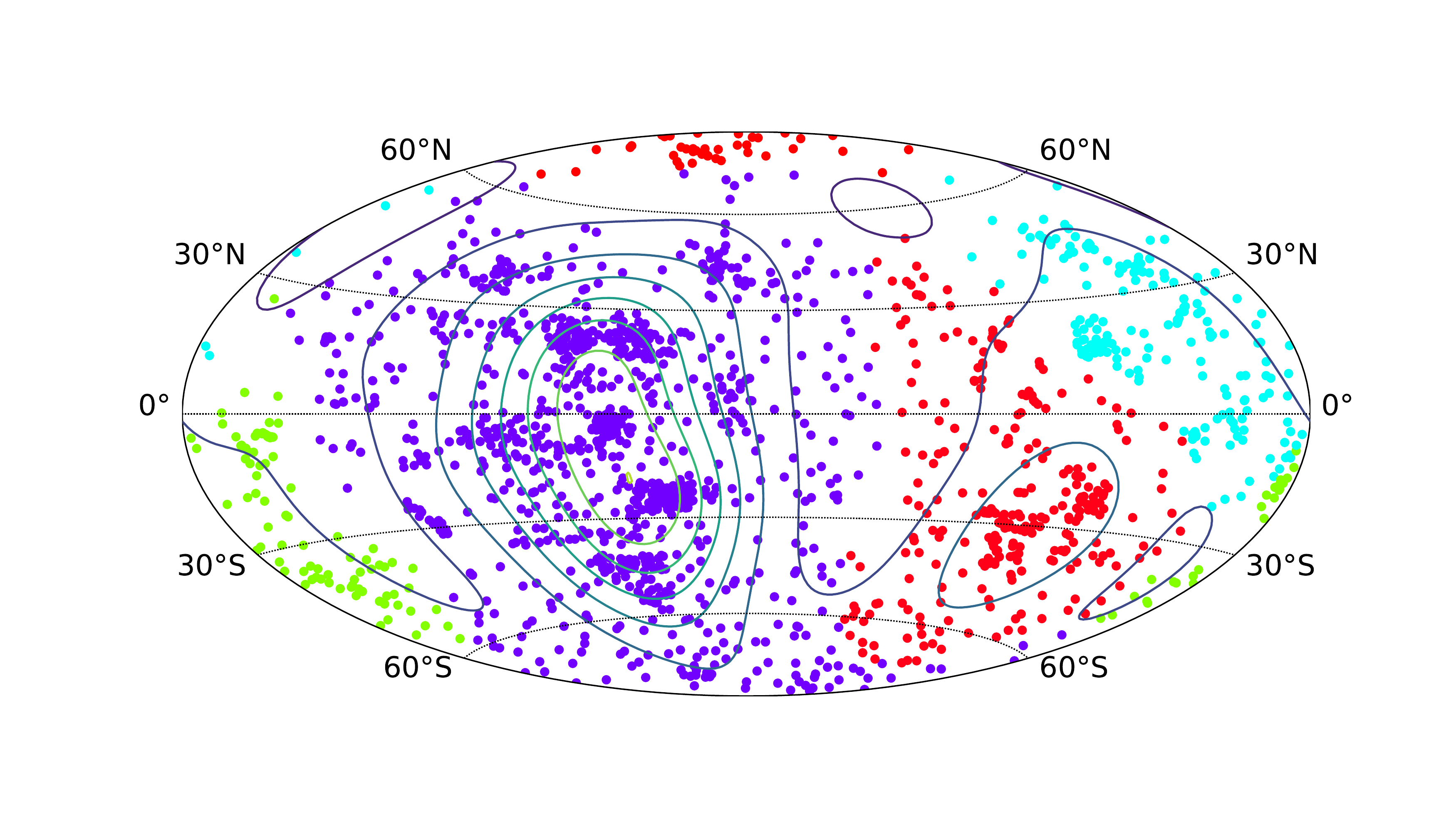}
				\caption{5 clusters (Step 237 (converged), $h=0.3$)}
			\end{subfigure}
			\hfil
			\begin{subfigure}[t]{.49\textwidth}
				\centering
				\includegraphics[width=1\linewidth,height=0.57\linewidth]{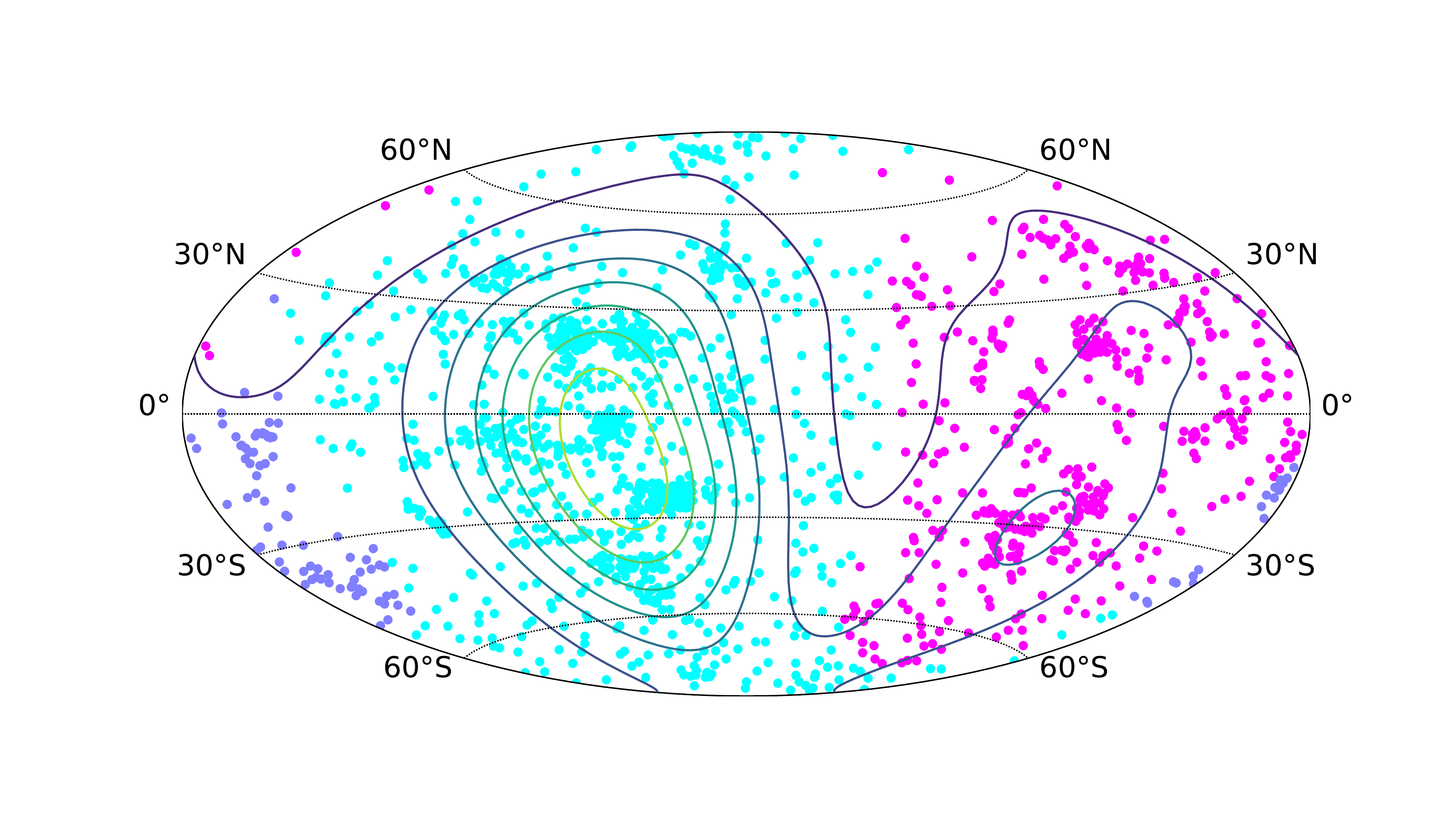}
				\caption{3 clusters (Step 118 (converged), $h=0.35$)}
			\end{subfigure}
			
			\begin{subfigure}[t]{.49\textwidth}
				\centering
				\includegraphics[width=1\linewidth,height=0.57\linewidth]{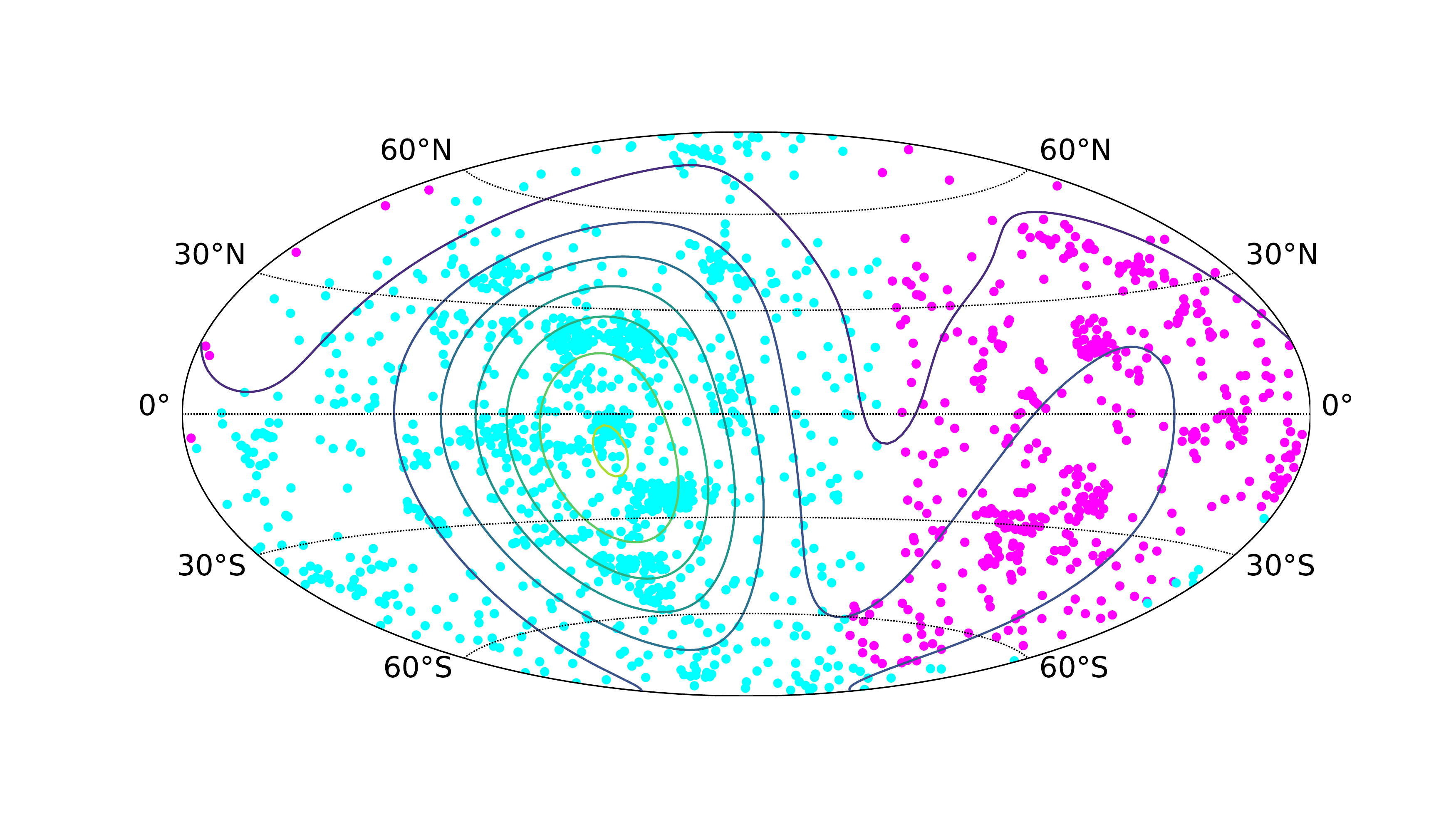}
				\caption{2 clusters (Step 155, 68, 57, and 61 (both converged) when $h=0.4, 0.45, 0.5, 0.55$). The plot here is for $h=0.4$.}
			\end{subfigure}
			\hfil
			\begin{subfigure}[t]{.49\textwidth}
				\centering
				\includegraphics[width=1\linewidth,height=0.57\linewidth]{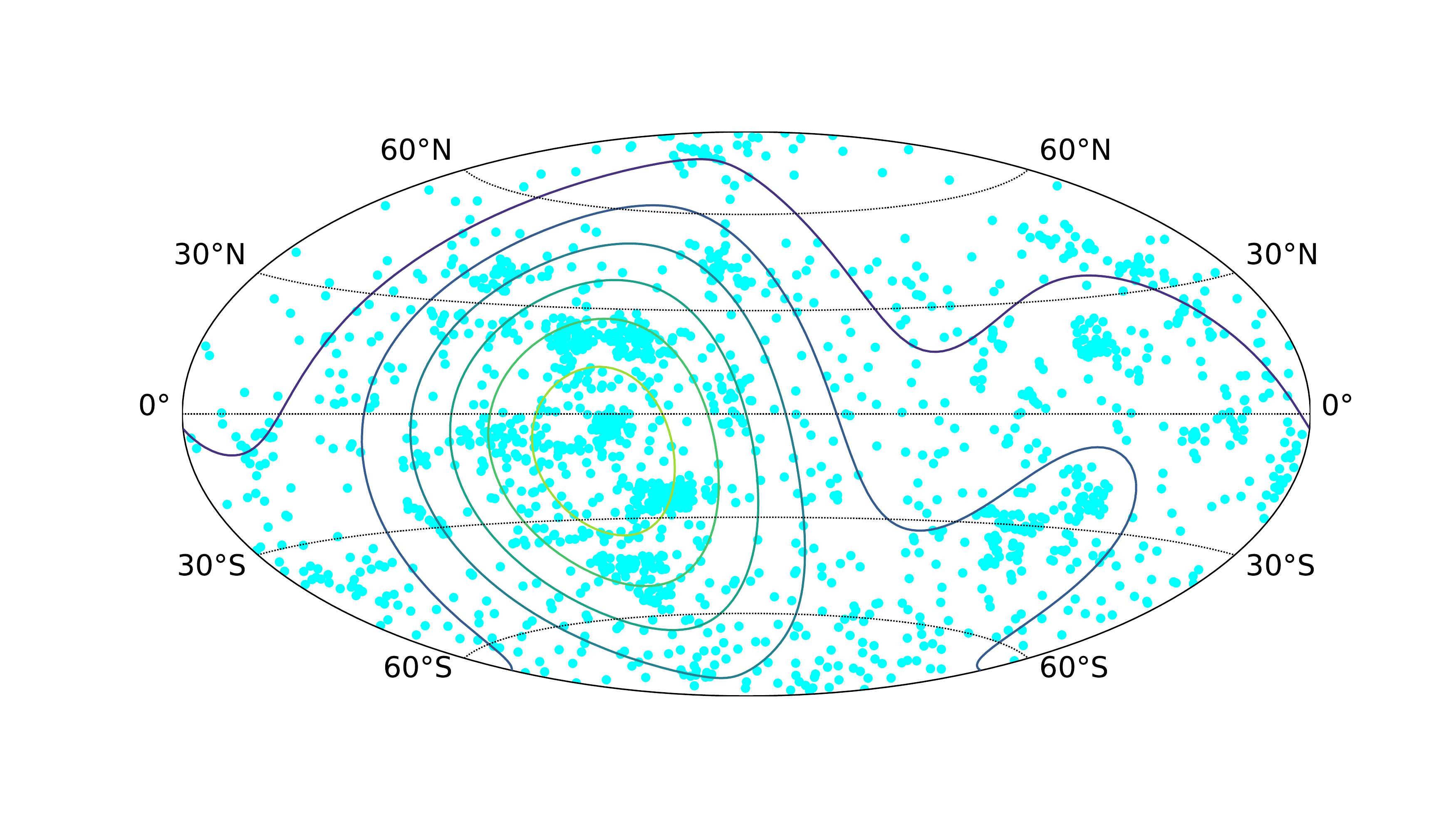}
				\caption{1 cluster (Step 69 (converged), $h=0.6$)}
			\end{subfigure}
			
			\caption{Directional mean shift algorithm with various bandwidth parameters performed on Martian crater data. The figures are visualized in their Hammer projections.}
			\label{fig:Mode_clu_Mars}
		\end{figure}
		
		\subsection{Preliminary Experiments on Blurring Directional Mean Shift Algorithm with the von-Mises Kernel}
		\label{Appendix:BMS}

		We randomly generate 1000 data points from von Mises-Fisher distributions with one ([1,0,0]), two ([0,1,0], [0,0,1]), and three ([0,1,0], [1,0,0], [0,-1,0]) true local modes via rejection sampling, respectively. Both the original directional mean shift algorithm with the von Mises-Fisher kernel and its blurring version are implemented on these simulated data sets. The stopping criterion for the blurring directional mean shift algorithm is adopted from the Gaussian Blurring mean shift algorithm with Euclidean data \citep{Fast_GBMS2006}, that is,
		$$\left(\left|H\left(\bm{e}^{(s+1)}\right) - H\left(\bm{e}^{(s)} \right) \right| \leq 10^{-8} \right) \quad \text{ OR } \quad \left(\frac{1}{n}\sum_{i=1}^n e_i^{(s+1)} \leq \epsilon \right),$$
		where $\bm{e}^{(s)} = (e_1^{(s)},...,e_n^{(s)})$, $e_i^{(s)} = \norm{\hat{\bm{y}}_i^{(s)} - \hat{\bm{y}}_i^{(s+1)}}_2$, $\left\{\hat{\bm{y}}_i^{(0)} \right\}_{i=1}^n = \{\bm{X}_i\}_{i=1}^n$ is the original data set, $H(\bm{e}) = -\sum_{i=1}^B f_i \log f_i$ is the entropy, $f_i$ is the relative frequency of bin $i$ (so $\sum_{i=1}^B f_i=1$), and the bins span the interval $[0, \max(\bm{e})]$. The number of bins $B$ is chosen as $B=0.9n$, where $n$ is the number of data points in the original data set. Among all the experiments, the bandwidth parameter is selected using the rule of thumb \eqref{bw_ROT}. The tolerance level is set to $\epsilon=10^{-7}$. The repeated experimental results are recorded in Table~\ref{table:BMS}, where the column ``Avg. Err. of Est. Modes'' presents the average distances between all the estimated local modes (identified by the original directional mean shift algorithm) and the nearest local mode estimates yielded by the blurring directional mean shift algorithm. As shown by Table 1, the blurring procedure is able to substantially reduce the total number of iterations for the directional mean shift algorithm with the von-Mises kernel. However, besides those valid estimated local modes identified by the original directional mean shift algorithm, the blurring version also recovers some spurious local mode estimates. The number of spurious local mode estimates from the blurring directional mean shift algorithm tends to decrease as the number of true local modes increases, given that the true local modes are well-separated. It illuminates a promising avenue to further accelerate the directional mean shift algorithm if a more delicate stopping criterion is designed.
		
		\begin{table}
			\centering
			\begin{tabular}{@{}|cccc|@{}}
				\hline\hline
				\bf Method (Scenario) & \bf \# Est. Modes & \bf \# Steps & \bf Avg. Err. of Est. Modes \\
				\hline\hline
				DMS (One mode)   & 4.25 (1.670)    & 86.30 (48.774) &  -- \\
				BDMS (One mode)  & 11.95 (2.156) &  17.10 (2.700) & 0.074 (0.0492) \\
				\hline
				DMS (Two modes)   & 2.40 (0.490) & 30.55 (5.757) &  -- \\
				BDMS (Two modes)  & 3.60 (1.114) & 9.90 (1.868) & 0.045 (0.0240)  \\
				\hline
				DMS (Three modes)   & 3.00 (0.000)  & 28.65 (5.790) &  -- \\
				BDMS (Three modes)  & 3.10 (0.300) &  7.75 (0.698) & 0.034 (0.0090)  \\
				\hline\hline
			\end{tabular}
			\caption{Comparisons between the Directional Mean Shift (DMS) and Blurring Directional Mean Shift (BDMS) algorithms. The means and standard errors (within round brackets) are calculated with 20 repeated experiments.}
			\label{table:BMS}
		\end{table}

	\section{Review of Geometry of Riemannian Manifolds}	
	\label{sec::GH}
		
	$\bullet$ \textbf{(Riemannian) Manifold}. A $m$-dimensional \emph{manifold} $M \subset \mathbb{R}^D$ with $D>m$ is a second countable Hausdorff space where each point has a neighborhood that is homeomorphic to the $m$-dimensional Euclidean space. For each point $p\in M$, it is possible to define a \emph{coordinate chart} $(U,\varphi)$ centered at $p$ as a homeomorphism $\varphi: U \to \varphi(U) \subset \mathbb{R}^m$, where $U$ is an open subset of $M$ containing $p$. Somewhat informally, if two coordinate charts $(U,\varphi)$ and $(V,\psi)$ are smoothly compatible, that is, either $U\cap V=\emptyset$ or the transition map $\psi\circ \varphi: \varphi(U\cap V) \to \psi(U \cap V)$ is a diffeomorphism, then $M$ is a smooth manifold. See Chapter 1 in \cite{Lee2012} for more formal definitions and discussions on smooth manifolds. A Riemannian manifold $(M, \mathfrak{g})$ is a real smooth manifold equipped with an inner product $\mathfrak{g}_p$ on the \emph{tangent space} $T_p(M)$ of every point $p\in M$, such that if $u,v$ are two vector fields on $M$ then $p\mapsto \langle u,v \rangle_p := \mathfrak{g}_p(u,v)$ is a smooth function. 
	
	\noindent $\bullet$ \textbf{Curvature}. The curvature of a Riemannian manifold is characterized by its Riemannian metric tensor at each point. \emph{Sectional curvature} is the Gaussian curvature of a two dimensional submanifold formed as the image of a two-dimensional subspace of a tangent space after exponential mapping. See Section 3-2 in \cite{doCarmo} for detailed discussions on the Gaussian curvature. It is known that a two-dimensional submanifold with positive, zero, or negative sectional curvature is locally isometric to a two-dimensional sphere, a Euclidean plane, or a hyperbolic plane with the same Gaussian curvature \citep{Geo_Convex_Op2016}.
	
	\noindent $\bullet$ \textbf{Differential}. Given a smooth $m$-dimensional manifold $M$, the \emph{differential} (or \emph{total gradient}) of a smooth function $f:U \subset M \to \mathbb{R}$ at $p\in U$ is defined as a linear map 
	$$df_p: T_p(M) \to T_{f(p)}(\mathbb{R}) \simeq \mathbb{R},$$ 
	where $U$ is an open subset of $M$,  $T_p(M)$ is the tangent space of $M$ at $p$, and $V_1 \simeq V_2$ means that these two vector spaces are isomorphic. Other commonly used notations for the differential are: $df_p(v) = v(f)(p) = v_p(f) = (v\cdot f)(p)$ for $v\in T_p(M)$. See Section 2-4 in \cite{doCarmo} and Section 3.1 in \cite{Morse_Homology2004} for more details.\\
		
	With an inner product structure on tangent spaces and the definition of differentials, one can define the gradient of a smooth function $f$ on $M$.
	
	\begin{definition}[Riemannian Gradient]
		\label{gradient}
		The (Riemannian) \emph{gradient} of a smooth function $f:M \to \mathbb{R}$ is a differentiable map $\grad f: M \to \mathcal{T} M$ which assigns to each point $p\in M$ a vector $\grad f(p) \in T_p(M) \subset \mathbb{R}^D$ such that
		\begin{equation}
		\label{grad_manifold}
		\langle \grad f(p), v \rangle_p = df_p(v) \quad \text{ for all } v\in T_p(M).
		\end{equation}
		Here $\mathcal{T}M$ is the tangent bundle, that is, the disjoint union of the tangent spaces at all points of $M$.
	\end{definition}

	In terms of the following definition, the Hessian matrices on a manifold are only well-defined at critical points, that is, those points whose differentials vanish, though an extension of the definition to non-critical points is possible.
	
	\begin{definition}[Riemannian Hessian]
		\label{Hessian2}
		The \emph{Hessian} $\mathcal{H}_p f$ of a smooth function $f: M\to \mathbb{R}$ at a critical point $p$ is a symmetric bilinear map
		$$\mathcal{H}_p f: T_p(M) \times T_p(M) \to \mathbb{R}$$
		defined as follows. For any tangent vectors $v,w \in T_p(M)$, we choose extensions $\tilde{v}$ and $\tilde{w}$ to vector fields on an open neighborhood of $p$ and set 
		$$\mathcal{H}_pf(v,w) = \left(\tilde{v} \cdot \left(\tilde{w} \cdot f \right)\right)(p) = v_p \left(\tilde{w} \cdot f \right).$$
		The expression above is \emph{independent} of the extensions $\tilde{v}$ of $v$ and $\tilde{w}$ of $w$, since
		$$\tilde{v} \cdot \left(\tilde{w} \cdot f \right)(p) - \tilde{w} \cdot \left(\tilde{v} \cdot f \right)(p) = [\tilde{v}, \tilde{w}]_p(f)=0$$
		at a critical point $p$, where $[\tilde{v}, \tilde{w}]_p$ is the commutator (or Lie bracket) of $\tilde{v}$ and $\tilde{w}$ at the point $p$. Thus, $\mathcal{H}_p f$ is a well-defined symmetric bilinear form on $T_p(M)$ at the critical point $p$.
	\end{definition}
	
	\begin{remark}
		Note that in general, $\tilde{v} \cdot \left(\tilde{w} \cdot f \right)(p)$ and $\tilde{w} \cdot \left(\tilde{v} \cdot f \right)(p)$ might be of different values when $p$ is not a critical point. This is essentially the definition of the vector $[\tilde{v},\tilde{w}]_p = \tilde{v} \cdot \left(\tilde{w} \cdot f \right)(p) - \tilde{w} \cdot \left(\tilde{v} \cdot f \right)(p)$.
	\end{remark}
	
	Given a coordinate chart $(U,\varphi)$ around $p\in M$, $\left\{\frac{\partial}{\partial x_1}\big|_p,..., \frac{\partial}{\partial x_m}\big|_p \right\}$ forms a basis for $T_p(M)$, and the matrix of $\mathcal{H}_pf$ with respect to this basis can be expressed by the $m\times m$ matrix of second partial derivatives:
	$$Q_pf := \left(\frac{\partial^2(f\circ \phi^{-1})}{\partial x_i \partial x_j} \phi(p) \right).$$
	
	It is possible to extend the definition of Hessian matrices of a smooth function $f:M \to \mathbb{R}$ to non-critical points based on the current definition \citep{Morse1963}. Given a local coordinate chart $(U,\varphi)$ near a non-critical point $q$ and $v = \sum\limits_{i=1}^m a_i \frac{\partial}{\partial x_i}\big|_q$, $w = \sum\limits_{j=1}^m b_j \frac{\partial}{\partial x_j}\big|_q$, we take $\tilde{w} = \sum\limits_{j=1}^m b_j \frac{\partial}{\partial x_j}\big|_q$, where $b_j$ now denotes a constant function. Then 
	$$\mathcal{H}_q f(v,w) = v(\tilde{w}(f))(q) = v\left(\sum_{j=1}^m b_j \frac{\partial f}{\partial x_j}\Big|_q \right) = \sum_{i=1}^m \sum_{j=1}^m a_i b_j \frac{\partial^2 f}{\partial x_i \partial x_j}(q);$$
	so the matrix $\left(\frac{\partial^2 f}{\partial x_i \partial x_j}(q) \right)_{i,j=1}^m$ represents the bilinear function $\mathcal{H}_q f$ with respect to the basis $\frac{\partial}{\partial x_1}\big|_q,..., \frac{\partial}{\partial x_m}\big|_q$. Another feasible avenue to define the Hessian on a Riemannian manifold starts from the notion of Riemannian gradient (Definition~\ref{gradient}) and covariant derivative (or affine connection). See \cite{Extrinsic_Look_Riem_Manifold} for more details.
	
	\begin{definition}[Non-degenerate Critical Points and Morse Functions (Definition 3.1 in \citealt{Morse_Homology2004})]
		\label{morse_fun}
		A critical point $p\in M$ of a differentiable function $f: M \to \mathbb{R}$ is \emph{non-degenerate} if the Hessian $\mathcal{H}_p f$ is non-degenerate. In other words, the determinant of $Q_pf$ is non-zero. Otherwise, $p$ is a \emph{degenerate} critical point. A differentiable function on $M$ is a \emph{Morse function} if all its critical points are non-degenerate.
	\end{definition}
	
	A standard result for a Morse function on a finite dimensional compact smooth manifold $M$, including $\Omega_q$, is that it has a finite number of critical points (Corollary 3.3 in \citealt{Morse_Homology2004}). Another remarkable fact in Morse theory is that integral curves on $M$ (equivalently, gradient ascent paths with infinitely small step sizes) never intersect except at critical points, so they partition the space \citep{Morse1925,Morse1930,Morse_Homology2004}. It thus serves as the backbone of mode clustering \citep{Mode_clu2016}. We have presented some mode clustering results on $\Omega_q$ using both synthetic and real-world data in Section~\ref{Sec:Experiments}.

	\subsection{Function Classes on Riemannian Manifolds}
	
	The key definitions in this subsection are modified from Section 2 in \cite{Geo_Convex_Op2016}. 
	
	\begin{definition}[Geodesic Concavity]
		A function $f: M \to \mathbb{R}$ is said to be \emph{geodesically concave} (or g-concave) if for any $p,q \in M$, a geodesic $\gamma$ such that $\gamma(0) =p$ and $\gamma(1)=q$, and $t\in [0,1]$, it holds that
		$$f(\gamma(t)) \geq (1-t) f(p) + tf(q).$$
		Equivalently, it can be shown that there exists a tangent vector $g_p \in T_p(M)$ such that 
		\begin{equation}
		\label{Geo_concave}
		f(q) \leq f(p) + \langle g_p, \Exp_p^{-1}(q) \rangle_p,
		\end{equation}
		where $g_p$ is called a \emph{subgradient} of $f$ at $p$, or the \emph{gradient} if $f$ is differentiable, and $\langle\cdot, \cdot \rangle_p$ denotes the inner product in the tangent space of $p$ induced by the Riemannian metric. 
	\end{definition}
	
	See, for instance, Section 1.2 in \cite{SB2015} for the definition of subgradients of convex functions.
	
	\begin{definition}[Geodesically Strong Concavity]
		A function $f: M \to \mathbb{R}$ is said to be geodesically $\mu$-strongly concave if for any $p,q \in M$,
		\begin{equation}
		\label{strong_Geo_concave}
		f(q) \leq f(p) + \langle g_p, \Exp_p^{-1}(q) \rangle_p - \frac{\mu}{2} \cdot d^2(p,q),
		\end{equation}
		where $d(p,q)= \sqrt{\langle \Exp_q^{-1}(q), \Exp_p^{-1}(q) \rangle_p}=\norm{\Exp_p^{-1}(q)}$.
	\end{definition}
	
	\begin{definition}[Lipschitzness]
		A function $f:M \to \mathbb{R}$ is said to be geodesically $L_f$-Lipschitz if for any $p,q\in M$,
		\begin{equation}
		\label{Lipschitz}
		|f(p) -f(q)| \leq L_f\cdot d(p,q).
		\end{equation}
	\end{definition}
	
	\begin{definition}[$\beta$-Smoothness]
		A differentiable function $f: M \to \mathbb{R}$ is said to be geodesically $\beta$-smooth if its gradient is $\beta$-Lipschitz. That is, for any $p,q \in M$,
		\begin{equation}
		\label{L_smooth}
		||g_p - \Gamma_q^p(g_q)|| \leq \beta \cdot d(x,y),
		\end{equation}
		where $\Gamma_q^p$ is the parallel transport from $q$ to $p$ and $\beta >0$ is a constant.
	\end{definition}

		\section{An alternative derivation of Algorithm \ref{Algo:MS}}
		\label{Appendix:Tang_MS_DR}
		
		From the expression of the Riemannian/tangent gradient estimator \eqref{tang_grad}, we obtain that
		\begin{align*}
		\grad \hat{f}_h(\bm{x}) &\equiv \Tang\left(\nabla \hat{f}_h(\bm{x}) \right) \\
		&= \frac{c_{h,q}(L)}{nh^2} \sum_{i=1}^n \left(\bm{x}^T \bm{X}_i \cdot \bm{x} - \bm{X}_i \right) \cdot L'\left(\frac{1-\bm{x}^T\bm{X}_i}{h^2} \right)\\
		&= \left[-\frac{c_{h,q}(L)}{nh^2} \sum_{i=1}^n \bm{x}^T \bm{X}_i L'\left(\frac{1-\bm{x}^T \bm{X}_i}{h^2} \right) \right] \cdot \left[\frac{\sum_{i=1}^n \bm{X}_i L'\left(\frac{1-\bm{x}^T \bm{X}_i}{h^2} \right)}{\sum_{i=1}^n \bm{x}^T \bm{X}_i L'\left(\frac{1-\bm{x}^T \bm{X}_i}{h^2} \right)} -\bm{x}\right]\\
		&= \left[\bm{x}^T \nabla \hat{f}_h(\bm{x}) \right] \cdot \left[\frac{\nabla \hat{f}_h(\bm{x})}{\bm{x}^T \nabla \hat{f}_h(\bm{x})} -\bm{x} \right],
		\end{align*}
		where we need to assume that $\bm{x}^T \nabla \hat{f}_h(\bm{x}) \neq 0$. (This is true in small neighborhoods of estimated local modes under condition (C2), which in turn holds with high probability as the sample size increases and bandwidth parameter decreases accordingly. This is guaranteed by Lemma~\ref{Rad_grad}.) By equating the alternative directional mean shift vector $\Xi_h'(\bm{x}) = \frac{\nabla \hat{f}_h(\bm{x})}{\bm{x}^T \nabla \hat{f}_h(\bm{x})} -\bm{x}$ to 0, we obtain that
		\begin{equation}
		\label{fix_point2}
		\hat{\bm{y}}_{s+1}' = \frac{\nabla \hat{f}_h(\hat{\bm{y}}_s)}{\hat{\bm{y}}_s^T \nabla \hat{f}_h(\hat{\bm{y}}_s)} \quad \text{ and } \quad \hat{\bm{y}}_{s+1} = \frac{\hat{\bm{y}}_{s+1}'}{\norm{\hat{\bm{y}}_{s+1}'}_2} = \text{sgn} \left(\hat{\bm{y}}_s^T \nabla \hat{f}_h(\hat{\bm{y}}_s) \right) \cdot \frac{\nabla \hat{f}_h(\hat{\bm{y}}_s)}{\norm{\nabla \hat{f}_h(\hat{\bm{y}}_s)}_2},
		\end{equation}
		where $\text{sgn}(x) = \mathbbm{1}_{\{x\geq 0\}} - \mathbbm{1}_{\{x\leq 0\}}$. Now, we discuss two mutually exclusive cases.
		\begin{itemize}
			\item (Case 1) If $\hat{\bm{y}}_s^T \nabla \hat{f}_h(\hat{\bm{y}}_s) >0$, then the directional mean shift vector $\Xi_h'(\hat{\bm{y}}_s) = \frac{\nabla \hat{f}_h(\hat{\bm{y}}_s)}{\hat{\bm{y}}_s^T \nabla \hat{f}_h(\hat{\bm{y}}_s)} -\hat{\bm{y}}_s$ is parallel to the Riemannian gradient at $\hat{\bm{y}}_s$ after being projected to the tangent space and points toward the direction of increasing the estimated density. Then, the preceding fixed-point iteration \eqref{fix_point2} is correct and can be simplified as
			$$\hat{\bm{y}}_{s+1} =  \frac{\nabla \hat{f}_h(\hat{\bm{y}}_s)}{\norm{\nabla \hat{f}_h(\hat{\bm{y}}_s)}_2} = -\frac{\sum_{i=1}^n \bm{X}_i L'\left(\frac{1-\hat{\bm{y}}_s \bm{X}_i}{h^2} \right)}{\sum_{i=1}^n L'\left(\frac{1-\hat{\bm{y}}_s \bm{X}_i}{h^2} \right)}.$$
			\item (Case 2) If $\hat{\bm{y}}_s^T \nabla \hat{f}_h(\hat{\bm{y}}_s) <0$, then the mean shift vector $\Xi_h'(\hat{\bm{y}}_s)$ is still parallel to the Riemannian gradient at $\hat{\bm{y}}_s$ after being projected to the tangent space but points toward the direction of decreasing the estimated density. Thus, the preceding fixed-point equation \eqref{fix_point2} goes as
			$$\hat{\bm{y}}_{s+1} = -\frac{\nabla \hat{f}_h(\hat{\bm{y}}_s)}{\norm{\nabla \hat{f}_h(\hat{\bm{y}}_s)}_2}$$
			but is \emph{not correct} in this case. We need to flip the sign of the fixed-point function and obtain that
			$$\hat{\bm{y}}_{s+1} =  \frac{\nabla \hat{f}_h(\hat{\bm{y}}_s)}{\norm{\nabla \hat{f}_h(\hat{\bm{y}}_s)}_2} = -\frac{\sum_{i=1}^n \bm{X}_i L'\left(\frac{1-\hat{\bm{y}}_s \bm{X}_i}{h^2} \right)}{\sum_{i=1}^n L'\left(\frac{1-\hat{\bm{y}}_s \bm{X}_i}{h^2} \right)}.$$
		\end{itemize}
		In both cases, the final fixed-point iteration equations coincide with our previous result in Equation \eqref{fix_point_eq} or \eqref{fix_point_grad}.

    \section{Proofs of Lemmas and Theorems}
	\label{Appendix:proofs}
	
	This section includes the proofs of our lemmas and theorems. Other auxiliary results are also presented along the way.
	
	\subsection{Proof of Lemma~\ref{lem:Hessian}}
	\label{Appendix:lem1_pf}
	
	\begin{customlem}{1}
	Assume that kernel $L$ is twice continuously differentiable. Then,
	$$\mathcal{H}\tilde f_h(\bm x)=\mathcal{H}\hat f_h(\bm x)$$
	for any point $\bm{x}\in\Omega_q$.
	\end{customlem}
		
	% \noindent \textsc{Proof of Lemma \ref{lem:Hessian}}. 
	\begin{proof}
	Some straightforward matrix calculus shows that
	\begin{align*}
	\nabla \nabla \tilde{f}_h(\bm{x}) &= \frac{c_{h,q}(L)}{nh^2} \sum_{i=1}^n I_{q+1} \cdot L'\left(\frac{1-\bm{x}^T\bm{X}_i}{h^2} \right) \\
	&\quad + \frac{c_{h,q}(L)}{nh^4} \sum_{i=1}^n (\bm{x} -\bm{X}_i) (\bm{x} -\bm{X}_i)^T \cdot L''\left(\frac{1-\bm{x}^T\bm{X}_i}{h^2} \right)\\
	& :=\hat{\mathcal{A}}_{\bm{x}} f
	\end{align*}
	and 
	$$\nabla\nabla \hat{f}_h(\bm{x}) = \frac{c_{h,q}(L)}{nh^4} \sum_{i=1}^n \bm{X}_i\bm{X}_i^T L''\left(\frac{1-\bm{x}^T\bm{X}_i}{h^2} \right).$$
	According to the generalized form of the Hessian matrix on $\Omega_q$ in \eqref{Hess_Dir}, we derive the Hessian estimator of the directional density $f$ as
	
	\begin{align*}
	&\left(I_{q+1} -\bm{x}\bm{x}^T \right) \left[\nabla \nabla \tilde{f}_h(\bm{x}) - \bm{x}^T \nabla \tilde{f}_h(\bm{x})\right]\left(I_{q+1} -\bm{x}\bm{x}^T \right)\\
	&= \frac{c_{h,q}(L)}{nh^2} \sum_{i=1}^n \left(I_{q+1} -\bm{x}\bm{x}^T \right) L'\left(\frac{1-\bm{x}^T\bm{X}_i}{h^2} \right) \\
	& \quad + \frac{c_{h,q}(L)}{nh^4} \sum_{i=1}^n \left(I_{q+1} -\bm{x}\bm{x}^T \right) \bm{X}_i\bm{X}_i^T \left(I_{q+1} -\bm{x}\bm{x}^T \right) L''\left(\frac{1-\bm{x}^T\bm{X}_i}{h^2} \right) \\
	& \quad - \frac{c_{h,q}(L)}{nh^2} \sum_{i=1}^n (1-\bm{x}^T \bm{X}_i)\left(I_{q+1} -\bm{x}\bm{x}^T \right) L'\left(\frac{1-\bm{x}^T\bm{X}_i}{h^2} \right)\\
	& = \left(I_{q+1} -\bm{x}\bm{x}^T \right) \Bigg[\frac{c_{h,q}(L)}{nh^4} \sum_{i=1}^n \bm{X}_i\bm{X}_i^T L''\left(\frac{1-\bm{x}^T\bm{X}_i}{h^2} \right) \\
	&\quad + \frac{c_{h,q}(L)}{nh^2} \sum_{i=1}^n \bm{x}^T \bm{X}_i I_{q+1}\cdot L'\left(\frac{1-\bm{x}^T\bm{X}_i}{h^2} \right) \Bigg] \left(I_{q+1} -\bm{x}\bm{x}^T \right)\\
	&= \left(I_{q+1} -\bm{x}\bm{x}^T \right) \left[\nabla \nabla \hat{f}_h(\bm{x}) - \bm{x}^T \nabla \hat{f}_h(\bm{x})\right]\left(I_{q+1} -\bm{x}\bm{x}^T \right),
	\end{align*}
	where we recall that $\nabla\tilde{f}_h(\bm{x}) = \frac{c_{h,q}(L)}{nh^2} \sum\limits_{i=1}^n (\bm{x} - \bm{X}_i) \cdot L'\left(\frac{1-\bm{x}^T\bm{X}_i}{h^2} \right)$ from \eqref{Dir_KDE_grad1} in the first equality. Thus, we conclude that the directional Hessian estimator at a point $\bm{x}\in \Omega_q$ is defined to be
	\begin{align}
	\label{Hess_KDE}
	\begin{split}
	\mathcal{H} \hat{f}_h(\bm x) &= \left(I_{q+1} -\bm{x}\bm{x}^T \right) \Bigg[\frac{c_{h,q}(L)}{nh^4} \sum_{i=1}^n \bm{X}_i\bm{X}_i^T L''\left(\frac{1-\bm{x}^T\bm{X}_i}{h^2} \right) \\
	&\quad + \frac{c_{h,q}(L)}{nh^2} \sum_{i=1}^n \bm{x}^T \bm{X}_i I_{q+1} \cdot L'\left(\frac{1-\bm{x}^T\bm{X}_i}{h^2} \right) \Bigg] \left(I_{q+1} -\bm{x}\bm{x}^T \right)\\
	&= \mathcal{H} \tilde{f}_h(\bm{x}).
	\end{split}
	\end{align} 
	The result follows.
	\end{proof}
		
	\subsection{Proof of Theorem~\ref{pw_conv_tang}}
	\label{Appendix:Thm2_pf}
		
	Before we dive into the (pointwise and uniform) consistency of the Riemannian gradient and Hessian estimators, we reiterate some common notation and terminology in directional data. For a variable $\bm{x}\in \Omega_q$ and a fixed point $\bm{y}\in \Omega_q$, we denote $t=\bm{x}^T \bm{y}$ the inner product between $\bm{x}$ and $\bm{y}$ and write
	$$\bm{x} = t\bm{y} + (1-t^2)^{\frac{1}{2}} \bm{\xi},$$
	where $\bm{\xi} \in \Omega_q$ is a unit vector orthogonal to $\bm{y}$. Further, an area element on $\Omega_q$ can be written as
	$$\omega_q(d\bm{x}) = (1-t^2)^{\frac{q}{2}-1} dt\, \omega_{q-1}(d\bm{\xi}).$$
	We will make extensive use of Lemmas 1, 2 and 3 in \cite{Dir_Linear2013} as well as their small extensions. Thus, we synthesize them in the following lemma.
		
	\begin{lemma}[A Change of Variables and Orthogonality in $\Omega_q$]
		\label{integ_lemma}
		The following results are extended from Lemmas 2 and 3 in \cite{Dir_Linear2013}:
			\begin{enumerate}[label=(\alph*)]
				\item Under condition (D2) or the stronger condition (D2'), we have that
				$$\lim_{h\to 0} \lambda_{h,q}(L) =\lambda_q(L) = 2^{\frac{q}{2}-1} \bar{\omega}_{q-1} \int_0^{\infty} L(r) r^{\frac{q}{2}-1} dr, $$
				where $\lambda_{h,q}(L) = \bar{\omega}_{q-1} \int_0^{2h^{-2}} L(r) r^{\frac{q}{2}-1} (2-rh^2)^{\frac{q}{2}-1} dr$ and $\bar{\omega}_q\equiv \omega_q(\Omega_q)$ is the surface area of $\Omega_q$ for $q\geq 1$. In other words, $\lambda_{h,q}(L)=\lambda_q(L) + o(1)$ as $h\to 0$.
				\item Let $f$ be a function defined in $\Omega_q$, and let $\bm{y} \in \Omega_q$ be a fixed point. The integral $\int_{\Omega_q} f(\bm{x}) \omega_q(d\bm{x})$ can be expressed in one of the following equivalent integrals:
				\begin{align}
				\label{change_of_var}
				\begin{split}
				\int_{\Omega_q} f(\bm{x}) \, \omega_q(d\bm{x}) &= \int_{-1}^1 \int_{\Omega_{q-1}} f\left(t, (1-t^2)^{\frac{1}{2}} \bm{\xi} \right) (1-t^2)^{\frac{q}{2}-1} \omega_{q-1}(d\bm{\xi}) dt\\
				&= \int_{-1}^1 \int_{\Omega_{q-1}} f\left(t\bm{y} + (1-t^2)^{\frac{1}{2}} \bm{B_y}\bm{\xi} \right) (1-t^2)^{\frac{q}{2}-1} \omega_{q-1}(d\bm{\xi}) dt,
				\end{split}
				\end{align}
				where $\bm{B_y} = (\bm{b}_1,...,\bm{b}_q)_{(q+1)\times q}$ is the semi-orthonormal matrix ($\bm{B_y}^T \bm{B_y} =I_q$ and $\bm{B_y} \bm{B_y}^T = I_{q+1}$) resulting from the completion of $\bm{y}$ to the orthonormal basis $\{\bm{y},\bm{b}_1,...,\bm{b}_q \}$.
				\item For any variable $\bm{x}=(x_1,...,x_{q+1})^T \in \Omega_q$, it holds that
				\[
				\int_{\Omega_q} x_i \omega_q(d\bm{x}) =0, \quad \int_{\Omega_q} x_ix_j\, \omega_q(d\bm{x}) = 
				\begin{cases}
				0, & i\neq j,\\
				\frac{\bar{\omega}_q}{q+1}, & i=j,
				\end{cases}
				\quad
				\int_{\Omega_q} x_ix_jx_k\, \omega_q(d\bm{x}) =0,
				\]
				\[
				\int_{\Omega_q} x_ix_j x_k x_m\, \omega_q(d\bm{x}) =
				\begin{cases}
				\frac{3\bar{\omega}_q}{(q+1)(q+3)}, & i=j=k=m,\\
				\frac{\bar{\omega}_q}{(q+1)(q+3)}, & i=k,j=m, i\neq j,\\
				0 & \text{otherwise},
				\end{cases} 
				\quad
				\int_{\Omega_q} x_ix_jx_k x_m x_{\ell} \, \omega_q(d\bm{x})=0
				\]
				for all $i,j,k,m,\ell =1,...,q+1$, where $\bar{\omega}_q$ is the surface area of $\Omega_q$ for $q\geq 1$. In particular, using the notation in (b), we have that 
				$$\int_{\Omega_{q-1}} \bm{B_x} \bm{\xi} \, \omega_{q-1}(d\bm{\xi}) = 0.$$
			\end{enumerate}
		\end{lemma}
		
		\begin{proof}
			As we will use the argument of (a) in our proof of Theorem~\ref{pw_conv_tang}, we reproduce the proof of Lemma 1 in \cite{Dir_Linear2013} here.\\
			\noindent (a) Consider the functions
			\begin{align*}
			\varpi_h(r) &= L(r) r^{\frac{q}{2}-1} (2-h^2r)^{\frac{q}{2}-1} \mathbbm{1}_{[0,2h^{-2})}(r),\\
			\varpi(r) &= \lim_{h\to 0} \varpi_h(r) =L(r) r^{\frac{q}{2}-1} 2^{\frac{q}{2}-1} \mathbbm{1}_{[0,\infty)}(r).
			\end{align*}
			Then, proving $\lim\limits_{h\to 0} \lambda_{h,q}(L) = \lambda_q(L)$ is equivalent to proving $\lim\limits_{h\to 0} \int_0^{\infty} \varpi_h(r) \,dr=\int_0^{\infty} \varpi(r) \,dr$.
			
			Consider first the case $q\geq 2$. As $\frac{q}{2}-1 \geq 0$, then $(2-h^2r)^{\frac{q}{2}-1} \leq 2^{\frac{q}{2}-1}$, $\forall h\geq 0, \forall r\in [0,2h^{-2})$. Then,
			$$|\varpi_h(r)| \leq L(r) r^{\frac{q}{2}-1} 2^{\frac{q}{2}-1} \mathbbm{1}_{[0,2h^{-2})}(r) \leq \varpi(r), \quad \forall r\in [0,\infty), \forall h >0.$$
			Since $\int_0^{\infty} \varpi(r) dr <\infty$ by condition (D2) on kernel $L$, by the Dominated Convergence Theorem, it follows that $\lim\limits_{h\to 0} \int_0^{\infty} \varpi_h(r) dr =\int_0^{\infty} \varpi(r) dr$.
			
			For the case $q=1$, $\varpi_h(r)=L(r) r^{-\frac{1}{2}} (2-h^2r)^{-\frac{1}{2}} \mathbbm{1}_{[0,2h^{-2})}(r)$. Consider now the following decomposition:
			\begin{align*}
			\int_0^{\infty} \varpi_h(r) dr &= \int_0^{\infty} L(r) r^{-\frac{1}{2}} (2-h^2r)^{-\frac{1}{2}} \mathbbm{1}_{[0,h^{-2})}(r) dr \\
			&\quad + \int_0^{\infty} L(r) r^{-\frac{1}{2}} (2-h^2r)^{-\frac{1}{2}} \mathbbm{1}_{[h^{-2},2h^{-2})}(r) dr.
			\end{align*}
			The limit of the first integral can be derived analogously with the Dominated Convergence Theorem. As $(2-h^2r)^{-\frac{1}{2}}$ is monotonically increasing with respect to $r \in [0,h^{-2})$, we know that $(2-h^2r)^{-\frac{1}{2}} \leq 1$, $\forall r\in [0,h^{-2})$, $\forall h>0$. Therefore,
			$$\left| L(r) r^{-\frac{1}{2}} (2-h^2r)^{-\frac{1}{2}} \mathbbm{1}_{[0,h^{-2})}(r) \right| \leq L(r) r^{-\frac{1}{2}} \mathbbm{1}_{[0,h^{-2})}(r) \leq \varpi(r), \quad \forall r\in [0,\infty), \forall h>0.$$
			Then, as $\lim\limits_{h\to 0} L(r) r^{-\frac{1}{2}} (2-h^2r)^{-\frac{1}{2}} \mathbbm{1}_{[0,h^{-2})}(r) =\varpi(r)$ and $\int_0^{\infty} \varpi(r) dr < \infty$ by condition (D2), the Dominated Convergence Theorem guarantees that 
			$$\lim\limits_{h\to 0} \int_0^{\infty} L(r) r^{-\frac{1}{2}} (2-h^2r)^{-\frac{1}{2}} \mathbbm{1}_{[0,h^{-2})}(r) dr = \int_0^{\infty} \varpi(r) dr.$$
			For the second integral, as a consequence of condition (D2), $L$ must be decrease faster than any power function in order for $0<\int_0^{\infty} L^k(r)r^{\frac{q}{2}-1} dr < \infty$ for all $q\geq 1$ and $k=1,2$. In particular, for some fixed $h_0 >0$, $L(r) \leq r^{-1}$, $\forall r\in [h^{-2}, 2h^{-2})$, $\forall h\in (0,h_0)$. Using this, it results in:
			$$\lim_{h\to 0} \int_{h^{-2}}^{2h^{-2}} L(r) r^{-\frac{1}{2}} (2-h^2r)^{-\frac{1}{2}} dr \leq \lim_{h\to 0}  \int_{h^{-2}}^{2h^{-2}} r^{-\frac{3}{2}} (2-h^2 r)^{-\frac{1}{2}} dr =\lim_{h\to 0} h=0.$$
			This completes the proof.\\
			
			The proofs of (b) and the first two integral results in (c) can be found in \cite{Dir_Linear2013} and thus omitted. We adopt some of the argument of Lemma 3 in \cite{Dir_Linear2013} to prove the last three integrals in (c).\\
			Recall that the $n$-dimensional spherical coordinates of $\bm{x}=(x_1,...,x_n)^T$ with norm $r:=\norm{\bm{x}}_2$ are given by
			\begin{equation}
			\label{sphe_coord}
			\begin{cases}
			x_1=r\cos \phi_1,\\
			x_j = r\cos \phi_j \prod\limits_{k=1}^{j-1} \sin \phi_k, \quad j=2,...,n-2,\\
			x_{n-1} = r\sin\theta \prod\limits_{k=1}^{n-2} \sin \phi_k,\\
			x_n = r\cos\theta \prod\limits_{k=1}^{n-2} \sin\phi_k,
			\end{cases}
			\quad J=r^{n-1} \prod\limits_{k=1}^{n-2} \sin^k \phi_{n-1-k},
			\end{equation}
			where $0\leq \phi_j \leq \pi$, $j=1,...,n-2$, $0\leq \theta \leq 2\pi$, and $0\leq r<\infty$. $J$ denotes the Jacobian of the transformation. Without loss of generality, we assume, by the $q$-spherical coordinates \eqref{sphe_coord}, that $x_i=\cos \phi_1$, $x_j=\cos\phi_2 \sin\phi_1$, and $x_k = \cos\phi_3 \sin\phi_2 \sin\phi_1$. Then,
			\begin{align*}
			\int_{\Omega_q} x_i^3 \, \omega_q(d\bm{x}) &= \int_0^{2\pi} \int_0^{\pi} \times \stackrel{(q-1)}{\cdots} \times \int_0^{\pi} \cos^3 \phi_1 \prod_{k=1}^{q-2} \sin^k \phi_{q-k} \sin^{q-1}\phi_1 \prod_{j=q-1}^1 d\phi_j d\theta\\
			&\hspace{-15mm}= \int_0^{2\pi} \int_0^{\pi} \times \stackrel{(q-2)}{\cdots} \times \int_0^{\pi} \prod_{k=1}^{q-2} \sin^k \phi_{q-k}  \prod_{j=q-1}^2 d\phi_j d\theta \times \int_0^{\pi} \cos^3\phi_1 \sin^{q-1} \phi_1 d\phi_1\\
			&\hspace{-15mm}= \int_0^{2\pi} \int_0^{\pi} \times \stackrel{(q-2)}{\cdots} \times \int_0^{\pi} \prod_{k=1}^{q-2} \sin^k \phi_{q-k}  \prod_{j=q-1}^2 d\phi_j d\theta \times \int_0^{\pi} (1-\sin^2\phi_1) \sin^{q-1} \phi_1 d(\sin\phi_1)\\
			&\hspace{-15mm}= \bar{\omega}_{q-1} \times 0= 0,
			\end{align*}
			\begin{align*}
			&\int_{\Omega_q} x_i^2x_j \omega_q(d\bm{x}) \\
			&= \int_0^{2\pi} \int_0^{\pi} \times \stackrel{(q-1)}{\cdots} \times \int_0^{\pi} \cos^2 \phi_1 \cos\phi_2 \sin\phi_1 \prod_{k=1}^{q-3} \sin^k \phi_{q-k} \sin^{q-2}\phi_2 \sin^{q-1} \phi_1 \prod_{j=q-1}^1 d\phi_j d\theta\\
			&= \int_0^{2\pi} \int_0^{\pi} \times \stackrel{(q-3)}{\cdots} \times \int_0^{\pi} \prod_{k=1}^{q-3} \sin^k \phi_{q-k}  \prod_{j=q-1}^3 d\phi_j d\theta \\
			&\quad \times \int_0^{\pi} \cos^2\phi_1 \sin^q \phi_1 d\phi_1 \int_0^{\pi} \cos\phi_2 \sin^{q-2} \phi_2 d\phi_2\\
			&=\bar{\omega}_{q-2} \times \int_0^{\pi} \cos^2\phi_1 \sin^q \phi_1 d\phi_1 \times 0 = 0,
			\end{align*}
			and 
			\begin{align*}
			\int_{\Omega_q} x_ix_j x_k \omega_q(d\bm{x}) &= \int_0^{2\pi} \int_0^{\pi} \times \stackrel{(q-1)}{\cdots} \times \int_0^{\pi} \cos\phi_1 \cos\phi_2 \sin\phi_1 \cos\phi_3 \sin\phi_2 \sin\phi_1 \\
			&\quad \times \prod_{k=1}^{q-4} \sin^k \phi_{q-k} \sin^{q-3}\phi_3 \sin^{q-2}\phi_2 \sin^{q-1} \phi_1 \prod_{j=q-1}^1 d\phi_j d\theta\\
			&= \int_0^{2\pi} \int_0^{\pi} \times \stackrel{(q-4)}{\cdots} \times \int_0^{\pi} \prod_{k=1}^{q-4} \sin^k \phi_{q-k}  \prod_{j=q-1}^4 d\phi_j d\theta \\
			&\quad \times \int_0^{\pi} \cos\phi_1 \sin^q \phi_1 d\phi_1 \int_0^{\pi} \cos\phi_2 \sin^{q-1} \phi_2 d\phi_2 \int_0^{\pi} \cos\phi_3 \sin^{q-2} \phi_3 d\phi_3\\
			&= \bar{\omega}_{q-3}\times 0 \times 0 \times 0 =0.
			\end{align*}
			The preceding argument teaches us that 
			$$\int_{\Omega_q} x_ix_jx_kx_m \, \omega_q(d\bm{x})= \int_{\Omega_q} x_ix_jx_kx_m x_{\ell} \, \omega_q(d\bm{x}) =0$$
			as long as one of the unique factors in the integrand has an odd multiplicity. (Indeed, any integration of a monomial with an odd degree on $\Omega_q$ will yield 0.) Thus, the only nonzero integrals in $\int_{\Omega_q} x_ix_jx_kx_m \, \omega_q(d\bm{x})$ and $\int_{\Omega_q} x_ix_jx_kx_m x_{\ell} \, \omega_q(d\bm{x})$ are
			$$\int_{\Omega_q} x_i^4 \, \omega_q(d\bm{x}) \quad \text{ and } \quad \int_{\Omega_q} x_i^2x_j^2 \, \omega_q(d\bm{x})$$
			with $i\neq j$. To compute the first integral, we define a vector field as 
			$$\bm{F}(\bm{x})=\left(F_1(\bm{x}),...,F_{q+1}(\bm{x}) \right)=(x_1^3,...,x_{q+1}^3)$$
			with $\bm{x}=(x_1,...,x_{q+1})\in \Omega_q$. By the divergence theorem (Theorem 10.51 in \citealt{Rudin1976}),
			\begin{align*}
			\int_{\Omega_q} x_i^4\, \omega_q(d\bm{x}) &= \frac{1}{q+1} \int_{\Omega_q} \left(\sum_{i=1}^{q+1} x_i^4 \right) \omega_q(d\bm{x})\\
			&= \frac{1}{q+1} \int_{\Omega_q} \langle \bm{F}, \bm{x} \rangle\, \omega_q(d\bm{x})\\
			&= \frac{1}{q+1} \int_{V_q} \mathtt{div}\,\bm{F} \, dV\\
			&= \frac{3}{q+1} \int_0^1 r^2 \cdot r^q\bar{\omega}_q dr = \frac{3\bar{\omega}_q}{(q+1)(q+3)},
			\end{align*}
			where $\langle \cdot,\cdot \rangle$ is the usual inner product in $\mathbb{R}^{q+1}$, $\mathtt{div}\,\bm{F} = \sum\limits_{i=1}^{q+1} \frac{\partial F_i}{\partial x_i}$, and $\int_{V_q} \cdots dV$ is integrating the solid $q$-dimensional sphere $V_q$ in $\mathbb{R}^{q+1}$. The second integral can be evaluated based on the preceding results as
			\begin{align*}
			\int_{\Omega_q} x_i^2x_j^2\, \omega_q(d\bm{x}) &= \frac{1}{q} \int_{\Omega_q} x_i^2\left(\sum_{j\neq i} x_j^2 \right) \, \omega_q(d\bm{x})\\
			&= \frac{1}{q} \int_{\Omega_q} (x_i^2-x_i^4) \, \omega_q(d\bm{x})\\
			&= \frac{1}{q}\left[\frac{\bar{\omega}_q}{q+1}- \frac{3\bar{\omega}_q}{(q+1)(q+3)} \right]= \frac{\bar{\omega}_q}{(q+1)(q+3)}. 
			\end{align*}
			As a specific application of our above results, we know that $\int_{\Omega_{q-1}} \bm{B_x} \bm{\xi} \, \omega_{q-1}(d\bm{\xi}) = 0$.
		\end{proof}
	
	\begin{remark}
		\label{integ_lemma_remark}
		\cite{Dir_Linear2013} also provided a key remark about how to generalize the arguments in (a) of Lemma~\ref{integ_lemma}. Under condition (D2'), one can apply the same techniques in (a) to prove the result with the functions
		\[
		\begin{cases}
		\varpi_{h,i,j,k}(r) = L^k(r) r^{\frac{q}{2}+i} (2-h^2r)^{\frac{q}{2}-j} \mathbbm{1}_{[0,2h^{-2})}(r),\\
		\varpi_{i,j,k}(r) = \lim\limits_{h\to 0} \varpi_{h,i,j,k}(r) = L^k(r) r^{\frac{q}{2}+i} 2^{\frac{q}{2}-j} \mathbbm{1}_{[0,\infty)}(r);
		\end{cases}
		\]
		\[
		\begin{cases}
		\varpi_{h,i,j,k}'(r) = [L'(r)]^k r^{\frac{q}{2}+i} (2-h^2r)^{\frac{q}{2}-j} \mathbbm{1}_{[0,2h^{-2})}(r),\\
		\varpi_{i,j,k}'(r) = \lim\limits_{h\to 0} \varpi_{h,i,j,k}'(r) = [L'(r)]^k r^{\frac{q}{2}+i} 2^{\frac{q}{2}-j} \mathbbm{1}_{[0,\infty)}(r);
		\end{cases}
		\]
		\[
		\begin{cases}
		\varpi_{h,i,j,k}''(r) = [L''(r)]^k r^{\frac{q}{2}+i} (2-h^2r)^{\frac{q}{2}-j} \mathbbm{1}_{[0,2h^{-2})}(r),\\
		\varpi_{i,j,k}''(r) = \lim\limits_{h\to 0} \varpi_{h,i,j,k}''(r) = [L''(r)]^k r^{\frac{q}{2}+i} 2^{\frac{q}{2}-j} \mathbbm{1}_{[0,\infty)}(r)
		\end{cases}
		\]
		with $i\geq -1$, $j\leq 1$, and $k=1,2$. For the case where $\frac{q}{2} -j \geq 0$, use the Dominated Convergence Theorem. For the other cases, subdivide the integral over $[0,2h^{-2})$ into the intervals $[0,h^{-2})$ and $[h^{-2},2h^{-2})$. Then apply the Dominated Convergence Theorem in the former and use a suitable power function to make the latter tend to 0 in the same way as described in the proof of (a) in Lemma~\ref{integ_lemma}.
	\end{remark}
		
	\begin{customthm}{2}
		Assume conditions (D1) and (D2').
		For any fixed $\bm{x}\in \Omega_q$, we have
		$$\grad \hat{f}_h(\bm{x})- \grad f(\bm{x}) = O(h^2) + O_P\left(\sqrt{\frac{1}{nh^{q+2}}} \right)$$
		as $h\to 0$ and $nh^{q+2} \to \infty$.\\
		Under the same condition, for any fixed $\bm{x}\in \Omega_q$, we have
		$$\mathcal{H} \hat{f}_h(\bm x) - \mathcal{H} f(\bm x) = O(h^2) + O_P\left(\sqrt{\frac{1}{nh^{q+4}}} \right)$$
		as $h\to 0$ and $nh^{q+4} \to \infty$.
	\end{customthm}

	\begin{proof}
		{\bf Part A: Pointwise convergence rate of the Riemannian gradient estimator $\grad \hat{f}_h(\bm{x})$}. Recall from Section~\ref{Sec:grad_Hess_Dir} that the tangent/Riemannian gradient estimator of a directional KDE is uniquely defined under a given kernel function $L$. Thus, we can establish the pointwise convergence rate under any total gradient (or differential) estimator, that is, $\grad \hat{f}_h(\bm{x})=\grad \tilde{f}_h(\bm{x}) \equiv \Tang\left(\nabla\tilde{f}_h(\bm{x}) \right)$. Here, we stick on the differential form \eqref{Dir_KDE_grad1}, $\nabla \tilde{f}_h(\bm{x})$. (One may also prove Theorem~\ref{pw_conv_tang} with $\nabla \hat{f}_h(\bm{x})$. The proof of Lemma~\ref{Rad_grad} provides a starting point for this direction.) 
		
		\noindent $\bullet$ {\bf Result 1}: The expectation of the Riemannian gradient estimator, $\mathbb{E}\left[\grad \hat{f}_h(\bm{x}) \right]$, has the following asymptotic behavior as $h\to 0$:
		\begin{align*}
		\mathbb{E}\left[\grad \tilde{f}_h(\bm{x})\right] &= \mathbb{E}\left[\Tang\left(\nabla \tilde{f}_h(\bm{x}) \right) \right] = (I_{q+1} -\bm{x}\bm{x}^T) \mathbb{E}\left[\nabla \tilde{f}_h(\bm{x}) \right]\\ 
		&\hspace{-20mm} = \left(I_{q+1} - \bm{x}\bm{x}^T \right) \nabla f(\bm{x}) + \frac{h^2}{2} \left(I_{q+1} - \bm{x}\bm{x}^T \right) \nabla f(\bm{x}) \cdot \frac{\int_0^{\infty} L(r) r^{\frac{q}{2}} dr}{\int_0^{\infty} L(r) r^{\frac{q}{2}-1} dr} \\
		&\hspace{-20mm} \quad + \frac{2h^2}{q} \sum_{i=1}^q \left(\bm{x}^T \nabla\nabla f(\bm{x}) \bm{b}_i \right) \bm{b}_i \cdot  \frac{\int_0^{\infty} L'(r) r^{\frac{q}{2}+1} dr}{\int_0^{\infty} L(r) r^{\frac{q}{2}-1} dr} +O(h^2) + o(h^2)\\
		&\hspace{-20mm} = \left(I_{q+1} - \bm{x}\bm{x}^T \right) \nabla f(\bm{x}) + O(h^2).
		\end{align*}
		
		\noindent \textit{Derivation of Result 1}. With the definition of $\bm{B_x}$ from Lemma~\ref{integ_lemma}, the expected value of $\nabla \tilde{f}_h(\bm{x})$ is
		\begin{align}
			\label{grad_expected1}
			\begin{split}
			\mathbb{E}\left[\nabla \tilde{f}_h(\bm{x}) \right]& = \frac{c_{h,q}(L)}{h^2} \int_{\Omega_q} (\bm{x} -\bm{y})\cdot  L'\left(\frac{1-\bm{x}^T \bm{y}}{h^2} \right) f(\bm{y}) \, \omega_q(d\bm{y})\\
			&= \frac{c_{h,q}(L)}{h^2} \int_{-1}^1 \int_{\Omega_{q-1}} \left(\bm{x} -t\bm{x} -\sqrt{1-t^2} \bm{B_x}\bm{\xi} \right) L'\left(\frac{1-t}{h^2} \right)\\
			&\quad \times f\left(t\bm{x} + \sqrt{1-t^2} \bm{B_x}\bm{\xi} \right) (1-t^2)^{\frac{q}{2}-1} \, \omega_{q-1}(d\bm{\xi}) dt\\
			&= c_{h,q}(L) h^{q-2} \int_0^{2h^{-2}} \int_{\Omega_{q-1}} \left(rh^2 \bm{x} -h\sqrt{r(2-h^2 r)} \bm{B_x}\bm{\xi} \right) L'(r)\\
			&\quad \times f(\bm{x} + \alpha_{\bm{x},\bm{\xi}}) \cdot r^{\frac{q}{2}-1} (2-h^2r)^{\frac{q}{2}-1} \omega_{q-1}(d\bm{\xi}) dr
			\end{split}
			\end{align}
			by (a) in Lemma \ref{integ_lemma} and a change of variable $r=\frac{1-t}{h^2}$, where $\alpha_{\bm{x},\bm{\xi}} = -rh^2 \bm{x} +h\sqrt{r(2-h^2 r)} \bm{B_x} \bm{\xi}$. By condition (D1), the Taylor's expansion of $f$ at $\bm{x}$ is
			\begin{align*}
			&f(\bm{x}+\alpha_{\bm{x},\bm{\xi}}) \\
			&= f(\bm{x}) + \alpha_{\bm{x},\bm{\xi}}^T \nabla f(\bm{x}) + \frac{1}{2} \alpha_{\bm{x},\bm{\xi}}^T \nabla\nabla f(\bm{x}) \alpha_{\bm{x},\bm{\xi}} + \frac{1}{6} \left(\sum_{i=1}^{q+1} (\alpha_{\bm{x},\bm{\xi}})_i \cdot \frac{\partial}{\partial x_i} \right)^3 f(\bm{x}) + o\left(||\alpha_{\bm{x},\bm{\xi}}||_2^3 \right)\\
			&\equiv \text{(I) + (II) + (III) + (IV)} + o(h^3),
			\end{align*}
			where $\norm{\alpha_{\bm{x},\bm{\xi}}}_2^2 = r^2h^4 + h^2r(2-h^2r) = 2rh^2$ by the orthogonality of $\bm{x}$ and columns of $\bm{B_x}$, and $(\alpha_{\bm{x},\bm{\xi}})_i$ stands for the $i^{th}$ entry of the vector $\alpha_{\bm{x},\bm{\xi}}$. Now we plug (I), (II), (III), (IV), and $o(h^3)$ back into (\ref{grad_expected1}) respectively to compute the dominating term of $\mathbb{E}\left[\nabla \tilde{f}_h(\bm{x}) \right]$.
			\begin{align*}
			&\text{Plug in (I)} \\
			&= c_{h,q}(L) h^{q-2} f(\bm{x}) \int_0^{2h^{-2}} \int_{\Omega_{q-1}} \left(rh^2 \bm{x} -h\sqrt{r(2-h^2 r)} \bm{B_x}\bm{\xi} \right) \\
			&\quad \quad \times L'(r)\, r^{\frac{q}{2}-1} (2-h^2r)^{\frac{q}{2}-1} \omega_{q-1}(d\bm{\xi}) dr\\
			&\stackrel{\text{(i)}}{=} \bar{\omega}_{q-1}\cdot \bm{x} f(\bm{x}) \int_0^{2h^{-2}} c_{h,q}(L) h^q L'(r) \cdot r^{\frac{q}{2}} (2-h^2r)^{\frac{q}{2}-1} dr +0 \\
			&\stackrel{\text{(ii)}}{=} \bar{\omega}_{q-1}\cdot \bm{x} f(\bm{x}) \cdot c_{h,q}(L) h^q \bigg\{L(r) r^{\frac{q}{2}} (2-h^2r)^{\frac{q}{2}-1} \Big|_{0}^{2h^{-2}} \\
			&\hspace{30mm} - \int_0^{2h^{-2}} L(r) \left[\frac{q}{2}\cdot r^{\frac{q}{2}-1}(2-h^2r)^{\frac{q}{2}-1} -h^2\left(\frac{q}{2}-1 \right) r^{\frac{q}{2}}(2-h^2r)^{\frac{q}{2}-2} \right] dr \bigg\} \\
			&= -\frac{q}{2} \cdot \bar{\omega}_{q-1} \cdot \bm{x} f(\bm{x}) \cdot c_{h,q}(L) h^q \int_0^{2h^{-2}} L(r) r^{\frac{q}{2}-1} (2-h^2r)^{\frac{q}{2}-1} dr \\
			&\quad + \left(\frac{q-2}{2}\right) \bar{\omega}_{q-1} \cdot \bm{x} f(\bm{x}) \cdot c_{h,q}(L) h^{q+2} \int_0^{2h^{-2}} L(r) r^{\frac{q}{2}} (2-h^2r)^{\frac{q}{2}-2} dr \\
			&\stackrel{\text{(iii)}}{=} -\frac{q}{2} \cdot \bm{x} f(\bm{x}) + \left(\frac{q-2}{2}\right)\bm{x} f(\bm{x}) h^2 \cdot \frac{\int_0^{2h^{-2}} L(r) r^{\frac{q}{2}} (2-h^2r)^{\frac{q}{2}-2} dr}{\int_0^{2h^{-2}} L(r) r^{\frac{q}{2}-1} (2-h^2r)^{\frac{q}{2}-1} dr}\\
			&\stackrel{\text{(iv)}}{=} -\frac{q}{2} \cdot \bm{x} f(\bm{x}) + \left(\frac{q-2}{4}\right)\bm{x} f(\bm{x}) h^2 \cdot \frac{\int_0^{\infty} L(r) r^{\frac{q}{2}} dr}{\int_0^{\infty} L(r) r^{\frac{q}{2}-1} dr} +o(h^2),
			\end{align*}
			as $h\to 0$, where we use (c) of Lemma~\ref{integ_lemma} with $\bm{B_x}\bm{\xi}=\sum_{i=1}^q \xi_i\bm{b}_i$ in (i), conduct integration by parts in (ii), plug in the expression \eqref{asym_norm_const} of $c_{h,q}(L)$ in (iii), and take $h\to 0$ with our argument in (a) of Lemma~\ref{integ_lemma} and Remark~\ref{integ_lemma_remark} to obtain (iv). The $o(h^2)$-term in (iv) takes into account those small error terms as $h\to 0$. Likewise,
			\begin{align*}
			&\text{Plug in (II)} \\
			&= c_{h,q}(L) h^{q-2} \int_0^{2h^{-2}} \int_{\Omega_{q-1}} \left(rh^2 \bm{x} -h\sqrt{r(2-h^2 r)} \bm{B_x}\bm{\xi} \right) \alpha_{\bm{x},\bm{\xi}}^T f(\bm{x}) \\
			&\quad \quad \times L'(r) \cdot r^{\frac{q}{2}-1} (2-h^2r)^{\frac{q}{2}-1} \omega_{q-1}(d\bm{\xi}) dr\\
			&= - c_{h,q}(L) h^{q+2} \int_0^{2h^{-2}} \int_{\Omega_{q-1}} \bm{x}\bm{x}^T \nabla f(\bm{x}) L'(r) r^{\frac{q}{2}+1} (2-h^2r)^{\frac{q}{2}-1} \omega_{q-1}(d\bm{\xi}) dr \\
			&\quad + c_{h,q}(L) h^{q+1} \int_0^{2h^{-2}} \int_{\Omega_{q-1}} \bm{x} \bm{\xi}^T \bm{B_x}^T \nabla f(\bm{x}) L'(r) r^{\frac{q+1}{2}} (2-h^2r)^{\frac{q-1}{2}} \omega_{q-1}(d\bm{\xi}) dr\\
			& \quad + c_{h,q}(L) h^{q+1} \int_0^{2h^{-2}} \int_{\Omega_{q-1}} \bm{B_x}\bm{\xi} \cdot \bm{x}^T \nabla f(\bm{x}) L'(r) r^{\frac{q+1}{2}} (2-h^2r)^{\frac{q-1}{2}} \omega_{q-1}(d\bm{\xi}) dr\\
			& \quad - c_{h,q}(L) h^q \int_0^{2h^{-2}} \int_{\Omega_{q-1}} \bm{B_x}\bm{\xi} \cdot \bm{\xi}^T \bm{B_x}^T \nabla f(\bm{x}) L'(r) r^{\frac{q}{2}} (2-h^2r)^{\frac{q}{2}} \omega_{q-1}(d\bm{\xi}) dr\\
			&\stackrel{\text{(i)}}{=} -c_{h,q}(L) h^{q+2}\cdot \bm{x}\bm{x}^T \nabla f(\bm{x}) \cdot \bar{\omega}_{q-1} \int_0^{2h^{-2}}   L'(r) r^{\frac{q}{2}+1} (2-h^2r)^{\frac{q}{2}-1} dr + 0 + 0\\
			&\quad - c_{h,q}(L) h^q \int_0^{2h^{-2}} \int_{\Omega_{q-1}} \left(\sum_{i=1}^q \xi_i \bm{b}_i \right) \left(\sum_{i=1}^q \xi_i \bm{b}_i^T \nabla f(\bm{x}) \right)  L'(r) r^{\frac{q}{2}} (2-h^2r)^{\frac{q}{2}} \omega_{q-1}(d\bm{\xi}) dr \\
			&\stackrel{\text{(ii)}}{=} -c_{h,q}(L) h^{q+2}\cdot \bm{x}\bm{x}^T \nabla f(\bm{x}) \cdot \bar{\omega}_{q-1} \Bigg\{r^{\frac{q}{2}+1} (2-h^2r)^{\frac{q}{2}-1} L(r) \Big|_0^{2h^{-2}} \\
			&\hspace{20mm} - \int_0^{2h^{-2}} L(r) \left[\left(\frac{q+2}{2}\right)r^{\frac{q}{2}} (2-h^2r)^{\frac{q}{2}-1} -h^2\left(\frac{q-2}{2} \right) r^{\frac{q}{2}+1} (2-h^2r)^{\frac{q}{2}-2} \right] dr \Bigg\}\\
			& \quad - \frac{\bar{\omega}_{q-1}}{q} \left(\sum_{i=1}^q \bm{b}_i \bm{b}_i^T \right) \nabla f(\bm{x}) \cdot  c_{h,q}(L) h^q \int_0^{2h^{-2}}  L'(r) r^{\frac{q}{2}} (2-h^2r)^{\frac{q}{2}} dr\\
			&\stackrel{\text{(iii)}}{=} c_{h,q}(L) h^{q+2}\cdot \bm{x}\bm{x}^T \nabla f(\bm{x}) \cdot \bar{\omega}_{q-1} \left(\frac{q+2}{2}\right)\int_0^{2h^{-2}} L(r) r^{\frac{q}{2}} (2-h^2r)^{\frac{q}{2}-1} dr \\
			&\quad - c_{h,q}(L) h^{q+4}\cdot \bm{x}\bm{x}^T \nabla f(\bm{x}) \cdot \bar{\omega}_{q-1} \left(\frac{q-2}{2}\right)\int_0^{2h^{-2}} r^{\frac{q}{2}+1} (2-h^2r)^{\frac{q}{2}-2} dr\\
			&\quad - \frac{\bar{\omega}_{q-1}}{q} (I_{q+1} - \bm{x}\bm{x}^T) \nabla f(\bm{x}) \cdot c_{h,q}(L) h^q \Bigg[L(r) r^{\frac{q}{2}} (2-h^2r)^{\frac{q}{2}} \Big|_0^{2h^{-2}} \\
			&\hspace{40mm} - \frac{q}{2} \int_0^{2h^{-2}}L(r) \left(r^{\frac{q}{2}-1}(2-h^2r)^{\frac{q}{2}} -h^2r^{\frac{q}{2}} (2-h^2 r)^{\frac{q}{2}-1} \right)dr \Bigg]\\
			&\stackrel{\text{(iv)}}{=} \left(\frac{q+2}{2} \right) h^2 \bm{x}\bm{x}^T \nabla f(\bm{x}) \cdot \frac{\int_0^{2h^{-2}} L(r) r^{\frac{q}{2}} (2-h^2r)^{\frac{q}{2}-1} dr}{\int_0^{2h^{-2}} L(r) r^{\frac{q}{2}-1} (2-h^2r)^{\frac{q}{2}-1} dr}\\
			&\quad - \left(\frac{q-2}{2} \right) h^4 \bm{x}\bm{x}^T \nabla f(\bm{x}) \cdot \frac{\int_0^{2h^{-2}} L(r) r^{\frac{q}{2}+1} (2-h^2r)^{\frac{q}{2}-2} dr}{\int_0^{2h^{-2}} L(r) r^{\frac{q}{2}-1} (2-h^2r)^{\frac{q}{2}-1} dr}\\
			&\quad + \frac{1}{2}\left(I_{q+1}-\bm{x}\bm{x}^T\right) \nabla f(\bm{x}) \Bigg[\frac{\int_0^{2h^{-2}} L(r) r^{\frac{q}{2}-1} (2-h^2r)^{\frac{q}{2}} dr}{\int_0^{2h^{-2}} L(r) r^{\frac{q}{2}-1} (2-h^2r)^{\frac{q}{2}-1} dr} \\
			&\hspace{50mm} - h^2 \cdot \frac{\int_0^{2h^{-2}} L(r) r^{\frac{q}{2}} (2-h^2r)^{\frac{q}{2}-1} dr}{\int_0^{2h^{-2}} L(r) r^{\frac{q}{2}-1} (2-h^2r)^{\frac{q}{2}-1} dr}\Bigg]\\
			&\stackrel{\text{(v)}}{=} \left(\frac{q+2}{2} \right) h^2 \cdot \bm{x}\bm{x}^T \nabla f(\bm{x}) \cdot \frac{\int_0^{\infty} L(r) r^{\frac{q}{2}} dr}{\int_0^{\infty} L(r) r^{\frac{q}{2}-1} dr} - \left(\frac{q-2}{4} \right) h^4 \cdot \bm{x}\bm{x}^T \nabla f(\bm{x}) \cdot \frac{\int_0^{\infty} L(r) r^{\frac{q}{2}+1} dr}{\int_0^{\infty} L(r) r^{\frac{q}{2}-1} dr}\\ 
			&\quad + \left(I_{q+1} - \bm{x}\bm{x}^T \right) \nabla f(\bm{x}) + O(h^2) - \frac{h^2}{2} \left(I_{q+1} - \bm{x}\bm{x}^T \right) \nabla f(\bm{x}) \cdot \frac{\int_0^{\infty} L(r) r^{\frac{q}{2}} dr}{\int_0^{\infty} L(r) r^{\frac{q}{2}-1} dr} + o(h^2) \\
			&= (I_{q+1} - \bm{x}\bm{x}^T) \nabla f(\bm{x}) + h^2\left(\frac{q+2}{2} \right) \bm{x}\bm{x}^T \nabla f(\bm{x}) \cdot \frac{\int_0^{\infty} L(r) r^{\frac{q}{2}} dr}{\int_0^{\infty} L(r) r^{\frac{q}{2}-1} dr} \\
			&\quad +\frac{h^2}{2} \left(I_{q+1} - \bm{x}\bm{x}^T \right) \nabla f(\bm{x}) \cdot \frac{\int_0^{\infty} L(r) r^{\frac{q}{2}} dr}{\int_0^{\infty} L(r) r^{\frac{q}{2}-1} dr} +O(h^2) + o(h^2),
			\end{align*}
			where we use (c) of Lemma~\ref{integ_lemma} in (i) and (ii), leverage the fact that $\sum_{i=1}^q \bm{b}_i \bm{b}_i^T = \bm{B_x} \bm{B_x}^T = I_{q+1} - \bm{x}\bm{x}^T$ in (iii), plug in the expression \eqref{asym_norm_const} of $c_{h,q}(L)$ in (iv), and take $h\to 0$ with arguments in Lemma~\ref{integ_lemma} and Remark~\ref{integ_lemma_remark} to obtain (v). The $o(h^2)$-term incorporates higher-order error terms, while the $O(h^2)$-term in (v) comes from the following arguments:
			% $$\frac{1}{2}\left(I_{q+1}-\bm{x}\bm{x}^T\right) \nabla f(\bm{x})\cdot \frac{\int_0^{2h^{-2}} L(r) r^{\frac{q}{2}-1} (2-h^2r)^{\frac{q}{2}} dr}{\int_0^{2h^{-2}} L(r) r^{\frac{q}{2}-1} (2-h^2r)^{\frac{q}{2}-1} dr} \to \left(I_{q+1}-\bm{x}\bm{x}^T\right) \nabla f(\bm{x}) \quad \text{ as } \,\, h\to 0,$$
			\begin{equation}
			\label{kernel_asymp_rate}
			\frac{\int_0^{2h^{-2}} L(r) r^{\frac{q}{2}-1} (2-h^2r)^{\frac{q}{2}} dr}{\int_0^{2h^{-2}} L(r) r^{\frac{q}{2}-1} (2-h^2r)^{\frac{q}{2}-1} dr} - 2 
			= -h^2 \cdot \frac{\int_0^{2h^{-2}} L(r) r^{\frac{q}{2}} (2-h^2r)^{\frac{q}{2}-1} dr}{\int_0^{2h^{-2}} L(r) r^{\frac{q}{2}-1} (2-h^2r)^{\frac{q}{2}-1} dr} = O(h^2).
			\end{equation}
			We now move on to the calculation of (III), which is more complicated.
			\begin{align}
			\label{eq_III}
			\begin{split}
			\text{Plug in (III)} &= c_{h,q}(L) h^q \int_0^{2h^{-2}} \int_{\Omega_{q-1}} \frac{1}{2} \alpha_{\bm{x},\bm{\xi}}^T \nabla\nabla f(\bm{x}) \alpha_{\bm{x},\bm{\xi}} \cdot \bm{x} L'(r) r^{\frac{q}{2}} (2-h^2r)^{\frac{q}{2}-1} \omega_{q-1}(d\bm{\xi}) dr\\
			&\hspace{-15mm} - c_{h,q}(L) h^{q-1} \int_0^{2h^{-2}} \int_{\Omega_{q-1}} \frac{1}{2} \alpha_{\bm{x},\bm{\xi}}^T \nabla\nabla f(\bm{x}) \alpha_{\bm{x},\bm{\xi}} \cdot \bm{B_x}\bm{\xi} L'(r) r^{\frac{q-1}{2}} (2-h^2r)^{\frac{q-1}{2}} \omega_{q-1}(d\bm{\xi}) dr.
			\end{split}
			\end{align}
			Notice that 
			\begin{align}
			\label{eq_III1}
			\begin{split}
			&\int_{\Omega_{q-1}} \alpha_{\bm{x},\bm{\xi}}^T \nabla\nabla f(\bm{x}) \alpha_{\bm{x},\bm{\xi}} \cdot \bm{x} \,\omega_{q-1}(\bm{\xi}) \\
			&= r^2h^4 \int_{\Omega_{q-1}} \bm{x}^T \nabla\nabla f(\bm{x}) \bm{x} \cdot \bm{x} \,\omega_{q-1}(d\bm{\xi}) \\
			&\quad -2rh^3 \sqrt{r(2-h^2 r)} \int_{\Omega_{q-1}} \bm{x}^T \nabla\nabla f(\bm{x}) \bm{B_x} \bm{\xi} \cdot \bm{x} \, \omega_{q-1}(d\bm{\xi})\\
			&\quad + h^2r(2-h^2r) \int_{\Omega_{q-1}} \bm{\xi}^T \bm{B_x}^T \nabla\nabla f(\bm{x}) \bm{B_x}\bm{\xi} \cdot \bm{x} \, \omega_{q-1}(d\bm{\xi})\\
			&= r^2h^4 \bar{\omega}_{q-1} \bm{x}^T \nabla\nabla f(\bm{x}) \bm{x} \cdot \bm{x} + h^2r(2-h^2r) \int_{\Omega_{q-1}} \left(\sum_{i,j=1}^q \bm{b}_i^T \nabla\nabla f(\bm{x}) \bm{b}_j \xi_i\xi_j \right) \bm{x}\, \omega_{q-1}(d\bm{\xi})\\
			&= r^2h^4 \bar{\omega}_{q-1} \bm{x}^T \nabla\nabla f(\bm{x}) \bm{x} \cdot \bm{x} + h^2r(2-h^2r) \int_{\Omega_{q-1}} \left(\sum_{i=1}^q \bm{b}_i^T \nabla\nabla f(\bm{x}) \bm{b}_i \xi_i^2 \right) \bm{x}\, \omega_{q-1}(d\bm{\xi})\\
			&= r^2h^4 \bar{\omega}_{q-1} \bm{x}^T \nabla\nabla f(\bm{x}) \bm{x} \cdot \bm{x} + h^2r(2-h^2r) \cdot \frac{\bar{\omega}_{q-1}}{q} \left[\Delta f(\bm{x}) -\bm{x}^T \nabla\nabla f(\bm{x}) \bm{x} \right] \cdot \bm{x},
			\end{split}
			\end{align}
			where we use (c) of Lemma~\ref{integ_lemma} in the second, third, and fourth equations and the fact that 
			\begin{align*}
			\sum_{i=1}^q \bm{b}_i^T \nabla\nabla f(\bm{x}) \bm{b}_i = \text{tr}\left[\nabla\nabla f(\bm{x}) \sum_{i=1}^q \bm{b}_i\bm{b}_i^T \right] &= \text{tr}\left[\nabla\nabla f(\bm{x}) (I_{q+1} -\bm{x}\bm{x}^T)\right] \\
			&= \Delta f(\bm{x}) -\bm{x}^T \nabla\nabla f(\bm{x}) \bm{x}.
			\end{align*}
			Here, $\Delta f(\bm{x}) = \sum_{i=1}^{q+1} \frac{\partial^2}{\partial x_i^2} f(\bm{x})$ is the Laplace of function $f$. At the same time,
			\begin{align*}
			%\label{eq_III2}
			%\begin{split}
			&\int_{\Omega_{q-1}} \alpha_{\bm{x},\bm{\xi}}^T \nabla\nabla f(\bm{x}) \alpha_{\bm{x},\bm{\xi}} \cdot \bm{B_x}\bm{\xi} \, \omega_{q-1}(d\bm{\xi}) \\
			&= r^2h^4 \int_{\Omega_{q-1}} \bm{x}^T \nabla\nabla f(\bm{x}) \bm{x} \cdot \bm{B_x}\bm{\xi} \, \omega_{q-1}(d\bm{\xi})\\
			&\quad - 2rh^3 \sqrt{r(2-h^2r)} \int_{\Omega_{q-1}} \bm{x}^T \nabla\nabla f(\bm{x}) \bm{B_x}\bm{\xi} \cdot \bm{B_x}\bm{\xi} \, \omega_{q-1}(d\bm{\xi})\\
			&\quad + h^2r(2-h^2r) \int_{\Omega_{q-1}} \bm{\xi}^T \bm{B_x}^T \nabla\nabla f(\bm{x}) \bm{B_x} \bm{\xi} \cdot \bm{B_x}\bm{\xi} \, \omega_{q-1}(d\bm{\xi})\\
			&= - 2rh^3 \sqrt{r(2-h^2r)} \int_{\Omega_{q-1}} \left(\sum_{i=1}^q \bm{x}^T \nabla\nabla f(\bm{x}) \bm{b}_i \cdot \bm{b}_i \xi_i^2 \right) \omega_{q-1}(d\bm{\xi}) \\
			&\quad + h^2r(2-h^2r) \int_{\Omega_{q-1}} \left(\sum_{i,j=1}^q \bm{b}_i^T \nabla\nabla f(\bm{x}) \bm{b}_j \xi_i\xi_j \right) \left(\sum_{k=1}^q \bm{b}_k \xi_k\right) \, \omega_{q-1}(d\bm{\xi})\\
			&= - 2rh^3 \sqrt{r(2-h^2r)} \cdot  \frac{\bar{\omega}_{q-1}}{q} \sum_{i=1}^q \left(\bm{x}^T \nabla\nabla f(\bm{x}) \bm{b}_i \right)  \bm{b}_i,
			%\end{split}
			\end{align*}
			where we apply (c) of Lemma~\ref{integ_lemma} in the last two equations. That is,
			\begin{align}
			\label{eq_III2}
			\begin{split}
			\int_{\Omega_{q-1}} \alpha_{\bm{x},\bm{\xi}}^T \nabla\nabla f(\bm{x}) \alpha_{\bm{x},\bm{\xi}} \cdot \bm{B_x}\bm{\xi} \, \omega_{q-1}(d\bm{\xi}) 
			&= - 2rh^3 \sqrt{r(2-h^2r)} \cdot  \frac{\bar{\omega}_{q-1}}{q} \sum_{i=1}^q \left(\bm{x}^T \nabla\nabla f(\bm{x}) \bm{b}_i \right)  \bm{b}_i.
			\end{split}
			\end{align}
			Plugging \eqref{eq_III1} and \eqref{eq_III2} back into \eqref{eq_III}, we proceed ``Plug in (III)'' as
			\begin{align*}
			\text{Plug in (III)} &= \frac{c_{h,q}(L) h^{q+4}}{2} \cdot \bar{\omega}_{q-1} \bm{x}^T \nabla\nabla f(\bm{x}) \bm{x}\cdot \bm{x} \int_0^{2h^{-2}} r^{\frac{q}{2}+2} (2-h^2r)^{\frac{q}{2}-1} L'(r) dr\\
			&\quad + \frac{c_{h,q}(L) h^{q+2}}{2q} \cdot \bar{\omega}_{q-1} \left[\Delta f(\bm{x}) -\bm{x}^T \nabla\nabla f(\bm{x}) \bm{x} \right] \bm{x} \int_0^{2h^{-2}} L'(r) r^{\frac{q}{2}+1} (2-h^2r)^{\frac{q}{2}} dr\\
			&\quad + \frac{c_{h,q}(L) h^{q+2}}{q} \cdot \bar{\omega}_{q-1} \sum_{i=1}^q \left(\bm{x}^T \nabla\nabla f(\bm{x}) \bm{b}_i \right) \bm{b}_i \int_0^{2h^{-2}} L'(r) r^{\frac{q}{2}+1} (2-h^2r)^{\frac{q}{2}} dr\\
			&\stackrel{\text{(i)}}{=} \frac{h^4}{2} \cdot \bm{x}^T \nabla\nabla f(\bm{x}) \bm{x}\cdot \bm{x} \cdot \frac{\int_0^{\infty} L'(r) r^{\frac{q}{2}+2}(2-h^2r)^{\frac{q}{2}-1} dr}{\int_0^{\infty} L(r) r^{\frac{q}{2}-1}(2-h^2r)^{\frac{q}{2}-1} dr}\\
			&\quad + \frac{h^2}{2q} \left[\Delta f(\bm{x}) -\bm{x}^T \nabla\nabla f(\bm{x}) \bm{x} \right] \bm{x} \cdot \frac{\int_0^{\infty} L'(r) r^{\frac{q}{2}+1}(2-h^2r)^{\frac{q}{2}} dr}{\int_0^{\infty} L(r) r^{\frac{q}{2}-1}(2-h^2r)^{\frac{q}{2}-1} dr} \\
			&\quad + \frac{h^2}{q} \sum_{i=1}^q \left(\bm{x}^T \nabla\nabla f(\bm{x}) \bm{b}_i \right)  \bm{b}_i \cdot \frac{\int_0^{\infty} L'(r) r^{\frac{q}{2}+1}(2-h^2r)^{\frac{q}{2}} dr}{\int_0^{\infty} L(r) r^{\frac{q}{2}-1}(2-h^2r)^{\frac{q}{2}-1} dr}\\
			&\stackrel{\text{(ii)}}{=} \frac{h^2}{q} \left[\Delta f(\bm{x}) -\bm{x}^T \nabla\nabla f(\bm{x}) \bm{x} \right] \bm{x} \cdot \frac{\int_0^{\infty} L'(r) r^{\frac{q}{2}+1} dr}{\int_0^{\infty} L(r) r^{\frac{q}{2}-1} dr}\\
			&\quad + \frac{2h^2}{q} \sum_{i=1}^q \left(\bm{x}^T \nabla\nabla f(\bm{x}) \bm{b}_i \right) \bm{b}_i \cdot  \frac{\int_0^{\infty} L'(r) r^{\frac{q}{2}+1} dr}{\int_0^{\infty} L(r) r^{\frac{q}{2}-1} dr} + o(h^2),
			\end{align*}
			where we plug in the expression \eqref{asym_norm_const} of $c_{h,q}(L)$ in (i) and take $h\to 0$ with arguments in Lemma~\ref{integ_lemma} and Remark~\ref{integ_lemma_remark} in (ii).
			
			We argue that after plugging (IV)+$o(h^3)$ back into \eqref{grad_expected1}, it yields a $o(h^2)$ term.
			\begin{align*}
			&\text{Plug in (IV)}+o(h^3) \\
			&= c_{h,q}(L) h^q \int_0^{2h^{-2}} \int_{\Omega_{q-1}} \frac{1}{6} \left[\left(\sum_{i=1}^{q+1} (\alpha_{\bm{x},\bm{\xi}})_i \cdot \frac{\partial}{\partial x_i} \right)^3 f(\bm{x}) \right] \bm{x} L'(r) r^{\frac{q}{2}} (2-h^2r)^{\frac{q}{2}-1} \omega_{q-1}(d\bm{\xi}) dr\\
			& + c_{h,q}(L) h^{q-1} \int_0^{2h^{-2}} \int_{\Omega_{q-1}} \frac{1}{6} \left[\left(\sum_{i=1}^{q+1} (\alpha_{\bm{x},\bm{\xi}})_i \cdot \frac{\partial}{\partial x_i} \right)^3 f(\bm{x}) \right] \bm{B_x}\bm{\xi} \\
			& \hspace{60mm} \times  L'(r) r^{\frac{q-1}{2}} (2-h^2r)^{\frac{q-1}{2}} \omega_{q-1}(d\bm{\xi}) dr\\
			& + c_{h,q}(L) \cdot o(h^{q+1}) \int_0^{2h^{-2}} \int_{\Omega_{q-1}} \left(rh^2\bm{x} -h\sqrt{r(2-h^2r)} \bm{B_x}\bm{\xi} \right) L'(r) r^{\frac{q-1}{2}} (2-h^2r)^{\frac{q-1}{2}} \omega_{q-1}(d\bm{\xi}) dr\\
			&= c_{h,q}(L) h^q \int_0^{2h^{-2}} \int_{\Omega_{q-1}} \frac{1}{6} \left[\left(\sum_{i=1}^{q+1} (\alpha_{\bm{x},\bm{\xi}})_i \cdot \frac{\partial}{\partial x_i} \right)^3 f(\bm{x}) \right] \bm{x} L'(r) r^{\frac{q}{2}} (2-h^2r)^{\frac{q}{2}-1} \omega_{q-1}(d\bm{\xi}) dr\\
			& + c_{h,q}(L) h^{q-1} \int_0^{2h^{-2}} \int_{\Omega_{q-1}} \frac{1}{6} \left[\left(\sum_{i=1}^{q+1} (\alpha_{\bm{x},\bm{\xi}})_i \cdot \frac{\partial}{\partial x_i} \right)^3 f(\bm{x}) \right] \bm{B_x}\bm{\xi} \\
			&\hspace{60mm} \times L'(r) r^{\frac{q-1}{2}} (2-h^2r)^{\frac{q-1}{2}} \omega_{q-1}(d\bm{\xi}) dr\\
			& + c_{h,q}(L) \cdot o(h^{q+3}) \cdot \bar{\omega}_{q-1} \bm{x} \int_0^{2h^{-2}} L'(r) r^{\frac{q+1}{2}} (2-h^2r)^{\frac{q-1}{2}} dr
			\end{align*}
			by (c) of Lemma~\ref{integ_lemma} in the last equality. As $c_{h,q}(L) = h^q \lambda_{h,q}(L) =O(h^q)$ by \eqref{asym_norm_const} and (a) of Lemma~\ref{integ_lemma}, we know that the third integral is of the order $o(h^3)$. Since $\bm{B_x}\bm{\xi}= \sum_{i=1}^q \xi_i \bm{b}_i$ and $\alpha_{\bm{x},\bm{\xi}} = -rh^2\bm{x} + h\sqrt{r(2-h^2r)} \bm{B_x}\bm{\xi}$, we derive that
			\begin{align*}
			&\left(\sum_{i=1}^{q+1} (\alpha_{\bm{x},\bm{\xi}})_i \cdot \frac{\partial}{\partial x_i} \right)^3 f(\bm{x}) \\
			&= A_{f,1}\cdot h^3 r^{\frac{3}{2}} (2-h^2r)^{\frac{3}{2}} \sum_{i,j,k} \bm{b}_i\bm{b}_j\bm{b}_k \xi_i\xi_j \xi_k + A_{f,2} \cdot h^4 r^2 (2-h^2r) \sum_{i,j} \bm{b}_i\bm{b}_j \xi_i\xi_j \\
			&\quad + A_{f,3} \cdot h^5 r^{\frac{5}{2}} (2-h^2r)^{\frac{1}{2}} \sum_i \bm{b}_i \xi_i +A_{f,4} h^6r^3
			\end{align*}
			and
			\begin{align*}
			&\left[\left(\sum_{i=1}^{q+1} (\alpha_{\bm{x},\bm{\xi}})_i \cdot \frac{\partial}{\partial x_i} \right)^3 f(\bm{x}) \right] \bm{B_x}\bm{\xi}\\
			&= \tilde{A}_{f,1}\cdot h^3 r^{\frac{3}{2}} (2-h^2r)^{\frac{3}{2}} \sum_{i,j,k,\ell} \bm{b}_i\bm{b}_j\bm{b}_k \bm{b}_{\ell} \xi_i\xi_j \xi_k \xi_{\ell} + \tilde{A}_{f,2} \cdot h^4 r^2 (2-h^2r) \sum_{i,j,k} \bm{b}_i\bm{b}_j \bm{b}_k \xi_i\xi_j \xi_k \\
			&\quad + \tilde{A}_{f,3} \cdot h^5 r^{\frac{5}{2}} (2-h^2r)^{\frac{1}{2}} \sum_{i,j} \bm{b}_i \bm{b}_j\xi_i \xi_j + \tilde{A}_{f,4} h^6r^3 \sum_i \bm{b}_i \xi_i,
			\end{align*}
			where $A_{f,i}, i=1,...,4$ and $\tilde{A}_{f,i}, i=1,...,4$ are some ``constants'' that depends on the partial derivatives of $f(\bm{x})$. Thus, by (c) of Lemma~\ref{integ_lemma} (that is, any integration of a monomial of $\bm{\xi}$ with an odd degree on $\Omega_{q-1}$ will yield 0), we know that
			$$\int_{\Omega_{q-1}} \left[\left(\sum_{i=1}^{q+1} (\alpha_{\bm{x},\bm{\xi}})_i \cdot \frac{\partial}{\partial x_i} \right)^3 f(\bm{x}) \right] \omega_{q-1}(d\bm{\xi}) \asymp h^4r^2(2-h^2r) + o(h^4),$$
			$$\int_{\Omega_{q-1}} \left[\left(\sum_{i=1}^{q+1} (\alpha_{\bm{x},\bm{\xi}})_i \cdot \frac{\partial}{\partial x_i} \right)^3 f(\bm{x}) \right] \bm{B_x}\bm{\xi}\, \omega_{q-1}(d\bm{\xi}) \asymp h^3 r^{\frac{3}{2}}(2-h^2r)^{\frac{3}{2}} + h^5r\sqrt{2-h^2r} + o(h^3),$$
			where ``$\asymp$'' means an asymptotic equivalence. With condition (D2') and our arguments of (a) in Lemma~\ref{integ_lemma} and Remark~\ref{integ_lemma_remark}, we obtain that ``Plug in (IV)+$o(h^3)$'' yields a $o(h^2)$ term. Therefore,
			\begin{equation}
			\label{grad_expect_asym}
			\begin{split}
			\mathbb{E}\left[\nabla \tilde{f}_h(\bm{x}) \right] &= -\frac{q}{2} \cdot \bm{x} f(\bm{x}) + \left(\frac{q-2}{4}\right)\bm{x} f(\bm{x}) h^2 \cdot \frac{\int_0^{\infty} L(r) r^{\frac{q}{2}} dr}{\int_0^{\infty} L(r) r^{\frac{q}{2}-1} dr} \\
			& \quad + (I_{q+1} - \bm{x}\bm{x}^T) \nabla f(\bm{x}) + h^2\left(\frac{q+2}{2} \right) \bm{x}\bm{x}^T \nabla f(\bm{x}) \cdot \frac{\int_0^{\infty} L(r) r^{\frac{q}{2}} dr}{\int_0^{\infty} L(r) r^{\frac{q}{2}-1} dr} \\
			&\quad +\frac{h^2}{2} \left(I_{q+1} - \bm{x}\bm{x}^T \right) \nabla f(\bm{x}) \cdot \frac{\int_0^{\infty} L(r) r^{\frac{q}{2}} dr}{\int_0^{\infty} L(r) r^{\frac{q}{2}-1} dr}\\
			&\quad + \frac{h^2}{q} \left[\Delta f(\bm{x}) -\bm{x}^T \nabla\nabla f(\bm{x}) \bm{x} \right] \bm{x} \cdot \frac{\int_0^{\infty} L'(r) r^{\frac{q}{2}+1} dr}{\int_0^{\infty} L(r) r^{\frac{q}{2}-1} dr}\\
			&\quad + \frac{2h^2}{q} \sum_{i=1}^q \left(\bm{x}^T \nabla\nabla f(\bm{x}) \bm{b}_i \right) \bm{b}_i \cdot  \frac{\int_0^{\infty} L'(r) r^{\frac{q}{2}+1} dr}{\int_0^{\infty} L(r) r^{\frac{q}{2}-1} dr} +O(h^2)+ o(h^2),
			\end{split}
			\end{equation}
			which in turn shows that
			\begin{align*}
			\mathbb{E}\left[\grad \tilde{f}_h(\bm{x})\right] &= \mathbb{E}\left[\Tang\left(\nabla \tilde{f}_h(\bm{x}) \right) \right] = (I_{q+1} -\bm{x}\bm{x}^T) \mathbb{E}\left[\nabla \tilde{f}_h(\bm{x}) \right]\\ 
			&\hspace{-20mm} = \left(I_{q+1} - \bm{x}\bm{x}^T \right) \nabla f(\bm{x}) + \frac{h^2}{2} \left(I_{q+1} - \bm{x}\bm{x}^T \right) \nabla f(\bm{x}) \cdot \frac{\int_0^{\infty} L(r) r^{\frac{q}{2}} dr}{\int_0^{\infty} L(r) r^{\frac{q}{2}-1} dr} \\
			&\hspace{-20mm} \quad + \frac{2h^2}{q} \sum_{i=1}^q \left(\bm{x}^T \nabla\nabla f(\bm{x}) \bm{b}_i \right) \bm{b}_i \cdot  \frac{\int_0^{\infty} L'(r) r^{\frac{q}{2}+1} dr}{\int_0^{\infty} L(r) r^{\frac{q}{2}-1} dr} +O(h^2)+ o(h^2)\\
			&\hspace{-20mm} = \left(I_{q+1} - \bm{x}\bm{x}^T \right) \nabla f(\bm{x}) + O(h^2)
			\end{align*}
			as $h\to 0$. {\bf Result 1} thus follows and the Riemannian gradient estimator is unbiased. 
			
			\noindent $\bullet$ {\bf Result 2} The covariance matrix of $\nabla \tilde{f}_h(\bm{x})$ has the following asymptotic rate as $h\to 0$ and $nh^{q+2} \to \infty$:
			$$\text{Cov}\left[\nabla \tilde{f}_h(\bm{x}) \right] = \frac{1}{nh^{q+2}} \cdot R(f,L) + o\left(\frac{1}{nh^{q+2}} \right),$$
			where $R(f,L) = \frac{f(\bm{x}) \int_0^{\infty} L'(r)^2 r^{\frac{q}{2}} dr}{2^{\frac{q}{2}-2}q \cdot \bar{\omega}_{q-1} \left(\int_0^{\infty} L(r) r^{\frac{q}{2}-1} dr \right)^2} \cdot (I_{q+1} -\bm{x}\bm{x}^T)$.
			
			\noindent \emph{Derivation of Result 2}. By \eqref{grad_expect_asym}, the covariance matrix of $\nabla \tilde{f}_h(\bm{x})$ can be calculated as
			\begin{align*}
			&\text{Cov}\left[\nabla \tilde{f}_h(\bm{x}) \right] \\
			&= \frac{c_{h,q}(L)^2}{nh^4} \cdot \text{Cov}\left[(\bm{x}-\bm{X}_1) L'\left(\frac{1-\bm{x}^T \bm{X}_1}{h^2} \right)\right]\\
			&= \frac{c_{h,q}(L)^2}{nh^4} \cdot \mathbb{E}\left[(\bm{x}-\bm{X}_1) (\bm{x}-\bm{X}_1)^T L'\left(\frac{1-\bm{x}^T \bm{X}_1}{h^2} \right)^2 \right] - \frac{1}{n}\cdot \mathbb{E}\left[\nabla \tilde{f}_h(\bm{x}) \right] \cdot \mathbb{E}\left[\nabla \tilde{f}_h(\bm{x}) \right]^T\\
			&= \frac{c_{h,q}(L)^2}{nh^4} \int_{\Omega_q} (\bm{x}-\bm{y}) (\bm{x}-\bm{y})^T L'\left(\frac{1-\bm{x}^T \bm{y}}{h^2} \right)^2 f(\bm{y}) \, \omega_q(d\bm{y}) + O\left(\frac{1}{n} \right) \\
			&= \frac{c_{h,q}(L)^2}{n} h^{q-4} \int_0^{2h^{-2}} \int_{\Omega_{q-1}} \alpha_{\bm{x},\bm{\xi}} \alpha_{\bm{x},\bm{\xi}}^T L'(r)^2 f(\bm{x} + \alpha_{\bm{x},\bm{\xi}}) r^{\frac{q}{2}-1} (2-h^2 r)^{\frac{q}{2}-1} \omega_{q-1}(d\bm{\xi}) dr\\
			&\quad + O\left(\frac{1}{n} \right),
			\end{align*}
			where $\alpha_{\bm{x},\bm{\xi}} = -rh^2\bm{x}+h\sqrt{r(2-h^2r)} \bm{B_x}\bm{\xi}$. By condition (D1), the first-order Taylor's expansion of $f$ at $\bm{x}\in \Omega_q$ is
			$$f(\bm{x}+ \alpha_{\bm{x},\bm{\xi}}) = f(\bm{x}) + O(||\alpha_{\bm{x},\bm{\xi}}||_2) = f(\bm{x}) + O(h).$$
			Thus,
			\begin{align*}
			&\text{Cov}\left[\nabla \tilde{f}_h(\bm{x}) \right] \\
			&= \frac{c_{h,q}(L)^2}{n} h^{q-4} f(\bm{x}) \int_0^{2h^{-2}} \int_{\Omega_{q-1}} \Big[r^2h^4 \bm{x}\bm{x}^T -rh^3 \sqrt{r(2-h^2r)} \bm{x} (\bm{B_x} \bm{\xi})^T \\
			& \quad - rh^3 \sqrt{r(2-h^2r)} (\bm{B_x} \bm{\xi}) \bm{x}^T + h^2 r(2-h^2r) \bm{B_x}\bm{\xi} (\bm{B_x}\bm{\xi})^T \Big] L'(r)^2 r^{\frac{q}{2}-1} (2-h^2r)^{\frac{q}{2}-1} \omega_{q-1}(d\bm{\xi}) dr\\
			&\quad + o\left(\frac{1}{nh^{q+2}} \right)\\
			& \stackrel{\text{(i)}}{=} \frac{c_{h,q}(L)^2}{n} h^q f(\bm{x}) \bm{x}\bm{x}^T \bar{\omega}_{q-1} \int_0^{2h^{-2}} L'(r)^2 r^{\frac{q}{2}+1} (2-h^2r)^{\frac{q}{2}-1} dr\\
			&\quad + \frac{c_{h,q}(L)^2}{n} h^{q-2} f(\bm{x}) \int_0^{2h^{-2}} \int_{\Omega_{q-1}} \left(\sum_{i=1}^q \xi_i\bm{b}_i \right) \left(\sum_{i=1}^q \xi_i\bm{b}_i^T \right) L'(r)^2 r^{\frac{q}{2}} (2-h^2r)^{\frac{q}{2}} \omega_{q-1}(d\bm{\xi}) dr\\
			& \quad + o\left(\frac{1}{nh^{q+2}} \right) \\
			& \stackrel{\text{(ii)}}{=} \frac{c_{h,q}(L)^2}{n} h^q f(\bm{x}) \bm{x}\bm{x}^T \bar{\omega}_{q-1} \int_0^{2h^{-2}} L'(r)^2 r^{\frac{q}{2}+1} (2-h^2r)^{\frac{q}{2}-1} dr \\
			& \quad + \frac{c_{h,q}(L)^2 \bar{\omega}_{q-1}}{nq} \cdot h^{q-2} f(\bm{x}) \int_0^{2h^{-2}} \left(\sum_{i=1}^q \bm{b}_i\bm{b}_i^T \right) L'(r)^2 r^{\frac{q}{2}} (2-h^2r)^{\frac{q}{2}} dr + o\left(\frac{1}{nh^{q+2}} \right) \\
			&\stackrel{\text{(iii)}}{=} \frac{c_{h,q}(L)}{n} f(\bm{x}) \bm{x}\bm{x}^T \cdot \frac{\int_0^{2h^{-2}} L'(r)^2 r^{\frac{q}{2}+1} (2-h^2r)^{\frac{q}{2}-1} dr}{\int_0^{2h^{-2}} L(r) r^{\frac{q}{2}-1} (2-h^2r)^{\frac{q}{2}-1} dr}\\
			& \quad + \frac{c_{h,q}(L)}{nq} \cdot h^{q-2} f(\bm{x}) \left(I_{q+1} -\bm{x}\bm{x}^T \right)\frac{\int_0^{2h^{-2}}  L'(r)^2 r^{\frac{q}{2}} (2-h^2r)^{\frac{q}{2}} dr}{\int_0^{2h^{-2}}  L(r) r^{\frac{q}{2}-1} (2-h^2r)^{\frac{q}{2}-1} dr} + o\left(\frac{1}{nh^{q+2}} \right) \\
			&\stackrel{\text{(iv)}}{=} \frac{1}{nh^{q+2}} \cdot R(f,L) + o\left(\frac{1}{nh^{q+2}} \right),
			\end{align*}
			where $R(f,L) = \frac{f(\bm{x}) \int_0^{\infty} L'(r)^2 r^{\frac{q}{2}} dr}{2^{\frac{q}{2}-2}q \cdot \bar{\omega}_{q-1} \left(\int_0^{\infty} L(r) r^{\frac{q}{2}-1} dr \right)^2} \cdot (I_{q+1} -\bm{x}\bm{x}^T)$ is a matrix whose columns lie in the tangent space of $\Omega_q$ at $\bm{x}$. During the derivation, we use (c) of Lemma~\ref{integ_lemma} in (i) and (ii), plug in the expression \eqref{asym_norm_const} of $c_{h,q}(L)$ in (iii), and take $h\to 0$ with arguments in (a) of Lemma~\ref{integ_lemma} and Remark~\ref{integ_lemma_remark} in (iv). {\bf Result 2} thus follows.
			
		    By the central limit theorem,
			\begin{align*}
			\nabla \tilde{f}_h(\bm{x}) - \mathbb{E}\left[ \nabla \tilde{f}_h(\bm{x}) \right] &= \text{Cov}\left[\nabla \tilde{f}_h(\bm{x}) \right]^{\frac{1}{2}} \cdot \text{Cov}\left[\nabla \tilde{f}_h(\bm{x}) \right]^{-\frac{1}{2}} \left\{ \nabla \tilde{f}_h(\bm{x}) - \mathbb{E}\left[ \nabla \tilde{f}_h(\bm{x}) \right] \right\}\\
			&= \left[\frac{1 }{nh^{q+2}} \cdot R(f,L) + o\left(\frac{1}{nh^{q+2}} \right) \right]^{\frac{1}{2}} \cdot \bm{Z}_n(\bm{x})\\
			&= O_P\left(\sqrt{\frac{1}{nh^{q+2}}} \right),
			\end{align*}
			where $\bm{Z}_n(\bm{x}) \stackrel{d}{\to} N_{q+1}(\bm{0},I_{q+1})$.\\ 
			The asymptotic rate for $\grad \tilde{f}_h(\bm{x}) - \mathbb{E}\left[\grad \tilde{f}_h(\bm{x}) \right] = \Tang\left(\nabla \tilde{f}_h(\bm{x}) \right) - \mathbb{E}\left[\Tang \left(\nabla \tilde{f}_h(\bm{x}) \right) \right]$ remains unchanged, since the dominating constant $R(f,L)$ are within the tangent space of $\Omega_q$ at $\bm{x}$.\\
			In a nutshell, we conclude with bias ({\bf Result 1}) and variance ({\bf Result 2}) estimation that
			$$\grad \hat{f}_h(\bm{x}) - \grad f(\bm{x})= \Tang\left(\nabla \hat{f}_h(\bm{x}) \right) - \Tang\left(\nabla f(\bm{x}) \right) = O(h^2) + O_P\left(\sqrt{\frac{1}{nh^{q+2}}} \right)$$
			for any fixed $\bm{x}\in \Omega_q$, as $h\to 0$ and $nh^{q+2} \to \infty$.\\
			
		\noindent {\bf Part B: Pointwise convergence rate of the Riemannian Hessian estimator $\mathcal{H} \hat{f}_h(\bm x)$}. As shown in Lemma~\ref{lem:Hessian}, $\mathcal{H} \hat{f}_h(\bm{x}) = \mathcal{H} \tilde{f}_h(\bm{x})$ for any $\bm{x}\in \Omega_q$ and we can establish the pointwise convergence rate using either Riemannian Hessian estimator. Here, we stick to the Riemannian Hessian estimator $\mathcal{H} \hat{f}_h(\bm x)$ in \eqref{Hess_KDE}. 
		
		\noindent $\bullet$ {\bf Result 3}. The expectation of the Riemannian Hessian estimator, $\mathbb{E}\left[\mathcal{H} \hat{f}_h(\bm{x}) \right]$, has the following asymptotic behavior as $h\to 0$:
		\begin{align*}
		\mathbb{E}\left[\mathcal{H} \hat{f}_h(\bm{x}) \right] &= \left(I_{q+1} -\bm{x}\bm{x}^T\right) \mathbb{E}\left[\mathcal{A} \hat{f}_h(\bm{x}) \right] \left(I_{q+1} -\bm{x}\bm{x}^T\right)\\
		&= \left(I_{q+1} -\bm{x}\bm{x}^T\right) \left[\nabla\nabla f(\bm{x}) -\bm{x}^T \nabla f(\bm{x}) \right] \left(I_{q+1} -\bm{x}\bm{x}^T\right) + O(h^2),
		\end{align*}
		where $\mathcal{A} \hat{f}_h(\bm{x})=\frac{c_{h,q}(L)}{nh^4} \sum\limits_{i=1}^n \bm{X}_i\bm{X}_i^T L''\left(\frac{1-\bm{x}^T\bm{X}_i}{h^2} \right) + \frac{c_{h,q}(L)}{nh^2} \sum\limits_{i=1}^n \bm{x}^T \bm{X}_i I_{q+1} \cdot L'\left(\frac{1-\bm{x}^T\bm{X}_i}{h^2} \right)$.
		
		\noindent\emph{Derivation of Result 3}. We first compute the expectation of $\mathcal{A} \hat{f}_h(\bm{x})$ and apply the left and right multiplications of $\left(I_{q+1} -\bm{x}\bm{x}^T\right)$ to simplify our calculation. Notice that
		\begin{align*}
			&\mathbb{E}\left[\mathcal{A} \hat{f}_h(\bm{x})\right] \\
			&= \frac{c_{h,q}(L)}{h^4}\int_{\Omega_q} \bm{y}\bm{y}^T L''\left(\frac{1-\bm{x}^T\bm{y}}{h^2}\right) f(\bm{y})\, \omega_q(d\bm{y}) \\
			&\quad + \frac{c_{h,q}(L)}{h^2}\int_{\Omega_q} \bm{x}^T\bm{y} \cdot I_{q+1} \cdot L'\left(\frac{1-\bm{x}^T\bm{y}}{h^2}\right) f(\bm{y})\, \omega_q(d\bm{y})\\
			&= \frac{c_{h,q}(L)}{h^4}\int_{-1}^1 \int_{\Omega_{q-1}} \left(t\bm{x} + \sqrt{1-t^2} \bm{B_x}\bm{\xi} \right) \left(t\bm{x} + \sqrt{1-t^2} \bm{B_x}\bm{\xi} \right)^T \\ 
			& \hspace{40mm} \times L''\left(\frac{1-t}{h^2} \right) \cdot f\left(t\bm{x} + \sqrt{1-t^2} \bm{B_x}\bm{\xi} \right) (1-t^2)^{\frac{q}{2}-1} \omega_{q-1}(d\bm{\xi})dt\\
			& \quad + \frac{c_{h,q}(L)}{h^2}\int_{-1}^1 \int_{\Omega_{q-1}} tI_{q+1} \cdot L'\left(\frac{1-t}{h^2} \right) \cdot f\left(t\bm{x} + \sqrt{1-t^2} \bm{B_x}\bm{\xi} \right) (1-t^2)^{\frac{q}{2}-1} \omega_{q-1}(d\bm{\xi})dt\\
			&= c_{h,q}(L) h^{q-4} \int_0^{2h^{-2}} \int_{\Omega_{q-1}} \left(\bm{x} +\alpha_{\bm{x},\bm{\xi}} \right)\left(\bm{x} +\alpha_{\bm{x},\bm{\xi}} \right)^T L''(r)\\
			&\hspace{60mm} \times f\left(\bm{x} +\alpha_{\bm{x},\bm{\xi}} \right) r^{\frac{q}{2}-1} (2-h^2r)^{\frac{q}{2}-1} \omega_{q-1}(d\bm{\xi})dr\\
			&\quad + c_{h,q}(L) h^{q-2} \int_0^{2h^{-2}} \int_{\Omega_{q-1}} (1-h^2r) I_{q+1} \cdot L'(r) f\left(\bm{x} +\alpha_{\bm{x},\bm{\xi}} \right) r^{\frac{q}{2}-1} (2-h^2r)^{\frac{q}{2}-1} \omega_{q-1}(d\bm{\xi})dr
			\end{align*}
			by $r=\frac{1-t}{h^2}$ and $\alpha_{\bm{x},\bm{\xi}} = -rh^2\bm{x} +h\sqrt{r(2-h^2r)} \bm{B_x} \bm{\xi}$. Since
			\begin{align*}
			\left(\bm{x} +\alpha_{\bm{x},\bm{\xi}} \right)\left(\bm{x} +\alpha_{\bm{x},\bm{\xi}} \right)^T &= (1-rh^2)^2 \bm{x}\bm{x}^T + h(1-rh^2) \sqrt{r(2-h^2 r)} \left[\bm{x} (\bm{B_x}\bm{\xi})^T + (\bm{B_x}\bm{\xi}) \bm{x}^T \right] \\
			&\quad + h^2r(2-h^2r)\cdot (\bm{B_x}\bm{\xi})(\bm{B_x}\bm{\xi})^T,
			\end{align*}
			the preceding calculation proceeds as
			\begin{align}
			\label{Hess_Expect1}
			\begin{split}
			&\mathbb{E}\left[{\mathcal{A}}\hat f_h (\bm x)\right] \\
			&= c_{h,q}(L) h^{q-4} \int_0^{2h^{-2}} \int_{\Omega_{q-1}} \bm{x}\bm{x}^T L''(r) \cdot f\left(\bm{x} +\alpha_{\bm{x},\bm{\xi}} \right) r^{\frac{q}{2}-1} (1-rh^2)^2 (2-h^2r)^{\frac{q}{2}-1} \omega_{q-1}(d\bm{\xi})dr\\
			&\quad + c_{h,q}(L) h^{q-3} \int_0^{2h^{-2}} \int_{\Omega_{q-1}} \left[\bm{x} (\bm{B_x}\bm{\xi})^T + (\bm{B_x}\bm{\xi}) \bm{x}^T \right] L''(r) \\
			&\hspace{50mm} \times f\left(\bm{x} +\alpha_{\bm{x},\bm{\xi}} \right) r^{\frac{q-1}{2}} (1-rh^2) (2-h^2r)^{\frac{q-1}{2}} \omega_{q-1}(d\bm{\xi})dr\\
			&\quad + c_{h,q}(L) h^{q-2} \int_0^{2h^{-2}} \int_{\Omega_{q-1}} (\bm{B_x}\bm{\xi})(\bm{B_x}\bm{\xi})^T L''(r) \cdot f\left(\bm{x} +\alpha_{\bm{x},\bm{\xi}} \right) r^{\frac{q}{2}} (2-h^2r)^{\frac{q}{2}}\, \omega_{q-1}(d\bm{\xi})dr\\
			&\quad + c_{h,q}(L) h^{q-2} \int_0^{2h^{-2}} \int_{\Omega_{q-1}} (1-h^2r) I_{q+1} \cdot L'(r) f\left(\bm{x} +\alpha_{\bm{x},\bm{\xi}} \right) r^{\frac{q}{2}-1} (2-h^2r)^{\frac{q}{2}-1} \omega_{q-1}(d\bm{\xi})dr\\
			&\equiv \text{(I) + (II) +(III) + (IV)}.
			\end{split}
			\end{align}
			The above terms (I) and (II), after we apply the congruence operation 
			$$\left(I_{q+1} -\bm{x}\bm{x}^T \right) \mathbb{E}\left[{\mathcal{A}}\hat f_h (\bm x)\right] \left(I_{q+1} -\bm{x}\bm{x}^T \right),$$ 
			yield zero. Hence, we will not continue to compute them. By condition (D1), the Taylor's expansion of $f$ at $\bm{x}$ is
			\begin{align*}
			&f(\bm{x}+\alpha_{\bm{x},\bm{\xi}}) \\
			&= f(\bm{x}) + \alpha_{\bm{x},\bm{\xi}}^T \nabla f(\bm{x}) + \frac{1}{2} \alpha_{\bm{x},\bm{\xi}}^T \nabla\nabla f(\bm{x}) \alpha_{\bm{x},\bm{\xi}} + \frac{1}{6} \left(\sum_{i=1}^{q+1} (\alpha_{\bm{x},\bm{\xi}})_i \cdot \frac{\partial}{\partial x_i} \right)^3 f(\bm{x}) + O\left(||\alpha_{\bm{x},\bm{\xi}}||_2^4 \right)\\
			&= f(\bm{x}) + \alpha_{\bm{x},\bm{\xi}}^T \nabla f(\bm{x}) + \frac{1}{2} \alpha_{\bm{x},\bm{\xi}}^T \nabla\nabla f(\bm{x}) \alpha_{\bm{x},\bm{\xi}} + \frac{1}{6} \left(\sum_{i=1}^{q+1} (\alpha_{\bm{x},\bm{\xi}})_i \cdot \frac{\partial}{\partial x_i} \right)^3 f(\bm{x}) + O(h^4),
			\end{align*}
			where $\norm{\alpha_{\bm{x},\bm{\xi}}}_2^2 = r^2h^4 + h^2r(2-h^2r) = 2rh^2$ by the orthogonality of $\bm{x}$ and columns of $\bm{B_x}$, and $(\alpha_{\bm{x},\bm{\xi}})_i$ stands for the $i^{th}$ entry of the vector $\alpha_{\bm{x},\bm{\xi}}$. Note that plugging $O(h^4)$ into (III) or (IV) in \eqref{Hess_Expect1} both leads to a $O(h^2)$ after integration, since with condition (D2') and our arguments in (a) of Lemma~\ref{integ_lemma} and Remark~\ref{integ_lemma_remark},
			$$O(h^4)\cdot c_{h,q}(L)h^{q-2} \int_0^{2h^{-2}} \int_{\Omega_{q-1}} \phi(r,\bm{\xi}) L'(r) \,\omega_{q-1}(d\bm{\xi}) dr \asymp O(h^2),$$
			where $\phi(r,\bm{\xi})$ is a square integrable function of $(r,\bm{\xi})$ and ``$\asymp$'' stands for the asymptotic equivalence. It shows that carrying out the Taylor's expansion of $f$ at $\bm{x}$ to the third order is sufficient in our context.\\ 
			More importantly, plugging the term $\frac{1}{6} \left(\sum_{i=1}^{q+1} (\alpha_{\bm{x},\bm{\xi}})_i \cdot \frac{\partial}{\partial x_i} \right)^3 f(\bm{x})$ into (II) and (III) also gives rise to a $O(h^2)$ term. Because $\bm{B_x}\bm{\xi} =\sum_{i=1}^q \xi_i\bm{b}_i$ and $\alpha_{\bm{x},\bm{\xi}} = -rh^2\bm{x}+ h\sqrt{r(2-h^2r)} \bm{B_x}\bm{\xi}$,
			\begin{align}
			\label{O_h2_arg1}
			\begin{split}
			&\int_{\Omega_{q-1}} \left(\sum_{i=1}^q \bm{b}_i\xi_i \right) \left(\sum_{i=1}^q \bm{b}_i\xi_i \right)^T \left[\frac{1}{6} \left(\sum_{i=1}^{q+1} (\alpha_{\bm{x},\bm{\xi}})_i \cdot \frac{\partial}{\partial x_i} \right)^3 f(\bm{x}) \right] \omega_{q-1}(d\bm{\xi}) \\
			&= \int_{\Omega_{q-1}} \left[P_r(\bm{\xi},2) h^6 + P_r(\bm{\xi},3) h^5 + P_r(\bm{\xi},4) h^4 + P_r(\bm{\xi},5) h^3 \right] \omega_{q-1}(d\bm{\xi})\\
			&= O(h^4) + \underbrace{\int_{\Omega_{q-1}} h^3P_r(\bm{\xi},5)\, \omega_{q-1}(d\bm{\xi})}_{=0},
			\end{split}
			\end{align}
			\begin{align}
			\label{O_h2_arg2}
			\begin{split}
			&\int_{\Omega_{q-1}} \left[\frac{1}{6} \left(\sum_{i=1}^{q+1} (\alpha_{\bm{x},\bm{\xi}})_i \cdot \frac{\partial}{\partial x_i} \right)^3 f(\bm{x}) \right] \omega_{q-1}(d\bm{\xi}) \\
			&= \int_{\Omega_{q-1}} \left[P_r(\bm{\xi},0)h^6 + P_r(\bm{\xi},1)h^5 + P_r(\bm{\xi},2)h^4 + P_r(\bm{\xi},3)h^3 \right] \omega_{q-1}(d\bm{\xi})\\
			&= O(h^4) + \underbrace{\int_{\Omega_{q-1}} h^3P_r(\bm{\xi},3)\, \omega_{q-1}(d\bm{\xi})}_{=0},
			\end{split}
			\end{align}
			where $P_r(\bm{\xi}, n)$ is a polynomial of elements of $\bm{\xi}=(\xi_1,...,\xi_q)$ with only degree $n$ terms, whose coefficients may involve the variable $r$. The integral 
			$\int_{\Omega_{q-1}} h^3P_r(\bm{\xi},5)\, \omega_{q-1}(d\bm{\xi})= \int_{\Omega_{q-1}} h^3P_r(\bm{\xi},3)\, \omega_{q-1}(d\bm{\xi}) = 0$
			is due to (c) of Lemma~\ref{integ_lemma} and the fact that the integrand is a linear combination of degree 5 or 3 monomials of elements of $\bm{\xi}$. With condition (D2') and our arguments in (a) of Lemma~\ref{integ_lemma} and Remark~\ref{integ_lemma_remark}, the final $O(h^4)$ terms in \eqref{O_h2_arg1} and \eqref{O_h2_arg2} both yield $O(h^2)$ terms after being plugged into (III) and (IV).\\
			We now plug the Taylor's expansion of $f(\bm{x}+\alpha_{\bm{x},\bm{\xi}})$ back into (III) and obtain that
			\begin{align}
			\label{II_integral}
			\begin{split}
			&\text{Plug in (III) in \eqref{Hess_Expect1}} \\
			&= c_{h,q}(L) h^{q-2} \int_0^{2h^{-2}} \int_{\Omega_{q-1}} \left(\sum_{i=1}^q \bm{b}_i\xi_i \right) \left(\sum_{i=1}^q \bm{b}_i\xi_i \right)^T L''(r) f(\bm{x}) r^{\frac{q}{2}} (2-h^2r)^{\frac{q}{2}} \, \omega_{q-1}(d\bm{\xi}) dr\\
			&\quad + c_{h,q}(L) h^{q-2} \int_0^{2h^{-2}} \int_{\Omega_{q-1}} \left(\sum_{i=1}^q \bm{b}_i\xi_i \right) \left(\sum_{i=1}^q \bm{b}_i\xi_i \right)^T L''(r) \nabla f(\bm{x})^T \alpha_{\bm{x},\bm{\xi}}\\
			&\hspace{80mm} \times r^{\frac{q}{2}} (2-h^2r)^{\frac{q}{2}} \, \omega_{q-1}(d\bm{\xi}) dr\\
			&\quad + c_{h,q}(L) h^{q-2} \int_0^{2h^{-2}}\int_{\Omega_{q-1}} \left(\sum_{i=1}^q \bm{b}_i\xi_i \right) \left(\sum_{i=1}^q \bm{b}_i\xi_i \right)^T L''(r)\\
			&\hspace{40mm} \times \frac{1}{2} \alpha_{\bm{x},\bm{\xi}}^T \nabla\nabla f(\bm{x}) \alpha_{\bm{x},\bm{\xi}} \cdot r^{\frac{q}{2}} (2-h^2r)^{\frac{q}{2}} \, \omega_{q-1}(d\bm{\xi}) dr +O(h^2)\\
			&= c_{h,q}(L) h^{q-2} \int_0^{2h^{-2}} \frac{\bar{\omega}_{q-1}}{q} \left(I_{q+1} -\bm{x}\bm{x}^T \right) L''(r) f(\bm{x}) r^{\frac{q}{2}} (2-h^2r)^{\frac{q}{2}} \, dr\\
			&\quad - c_{h,q}(L) h^q \int_0^{2h^{-2}} \frac{\bar{\omega}_{q-1}}{q} \left(I_{q+1} -\bm{x}\bm{x}^T \right) L''(r) \nabla f(\bm{x})^T \bm{x} r^{\frac{q}{2}+1} (2-h^2r)^{\frac{q}{2}} \, dr\\
			&\quad + c_{h,q}(L) h^{q+2} \int_0^{2h^{-2}} \int_{\Omega_{q-1}} \left(\sum_{i=1}^q \bm{b}_i\xi_i \right) \left(\sum_{i=1}^q \bm{b}_i\xi_i \right)^T L''(r) \\
			&\hspace{40mm} \times \frac{1}{2} \bm{x}^T \nabla\nabla f(\bm{x}) \bm{x} \cdot r^{\frac{q}{2}+2} (2-h^2r)^{\frac{q}{2}}\, \omega_{q-1}(d\bm{\xi}) dr\\
			& \quad - c_{h,q}(L) h^{q+1} \int_0^{2h^{-2}} \int_{\Omega_{q-1}} \left(\sum_{i=1}^q \bm{b}_i\xi_i \right) \left(\sum_{i=1}^q \bm{b}_i\xi_i \right)^T L''(r) \\
			&\hspace{40mm} \times  \bm{x}^T \nabla\nabla f(\bm{x}) (\bm{B_x}\bm{\xi}) \cdot r^{\frac{q+3}{2}} (2-h^2r)^{\frac{q+1}{2}}\, \omega_{q-1}(d\bm{\xi}) dr\\
			&\quad + c_{h,q}(L) h^q \int_0^{2h^{-2}} \int_{\Omega_{q-1}} \left(\sum_{i=1}^q \bm{b}_i\xi_i \right) \left(\sum_{i=1}^q \bm{b}_i\xi_i \right)^T L''(r) \\
			&\hspace{40mm} \times \frac{1}{2} (\bm{B_x}\bm{\xi})^T \nabla\nabla f(\bm{x}) (\bm{B_x}\bm{\xi}) \cdot r^{\frac{q}{2}+1} (2-h^2r)^{\frac{q}{2}+1}\, \omega_{q-1}(d\bm{\xi}) dr + O(h^2)\\
			&= c_{h,q}(L) h^{q-2} \int_0^{2h^{-2}} \frac{\bar{\omega}_{q-1}}{q} \left(I_{q+1} -\bm{x}\bm{x}^T \right) L''(r) f(\bm{x})\cdot r^{\frac{q}{2}} (2-h^2r)^{\frac{q}{2}} \, dr\\
			&\quad - c_{h,q}(L) h^q \int_0^{2h^{-2}} \frac{\bar{\omega}_{q-1}}{q} \left(I_{q+1} -\bm{x}\bm{x}^T \right) L''(r) \nabla f(\bm{x})^T \bm{x} \cdot r^{\frac{q}{2}+1} (2-h^2r)^{\frac{q}{2}} \, dr \quad - \quad 0\\
			&\quad + c_{h,q}(L) h^q \int_0^{2h^{-2}} \int_{\Omega_{q-1}} \left(\sum_{i=1}^q \bm{b}_i\xi_i \right) \left(\sum_{i=1}^q \bm{b}_i\xi_i \right)^T L''(r)\\
			&\hspace{30mm} \times \frac{1}{2} (\bm{B_x}\bm{\xi})^T \nabla\nabla f(\bm{x}) (\bm{B_x}\bm{\xi}) \cdot r^{\frac{q}{2}+1} (2-h^2r)^{\frac{q}{2}+1}\, \omega_{q-1}(d\bm{\xi}) dr + O(h^2),
			\end{split}
			\end{align}
			where we apply (c) of Lemma \ref{integ_lemma} and the fact that $\sum_{i=1}^q \bm{b}_i\bm{b}_i^T =I_{q+1}- \bm{x}\bm{x}^T$. We also absorb the third integral in the second equality into $O(h^2)$ to obtain the third equality, given condition (D2') and our arguments in (a) of Lemma~\ref{integ_lemma}. The ``0'' term in the third equality is due to the fact that 
			\begin{align*}
			\int_{\Omega_{q-1}} \left(\sum_{i=1}^q \bm{b}_i\xi_i \right) \left(\sum_{i=1}^q \bm{b}_i\xi_i \right)^T \bm{x}^T \nabla\nabla f(\bm{x}) (\bm{B_x}\bm{\xi}) \, \omega_{q-1}(d\bm{\xi}) = \int_{\Omega_{q-1}} P_r(\bm{\xi},3) \, \omega_{q-1}(d\bm{\xi}) = 0
			\end{align*}
			by (c) of Lemma~\ref{integ_lemma}, where the notation $P_r(\bm{\xi},3)$ is defined in \eqref{O_h2_arg2}.
			Now, we consider the inner integral inside the last integration.
			\begin{align*}
			&\int_{\Omega_{q-1}} \left(\sum_{i=1}^q \bm{b}_i\xi_i \right) \left(\sum_{i=1}^q \bm{b}_i\xi_i \right)^T \cdot \left(\sum_{i=1}^q \bm{b}_i\xi_i \right)^T \nabla\nabla f(\bm{x}) \left(\sum_{i=1}^q \bm{b}_i\xi_i \right) \omega_{q-1}(d\bm{\xi})\\
			&= \int_{\Omega_{q-1}} \left(\sum_{i=1}^q \bm{b}_i\bm{b}_i^T \xi_i^2 + \sum_{i\neq j} \bm{b}_i\bm{b}_j^T \xi_i \xi_j \right) \left(\sum_{i=1}^q \bm{b}_i^T \nabla\nabla f(\bm{x}) \bm{b}_i \xi_i^2 + \sum_{i\neq j} \bm{b}_i^T \nabla\nabla f(\bm{x}) \bm{b}_j \xi_i\xi_j \right) \omega_{q-1}(d\bm{\xi})\\
			&\stackrel{(*)}{=} \int_{\Omega_{q-1}} \Bigg[\left(\sum_{i=1}^q \bm{b}_i\bm{b}_i^T \xi_i^2 \right) \left(\sum_{i=1}^q \bm{b}_i^T \nabla\nabla f(\bm{x}) \bm{b}_i \xi_i^2 \right) \\
			&\quad + \left(\sum_{i\neq j} \bm{b}_i\bm{b}_j^T \xi_i \xi_j \right) \left(\sum_{i\neq j} \bm{b}_i^T \nabla\nabla f(\bm{x}) \bm{b}_j \xi_i\xi_j \right)\Bigg] \omega_{q-1}(d\bm{\xi})\\
			& = \int_{\Omega_{q-1}} \Bigg[\sum_{i=1}^q \bm{b}_i\bm{b}_i^T \cdot \bm{b}_i^T \nabla\nabla f(\bm{x})\bm{b}_i \xi_i^4 + \sum_{i\neq j} \bm{b}_i\bm{b}_i^T\cdot \bm{b}_j^T \nabla\nabla f(\bm{x})\bm{b}_j \xi_i^2\xi_j^2 \\
			&\quad + 2\sum_{i\neq j} \bm{b}_i\bm{b}_j^T \cdot \bm{b}_i^T\nabla\nabla f(\bm{x}) \bm{b}_j \xi_i^2\xi_j^2 \Bigg] \omega_{q-1}(d\bm{\xi}),
			\end{align*}
			where the cross product terms vanish after integration by Lemma \ref{integ_lemma} and the factor 2 in front of the last summation emerges because any fixed $(i,j)$ term ($i\neq j$) in the first factor of the second product in equality $(*)$ can be matched up with both $(i,j)$ and $(j,i)$ terms in the second factor to yield a summand. Using Lemma \ref{integ_lemma} and the facts that $\sum_{i=1}^q \bm{b}_i\bm{b}_i^T=I_{q+1}-\bm{x}\bm{x}^T$ and 
			$$\sum_{i=1}^q \bm{b}_i^T\nabla\nabla f(\bm{x}) \bm{b}_i= \text{tr}\left[\nabla\nabla f(\bm{x})\sum_{i=1}^q \bm{b}_i\bm{b}_i^T \right] = \Delta f(\bm{x}) -\bm{x}^T \nabla\nabla f(\bm{x}) \bm{x},$$ 
			the preceding display continues as
			\begin{align*}
			&\int_{\Omega_{q-1}} \left(\sum_{i=1}^q \bm{b}_i\xi_i \right) \left(\sum_{i=1}^q \bm{b}_i\xi_i \right)^T \cdot \left(\sum_{i=1}^q \bm{b}_i\xi_i \right)^T \nabla\nabla f(\bm{x}) \left(\sum_{i=1}^q \bm{b}_i\xi_i \right) \omega_{q-1}(d\bm{\xi})\\
			&= \frac{3\bar{\omega}_{q-1}}{q(q+2)} \left(\sum_{i=1}^q \bm{b}_i\bm{b}_i^T \cdot \bm{b}_i^T \nabla\nabla f(\bm{x})\bm{b}_i\right) + \frac{\bar{\omega}_{q-1}}{q(q+2)} \left(\sum_{i\neq j} \bm{b}_i\bm{b}_i^T \cdot \bm{b}_j^T \nabla\nabla f(\bm{x})\bm{b}_j\right)\\ 
			&\quad + \frac{2\bar{\omega}_{q-1}}{q(q+2)} \left(\sum_{i\neq j} \bm{b}_i\bm{b}_j^T \cdot \bm{b}_i^T \nabla\nabla f(\bm{x})\bm{b}_j\right)\\
			&= \frac{\bar{\omega}_{q-1}}{q(q+2)} \left[\sum_{i=1}^q \bm{b}_i\bm{b}_i^T \left(\sum_{j=1}^q\bm{b}_j^T \nabla\nabla f(\bm{x})\bm{b}_j \right)\right] \\
			&\quad + \frac{2\bar{\omega}_{q-1}}{q(q+2)} \left[\sum_{i=1}^q \bm{b}_i\left(\bm{b}_i^T \nabla\nabla f(\bm{x}) \bm{b}_i \right) \bm{b}_i^T + \sum_{i\neq j} \bm{b}_i\left(\bm{b}_i^T \nabla\nabla f(\bm{x})\bm{b}_j \right)\bm{b}_j^T \right]\\
			&= \frac{\bar{\omega}_{q-1}}{q(q+2)} \left(I_{q+1}-\bm{x}\bm{x}^T \right) \left[\Delta f(\bm{x}) -\bm{x}^T \nabla\nabla f(\bm{x}) \bm{x} \right] \\
			&\quad + \frac{2\bar{\omega}_{q-1}}{q(q+2)} \left(I_{q+1}-\bm{x}\bm{x}^T \right) \nabla\nabla f(\bm{x}) \left(I_{q+1}-\bm{x}\bm{x}^T \right).
			\end{align*}
			Plugging this result back into \eqref{II_integral} and conduct some integration by parts, we obtain that
			\begin{align*}
			&\text{Plug in (III) in \eqref{Hess_Expect1}} \\
			&= c_{h,q}(L) h^{q-2} \int_0^{2h^{-2}} \frac{\bar{\omega}_{q-1}}{q} \left(I_{q+1} -\bm{x}\bm{x}^T \right) f(\bm{x}) L''(r)\cdot r^{\frac{q}{2}} (2-h^2r)^{\frac{q}{2}} \, dr\\
			&\quad - c_{h,q}(L) h^q \int_0^{2h^{-2}} \frac{\bar{\omega}_{q-1}}{q} \left(I_{q+1} -\bm{x}\bm{x}^T \right) \nabla f(\bm{x})^T \bm{x} L''(r) \cdot r^{\frac{q}{2}+1} (2-h^2r)^{\frac{q}{2}} \, dr\\
			&\quad + c_{h,q}(L) h^q \cdot\frac{\bar{\omega}_{q-1}}{2q(q+2)} \int_0^{2h^{-2}} \left(I_{q+1}-\bm{x}\bm{x}^T \right) \left[\Delta f(\bm{x}) -\bm{x}^T \nabla\nabla f(\bm{x}) \bm{x} \right] L''(r) r^{\frac{q}{2}+1}(2-h^2r)^{\frac{q}{2}+1} dr\\
			&\quad + c_{h,q}(L) h^q \cdot\frac{\bar{\omega}_{q-1}}{q(q+2)} \int_0^{2h^{-2}} \left(I_{q+1}-\bm{x}\bm{x}^T \right) \nabla\nabla f(\bm{x}) \left(I_{q+1}-\bm{x}\bm{x}^T \right) L''(r) r^{\frac{q}{2}+1} (2-h^2r)^{\frac{q}{2}+1} dr\\
			&\quad +O(h^2)\\
			&= c_{h,q}(L) h^{q-2} \cdot \frac{\bar{\omega}_{q-1}}{q} \left(I_{q+1} -\bm{x}\bm{x}^T \right) L'(r) f(\bm{x})\cdot r^{\frac{q}{2}} (2-h^2r)^{\frac{q}{2}} \Big|_0^{2h^{-2}}\\
			&\quad - c_{h,q}(L) h^{q-2} \cdot \frac{\bar{\omega}_{q-1}}{2}\int_0^{2h^{-2}} \left(I_{q+1} -\bm{x}\bm{x}^T \right) f(\bm{x}) L'(r)\left[r^{\frac{q}{2}-1}(2-h^2r)^{\frac{q}{2}} -h^2 r^{\frac{q}{2}}(2-h^2r)^{\frac{q}{2}-1} \right]dr\\
			&\quad - c_{h,q}(L) h^q \cdot \frac{\bar{\omega}_{q-1}}{q} \left(I_{q+1} -\bm{x}\bm{x}^T \right) \nabla f(\bm{x})^T \bm{x} L'(r) \cdot r^{\frac{q}{2}+1} (2-h^2r)^{\frac{q}{2}} \Big|_0^{2h^{-2}}\\
			&\quad + c_{h,q}(L) h^q \int_0^{2h^{-2}} \frac{\bar{\omega}_{q-1}}{q} \left(I_{q+1} -\bm{x}\bm{x}^T \right) \nabla f(\bm{x})^T \bm{x} L'(r)\\
			&\hspace{50mm} \times \left[\left(\frac{q+2}{2} \right) r^{\frac{q}{2}} (2-h^2r)^{\frac{q}{2}} - \frac{qh^2}{2} r^{\frac{q}{2}+1}(2-h^2r)^{\frac{q}{2}-1} \right] dr\\
			&\quad + c_{h,q}(L) h^q \cdot \frac{\bar{\omega}_{q-1}}{2q(q+2)} \left(I_{q+1}-\bm{x}\bm{x}^T \right) \left[\Delta f(\bm{x}) -\bm{x}^T \nabla\nabla f(\bm{x}) \bm{x} \right] L'(r) r^{\frac{q}{2}+1}(2-h^2r)^{\frac{q}{2}+1} \Big|_0^{2h^{-2}}\\
			&\quad - c_{h,q}(L) h^q \cdot \frac{\bar{\omega}_{q-1}}{4q} \int_0^{2h^{-2}} \left(I_{q+1}-\bm{x}\bm{x}^T \right) \left[\Delta f(\bm{x}) -\bm{x}^T \nabla\nabla f(\bm{x}) \bm{x} \right] L'(r)\\
			&\hspace{50mm} \times \left[r^{\frac{q}{2}}(2-h^2r)^{\frac{q}{2}+1} -h^2 r^{\frac{q}{2}+1} (2-h^2r)^{\frac{q}{2}} \right] dr\\
			&\quad + c_{h,q}(L) h^q \cdot \frac{\bar{\omega}_{q-1}}{q(q+2)} \left(I_{q+1}-\bm{x}\bm{x}^T \right) \nabla\nabla f(\bm{x}) \left(I_{q+1}-\bm{x}\bm{x}^T \right) L'(r) r^{\frac{q}{2}+1} (2-h^2r)^{\frac{q}{2}+1} \Big|_0^{2h^{-2}}\\
			&\quad - c_{h,q}(L) h^q \cdot \frac{\bar{\omega}_{q-1}}{2q} \int_0^{2h^{-2}} \left(I_{q+1}-\bm{x}\bm{x}^T \right) \nabla\nabla f(\bm{x}) \left(I_{q+1}-\bm{x}\bm{x}^T \right) L'(r)\\
			&\hspace{50mm} \times \left[r^{\frac{q}{2}}(2-h^2r)^{\frac{q}{2}+1} -h^2 r^{\frac{q}{2}+1} (2-h^2r)^{\frac{q}{2}} \right] dr \\
			&\quad +O(h^2)\\
			&= - c_{h,q}(L) h^{q-2} \bar{\omega}_{q-1} \int_0^{2h^{-2}} \left(I_{q+1}-\bm{x}\bm{x}^T \right) f(\bm{x}) L'(r) r^{\frac{q}{2}-1}(2-h^2r)^{\frac{q}{2}-1}(1-h^2r) dr\\
			&\quad + c_{h,q}(L) h^q \cdot \frac{\bar{\omega}_{q-1}}{2} \int_0^{2h^{-2}} \left(I_{q+1}-\bm{x}\bm{x}^T \right) \nabla f(\bm{x})^T \bm{x} \cdot L'(r) r^{\frac{q}{2}} (2-h^2r)^{\frac{q}{2}} dr\\
			&\quad + c_{h,q}(L) h^q \cdot \frac{\bar{\omega}_{q-1}}{q} \int_0^{2h^{-2}} \left(I_{q+1}-\bm{x}\bm{x}^T \right) \nabla f(\bm{x})^T \bm{x} \cdot L'(r) r^{\frac{q}{2}} (2-h^2r)^{\frac{q}{2}} dr\\
			&\quad - c_{h,q}(L) h^q \cdot \frac{\bar{\omega}_{q-1}}{4q} \int_0^{2h^{-2}} \left(I_{q+1}-\bm{x}\bm{x}^T \right) \left[\Delta f(\bm{x}) -\bm{x}^T \nabla\nabla f(\bm{x}) \bm{x} \right] L'(r) r^{\frac{q}{2}}(2-h^2r)^{\frac{q}{2}+1} dr\\
			&\quad - c_{h,q}(L) h^q \cdot \frac{\bar{\omega}_{q-1}}{2q} \int_0^{2h^{-2}} \left(I_{q+1}-\bm{x}\bm{x}^T \right) \nabla\nabla f(\bm{x}) \left(I_{q+1}-\bm{x}\bm{x}^T \right) L'(r) r^{\frac{q}{2}}(2-h^2r)^{\frac{q}{2}+1} dr \\
			&\quad +O(h^2),
			\end{align*}
			where we use the fact that 
			\begin{equation}
			\label{mid_asym_O_h2}
			c_{h,q}(L)h^{q+2} \int_0^{2h^{-2}} L'(r) r^j (2-h^2r)^k dr = O(h^2)
			\end{equation}
			for any $k,j>0$ via Remark~\ref{integ_lemma_remark}. With extra integration by parts on the third and fifth term in the preceding display, we obtain that
			\begin{align}
			\label{Hess_II_expect}
			\begin{split}
			&\text{Plug in (III) in \eqref{Hess_Expect1}} \\
			&= - c_{h,q}(L) h^{q-2} \bar{\omega}_{q-1} \int_0^{2h^{-2}} \left(I_{q+1}-\bm{x}\bm{x}^T \right) f(\bm{x}) L'(r) r^{\frac{q}{2}-1}(2-h^2r)^{\frac{q}{2}-1}(1-h^2r) dr\\
			&\quad + c_{h,q}(L) h^q \cdot \frac{\bar{\omega}_{q-1}}{2} \int_0^{2h^{-2}} \left(I_{q+1}-\bm{x}\bm{x}^T \right) \nabla f(\bm{x})^T \bm{x}L'(r) r^{\frac{q}{2}} (2-h^2r)^{\frac{q}{2}} dr\\
			&\quad - c_{h,q}(L) h^q \cdot \frac{\bar{\omega}_{q-1}}{2} \int_0^{2h^{-2}} \left(I_{q+1}-\bm{x}\bm{x}^T \right) \nabla f(\bm{x})^T \bm{x} L(r) r^{\frac{q}{2}-1} (2-h^2r)^{\frac{q}{2}} dr\\
			&\quad - c_{h,q}(L) h^q \cdot \frac{\bar{\omega}_{q-1}}{4q} \int_0^{2h^{-2}} \left(I_{q+1}-\bm{x}\bm{x}^T \right) \left[\Delta f(\bm{x}) -\bm{x}^T \nabla\nabla f(\bm{x}) \bm{x} \right] L'(r) r^{\frac{q}{2}}(2-h^2r)^{\frac{q}{2}+1} dr\\
			&\quad + c_{h,q}(L) h^q \cdot \frac{\bar{\omega}_{q-1}}{4} \int_0^{2h^{-2}} \left(I_{q+1}-\bm{x}\bm{x}^T \right) \nabla\nabla f(\bm{x}) \left(I_{q+1}-\bm{x}\bm{x}^T \right) L(r) r^{\frac{q}{2}-1}(2-h^2r)^{\frac{q}{2}+1} dr \\
			&\quad +O(h^2).
			\end{split}
			\end{align}
			We now plug the Taylor's expansion of $f(\bm{x}+\alpha_{\bm{x},\bm{\xi}})$ at $\bm{x}$ back into (IV) in \eqref{Hess_Expect1} and deduce that
			\begin{align}
			\label{Hess_III_expect}
			\begin{split}
			&\text{Plug in (IV) in \eqref{Hess_Expect1}} \\
			&= c_{h,q}(L) h^{q-2} \bar{\omega}_{q-1} \int_0^{2h^{-2}} I_{q+1} f(\bm{x}) L'(r)\cdot r^{\frac{q}{2}-1}(1-h^2r) (2-h^2r)^{\frac{q}{2}-1} dr\\
			&\quad + c_{h,q}(L) h^{q-2} \int_0^{2h^{-2}} \int_{\Omega_{q-1}} I_{q+1} \nabla f(\bm{x})^T \alpha_{\bm{x},\bm{\xi}} L'(r) \cdot r^{\frac{q}{2}-1}(1-h^2r) (2-h^2r)^{\frac{q}{2}-1} \,\omega_{q-1}(d\bm{\xi}) dr\\
			&\quad + c_{h,q}(L) h^{q-2} \int_0^{2h^{-2}} \int_{\Omega_{q-1}} I_{q+1}\cdot \frac{1}{2} \alpha_{\bm{x},\bm{\xi}}^T \nabla\nabla f(\bm{x}) \alpha_{\bm{x},\bm{\xi}} r^{\frac{q}{2}-1}(1-h^2r) (2-h^2r)^{\frac{q}{2}-1} \,\omega_{q-1}(d\bm{\xi}) dr \\
			&\quad +O(h^2)\\
			&= c_{h,q}(L) h^{q-2} \bar{\omega}_{q-1} \int_0^{2h^{-2}} I_{q+1} f(\bm{x}) L'(r)\cdot r^{\frac{q}{2}-1}(1-h^2r) (2-h^2r)^{\frac{q}{2}-1} dr\\
			&\quad - c_{h,q}(L) h^q \bar{\omega}_{q-1} \int_0^{2h^{-2}} I_{q+1} \nabla f(\bm{x})^T \bm{x} L'(r) r^{\frac{q}{2}}(1-h^2r)(2-h^2r)^{\frac{q}{2}-1} dr\\
			& \quad +c_{h,q}(L) h^{q-1} \int_0^{2h^{-2}} \underbrace{\int_{\Omega_{q-1}} I_{q+1} \cdot \nabla f(\bm{x})^T (\bm{B_x}\bm{\xi}) r^{\frac{q-1}{2}} (2-h^2r)^{\frac{q-1}{2}} (1-h^2r) \,\omega_{q-1}(d\bm{\xi})}_{=0} dr\\
			&\quad + c_{h,q}(L) h^{q-2} \int_0^{2h^{-2}} \int_{\Omega_{q-1}} I_{q+1}\cdot \frac{1}{2} (\bm{B_x}\bm{\xi})^T \nabla\nabla f(\bm{x}) (\bm{B_x}\bm{\xi}) L'(r) r^{\frac{q}{2}} (1-h^2r) (2-h^2r)^{\frac{q}{2}} \\
			&\quad +O(h^2)\\
			&= c_{h,q}(L) h^{q-2} \bar{\omega}_{q-1} \int_0^{2h^{-2}} I_{q+1} f(\bm{x}) L'(r)\cdot r^{\frac{q}{2}-1}(1-h^2r) (2-h^2r)^{\frac{q}{2}-1} dr\\
			&\quad - c_{h,q}(L) h^q \bar{\omega}_{q-1} \int_0^{2h^{-2}} I_{q+1} \nabla f(\bm{x})^T \bm{x} L'(r) r^{\frac{q}{2}}(1-h^2r)(2-h^2r)^{\frac{q}{2}-1} dr\\
			&\quad + c_{h,q}(L) h^q \cdot \frac{\bar{\omega}_{q-1}}{2q} \int_0^{2h^{-2}} I_{q+1} \left[\Delta f(\bm{x}) -\bm{x}^T\nabla\nabla f(\bm{x}) \bm{x} \right] L'(r) r^{\frac{q}{2}}(1-h^2r)(2-h^2r)^{\frac{q}{2}} dr \\
			& \quad + O(h^2),
			\end{split}
			\end{align}
			where we expand $\alpha_{\bm{x},\bm{\xi}}=-rh^2\bm{x}+h\sqrt{r(2-h^2r)} \bm{B_x}\bm{\xi}$, absorb $O(h^2)$ terms via \eqref{mid_asym_O_h2}, make use of (c) in Lemma \ref{integ_lemma}, and leverage our argument in \eqref{eq_III1}. Combining \eqref{Hess_Expect1}, \eqref{Hess_II_expect}, and \eqref{Hess_III_expect}, we conclude that
			\begin{align*}
			&\mathbb{E}\left[\mathcal{H} \hat{f}_h(\bm x) \right] \\
			&= \left(I_{q+1}-\bm{x}\bm{x}^T \right) \mathbb{E}\left[{\mathcal{A}}\hat f_h (\bm x)\right] \left(I_{q+1}-\bm{x}\bm{x}^T \right)\\
			&= \left(I_{q+1}-\bm{x}\bm{x}^T \right) \cdot \text{(III)}\cdot \left(I_{q+1}-\bm{x}\bm{x}^T \right) + \left(I_{q+1}-\bm{x}\bm{x}^T \right) \cdot \text{(IV)}\cdot \left(I_{q+1}-\bm{x}\bm{x}^T \right)\\
			&= - c_{h,q}(L) h^{q-2} \bar{\omega}_{q-1} \int_0^{2h^{-2}} \left(I_{q+1}-\bm{x}\bm{x}^T \right) f(\bm{x}) L'(r) r^{\frac{q}{2}-1}(2-h^2r)^{\frac{q}{2}-1}(1-h^2r) dr\\
			&\quad + c_{h,q}(L) h^q \cdot \frac{\bar{\omega}_{q-1}}{2} \int_0^{2h^{-2}} \left(I_{q+1}-\bm{x}\bm{x}^T \right) \nabla f(\bm{x})^T \bm{x}L'(r) r^{\frac{q}{2}} (2-h^2r)^{\frac{q}{2}} dr\\
			&\quad - c_{h,q}(L) h^q \cdot \frac{\bar{\omega}_{q-1}}{2} \int_0^{2h^{-2}} \left(I_{q+1}-\bm{x}\bm{x}^T \right) \nabla f(\bm{x})^T \bm{x} L(r) r^{\frac{q}{2}-1} (2-h^2r)^{\frac{q}{2}} dr\\
			&\quad - c_{h,q}(L) h^q \cdot \frac{\bar{\omega}_{q-1}}{4q} \int_0^{2h^{-2}} \left(I_{q+1}-\bm{x}\bm{x}^T \right) \left[\Delta f(\bm{x}) -\bm{x}^T \nabla\nabla f(\bm{x}) \bm{x} \right] L'(r) r^{\frac{q}{2}}(2-h^2r)^{\frac{q}{2}+1} dr\\
			&\quad + c_{h,q}(L) h^q \cdot \frac{\bar{\omega}_{q-1}}{4} \int_0^{2h^{-2}} \left(I_{q+1}-\bm{x}\bm{x}^T \right) \nabla\nabla f(\bm{x}) \left(I_{q+1}-\bm{x}\bm{x}^T \right) L(r) r^{\frac{q}{2}-1}(2-h^2r)^{\frac{q}{2}+1} dr\\
			&\quad + c_{h,q}(L) h^{q-2} \cdot \bar{\omega}_{q-1} \int_0^{2h^{-2}} \left(I_{q+1}-\bm{x}\bm{x}^T \right) f(\bm{x}) L'(r)\cdot r^{\frac{q}{2}-1}(1-h^2r) (2-h^2r)^{\frac{q}{2}-1} dr\\
			&\quad - c_{h,q}(L) h^q \cdot \bar{\omega}_{q-1} \int_0^{2h^{-2}} \left(I_{q+1}-\bm{x}\bm{x}^T \right) \nabla f(\bm{x})^T \bm{x} L'(r) r^{\frac{q}{2}}(1-h^2r)(2-h^2r)^{\frac{q}{2}-1} dr\\
			&\quad + c_{h,q}(L) h^q \cdot \frac{\bar{\omega}_{q-1}}{2q} \int_0^{2h^{-2}} \left(I_{q+1}-\bm{x}\bm{x}^T \right) \left[\Delta f(\bm{x}) -\bm{x}^T\nabla\nabla f(\bm{x}) \bm{x} \right] \\
			&\hspace{60mm} \times L'(r) r^{\frac{q}{2}}(1-h^2r)(2-h^2r)^{\frac{q}{2}} dr \\
			&\quad + O(h^2)\\
			&\stackrel{(**)}{=} 0+c_{h,q}(L) h^q \cdot \frac{\bar{\omega}_{q-1}}{2} \int_0^{2h^{-2}} \left(I_{q+1}-\bm{x}\bm{x}^T \right) \nabla f(\bm{x})^T \bm{x}\\
			&\hspace{50mm} \times L'(r) r^{\frac{q}{2}} (2-h^2r)^{\frac{q}{2}-1}(2-h^2r-2+2rh^2) dr\\
			&\quad - \lambda_{h,q}(L)^{-1} \left(I_{q+1}-\bm{x}\bm{x}^T \right) \nabla f(\bm{x})^T \bm{x} \cdot \frac{\bar{\omega}_{q-1}}{2} \int_0^{2h^{-2}} L(r) r^{\frac{q}{2}-1} (2-h^2r)^{\frac{q}{2}} dr\\
			&\quad + c_{h,q}(L) h^q \cdot \frac{\bar{\omega}_{q-1}}{4q} \left(I_{q+1}-\bm{x}\bm{x}^T \right)\left[\Delta f(\bm{x}) -\bm{x}^T \nabla\nabla f(\bm{x}) \bm{x} \right]\\
			&\hspace{50mm} \times \int_0^{2h^{-2}}  L'(r) r^{\frac{q}{2}}(2-h^2r)^{\frac{q}{2}}(2-2h^2r-2+h^2r)dr\\
			&\quad + \lambda_{h,q}(L)^{-1} \cdot \frac{\bar{\omega}_{q-1}}{4} \int_0^{2h^{-2}} \left(I_{q+1}-\bm{x}\bm{x}^T \right)\nabla\nabla f(\bm{x})\left(I_{q+1}-\bm{x}\bm{x}^T \right)\\
			&\hspace{50mm} \times L(r)r^{\frac{q}{2}-1}(2-h^2r)^{\frac{q}{2}+1} dr \\
			&\quad + O(h^2)\\
			& \stackrel{(***)}{=} O(h^2)- \left(I_{q+1}-\bm{x}\bm{x}^T \right) \nabla f(\bm{x})^T \bm{x} \cdot \frac{\bar{\omega}_{q-1}}{2} \lambda_q(L)^{-1} \int_0^{\infty} L(r)r^{\frac{q}{2}-1} 2^{\frac{q}{2}} dr + O(h^2)\\
			& \quad + \lambda_q(L)^{-1} \cdot \frac{\bar{\omega}_{q-1}}{4} \left(I_{q+1}-\bm{x}\bm{x}^T \right)\nabla\nabla f(\bm{x})\left(I_{q+1}-\bm{x}\bm{x}^T \right) \int_0^{\infty} L(r) r^{\frac{q}{2}-1} 2^{\frac{q}{2}+1} dr +O(h^2)\\
			& = - \left(I_{q+1}-\bm{x}\bm{x}^T \right) \nabla f(\bm{x})^T \bm{x} + \left(I_{q+1}-\bm{x}\bm{x}^T \right)\nabla\nabla f(\bm{x})\left(I_{q+1}-\bm{x}\bm{x}^T \right) + O(h^2),
			\end{align*}
			where the first term matches up with the sixth term, the second term with the seventh term, the fourth term with the eighth term, and \eqref{asym_norm_const} is applied to the rest terms when $h\to 0$ in $(**)$. In addition, we leverage the asymptotic rates \eqref{mid_asym_O_h2} and \eqref{kernel_asymp_rate} as well as recall that $\lambda_q(L)=2^{\frac{q}{2}-1} \bar{\omega}_{q-1} \int_0^{\infty} L(r) r^{\frac{q}{2}-1} dr$ from (a) of Lemma~\ref{integ_lemma} in $(***)$. {\bf Result 3} thus follows. It implies that the bias $\mathbb{E}\left[\mathcal{H} \hat{f}_h(\bm x) \right] - \mathcal{H} f(\bm x)$ is of the rate $O(h^2)$ and the Riemannian Hessian estimator is asymptotically unbiased.
			
			Now, we proceed to bound 
			$$\mathcal{H} \hat{f}_h(\bm x) -\mathbb{E}\left[\mathcal{H} \hat{f}_h(\bm x) \right] = \left(I_{q+1}-\bm{x}\bm{x}^T \right) \left[{\mathcal{A}}\hat f_h (\bm x)-\mathbb{E}\left({\mathcal{A}}\hat f_h (\bm x)\right) \right] \left(I_{q+1}-\bm{x}\bm{x}^T \right).$$
			\noindent $\bullet$ {\bf Result 4}. The covariance matrix $\text{Cov}\left[\mathtt{vec}\left(\mathcal{H}\hat f_h (\bm x)\right) \right]$ has the following asymptotic rate as $h\to 0$ and $nh^{q+4} \to \infty$:
			$$\text{Cov}\left[\mathtt{vec}\left(\mathcal{H}\hat f_h (\bm x)\right) \right] = O\left(\frac{1}{nh^{q+4}} \right),$$
			where we define the matrix $\mathtt{vec}$ operator, which converts a matrix into a vector by stacking the columns. That is, given a matrix $A\in \mathbb{R}^{m\times n}$, $\mathtt{vec}(A)$ is a vector of length $mn$.
			
			\noindent \emph{Derivation of Result 4}. We first calculate the covariance matrix of $\mathtt{vec}\left({\mathcal{A}}\hat f_h (\bm x)\right)$ as
			\begin{align*}
			&\text{Cov}\left[\mathtt{vec}\left({\mathcal{A}}\hat f_h (\bm x)\right) \right]\\
			&= \frac{c_{h,q}(L)^2}{nh^8} \cdot \mathbb{E}\left[\mathtt{vec}(\bm{X}_1\bm{X}_1^T) \cdot \mathtt{vec}(\bm{X}_1\bm{X}_1^T)^T L''\left(\frac{1-\bm{x}^T\bm{X}_i}{h^2}\right)^2\right]\\
			&\quad + \frac{2c_{h,q}(L)^2}{nh^6} \cdot \mathbb{E}\left[\mathtt{vec}(\bm{X}_i\bm{X}_i^T)\cdot \mathtt{vec}(\bm{x}^T\bm{X}_1 I_{q+1})^T L''\left(\frac{1-\bm{x}^T\bm{X}_i}{h^2}\right) L'\left(\frac{1-\bm{x}^T\bm{X}_i}{h^2}\right) \right]\\
			&\quad + \frac{c_{h,q}(L)^2}{nh^4} \cdot \mathbb{E}\left[\mathtt{vec}(\bm{x}^T\bm{X}_1 I_{q+1}) \cdot \mathtt{vec}(\bm{x}^T\bm{X}_1 I_{q+1})^T L'\left(\frac{1-\bm{x}^T\bm{X}_i}{h^2}\right)^2\right]\\
			&\quad - \frac{1}{n}\cdot \mathbb{E}\left[\mathtt{vec}\left({\mathcal{A}}\hat f_h (\bm x)\right) \right] \mathbb{E}\left[\mathtt{vec}\left({\mathcal{A}}\hat f_h (\bm x)\right) \right]^T\\
			&= \frac{c_{h,q}(L)^2}{nh^8} \int_{\Omega_q} \mathtt{vec}(\bm{y}\bm{y}^T)\cdot \mathtt{vec}(\bm{y}\bm{y}^T)^T L''\left(\frac{1-\bm{x}^T\bm{y}}{h^2}\right)^2 f(\bm{y})\, \omega_q(d\bm{y})\\
			&\quad + \frac{2c_{h,q}(L)^2}{nh^6} \int_{\Omega_q} \mathtt{vec}(\bm{y}\bm{y}^T)\cdot \mathtt{vec}(\bm{x}^T\bm{y}I_{q+1})^TL''\left(\frac{1-\bm{x}^T\bm{y}}{h^2}\right) L'\left(\frac{1-\bm{x}^T\bm{y}}{h^2}\right) f(\bm{y}) \, \omega_q(d\bm{y})\\
			&\quad + \frac{c_{h,q}(L)^2}{nh^4} \int_{\Omega_q} \mathtt{vec}(\bm{x}^T\bm{y}I_{q+1}) \cdot \mathtt{vec}(\bm{x}^T\bm{y}I_{q+1})^TL'\left(\frac{1-\bm{x}^T\bm{y}}{h^2}\right)^2 f(\bm{y}) \, \omega_q(d\bm{y})\\
			&\quad - \frac{1}{n}\cdot \mathbb{E}\left[\mathtt{vec}\left({\mathcal{A}}\hat f_h (\bm x)\right) \right] \mathbb{E}\left[\mathtt{vec}\left({\mathcal{A}}\hat f_h (\bm x)\right) \right]^T\\
			&= \frac{c_{h,q}(L)^2 h^q}{nh^8} \int_0^{2h^{-2}} \int_{\Omega_{q-1}} \mathtt{vec}\left[(\bm{x}+ \alpha_{\bm{x},\bm{\xi}}) (\bm{x}+ \alpha_{\bm{x},\bm{\xi}})^T\right] \mathtt{vec}\left[(\bm{x}+ \alpha_{\bm{x},\bm{\xi}}) (\bm{x}+ \alpha_{\bm{x},\bm{\xi}})^T\right]^T\\ 
			&\hspace{50mm} \times L''(r)^2 f(\bm{x}+\alpha_{\bm{x},\bm{\xi}})\cdot r^{\frac{q}{2}-1} (2-h^2r)^{\frac{q}{2}-1}\, \omega_{q-1}(d\bm{\xi}) dr\\
			&\quad + \frac{2c_{h,q}(L)^2 h^q}{nh^6} \int_0^{2h^{-2}} \int_{\Omega_{q-1}} \mathtt{vec}\left[(\bm{x}+ \alpha_{\bm{x},\bm{\xi}}) (\bm{x}+ \alpha_{\bm{x},\bm{\xi}})^T\right] \mathtt{vec}\left[I_{q+1}(1-rh^2)\right]^T\\ 
			&\hspace{50mm} \times L''(r) L'(r) f(\bm{x}+\alpha_{\bm{x},\bm{\xi}}) \cdot r^{\frac{q}{2}-1} (2-h^2r)^{\frac{q}{2}-1}\, \omega_{q-1}(d\bm{\xi}) dr\\
			& \quad+ \frac{c_{h,q}(L)^2 h^q}{nh^4} \int_0^{2h^{-2}} \int_{\Omega_{q-1}} \mathtt{vec}\left[I_{q+1}(1-rh^2)\right] \mathtt{vec}\left[I_{q+1}(1-rh^2)\right]^T\\ 
			&\hspace{50mm} \times L'(r)^2 f(\bm{x}+\alpha_{\bm{x},\bm{\xi}}) \cdot r^{\frac{q}{2}-1} (2-h^2r)^{\frac{q}{2}-1}\, \omega_{q-1}(d\bm{\xi}) dr\\
			&\quad - \frac{1}{n}\cdot \mathbb{E}\left[\mathtt{vec}\left({\mathcal{A}}\hat f_h (\bm x)\right) \right] \mathbb{E}\left[\mathtt{vec}\left({\mathcal{A}}\hat f_h (\bm x)\right) \right]^T,
			\end{align*}
			where $\alpha_{\bm{x},\bm{\xi}} = -rh^2\bm{x} + h\sqrt{r(2-h^2r)} \bm{B_x}\bm{\xi}$. Note that by condition (D1), the first-order Taylor's expansion of $f$ at $\bm{x}\in \Omega_q$ is
			$$f(\bm{x}+\alpha_{\bm{x},\bm{\xi}}) = f(\bm{x}) + O(||\alpha_{\bm{x},\bm{\xi}}||_2) =f(\bm{x}) +O(h).$$
			In addition,
			\begin{align*}
			(\bm{x}+\alpha_{\bm{x},\bm{\xi}})(\bm{x}+\alpha_{\bm{x},\bm{\xi}})^T &= (1-rh^2)^2 \bm{x}\bm{x}^T +h(1-r^2h) \sqrt{r(2-h^2r)}\left[\bm{x}(\bm{B_x}\bm{\xi})^T + (\bm{B_x}) \bm{x}^T \right]\\
			&\quad + h^2r(2-h^2r) (\bm{B_x}\bm{\xi}) (\bm{B_x}\bm{\xi})^T.
			\end{align*} 
			Moreover, when the congruence operation $\left(I_{q+1}-\bm{x}\bm{x}^T \right) {\mathcal{A}}\hat f_h (\bm x)\left(I_{q+1}-\bm{x}\bm{x}^T \right)$ is introduced, it will be applied inside the $\mathtt{vec}$ operation. Thus, after applying the congruence operation,
			\begin{align*}
			&\mathtt{vec}\left[\left(I_{q+1}-\bm{x}\bm{x}^T \right)(\bm{x}+ \alpha_{\bm{x},\bm{\xi}}) (\bm{x}+ \alpha_{\bm{x},\bm{\xi}})^T \left(I_{q+1}-\bm{x}\bm{x}^T \right)\right]\\
			&\quad \times \mathtt{vec}\left[\left(I_{q+1}-\bm{x}\bm{x}^T \right)(\bm{x}+ \alpha_{\bm{x},\bm{\xi}}) (\bm{x}+ \alpha_{\bm{x},\bm{\xi}})^T \left(I_{q+1}-\bm{x}\bm{x}^T \right)\right]^T = O(h^4)
			\end{align*}
			and
			\begin{align*}
			& \mathtt{vec}\left[\left(I_{q+1}-\bm{x}\bm{x}^T \right)(\bm{x}+ \alpha_{\bm{x},\bm{\xi}}) (\bm{x}+ \alpha_{\bm{x},\bm{\xi}})^T \left(I_{q+1}-\bm{x}\bm{x}^T \right)\right]\\
			&\quad \times \mathtt{vec}\left[\left(I_{q+1}-\bm{x}\bm{x}^T \right)(1-rh^2)\right]^T =O(h^2).
			\end{align*}
			Together with condition (D2'), \eqref{asym_norm_const}, \eqref{mid_asym_O_h2}, and the bias bound $\mathbb{E}\left[\mathtt{vec}\left(\mathcal{H} \hat{f}_h(\bm x) \right) \right] = \mathcal{H} f(\bm x)+ O(h^2)$, we conclude that 
			$$\text{Cov}\left[\mathtt{vec}\left(\mathcal{H} \hat{f}_h(\bm x) \right) \right] = O\left(\frac{1}{nh^{q+4}} \right).$$
			{\bf Result 4} is thus proved. Finally, by the central limit theorem, 
			\begin{align*}
			&\mathtt{vec}\left\{\mathcal{H} \hat{f}_h(\bm x) - \mathbb{E}\left[\mathcal{H} \hat{f}_h(\bm x) \right]\right\} \\
			&= \text{Cov}\left[\mathtt{vec}\left(\mathcal{H} \hat{f}_h(\bm x) \right) \right]^{\frac{1}{2}} \text{Cov}\left[\mathtt{vec}\left(\mathcal{H} \hat{f}_h(\bm x) \right) \right]^{-\frac{1}{2}} \mathtt{vec}\left\{\mathcal{H} \hat{f}_h(\bm x) - \mathbb{E}\left[\mathcal{H} \hat{f}_h(\bm x) \right]\right\}\\
			&= O\left(\sqrt{\frac{1}{nh^{q+4}}} \right) \cdot \tilde{\bm{Z}}_{n}(\bm{x})\\
			&= O_P\left(\sqrt{\frac{1}{nh^{q+4}}} \right),
			\end{align*}
			where $\tilde{\bm{Z}}_{n}(\bm{x}) \stackrel{d}{\to} N_{(q+1)^2}\left(\bm{0}, I_{(q+1)^2} \right)$. In total, we conclude with our bias and stochastic variation bounds that
			$$\mathcal{H} \hat{f}_h(\bm x) - \mathcal{H} f(\bm x) = O(h^2) + O_P\left(\sqrt{\frac{1}{nh^{q+4}}} \right)$$
			for any fixed $\bm{x}\in \Omega_q$ as $h\to 0$ and $nh^{q+4} \to \infty$, where 
			$$\mathcal{H} f(\bm x)= \left(I_{q+1}-\bm{x}\bm{x}^T \right) \nabla\nabla f(\bm{x}) \left(I_{q+1}-\bm{x}\bm{x}^T \right) - \nabla f(\bm{x})^T \bm{x} \left(I_{q+1}-\bm{x}\bm{x}^T \right).$$
		\end{proof}
	
	\subsection{Proof of Theorem~\ref{unif_conv_tang}}
	\label{Appendix:Thm4_pf}
	
	\begin{customthm}{4}
		Assume (D1), (D2'), and (K1). The uniform convergence rate of $\hat{f}_h$ is given by
		$$\sup_{\bm x \in\Omega_q}|\hat{f}_h(\bm x) -f(\bm x)| = O(h^2) + O_P\left(\sqrt{\frac{|\log h|}{nh^q}} \right)$$
		as $h\to 0$ and $\frac{nh^q}{|\log h|} \to \infty$. \\
		Furthermore, the uniform convergence rate of $\grad \hat{f}_h(\bm{x})$ on $\Omega_q$ is 
		$$
		\sup_{\bm x\in\Omega_q}\norm{\grad \hat{f}_h (\bm x)- \grad f(\bm x) }_{\max} = O(h^2) + O_P\left(\sqrt{\frac{|\log h|}{nh^{q+2}}} \right),
		$$
		as $h\to 0$ and $\frac{nh^{q+2}}{|\log h|} \to \infty$.
		Finally, the uniform convergence rate of $\mathcal{H} \hat{f}_h(\bm x)$ on $\Omega_q$ is
		$$
		\sup_{\bm{x}\in \Omega_q}\norm{\mathcal{H} \hat{f}_h(\bm x) - \mathcal{H} f(\bm x)}_{\max} 
		= O(h^2) + O_P\left(\sqrt{\frac{|\log h|}{nh^{q+4}}} \right),
		$$
		as $h\to 0$ and $\frac{nh^{q+4}}{|\log h|} \to \infty$, where $\norm{\cdot}_{\max} $ is the elementwise maximum norm for a vector in $\mathbb{R}^{q+1}$ or a matrix in $\mathbb{R}^{(q+1)\times (q+1)}$.
	\end{customthm}
		
	\begin{proof}
		Note that with the directional KDE form \eqref{Dir_KDE2}, we have that
		$$D^{\tau_j} \hat{f}_h(\bm{x}) = \frac{c_{h,q}(L)}{nh} \sum_{i=1}^n \left(\frac{x_{\tau_j} -X_{\tau_j}}{h} \right) L'\left(\frac{1-\bm{x}^T\bm{X}_i}{h^2} \right),$$
			\[
			D^{[\tau_j,\tau_k]} \hat{f}_h(\bm{x}) = 
			\begin{cases}
			\frac{c_{h,q}(L)}{nh^2} \sum_{i=1}^n \left(\frac{x_{\tau_j} -X_{\tau_j}}{h} \right) \left(\frac{x_{\tau_k} -X_{\tau_k}}{h} \right) L''\left(\frac{1-\bm{x}^T\bm{X}_i}{h^2} \right) & j\neq k,\\
			\frac{c_{h,q}(L)}{nh^2} \sum_{i=1}^n \left(\frac{x_{\tau_j} -X_{\tau_j}}{h} \right)^2 L''\left(\frac{1-\bm{x}^T\bm{X}_i}{h^2} \right) + \frac{c_{h,q}(L)}{nh^2} \sum_{i=1}^n L'\left(\frac{1-\bm{x}^T\bm{X}_i}{h^2} \right) & j=k,
			\end{cases}
			\]
			and
			$$\norm{D^{[\tau]} \hat{f}_h - D^{[\tau]} f}_{\infty} \leq \norm{ \mathbb{E}\left[ D^{[\tau]} \hat{f}_h \right] - D^{[\tau]} f}_{\infty} + \norm{D^{[\tau]} \hat{f}_h - \mathbb{E}\left[ D^{[\tau]} \hat{f}_h \right] }_{\infty}$$
			for $[[\tau]]=0,1,2$. The first term in the preceding display is of order $O(h^2)$ inside the tangent space by Theorem \ref{pw_conv_tang} and the differentiability of $f$ under condition (D1). The proof of $\norm{D^{[\tau]} \hat{f}_h - \mathbb{E}\left[ D^{[\tau]} \hat{f}_h \right]}_{\infty} = O_P\left(\sqrt{\frac{|\log h|}{nh^{q+2[[\tau]]}}} \right)$ follows directly from the argument of Theorem 2.3 in \cite{Gine2002} and the following calculations:
			\begin{align*}
			\mathbb{E}\left[L^2 \left(\frac{1}{2} \norm{\frac{\bm{x}-\bm{X}}{h}}_2^2 \right) \right] &= \int_{\Omega_q} L^2\left(\frac{1-\bm{x}^T\bm{y}}{h^2} \right) \cdot f(\bm{y}) \, \omega_q(d\bm{y})\\
			&\hspace{-10mm} = \int_{-1}^1 \int_{\Omega_{q-1}} L^2\left(\frac{1-t}{h^2} \right) \cdot f\left(t\bm{x} + \sqrt{1-t^2} \bm{B_x} \bm{\xi} \right) (1-t^2)^{\frac{q}{2}-1} \omega_{q-1}(d\bm{\xi}) dt\\
			&\hspace{-10mm} = h^q \int_0^{2h^{-2}} \int_{\Omega_{q-1}} f(\bm{x}+\alpha_{\bm{x},\bm{\xi}})\cdot L^2(r) r^{\frac{q}{2}-1} (2-h^2r)^{\frac{q}{2}-1} \omega_{q-1}(d\bm{\xi}) dt\\
			&\hspace{-10mm} \leq h^q ||f||_{\infty} \omega_{q-1} 2^{\frac{q}{2}-1} \int_0^{\infty} L^2(r) r^{\frac{q}{2}-1} dr,
			\end{align*}
			\begin{align*}
			&\mathbb{E}\left[\left(\frac{x_i-X_i}{h} \right)^2 \left| L'\left(\frac{1}{2} \norm{\frac{\bm{x}-\bm{X}}{h}}_2^2 \right)\right|^2\right] \\
			& \leq \mathbb{E}\left[\norm{\frac{\bm{x}-\bm{X}}{h}}_2^2 \cdot \left| L'\left(\frac{1}{2} \norm{\frac{\bm{x}- \bm{X}}{h}}_2^2 \right) \right|^2 \right]\\
			& = 2\int_{\Omega_q} \left(\frac{1-\bm{x}^T\bm{y}}{h^2}\right) \left| L'\left(\frac{1-\bm{x}^T\bm{y}}{h^2} \right)\right|^2 f(\bm{y}) \, \omega_q(d\bm{y})\\
			& = 2\int_{-1}^1 \int_{\Omega_{q-1}} \left(\frac{1-t}{h^2} \right) \left|L'\left(\frac{1-t}{h^2} \right) \right|^2 f\left(t\bm{x} + \sqrt{1-t^2} \bm{B_x}\bm{\xi} \right) (1-t^2)^{\frac{q}{2}-1} \,\omega_{q-1}(d\bm{\xi}) dt\\
			& = 2h^q \int_0^{2h^{-2}} \int_{\Omega_{q-1}} f(x+\alpha_{\bm{x},\bm{\xi}}) \cdot |L'(r)|^2 \cdot r^{\frac{q}{2}} (2-h^2r)^{\frac{q}{2}-1} \omega_{q-1}(d\bm{\xi}) dr\\
			& \leq 2h^q ||f||_{\infty} \omega_{q-1} 2^{\frac{q}{2}-1} \int_0^{\infty} |L'(r)|^2 r^{\frac{q}{2}} dr,
			\end{align*}
			and
			\begin{align*}
			&\mathbb{E}\left[\max\left\{\left(\frac{x_i-X_i}{h} \right)^4 \left| L''\left(\frac{1}{2} \norm{\frac{\bm{x}-\bm{X}}{h}}_2^2 \right)\right|^2, \left(\frac{x_i-X_i}{h} \right)^2 \left(\frac{x_j-X_j}{h} \right)^2 \left| L''\left(\frac{1}{2} \norm{\frac{\bm{x}-\bm{X}}{h}}_2^2 \right)\right|^2 \right\} \right]\\
			&\leq \mathbb{E} \left[\norm{\frac{\bm{x}-\bm{X}}{h}}_2^4 L''\left(\frac{1}{2}\norm{\frac{\bm{x}-\bm{X}}{h}}_2^2 \right)^2 \right]\\
			&= 4 \int_{\Omega_q} \left(\frac{1-\bm{x}^T\bm{y}}{h^2} \right)^2 \left|L''\left(\frac{1-\bm{x}^T\bm{y}}{h^2} \right) \right|^2 f(\bm{y})\, \omega_q(d\bm{y})\\
			&= 4\int_{-1}^1 \int_{\Omega_{q-1}} \left(\frac{1-t}{h^2} \right)^2 \left|L''\left(\frac{1-t}{h^2} \right) \right|^2 f\left(t\bm{x} + \sqrt{1-t^2} \bm{B_x}\bm{\xi} \right) (1-t^2)^{\frac{q}{2}-1} \,\omega_{q-1}(d\bm{\xi}) dt \\
			&= 4h^q \int_0^{2h^{-2}} \int_{\Omega_{q-1}} f(\bm{x} + \alpha_{\bm{x},\bm{\xi}}) \cdot L''(r)^2 r^{\frac{q}{2}+1} (2-h^2r)^{\frac{q}{2}-1} \, \omega_{q-1}(d\bm{\xi}) dr\\
			&\leq 4h^q ||f||_{\infty} \omega_{q-1} 2^{\frac{q}{2}-1} \int_0^{\infty} r^{\frac{q}{2}+1} L''(r)^2 dr
			\end{align*}
			for $i=1,...,q+1$, where we apply (a) in Lemma \ref{integ_lemma}, the change of variable $r=\frac{1-t}{h^2}$, and $\alpha_{\bm{x},\bm{\xi}}=-rh^2 \bm{x} + h\sqrt{r (2-h^2 r)} \bm{B_x}\bm{\xi}$ in the preceding three displays.
		\end{proof}
		
	\subsection{Proof of Theorem~\ref{Mode_cons}}
	\label{Appendix:Thm6_pf}
	
	\begin{customthm}{6}
	Assume (D1), (D2'), (K1), and (M1-2). For any $\delta \in (0,1)$, when $h$ is sufficiently small and $n$ is sufficiently large,
	\begin{enumerate}[label=(\alph*)]
		\item there must be at least one estimated local mode $\hat{\bm{m}}_k$ within $S_k = \bm{m}_k \oplus \rho_*$ for every $\bm{m}_k \in \mathcal{M}$, and
		\item the collection of estimated modes satisfies $\hat{\mathcal{M}}_n \subset \mathcal{M} \oplus \rho_*$ and there is a unique estimated local mode $\hat{\bm{m}}_k$ within $S_k=\bm{m}_k\oplus \rho_*$
	\end{enumerate}
	with probability at least $1-\delta$. In total, when $h$ is sufficiently small and $n$ is sufficiently large, there exist some constants $A_3, B_3 >0$ such that
	$$\mathbb{P}\left(\hat{K}_n \neq K \right) \leq B_3 e^{-A_3nh^{q+4}}.$$
	\begin{enumerate}[label=(c)]
		\item The Hausdorff distance between the collection of local modes and its estimator satisfies $$\Haus\left(\mathcal{M},\hat{\mathcal{M}}_n \right) = O(h^2) + O_P\left(\sqrt{\frac{1}{nh^{q+2}}} \right),$$
		as $h\to 0$ and $nh^{q+2} \to \infty$.
	\end{enumerate}
	\end{customthm}

	\begin{proof}
	The proof is partially adopted from the proof of Theorem 1 in \cite{Mode_clu2016}. \\
	
	\noindent {\bf Statement (a).} Without loss of generality, we consider the local mode $\bm{m}_k$ and the set $S_k=\left\{\bm{x}\in \Omega_q: \norm{\bm{x} - \bm{m}_k}\leq \rho_* \right\}$. With condition (D1), we can apply the Taylor's expansion on the exponential map $\Exp_{\bm{m}_k}: D_{\epsilon} \subset T_{\bm{m}_k}(\Omega_q) \to \Omega_q$ with $\Exp_{\bm{m}_k}(0)=\bm{m}_k$, where $D_{\epsilon}$ is a disk of radius $\epsilon$ in $T_{\bm{m}_k}(\Omega_q)$ with center in the origin and $\epsilon > \arccos\left(1- \frac{\rho_*^2}{2} \right)$. Here, $\arccos\left(1- \frac{\rho_*^2}{2} \right)$ is the geodesic distance from the center $\bm{m}_k$ to $\partial S_k$ on $\Omega_q$, where $\partial S_k = \{\bm{x}\in \Omega_q: ||\bm{x}-\bm{m}_k||_2 = \rho_* \}$ is the boundary of $S_k$. With (M1) and the fact that the third order partial derivatives of $f$ are upper bounded by $C_3$, 
			\begin{align}
			\label{bound_boundary_mode}
			\begin{split}
			\sup_{\bm{x}\in \partial S_k} f(\bm{x}) &\leq \sup_{\bm{x}\in \partial S_k} \Bigg[f(\bm{m}_k) + \left[\grad f(\bm{m}_k)\right]^T \Exp_{\bm{m}_k}^{-1}(\bm{x}) \\
			&\quad + \frac{1}{2} \Exp_{\bm{m}_k}^{-1}(\bm{x})^T \left(\mathcal{H}_{\bm{m}_k} f \right) \Exp_{\bm{m}_k}^{-1}(\bm{x}) + \frac{C_3}{6} ||\Exp_{\bm{m}_k}^{-1}(\bm{x})||_2^3 \Bigg]\\
			&\leq f(\bm{m}_k) - \frac{\lambda_*}{2} \left(\frac{3\lambda_*}{2C_3} \right)^2 + \frac{C_3}{6} \left(\frac{3\lambda_*}{2C_3} \right)^3 = f(\bm{m}_k) - \frac{9\lambda_*^3}{8C_3^2},
			\end{split}
			\end{align}
			where recall that $\Exp_{\bm{m}_k}^{-1}(\bm{x}) \in T_{\bm{m}_k}(\Omega_q)$ is in the direction from $\bm{m}_k$ to $\bm{x}$ with the length equal to the great-circle (or geodesic) distance on $\Omega_q$. (We indeed apply the Cauchy-Schwarz inequality implicitly to obtain the first inequality in \eqref{bound_boundary_mode}.) Then, by the uniform consistency of $\hat{f}_h$ (Theorem~\ref{unif_conv_tang}), when $h$ is sufficiently small and $\frac{nh^q}{|\log h|}$ is large enough,
			\begin{equation}
			\label{mode_consist_cond1}
			\norm{\hat{f}_h - f}_{\infty} < \frac{9\lambda_*}{16C_3^2}
			\end{equation}
			with probability at least $1-\delta$ for any $0<\delta <1$. We thus conclude that there must be at least one estimated local mode $\hat{\bm{m}}_k$ within $S_k$. (If, on the contrary, there exists no $\hat{\bm{m}}_k \in \hat{\mathcal{M}}_n$ within $S_k$, then the maximum of $\hat{f}_h$ on $S_k$ is attained at the boundary $\partial S_k$, that is, $\max_{\bm{x}\in S_k} \hat{f}_h(\bm{x}) = \max_{\bm{x}\in \partial S_k} \hat{f}_h(\bm{x})$. However, $\max_{\bm{x}\in \partial S_k} \hat{f}_h(\bm{x}) < \max_{\bm{x}\in \partial S_k} f(\bm{x}) + \frac{9\lambda_*}{16C_3^2} \leq f(\bm{m}_k) - \frac{9\lambda_*}{16C_3^2} < \hat{f}_h(\hat{\bm{m}}_k)$ by (\ref{bound_boundary_mode}), contradiction.) Note that this argument can be generalized to each $k=1,...,K$.\\
			
			\noindent {\bf Statement (b).} With (M2), we know that whenever
			\begin{align}
			\label{mode_consist_cond2}
			\begin{split}
			\sup_{\bm{x} \in \Omega_q} \norm{\grad \hat{f}_h(\bm{x}) - \grad f(\bm{x})}_{\max} &=\sup_{\bm{x} \in \Omega_q} \norm{\Tang\left(\nabla \hat{f}_h(\bm{x}) \right) - \Tang\left(\nabla f(\bm{x}) \right)}_{\max} \leq \Theta_1,\\
			\sup_{\bm{x} \in \Omega_q} \norm{\mathcal{H} \hat{f}_h(\bm{x}) -\mathcal{H} f(\bm{x})}_{\max} &\leq \Theta_2
			\end{split}
			\end{align}
			for some $\Theta_2>0$, the followings hold simultaneously:
			\begin{enumerate}[label=(\roman*)]
				\item $\norm{\grad f(\hat{\bm{m}}_k)}_{\max} = \norm{\grad f(\hat{\bm{m}}_k) - \underbrace{\grad\hat{f}_h(\hat{\bm{m}}_k)}_{=0}}_{\max} \leq \Theta_1$,
				\item $\sup_{\bm{x}\in S_k} \lambda_1\left(\mathcal{H} \hat{f}_h(\bm{x}) \right) <0$ and $\lambda_1\left(\mathcal{H} \hat{f}_h(\hat{\bm{m}}_k) \right) \leq -\frac{\lambda_*}{2} - (q+1)\Theta_2$ by choosing $\Theta_2>0$ properly,
				\item and 
				\begin{align*}
				\lambda_1\left(\mathcal{H} f(\hat{\bm{m}}_k) \right) &\leq \lambda_1\left(\mathcal{H} \hat{f}_h(\hat{\bm{m}}_k) \right) + \lambda_{q-1}\left(\mathcal{H} f(\hat{\bm{m}}_k) - \mathcal{H} \hat{f}_h(\hat{\bm{m}}_k) \right) \\
				&\leq -\frac{\lambda_*}{2} - (q+1)\Theta_2 + (q+1)\Theta_2 = -\frac{\lambda_*}{2}
				\end{align*}
				by Weyl's theorem (Theorem 4.3.1 in \citealt{HJ2012}) and the fact that 
				\begin{align*}
				\left|\lambda_{q-1}(\mathcal{H} f(\hat{\bm{m}}_k) -\mathcal{H} \hat{f}_h(\hat{\bm{m}}_k) \right| &\leq \sup_{\norm{v}_2=1} \norm{\left[\mathcal{H} f(\hat{\bm{m}}_k) -\mathcal{H} \hat{f}_h(\hat{\bm{m}}_k) \right] v}_2\\
				&\leq \sqrt{(q+1)\times (q+1)} \norm{\mathcal{H} f(\hat{\bm{m}}_k) -\mathcal{H} \hat{f}_h(\hat{\bm{m}}_k)}_{\max} \\
				&\leq (q+1)\Theta_2.
				\end{align*}
				See Section 3.3 in \cite{Non_ridge_est2014} for detailed relations between different types of matrix norms.
			\end{enumerate}
		Notice that (ii) is true because $\lambda_1(\bm{m}_k)\leq \lambda_*$ by (M1) and the difference between $\mathcal{H}\hat{f}_h$ and $\mathcal{H}f$ will be minute given a small $\Theta_2$. By (i) and (iii), we conclude that $\hat{\mathcal{M}}_n \subset \mathcal{M} \oplus \rho_*$. By (ii) and Lemma 3.2 in \cite{Morse_Homology2004}, there is only one estimated local mode $\hat{\bm{m}}_k$ within $S_k$. They both hold with probability at least $1-\delta$ for any $\delta \in(0,1)$.\\
			
			In total, a sufficient condition for the number of true local modes and estimated local modes being the same is a combination of the inequalities in \eqref{mode_consist_cond1} and \eqref{mode_consist_cond2}. That is,
			\begin{align}
			\label{mode_consist_cond_com}
			\begin{split}
			\norm{\hat{f}_h - f}_{\infty} &< \frac{9\lambda_*}{16C_3^2}\\
			\norm{\Tang\left(\nabla \hat{f}_h \right) - \Tang\left(\nabla f \right)}_{\max,\infty} &\leq \Theta_1\\
			\norm{\hat{\mathcal{H}} f -\mathcal{H} f}_{\max,\infty} &\leq \Theta_2.
			\end{split}
			\end{align}
			By bias bounds in Theorem~\ref{pw_conv_tang} or Theorem~\ref{unif_conv_tang}, as $h$ is sufficiently small, we have
			\begin{align*}
			\norm{\mathbb{E}\left[\hat{f}_h \right] - f}_{\infty} < \frac{9\lambda_*}{32C_3^2}, &\quad
			\sup_{\bm{x} \in \Omega_q} \norm{\mathbb{E}\left[\grad \hat{f}_h(\bm{x}) \right] - \grad f(\bm{x})}_{\max} \leq \frac{\Theta_1}{2},\\
			& \text{ and }
			\sup_{\bm{x} \in \Omega_q} \norm{\mathbb{E}\left[\mathcal{H} \hat{f}_h(\bm{x})\right] -\mathcal{H} f(\bm{x})}_{\max} \leq \frac{\Theta_2}{2}.
			\end{align*}
			Therefore, (\ref{mode_consist_cond_com}) holds whenever
			\begin{align}
			\label{mode_consist_cond_com2}
			\begin{split}
			\norm{\hat{f}_h - \mathbb{E}\left[\hat{f}_h\right]}_{\infty} &< \frac{9\lambda_*}{32C_3^2}\\
			\sup_{\bm{x} \in \Omega_q} \norm{\grad \hat{f}_h(\bm{x}) - \mathbb{E}\left[\grad \hat{f}_h(\bm{x}) \right]}_{\max} &\leq \frac{\Theta_1}{2}\\
			\sup_{\bm{x} \in \Omega_q} \norm{\mathcal{H} \hat{f}_h(\bm{x}) -\mathbb{E}\left[\mathcal{H} \hat{f}_h(\bm{x})\right]}_{\max} &\leq \frac{\Theta_2}{2}
			\end{split}
			\end{align}
			and $h$ is sufficiently small. Now applying Talagrand's inequality \cite{Talagrand1996, Gine2002}, there exist constants $A_0,A_1,A_2>0$ and $B_0,B_1,B_2>0$ such that when $n$ is large enough,
			\begin{align}
			\label{Talagrand_KDE}
			\begin{split}
			\mathbb{P}\left(\norm{\hat{f}_h - \mathbb{E}\left[\hat{f}_h\right]}_{\infty} \geq \epsilon \right) &\leq B_0 e^{-A_0\epsilon^2 nh^q}\\
			\mathbb{P}\left(\sup_{\bm{x} \in \Omega_q} \norm{\grad \hat{f}_h(\bm{x}) - \mathbb{E}\left[\grad \hat{f}_h(\bm{x}) \right]}_{\max} \geq \epsilon \right) &\leq B_1 e^{-A_1\epsilon^2 nh^{q+2}}\\
			\mathbb{P}\left(\sup_{\bm{x} \in \Omega_q} \norm{\mathcal{H} \hat{f}_h(\bm{x}) -\mathbb{E}\left[\mathcal{H} \hat{f}_h(\bm{x})\right]}_{\max} \geq \epsilon \right) &\leq B_2 e^{-A_2\epsilon^2 nh^{q+4}}.
			\end{split}
			\end{align}
			Combining (\ref{mode_consist_cond_com2}) and (\ref{Talagrand_KDE}), we conclude that there exist some constants $A_3,B_3>0$ such that 
			$$\mathbb{P}\Big(\text{(\ref{mode_consist_cond_com}) holds}\, \Big) > 1 -B_3e^{-A_3nh^{q+4}}$$
			when $h$ is sufficiently small. Since the condition (\ref{mode_consist_cond_com}) implies $\hat{K}_n = K$, we conclude that 
			$$\mathbb{P}\left(\hat{K}_n \neq K \right) \leq B_3 e^{-A_3nh^{q+4}}$$
			for some constants $A_3,B_3>0$ as $h$ is sufficiently small. This proves the so-called modal consistency.\\
			
			\noindent {\bf Statement (c).} To establish the convergence rate of the Hausdorff distance between $\hat{\mathcal{M}}_n$ and $\mathcal{M}$, we assume that (\ref{mode_consist_cond_com}) holds so that $K = \hat{K}_n$ and each local mode is approximating by an unique estimated local mode. Notice that $\norm{\bm{m}_k -\hat{\bm{m}}_k}_2$ is upper bounded by the great-circle distance between these two points. Then,
			\begin{align}
			\label{Tang_grad_expand}
			\begin{split}
			&\grad f(\hat{\bm{m}}_k) \\
			&= \Tang\left(\nabla f(\hat{\bm{m}}_k) \right) - \underbrace{\Tang\left(\nabla f(\bm{m}_k) \right)}_{=0}\\
			&= \nabla \Tang\left(\nabla f(\bm{m}_k) \right) \cdot \Exp_{\bm{m}_k}^{-1}(\hat{\bm{m}}_k) + o\left(\norm{\Exp_{\bm{m}_k}^{-1}(\hat{\bm{m}}_k)}_2 \right)\\
			&= \left[(I_{q+1}-\bm{m}_k\bm{m}_k^T)\nabla\nabla f(\bm{m}_k) - \bm{m}_k^T \nabla f(\bm{m}_k) I_{q+1} -\bm{m}_k \nabla f(\bm{m}_k)^T \right] \cdot \Exp_{\bm{m}_k}^{-1}(\hat{\bm{m}}_k)\\
			&\quad  + o\left(\norm{\Exp_{\bm{m}_k}^{-1}(\hat{\bm{m}}_k)}_2 \right)\\
			&= \left[\mathcal{H} f(\bm{m}_k) \right] \Exp_{\bm{m}_k}^{-1}(\hat{\bm{m}}_k) + o\left(\norm{\Exp_{\bm{m}_k}^{-1}(\hat{\bm{m}}_k)}_2 \right),
			\end{split}
			\end{align}
			because $\nabla f(\bm{m}_k) = \norm{\nabla f(\bm{m}_k)}_2 \cdot \bm{m}_k$ when $\bm{m}_k$ is a local mode and $\Exp_{\bm{m}_k}^{-1}(\hat{\bm{m}}_k)\in T_{\bm{m}_k}$ is orthogonal to $\bm{m}_k$. Under (M1), the matrices $\mathcal{H} f{\bm{m}_k}$ are nonsingular for all $\bm{m}_k \in \mathcal{M}$ inside the tangent space $T_{\bm{m}_k}$, respectively. As the chord distance between two points on $\Omega_q$ is bounded by their great-circle distance, we multiply $\left[\mathcal{H} f(\bm{m}_k) \right]^{-1}$ on both sides of \eqref{Tang_grad_expand} and obtain that 
			$$\norm{\hat{\bm{m}}_k - \bm{m}_k}_2 \leq \norm{\Exp_{\bm{m}_k}^{-1}(\hat{\bm{m}}_k)}_2 = \left[\mathcal{H} f(\bm{m}_k) \right]^{-1} \grad f(\hat{\bm{m}}_k) + o\left(\norm{\Exp_{\bm{m}_k}^{-1}(\hat{\bm{m}}_k)}_2 \right),$$
			where the matrix inverse, strictly speaking, is taken with respect to the local coordinate system near $\bm{m}_k$. Note that $\norm{\left[\mathcal{H} f(\bm{m}_k) \right]^{-1}}_{\max}$ is bounded within $T_{\bm{m}_k}$ for all $\bm{m}_k \in \mathcal{M}$ under the assumption (\ref{mode_consist_cond_com}). Moreover, by Theorem \ref{pw_conv_tang},
			\begin{align*}
			\grad f(\hat{\bm{m}}_k) &= \grad f(\hat{\bm{m}}_k) - \underbrace{\grad \hat{f}_h(\hat{\bm{m}}_k) }_{=0}\\
			&= O(h^2) + O_P\left(\sqrt{\frac{1}{nh^{q+2}}} \right).
			\end{align*}
			Now applying this rate of convergence to each local mode and using the fact that
			$$\Haus\left(\hat{\mathcal{M}}_n,\mathcal{M} \right) = \max_{k=1,...,K} ||\hat{\bm{m}}_k -\bm{m}_k||_2,$$
			we obtain the final conclusion.
		\end{proof}
	
	\subsection{Proofs of Theorem~\ref{MS_asc}, Lemma~\ref{Rad_grad}, and Theorem~\ref{MS_conv}}
	\label{Appendix:Thm8_10_11_pf}
	
	\begin{customthm}{8}[Ascending Property]
	If kernel $L:[0,\infty) \to [0,\infty)$ is monotonically decreasing, differentiable, and convex with $L(0)<\infty$, then the sequence $\left\{\hat{f}_h(\hat{\bm{y}}_s) \right\}_{s=0}^{\infty}$ is monotonically increasing and thus converges.
	\end{customthm}
		
	\begin{proof}
		Obviously, the sequence $\left\{\hat{f}_h(\hat{\bm{y}}_s) \right\}_{s=0,1,...}$ is bounded if the kernel function $L$ is monotonically decreasing with $L(0) <\infty$. Hence, it suffices to show that it is monotonically increasing. The convexity and differentiability of kernel $L$ imply that 
		\begin{equation}
			\label{conx_prop}
			L(x_2) -L(x_1) \geq L'(x_1) \cdot (x_2-x_1)
		\end{equation}
		for all $x_1,x_2 \in [0,\infty), x_1\neq x_2$. Using the fact that (rearrangement from Algorithm~\ref{Algo:MS})
		$$\sum\limits_{i=1}^n \bm{X}_i L'\left(\frac{1-\hat{\bm{y}}_s^T\bm{X}_i}{h^2} \right) = -\hat{\bm{y}}_{s+1} \norm{\sum\limits_{i=1}^n \bm{X}_i L'\left(\frac{1-\bm{y}_s^T\bm{X}_i}{h^2} \right)}_2$$ 
			we have that
			\begin{align}
			\label{conv_inequ_fh}
			\begin{split}
			\hat{f}_h(\bm{y}_{s+1}) -\hat{f}_h(\bm{y}_s) &= \frac{c_{h,q}(L)}{n} \sum_{i=1}^n \left[L\left(\frac{1-\bm{y}_{s+1}^T\bm{X}_i}{h^2} \right) - L\left(\frac{1-\bm{y}_s^T\bm{X}_i}{h^2} \right) \right]\\
			&\geq \frac{c_{h,q}(L)}{nh^2} \sum_{i=1}^n L'\left(\frac{1-\bm{y}_s^T\bm{X}_i}{h^2} \right) \cdot (\bm{y}_s -\bm{y}_{s+1})^T \bm{X}_i \\
			&= \frac{c_{h,q}(L)}{nh^2} \cdot (\bm{y}_{s+1}-\bm{y}_s)^T \bm{y}_{s+1} \cdot \norm{\sum\limits_{i=1}^n \bm{X}_i L'\left(\frac{1-\bm{y}_s^T\bm{X}_i}{h^2} \right)}_2\\
			&= \frac{c_{h,q}(L)}{2nh^2} \norm{\bm{y}_{s+1} -\bm{y}_s}_2^2 \cdot \norm{\sum\limits_{i=1}^n \bm{X}_i L'\left(\frac{1-\bm{y}_s^T\bm{X}_i}{h^2} \right)}_2\\
			&\geq 0,
			\end{split}
			\end{align}
			where we use the fact that $2(\bm{y}_{s+1}-\bm{y}_s)^T \bm{y}_{s+1} = 2- 2\bm{y}_s^T \bm{y}_{s+1} = \norm{\bm{y}_{s+1} -\bm{y}_s}_2^2$ between the third and fourth lines, given that $\norm{\bm{y}_s}_2=\norm{\bm{y}_{s+1}}_2=1$. 
		\end{proof}
	
	\begin{customlem}{10}
		Assume conditions (D1) and (D2'). For any fixed $\bm{x} \in \Omega_q$, we have
		$$h^2 \cdot \Rad\left(\nabla \hat{f}_h(\bm{x}) \right) \asymp h^2 \cdot \nabla\hat{f}_h(\bm{x})  = \bm{x} f(\bm{x}) C_{L,q} + o\left(1 \right) + O_P\left(\sqrt{\frac{1}{nh^q}} \right)$$
		as $nh^q \to \infty$ and $h\to 0$, where $C_{L,q}=-\frac{\int_0^{\infty} L'(r) r^{\frac{q}{2}-1} dr}{\int_0^{\infty} L(r) r^{\frac{q}{2}-1} dr} > 0$ is a constant depending only on kernel $L$ and dimension $q$ and ``$\asymp$'' stands for an asymptotic equivalence.
	\end{customlem}
	
    \begin{proof}
    	The proof follows the same logic as the one for Theorem~\ref{pw_conv_tang}. Note that
    	\begin{equation}
    	\label{rad_decomp}
    	\nabla \hat{f}_h(\bm{x}) = \mathbb{E}\left[\nabla \hat{f}_h(\bm{x}) \right] + \nabla \hat{f}_h(\bm{x}) - \mathbb{E}\left[\nabla \hat{f}_h(\bm{x}) \right].
    	\end{equation}
    	Recall that $\nabla \hat{f}_h(\bm{x}) = -\frac{c_{h,q}(L)}{nh^2} \sum\limits_{i=1}^n \bm{X}_i L'\left(\frac{1-\bm{x}^T\bm{X}_i}{h^2} \right)$. The expectation of $\nabla \hat{f}_h(\bm{x})$ is
    	\begin{align}
    	\label{f_hat_grad_expected}
    	\begin{split}
    	\mathbb{E}\left[\nabla \hat{f}_h(\bm{x}) \right] &= \frac{c_{h,q}(L)}{h^2} \int_{\Omega_q} (-\bm{y}) L'\left(\frac{1-\bm{x}^T\bm{y}}{h^2} \right) f(\bm{y})\, \omega_q(d\bm{y}) \\
    	&= \frac{c_{h,q}(L)}{h^2} \int_{-1}^1 \int_{\Omega_{q-1}} \left(-t\bm{x} -\sqrt{1-t^2} \bm{B_x}\bm{\xi} \right) L'\left(\frac{1-t}{h^2} \right)\\ 
    	&\quad \times f\left(-t\bm{x} -\sqrt{1-t^2} \bm{B_x}\bm{\xi} \right) (1-t^2)^{\frac{q}{2}-1} \omega_{q-1}(d\bm{\xi}) dt\\
    	&= c_{h,q}(L) h^{q-2} \int_0^{2h^{-2}} \int_{\Omega_{q-1}} (-\bm{x}-\alpha_{\bm{x},\bm{\xi}}) \cdot L'(r)\\
    	&\quad \times f(\bm{x}+\alpha_{\bm{x},\bm{\xi}}) \cdot r^{\frac{q}{2}-1} (2-h^2r)^{\frac{q}{2}-1} \omega_{q-1}(d\bm{\xi}) dr
    	\end{split}
    	\end{align}
    	by (a) in Lemma~\ref{integ_lemma} and a change of variable $r=\frac{1-t}{h^2}$, where $\alpha_{\bm{x},\bm{\xi}} = -rh^2\bm{x} + h\sqrt{r(2-h^2r)} \bm{B_x} \bm{\xi}$. By condition (D1), the first-order Taylor's expansion of $f$ at $\bm{x}\in \Omega_q$ is
    	$$f(\bm{x}+ \alpha_{\bm{x},\bm{\xi}}) = f(\bm{x}) + O\left(\norm{\alpha_{\bm{x},\bm{\xi}}}_2 \right),$$
    	where $\norm{\alpha_{\bm{x},\bm{\xi}}}_2^2 = 2rh^2$ by the orthogonality of $\bm{x}$ and columns of $\bm{B_x}$. Now we plug it back into \eqref{f_hat_grad_expected} respectively to compute the dominating term of $\mathbb{E}\left[\nabla \hat{f}_h(\bm{x}) \right]$.
    \begin{align*}
    &\mathbb{E}\left[\nabla \hat{f}_h(\bm{x}) \right]\\ 
    &= -c_{h,q}(L) h^{q-2} \bm{x} f(\bm{x}) \int_0^{2h^{-2}} \int_{\Omega_{q-1}} L'(r) r^{\frac{q}{2}-1} (2-h^2r)^{\frac{q}{2}-1} \omega_{q-1}(d\bm{\xi})dr \\
    &\quad - c_{h,q}(L) h^{q-2} f(\bm{x}) \int_0^{2h^{-2}} \int_{\Omega_{q-1}} \alpha_{\bm{x},\bm{\xi}} L'(r) r^{\frac{q}{2}-1} (2-h^2r)^{\frac{q}{2}-1} \omega_{q-1}(d\bm{\xi})dr\\
    & \quad + O(h) \cdot c_{h,q}(L) h^{q-2} \int_0^{2h^{-2}} \int_{\Omega_{q-1}} (-\bm{x}-\alpha_{\bm{x},\bm{\xi}}) L'(r) r^{\frac{q}{2}-1} (2-h^2r)^{\frac{q}{2}-1} \omega_{q-1}(d\bm{\xi}) dr\\
    &\stackrel{\text{(i)}}{=} -c_{h,q}(L) h^{q-2} \bm{x} f(\bm{x}) \cdot \omega_{q-1} \int_0^{2h^{-2}} L'(r) \cdot  r^{\frac{q}{2}-1} (2-h^2r)^{\frac{q}{2}-1} \, dr\\
    &\quad + c_{h,q}(L) h^q \bm{x} f(\bm{x}) \cdot \omega_{q-1} \int_0^{2h^{-2}} L'(r) \cdot  r^{\frac{q}{2}} (2-h^2r)^{\frac{q}{2}-1} \, dr\\
    &\quad - c_{h,q}(L) h^{q-1} f(\bm{x}) \int_0^{2h^{-2}} \int_{\Omega_{q-1}} \bm{B_x}\bm{\xi} \cdot  L'(r) r^{\frac{q-1}{2}} (2-h^2r)^{\frac{q-1}{2}} \omega_{q-1}(d\bm{\xi})dr + O(h^{-1}),\\
    &\stackrel{\text{(ii)}}{=} -c_{h,q}(L) h^{q-2} \bm{x} f(\bm{x}) \cdot \omega_{q-1} \int_0^{2h^{-2}} L'(r) \cdot  r^{\frac{q}{2}-1} (2-h^2r)^{\frac{q}{2}-1} \, dr + O(1) + 0 + O(h^{-1}) \\
    &\stackrel{\text{(iii)}}{=} -\frac{\bm{x} f(\bm{x})}{h^2} \cdot \frac{\int_0^{\infty} L'(r) r^{\frac{q}{2}-1} dr}{\int_0^{\infty} L(r) r^{\frac{q}{2}-1} dr} + o(h^{-2})\\
    &\equiv C_{L,q} \cdot \bm{x} f(\bm{x}) h^{-2} + o(h^{-2}),
    \end{align*}
    where we use (b) of Lemma \ref{integ_lemma} and the fact that $\bm{B_x}\bm{\xi} = \sum_{i=1}^q \xi_i\bm{b}_i$ in (ii), and apply condition (D2') and \eqref{asym_norm_const} to argue that
    \begin{equation}
    \label{Asymp_relation}
    c_{h,q}(L) h^q \int_0^{2h^{-2}} \int_{\Omega_{q-1}} L'(r) \cdot \phi(r,\bm{\xi})\,\, \omega_{q-1}(d\bm{\xi}) dr \asymp O(1)
    \end{equation}
    in both (i) and (ii). We also use the asymptotic relation \eqref{asym_norm_const} in (iii) and denote $C_{L,q} = -\frac{\int_0^{\infty} L'(r) r^{\frac{q}{2}-1} dr}{\int_0^{\infty} L(r) r^{\frac{q}{2}-1} dr}$ in the last equality. Thus, 
    	$$\mathbb{E}\left[\Rad\left(\nabla \hat{f}_h(\bm{x}) \right) \right] = \bm{x}\bm{x}^T \mathbb{E}\left[\nabla \hat{f}_h(\bm{x}) \right] = C_{L,q} \cdot \bm{x} f(\bm{x}) h^{-2} + o(h^{-2}).$$
    	Based on the asymptotic rate of $\mathbb{E}\left[\nabla \hat{f}_h(\bm{x}) \right]$, we calculate the covariance matrix of $\nabla \hat{f}_h(\bm{x})$ as
    	
    	\begin{align*}
    	\text{Cov}\left[\nabla \hat{f}_h(\bm{x}) \right] &= \frac{c_{h,q}(L)^2}{nh^4} \cdot \text{Cov}\left[\bm{X}_1 \cdot L'\left(\frac{1-\bm{x}^T\bm{X}_1}{h^2} \right) \right]\\
    	&= \frac{c_{h,q}(L)^2}{nh^4} \cdot \mathbb{E}\left[\bm{X}_1\bm{X}_1^T \cdot L'\left(\frac{1-\bm{x}^T\bm{X}_1}{h^2} \right)^2 \right] - \frac{1}{n} \cdot \mathbb{E}\left[\nabla \hat{f}_h(\bm{x}) \right] \mathbb{E}\left[\nabla \hat{f}_h(\bm{x}) \right]^T\\
    	&= \frac{c_{h,q}(L)^2}{nh^4} \int_{\Omega_q} \bm{y}\bm{y}^T L'\left(\frac{1-\bm{x}^T\bm{y}}{h^2} \right)^2 f(\bm{y})\, \omega_q(d\bm{y}) + O\left(\frac{1}{nh^4} \right)\\
    	&= \frac{c_{h,q}(L)^2}{n} h^{q-4} \int_0^{2h^{-2}} \int_{\Omega_{q-1}} (\bm{x} +\alpha_{\bm{x},\bm{\xi}}) (\bm{x} + \alpha_{\bm{x},\bm{\xi}})^T L'(r)^2\\
    	&\quad \times f(\bm{x}+\alpha_{\bm{x},\bm{\xi}}) r^{\frac{q}{2}-1} (2-h^2r)^{\frac{q}{2}-1} \omega_{q-1}(d\bm{\xi}) dr + O\left(\frac{1}{nh^4} \right).
    	\end{align*}
    	With condition (D1), we carry out the first-order Taylor's expansion of $f$ at $\bm{x}\in \Omega_q$ as
    	$$f(\bm{x}+ \alpha_{\bm{x},\bm{\xi}}) = f(\bm{x}) + O\left(\norm{\alpha_{\bm{x},\bm{\xi}}}_2 \right)=f(\bm{x}) + O(h).$$
    	Therefore, 
    	\begin{align*}
    	\text{Cov}\left[\nabla \hat{f}_h(\bm{x}) \right] &= \frac{c_{h,q}(L)^2}{n} h^{q-4} \bm{x}\bm{x}^T f(\bm{x}) \omega_{q-1} \int_0^{2h^{-2}} L'(r)^2 r^{\frac{q}{2}-1} (2-h^2r)^{\frac{q}{2}-1} dr + o\left(\frac{1}{nh^{q+4}} \right)\\
    	&= \frac{\bm{x}\bm{x}^T f(\bm{x})}{nh^{q+4}} \cdot \frac{\int_0^{\infty} L'(r)^2 r^{\frac{q}{2}-1}dr}{\omega_{q-1} 2^{\frac{q}{2}-1} \left(\int_0^{\infty} L(r) r^{\frac{q}{2}-1}dr \right)^2} + o\left(\frac{1}{nh^{q+4}} \right), 
    	\end{align*}
    	where we use (b) of Lemma~\ref{integ_lemma}, asymptotic rate \eqref{asym_norm_const}, and \eqref{Asymp_relation} to absorb some higher order terms into $o\left(\frac{1}{nh^{q+4}} \right)$. The dominating term of $\text{Cov}\left[\nabla \hat{f}_h(\bm{x}) \right]$ is in the radial direction, so by the central limit theorem,
    	\begin{align*}
    	&\Rad\left(\nabla \hat{f}_h(\bm{x}) \right) - \mathbb{E}\left[\Rad\left(\nabla \hat{f}_h(\bm{x}) \right) \right] \\
    	&\asymp \nabla\hat{f}_h(\bm{x}) - \mathbb{E}\left[\nabla\hat{f}_h(\bm{x}) \right]\\
    	&= \text{Cov}\left[\nabla \hat{f}_h(\bm{x}) \right]^{\frac{1}{2}} \cdot \text{Cov}\left[\nabla \hat{f}_h(\bm{x}) \right]^{-\frac{1}{2}} \left\{\nabla \hat{f}_h(\bm{x}) - \mathbb{E}\left[\nabla \hat{f}_h(\bm{x}) \right] \right\}\\
    	&= \left[\frac{\bm{x}\bm{x}^T f(\bm{x})}{nh^{q+4}} \cdot \frac{\int_0^{\infty} L'(r)^2 r^{\frac{q}{2}-1}dr}{\omega_{q-1} 2^{\frac{q}{2}-1} \left(\int_0^{\infty} L(r) r^{\frac{q}{2}-1}dr \right)^2} + o\left(\frac{1}{nh^{q+4}} \right) \right]^{\frac{1}{2}} \cdot \hat{\bm{Z}}_n(\bm{x})\\
    	&= O_P\left(\sqrt{\frac{1}{nh^{q+4}}} \right),
    	\end{align*}
    	where $\hat{\bm{Z}}_n(\bm{x}) \stackrel{d}{\to} N_{q+1}(\bm{0}, I_{q+1})$. In total, we conclude with \eqref{rad_decomp} that 
    	$$\Rad\left(h^2\nabla \hat{f}_h(\bm{x}) \right) \asymp h^2 \nabla\hat{f}_h(\bm{x}) = \bm{x} f(\bm{x}) C_{L,q} + o\left(1\right) + O_P\left(\sqrt{\frac{1}{nh^q}} \right)$$
    	for any fixed $\bm{x}\in \Omega_q$, as $h\to 0$ and $nh^q \to \infty$.
    \end{proof}
		
	Before we prove Theorem~\ref{MS_conv}, we first note the following useful result.
		
	\begin{proposition}
	\label{Mode_Dir}
	Assume (C1) and the conditions on the kernel $L$ in Theorem \ref{MS_asc}. 
%		That is, $L:[0,\infty) \to [0,\infty)$ is monotonically decreasing, differentiable, and convex with $L(0)<\infty$. 
	Then for any mode $\hat{\bm{m}}_k \in \hat{\mathcal{M}}_n$ satisfying (C2), we have that $\hat{\bm{m}}_k^T \nabla \hat{f}_h(\hat{\bm{m}}_k) >0$. 
	\end{proposition}

	\begin{proof}
		Suppose, on the contrary, that $\hat{\bm{m}}_k^T \nabla \hat{f}_h(\hat{\bm{m}}_k) <0$. By the definition of a local mode $\hat{\bm{m}}_k$ of $\hat{f}_h$ on $\Omega_q$, we know that $\norm{\grad \hat{f}_h(\hat{\bm{m}}_k)}_2=\norm{\Tang\left(\nabla \hat{f}_h(\hat{\bm{m}}_k) \right)}_2 = 0$. Then 
		$$\frac{\nabla \hat{f}_h(\hat{\bm{m}}_k)}{\norm{\nabla \hat{f}_h(\hat{\bm{m}}_k)}_2} = -\hat{\bm{m}}_k$$
			and by (C1), there exist a $\hat{r}_k \in (0, 2]$ such that $\Tang\left(\nabla \hat{f}_h(\bm{y}) \right) \neq 0$ and $\hat{f}_h(\bm{y}) \leq \hat{f}_h(\hat{\bm{m}}_k)$ for any $\bm{y} \in \left\{\bm{z} \in \Omega_q: \bm{z}^T \hat{\bm{m}}_k \geq 1-\frac{\hat{r}_k^2}{2} \right\} \setminus \left\{\hat{\bm{m}}_k \right\} = \left\{\bm{z} \in \Omega_q: \norm{\bm{z} - \hat{\bm{m}}_k}_2 \leq \hat{r}_k \right\} \setminus \left\{\hat{\bm{m}}_k \right\}$. That is, $\hat{\bm{m}}_k$ is the unique mode inside its neighborhood $\left\{\bm{z} \in \Omega_q: \norm{\bm{z} - \hat{\bm{m}}_k}_2 \leq \hat{r}_k \right\}$. Since the sum of convex functions is convex, $\hat{f}_h$ is indeed convex and we deduce that when $\bm{y} \in \left\{\bm{z} \in \Omega_q: \norm{\bm{z} - \hat{\bm{m}}_k}_2 \leq \hat{r}_k \right\} \setminus \left\{\hat{\bm{m}}_k \right\}$,
			\begin{align*}
			\hat{f}_h(\hat{\bm{m}}_k) - \hat{f}_h(\bm{y}) &\leq \frac{c_{h,q}(L)}{nh^2} \sum_{i=1}^n L'\left(\frac{1-\hat{\bm{m}}_k^T \bm{X}_i}{h^2} \right) \bm{X}_i^T(\bm{y} -\hat{\bm{m}}_k)\\
			&= \norm{ \nabla \hat{f}_h(\hat{\bm{m}}_k)}_2 \cdot (-\hat{\bm{m}}_k)^T (\hat{\bm{m}}_k - \bm{y})\\
			&= \norm{\nabla \hat{f}_h(\hat{\bm{m}}_k)}_2 \cdot (\hat{\bm{m}}_k^T \bm{y} -1)\\
			&< 0
			\end{align*}
		   contradicting to the fact that $\hat{\bm{m}}_k$ is the unique local mode in $\left\{\bm{z} \in \Omega_q: \norm{\bm{z} - \hat{\bm{m}}_k}_2 \leq \hat{r}_k \right\}$. The result follows. 
		\end{proof}
	
	\begin{customthm}{11}
		Assume (C1) and (C2) and the conditions on kernel $L$ in Theorem \ref{MS_asc}. We further assume that $L$ is continuously differentiable. Then, for each local mode $\hat{\bm{m}}_k \in \hat{\mathcal{M}}_n$, there exists a $\hat{r}_k >0$ such that the sequence $\{\hat{\bm{y}}_s\}_{s=0}^{\infty}$ converges to $\hat{\bm{m}}_k$ whenever the initial point $\hat{\bm{y}}_0 \in \Omega_q$ satisfies $\norm{\hat{\bm{y}}_0 -\hat{\bm{m}}_k}_2 \leq \hat{r}_k$. Moreover, under conditions (D1) and (D2'), there exists a fixed constant $r^* >0$ such that $\mathbb{P}(\hat{r}_k \geq r^*) \to 1$ as $h\to 0$ and $nh^q \to \infty$.
	\end{customthm}	
		
	\begin{proof}
		By the definition of a local mode $\hat{\bm{m}}_k$ of $\hat{f}_h$ on $\Omega_q$, 
		$$\norm{\grad \hat{f}_h(\hat{\bm{m}}_k)}_2=\norm{\Tang\left(\nabla \hat{f}_h(\hat{\bm{m}}_k) \right)}_2 = 0.$$ 
		Hence, with condition (C2) imposed on $\hat{\bm{m}}_k$ and Proposition \ref{Mode_Dir},
		$$\frac{\nabla \hat{f}_h(\hat{\bm{m}}_k)}{\norm{\nabla \hat{f}_h(\hat{\bm{m}}_k)}_2} =\hat{\bm{m}}_k \quad \text{ and } \quad \hat{\bm{m}}_k^T \cdot \frac{\nabla \hat{f}_h(\hat{\bm{m}}_k)}{\norm{\nabla \hat{f}_h(\hat{\bm{m}}_k)}_2}=1.$$
		It indicates that our one-step fixed-point iteration of Algorithm~\ref{Algo:MS} on the local mode $\hat{\bm{m}}_k$ will yield $\hat{\bm{m}}_k$ itself. (This is the so-called consistency of fixed-point iterations.) Moreover, there exists a $\hat{r}_k >0$ such that $\hat{\bm{m}}_k$ is the only point in $\{\bm{y} \in \Omega_q: ||\bm{y} - \hat{\bm{m}}_k||_2 \leq \hat{r}_k \}$ satisfying $\norm{\Tang\left(\nabla \hat{f}_h(\bm{y}) \right)}_2 = 0$. See Figure \ref{fig:MS_One_Step} for a graphical illustration.\\
		In addition, given that $L$ is continuously differentiable, we may shrink $\hat{r}_k >0$ if necessary so that 
			\begin{equation}
			\label{MS_conv:cond1}
			\hat{\bm{m}}_k^T \cdot \frac{\nabla \hat{f}_h(\bm{y})}{\norm{\nabla \hat{f}_h(\bm{y})}_2} \geq 1-\frac{\hat{r}_k^2}{2} \quad \text{ and } \quad \norm{\sum\limits_{i=1}^n \bm{X}_i L'\left(\frac{1-\bm{y}^T\bm{X}_i}{h^2} \right)}_2 = \frac{nh^2}{c_{h,q}(L)}\norm{\nabla \hat{f}_h(\bm{y})}_2 \geq \hat{C}_k
			\end{equation}
			for all $\bm{y} \in \left\{\bm{z} \in \Omega_q: \norm{\bm{z} -\hat{\bm{m}}_k}_2 \leq \hat{r}_k \right\}$ and some constant $\hat{C}_k >0$. The first inequality in (\ref{MS_conv:cond1}) ensures that the sequence $\{\hat{\bm{y}}_s\}_{s=0}^{\infty}$ yielded by our fixed-point iteration will not jump outside of the set $\{\bm{y} \in \Omega_q: \norm{\bm{y} - \hat{\bm{m}}_k}_2 \leq \hat{r}_k \}$ as long as the initial point $\hat{\bm{y}}_0$ is in the set. It also guarantees the correctness of the second inequality in (\ref{MS_conv:cond1}) for the iterative sequence $\left\{\hat{\bm{y}}_s\right\}_{s=0}^{\infty}$. By (\ref{conv_inequ_fh}) in the proof of Theorem~\ref{MS_asc}, we know that
			\begin{align*}
			\hat{f}_h(\hat{\bm{y}}_{s+1}) -\hat{f}_h(\hat{\bm{y}}_s) &\geq \frac{c_{h,q}(L)}{2nh^2} ||\hat{\bm{y}}_{s+1} -\hat{\bm{y}}_s||_2^2 \cdot \norm{\sum\limits_{i=1}^n \bm{X}_i L'\left(\frac{1-\hat{\bm{y}}_s^T\bm{X}_i}{h^2} \right)}_2\\
			&\geq \frac{c_{h,q}(L)\cdot \hat{C}_k}{2nh^2} \cdot ||\hat{\bm{y}}_{s+1} -\hat{\bm{y}}_s||_2^2,
			\end{align*}
			where we used (\ref{MS_conv:cond1}) in the last strict inequality. Since $\left\{\hat{f}_h(\hat{\bm{y}}_s) \right\}_{s=0}^{\infty}$ converges by Theorem \ref{MS_asc} as $s\to \infty$, we conclude that
			\begin{equation}
			\label{pts_diff_limit}
			\lim_{s \to \infty} ||\hat{\bm{y}}_{s+1} -\hat{\bm{y}}_s||_2^2 =0 \quad \text{ or equivalently, } \quad \lim_{s \to \infty} \hat{\bm{y}}_{s+1}^T \hat{\bm{y}}_s = 1.
			\end{equation}
			Now with the expression (\ref{tang_grad}), 
			\begin{align*}
			\norm{\Tang\left(\nabla \hat{f}_h(\hat{\bm{y}}_s) \right)}_2^2 &= \norm{\nabla \hat{f}_h(\hat{\bm{y}}_s) - \hat{\bm{y}}_s^T \nabla \hat{f}_h(\hat{\bm{y}}_s) \cdot \hat{\bm{y}}_s}_2^2\\
			&= \norm{\nabla \hat{f}_h(\hat{\bm{y}}_s)}_2^2 \cdot \norm{\hat{\bm{y}}_{s+1} -\hat{\bm{y}}_{s+1}^T \hat{\bm{y}}_s \cdot \hat{\bm{y}}_s}_2^2\\
			&= \norm{\nabla \hat{f}_h(\hat{\bm{y}}_s)}_2^2 \cdot \left[1- \left(\hat{\bm{y}}_{s+1}^T \hat{\bm{y}}_s \right)^2 \right],
			\end{align*}
			where we plug in (\ref{fix_point_grad}) in the second equality. As the function $\bm{u} \mapsto \norm{\nabla \hat{f}_h(\bm{u})}_2^2$ is continuous on a compact set $\Omega_q$, it is upper bounded on $\Omega_q$. As $s\to \infty$, $\norm{\Tang\left(\nabla \hat{f}_h(\hat{\bm{y}}_s) \right)}_2 \to 0$ by \eqref{pts_diff_limit}. Given that $\hat{\bm{m}}_k$ is the unique point in $\left\{\bm{z} \in \Omega_q: \norm{\bm{z}- \hat{\bm{m}}_k}_2 \leq \hat{r}_k \right\}$ satisfying this $\norm{\Tang\left(\nabla \hat{f}_h(\bm{y}) \right)}_2 = 0$, we conclude that $\hat{\bm{y}}_s \to \hat{\bm{m}}_k$ as $s\to \infty$ and $\hat{\bm{y}}_0 \in \left\{\bm{z} \in \Omega_q: \norm{\bm{z}- \hat{\bm{m}}_k}_2 \leq \hat{r}_k \right\}$. \\
			Now with Lemma~\ref{Rad_grad}, we know that $\hat{\bm{m}}_k^T \nabla \hat{f}_h(\hat{\bm{m}}_k) >0$ for any $\hat{\bm{m}}_k \in \hat{\mathcal{M}}_n$ with probability tending to 1 as $h\to 0$ and $nh^q \to \infty$. Therefore, as $h$ is small enough and $n$ is sufficiently large, there exists a fixed constant $r^*>0$ such that $r^* \leq \min_k \hat{r}_k$ with high probability. The results follow.
		\end{proof}

	\subsection{Proof of Theorem~\ref{Linear_Conv_GA}}
	\label{Appendix:Thm12_pf}	
		
	Before proving Theorem~\ref{Linear_Conv_GA}, we introduce the following two useful results.
	As pointed out in \cite{Geo_Convex_Op2016}, a main hurdle in analyzing non-asymptotic convergence of first-order methods on smooth manifolds is that the Euclidean law of cosines does not hold. Fortunately, there is a trigonometric distance bound stated below for Alexandrov space \citep{Burago1992} with curvature bounded below.
	
	\begin{lemma}[Lemma 5 in \citealt{Geo_Convex_Op2016}; see also \citealt{SGD_Riem2013}]
		\label{trigo_inequality}
		If $a,b,c$ are the sides (that is, side lengths) of a geodesic triangle in an Alexandrov space with sectional curvature (see Appendix~\ref{sec::GH}) lower bounded by $\kappa$, and $A$ is the angle between sides $b$ and $c$, then
		\begin{equation}
		\label{trigo_inequality_eq}
		a^2 \leq \frac{\sqrt{|\kappa|} c}{\tanh(\sqrt{|\kappa|}c)} b^2 + c^2 -2bc \cos(A).
		\end{equation}
	\end{lemma}
	
	The sketching proof of Lemma \ref{trigo_inequality} can be founded in Lemma 5 of \cite{Geo_Convex_Op2016}. Note that $\kappa=1$ on $\Omega_q$. We inherit the notation in \cite{Geo_Convex_Op2016} and denote $\frac{\sqrt{|\kappa|} c}{\tanh(\sqrt{|\kappa|}c)}$ by $\zeta(\kappa, c)$ for the curvature dependent quantity from inequality (\ref{trigo_inequality_eq}). One can show by differentiating $\zeta(\kappa, c)$ with respect to $c$ that $\zeta(\kappa, c)$ is strictly increasing and greater than 1 for any $c>0$ and fixed $\kappa\neq 0$. With Lemma~\ref{trigo_inequality} in hand, we are able to state a straightforward corollary indicating an important relation between two consecutive updates of a gradient ascent algorithm on $\Omega_q$.
	
	\begin{corollary}
		\label{Riemannian_tri_update}
		For any point $\bm{x}, \bm{y}_s$ in a convex set on $\Omega_q$, the update in (\ref{grad_ascent_Manifold}) satisfies
		$$2\eta\langle \grad f(\bm{y}_s), \Exp_{\bm{y}_s}^{-1}(\bm{x}) \rangle \leq d^2(\bm{y}_s, \bm{x}) - d^2(\bm{y}_{s+1}, \bm{x}) + \zeta(1,d(\bm{y}_s, \bm{x})) \cdot \eta^2 ||\grad f(\bm{y}_s)||_2^2,$$
		recalling that $d(\bm{x},\bm{y}) = \sqrt{\langle \Exp_{\bm{x}}^{-1}(\bm{y}), \Exp_{\bm{x}}^{-1}(\bm{y}) \rangle} = \norm{\Exp_{\bm{x}}^{-1}(\bm{y})}_2$ on $\Omega_q$.
	\end{corollary}
	
	The proof is similar to Corollary 8 in \cite{Geo_Convex_Op2016} and thus omitted.
	
	\begin{customthm}{12}
	Assume (D1) and (M1).
	\begin{enumerate}[label=(\alph*)]
		\item \textbf{Linear convergence of gradient ascent with $f$}: Given a convergence radius $r_0$ with $0< r_0 \leq \sqrt{2-2\cos\left[\frac{3\lambda_*}{2(q+1)^{\frac{3}{2}}C_3} \right]}$, the iterative sequence $\left\{\bm{y}_s\right\}_{s=0}^{\infty}$ defined by the population-level gradient ascent algorithm \eqref{grad_ascent_Manifold} satisfies
		$$d(\bm{y}_s, \bm{m}_k) \leq \Upsilon^s \cdot d(\bm{y}_0, \bm{m}_k) \quad \text{ with } \quad \Upsilon = \sqrt{1-\frac{\eta\lambda_*}{2}},$$
		whenever $\eta \leq \min\left\{\frac{2}{\lambda_*}, \frac{1}{(q+1)C_3\zeta(1,r_0)} \right\}$ and the initial point $\bm{y}_0 \in \left\{\bm{z}\in \Omega_q: \norm{\bm{z}-\bm{m}_k}_2 \leq r_0 \right\}$ for some $\bm{m}_k \in \mathcal{M}$.
		We recall from Section~\ref{Sec:Mode_Const} that $C_3$ is an upper bound for the derivatives of the directional density $f$ up to the third order, $\lambda_*>0$ is defined in (M1), and $\mathcal{M}$ is the set of local modes of the directional density $f$.
%		There exists a positive number $r_0 >0$ (convergence radius) such that whenever $\eta \leq \min\left\{\frac{2}{\lambda_*}, \frac{1}{(q+1)C_3\zeta(1,r_0)} \right\}$ and the initial point $\bm{y}_0 \in \left\{\bm{z}\in \Omega_q: \norm{\bm{z}-\bm{m}_k}_2 \leq r_0 \right\}$ for some $\bm{m}_k \in \mathcal{M}$, the iterative sequence $\left\{\bm{y}_s\right\}_{s=0}^{\infty}$ defined by the population-level gradient ascent algorithm \eqref{grad_ascent_Manifold} satisfies
%		$$d(\bm{y}_s, \bm{m}_k) \leq \Upsilon^s \cdot d(\bm{y}_0, \bm{m}_k) \quad \text{ with } \quad \Upsilon = \sqrt{1-\frac{\eta\lambda_*}{2}},$$
%		recalling from Section~\ref{Sec:Mode_Const} that $C_3$ is an upper bound for the derivatives of the directional density $f$ up to the third order, $\lambda_*>0$ is defined in (M1), and $\mathcal{M}$ is the set of local modes of the directional density $f$.
		\end{enumerate}
		We further assume (D2') and (K1) in the sequel.
		\begin{enumerate}[label=(b)]
			\item \textbf{Linear convergence of gradient ascent with $\hat{f}_h$}: Let the sample-based gradient ascent update on $\Omega_q$ be $\hat{\bm{y}}_{s+1} = \Exp_{\bm{y}_s}\left(\eta\cdot \grad \hat{f}_h(\hat{\bm{y}}_s) \right)$. With the same choice of the convergence radius $r_0>0$ and $\Upsilon=\sqrt{1-\frac{\eta\lambda_*}{2}}$ as in (a), if $h\to 0$ and $\frac{nh^{q+2}}{|\log h|} \to \infty$, then for any $\delta \in (0,1)$,
			$$d\left(\hat{\bm{y}}_s,\bm{m}_k \right) \leq \Upsilon^s \cdot d\left(\hat{\bm{y}}_0,\bm{m}_k \right) + O(h^2) + O_P\left(\sqrt{\frac{|\log h|}{nh^{q+2}}} \right)$$
			with probability at least $1-\delta$, whenever $\eta \leq \min\left\{\frac{2}{\lambda_*}, \frac{1}{(q+1)C_3\cdot\zeta(1,r_0)} \right\}$ and the initial point $\hat{\bm{y}}_0 \in \left\{\bm{z}\in \Omega_q: \norm{\bm{z}-\bm{m}_k}_2 \leq r_0 \right\}$ for some $\bm{m}_k \in \mathcal{M}$.
		\end{enumerate}
	\end{customthm}
		
	\begin{proof}
		(a) \textbf{Linear convergence of gradient ascent with $f$}: The proof of the linear convergence of the population-level gradient ascent algorithm \eqref{grad_ascent_Manifold} is similar to some standard results in optimization theory, except that we are under the manifold context now. Recall from \eqref{grad_ascent_Manifold} that the iterative formula reads $\bm{y}_{s+1} = \Exp_{\bm{y}_s}\left(\eta \cdot \grad f(\bm{y}_s) \right)$ for $s=0,1,...$. We begin by deriving the following three facts.\\
		$\bullet$ \emph{Fact 1}: Given (M1), $f$ is geodesically strongly concave around some small neighborhoods of $\mathcal{M}$. In particular, when $0< r_0 \leq \sqrt{2-2\cos\left[\frac{3\lambda_*}{2(q+1)^{\frac{3}{2}} C_3}\right]}$,
		\begin{equation}
		\label{Fact1}
		f(\bm{y}) -f(\bm{m}_k) - \langle \grad f(\bm{m}_k),\, \Exp_{\bm{m}_k}^{-1}(\bm{y}) \rangle \leq -\frac{\lambda_*}{4} \norm{\Exp_{\bm{m}_k}^{-1}(\bm{y})}_2^2
		\end{equation}
		for any $\bm{y} \in \{\bm{z}\in \Omega_q: \norm{\bm{z}-\bm{m}_k}_2\leq r_0 \}$ and any $\bm{m}_k \in \mathcal{M}$.\\
%		$\bullet$ \emph{Fact 1}: Given (M1), there exists an $r_0>0$ such that $f$ is geodesically $\frac{\lambda_*}{2}$-strongly concave within $\cup_k \{\bm{z}\in \Omega_q: \norm{\bm{z}-\bm{m}_k}_2\leq r_0 \}$. That is,
%		\begin{equation}
%			\label{Fact1}
%			f(\bm{y}) -f(\bm{x}) - \langle \grad f(\bm{x}),\, \Exp_{\bm{x}}^{-1}(\bm{y}) \rangle \leq -\frac{\lambda_*}{4} \norm{\Exp_{\bm{x}}^{-1}(\bm{y})}_2^2
%		\end{equation}
%		for any $\bm{x}, \bm{y} \in \{\bm{z}\in \Omega_q: \norm{\bm{z}-\bm{m}_k}_2\leq r_0 \}$ and any $\bm{m}_k \in \mathcal{M}$.\\
        $\bullet$ \emph{Fact 2}. Given (D1) and (M1), we know that $\norm{\grad f(\bm{x})}_2\equiv \norm{\Tang\left(\nabla f(\bm{x}) \right)}_2 >0$ and 
        $$f(\bm{m}_k) - f\left(\Exp_{\bm{x}}\left(\frac{1}{(q+1)C_3} \grad f(\bm{x}) \right) \right) \geq 0$$
        for all $\bm{x} \in \{\bm{z}\in \Omega_q: \norm{\bm{z}-\bm{m}_k}_2\leq r_0 \}\setminus \{\bm{m}_k\}$ and any $\bm{m}_k \in \mathcal{M}$.\\
%		$\bullet$ \emph{Fact 2}. Given (D1) and (M1), we can choose $r_0>0$ properly such that $\norm{\grad f(\bm{x})}_2\equiv \norm{\Tang\left(\nabla f(\bm{x}) \right)}_2 >0$ and 
%		$$f(\bm{m}_k) - f\left(\Exp_{\bm{x}}\left(\frac{1}{(q+1)C_3} \grad f(\bm{x}) \right) \right) \geq 0$$
%		for all $\bm{x} \in \{\bm{z}\in \Omega_q: \norm{\bm{z}-\bm{m}_k}_2\leq r_0 \}\setminus \{\bm{m}_k\}$ and any $\bm{m}_k \in \mathcal{M}$.\\
		$\bullet$ \emph{Fact 3}. Given (D1), the directional density $f$ is $(q+1)C_3$-smooth.\\
			
		As for \emph{Fact 1}, it follows from the differentiability of $f$ guaranteed by (D1) and the eigenvalue condition (M1). By Taylor's expansion on manifolds \citep{pennec2006intrinsic} and (M1),
		\begin{align}
		\label{LGSC1}
		\begin{split}
		&f(\bm{y}) -f(\bm{m}_k) \\
		&= \langle \grad f(\bm{m}_k), \Exp_{\bm{m}_k}^{-1}(\bm{y}) \rangle + \frac{1}{2}\cdot\Exp_{\bm{m}_k}^{-1}(\bm{y})^T \left[\mathcal{H} f(\bm{m}_k) \right] \Exp_{\bm{m}_k}^{-1}(\bm{y}) + o\left(\norm{\Exp_{\bm{m}_k}^{-1}(\bm{y})}_2^2 \right)\\
		&\leq \langle \grad f(\bm{m}_k), \Exp_{\bm{m}_k}^{-1}(\bm{y}) \rangle -\frac{\lambda_*}{2} \norm{\Exp_{\bm{m}_k}^{-1}(\bm{y})}_2^2 + \frac{(q+1)^{\frac{3}{2}} C_3}{6} \cdot \norm{\Exp_{\bm{m}_k}^{-1}(\bm{y})}_2^3
		\end{split}
		\end{align}
		for any $\bm{y} \in \{\bm{z}\in \Omega_q: \norm{\bm{z}-\bm{m}_k}_2 \leq r_0 \}$ and $\bm{m}_k \in \mathcal{M}$. Since $\norm{\bm{y}-\bm{m}_k}_2 \leq r_0$, the geodesic distance between $\bm{y}$ and $\bm{m}_k$ satisfies $d_g(\bm{y},\bm{m}_k) = \norm{\Exp_{\bm{m}_k}^{-1}(\bm{y})}_2 = \arccos(\bm{y}^T \bm{m}_k) \leq \frac{3\lambda_*}{2(q+1)^{\frac{3}{2}} C_3}$. Plugging this result back into \eqref{LGSC1} yields that
		$$f(\bm{y}) -f(\bm{m}_k) \leq \langle \grad f(\bm{m}_k), \Exp_{\bm{m}_k}^{-1}(\bm{y}) \rangle -\frac{\lambda_*}{4} \norm{\Exp_{\bm{m}_k}^{-1}(\bm{y})}_2^2.$$
		For our purpose, it suffices to only prove \eqref{Fact1} as above. One can shrink the upper bound of the convergence radius $r_0>0$ so that the geodesically strong concavity is valid for any pair of points within $\{\bm{z}\in \Omega_q: \norm{\bm{z}-\bm{m}_k}_2 \leq r_0 \}$. Indeed, the local strong concavity of $f$ is a natural consequence of Morse Lemma (Lemma 3.11 in \cite{Morse_Homology2004}) given (M1).
		
		\emph{Fact 2} is an obvious result under the eigenvalue condition (M1) and differentiable condition (D1). This is because $\bm{m}_k$ is an unique local mode of $f$ within the neighborhood $\{\bm{z}\in \Omega_q: \norm{\bm{z}-\bm{m}_k}_2 \leq r_0 \}$ and the geodesic distance between $\bm{x}$ and one-step gradient ascent iteration from $\bm{x}$ with the step size $\frac{1}{(q+1)C_3}$ satisfies
		\begin{align*}
		d\left(\Exp_{\bm x}\left(\frac{1}{(q+1)C_3} \grad f(\bm{x}) \right), \bm{x}\right) & = \frac{1}{(q+1)C_3} \norm{\grad f(\bm{x})}_2\\
		&= \frac{1}{(q+1)C_3} \norm{\grad f(\bm{x}) - \underbrace{\Gamma_{\bm{m}_k}^{\bm x}\left(\grad f(\bm{m}_k) \right)}_{=0}}_2\\
		&\leq \frac{1}{(q+1)C_3} \norm{\mathcal{H} f(\bm{x})}_2 \cdot \norm{\Exp_{\bm x}^{-1}(\bm{m}_k)}_2\\
		&\leq \norm{\Exp_{\bm x}^{-1}(\bm{m}_k)}_2 = d_g(\bm{x}, \bm{m}_k),
		\end{align*}
		where we use the fact that $\norm{\mathcal{H} f(\bm{x})}_2 \leq (q+1)\norm{\mathcal{H} f(\bm{x})}_{\max} \leq (q+1)C_3$ to deduce the last inequality.
		This shows that the one-step gradient ascent iteration $\Exp_{\bm{x}}\left(\frac{1}{(q+1)C_3} \cdot \grad f(\bm{x}) \right)$ on $\Omega_q$ will stay within the neighborhood $\{\bm{z}\in \Omega_q: \norm{\bm{z}-\bm{m}_k}_2\leq r_0 \}$ whenever $\bm{x} \in \{\bm{z}\in \Omega_q: \norm{\bm{z}-\bm{m}_k}_2\leq r_0 \}\setminus \{\bm{m}_k\}$. Therefore, $f(\bm{m}_k) - f\left(\Exp_{\bm{x}}\left(\frac{1}{(q+1)C_3} \cdot \grad f(\bm{x}) \right) \right) \geq 0$.
        
%        \emph{Fact 2} ensures that the population-level gradient ascent algorithm will eventually converge to the nearest local mode $\bm{m}_k$ when it is initialized at $\bm{y}_0 \in \{\bm{z}\in \Omega_q: \norm{\bm{z}-\bm{m}_k}_2\leq r_0 \}$ for some $\bm{m}_k \in \mathcal{M}$. This is valid because $\bm{m}_k$ is an unique local mode of $f$ within the neighborhood $\{\bm{z}\in \Omega_q: \norm{\bm{z}-\bm{m}_k}_2 \leq r_0 \}$ under a small radius $r_0>0$.
			
		As for \emph{Fact 3}, note that $\norm{\mathcal{H} f(\bm x)}_{\max} \leq C_3$ for all $\bm{x}\in \Omega_q$. Thus,
		\begin{align*}
			\norm{\grad f(\bm{x}) - \Gamma_{\bm{y}}^{\bm{x}}\left(\grad f(\bm{y}) \right)}_2 &= \norm{\left(\mathcal{H} f(\bm x) \right) \Exp_{\bm{x}}^{-1}(\bm{x}^*)}_2\\
			&\leq \norm{\mathcal{H} f(\bm x)}_2 \cdot \norm{\Exp_{\bm{x}}^{-1}(\bm{y})}_2\\
			&\leq (q+1)C_3 \norm{\Exp_{\bm{x}}^{-1}(\bm{y})}_2,
		\end{align*}
		where $\bm{x}^* \in \left\{\bm{z}\in \Omega_q: \norm{\bm{z}-\bm{x}}_2 \leq \norm{\bm{y}-\bm{x}}_2 \right\}$, and we use the fact that $||A||_2 \leq \sqrt{mn}||A||_{\max}$ for any $A\in \mathbb{R}^{m\times n}$. See Section 3.3 in \cite{Non_ridge_est2014} or Section 5.6 in \cite{HJ2012} for detailed relations between different types of matrix norms.
			
		With \emph{Fact 1} and \emph{Fact 3}, we have that
		\begin{align}
		\label{f_cond}
		\begin{split}
		&-\frac{(q+1)C_3}{2} \norm{\Exp_{\bm{m}_k}^{-1}(\bm{y})}_2^2 \leq f(\bm{y}) -f(\bm{m}_k) - \langle \grad f(\bm{m}_k), \Exp_{\bm{m}_k}^{-1}(\bm{y}) \rangle, \\
		&f(\bm{y}) -f(\bm{m}_k) - \langle \grad f(\bm{m}_k), \Exp_{\bm{m}_k}^{-1}(\bm{y}) \rangle \leq -\frac{\lambda_*}{4} \norm{\Exp_{\bm{m}_k}^{-1}(\bm{y})}_2^2<0
		\end{split}
		\end{align}
		for any $\bm{y} \in \left\{\bm{z}\in \Omega_q: \norm{\bm{z}-\bm{m}_k}_2 \leq r_0 \right\}$ and $\bm{m}_k \in \mathcal{M}$. \\
		Hence, given a point $\bm{y} \in \left\{\bm{z}\in \Omega_q: \norm{\bm{z}-\bm{m}_k}_2 \leq r_0 \right\}$ and using \emph{Fact 2}, 
		\begin{align*}
		&f(\bm{y}) -f(\bm{m}_k)\\
		&\leq f(\bm{y}) -f(\bm{m}_k) + f(\bm{m}_k) - f\left(\Exp_{\bm{y}}\left(\frac{1}{(q+1)C_3} \cdot \grad f(\bm{y}) \right) \right)\\
		&= -\left[f\left(\Exp_{\bm{y}}\left(\frac{1}{(q+1)C_3} \cdot \grad f(\bm{y}) \right) \right) -f(\bm{y}) \right]\\
		&\leq -\left[\langle \grad f(\bm{y}), \frac{1}{(q+1)C_3} \grad f(\bm{y}) \rangle - \frac{(q+1)C_3}{2} \norm{\Exp_{\bm{y}}^{-1}\left(\frac{1}{(q+1)C_3} \cdot \grad f(\bm{y}) \right)}_2^2 \right]\\
		&= -\frac{1}{2(q+1)C_3} \norm{\grad f(\bm{y})}_2^2,
		\end{align*}
		where we apply the first inequality in \eqref{f_cond} to obtain the fourth line. Thus, for any $\bm{x} \in \left\{\bm{z}\in \Omega_q: \norm{\bm{z}-\bm{m}_k}_2 \leq r_0 \right\}$,
			\begin{equation}
			\label{grad_bound_f}
			\norm{\grad f(\bm{x})}_2^2 \leq 2(q+1)C_3 \left[f(\bm{m}_k) -f(\bm{x}) \right].
			\end{equation}
		With $\bm{y}_0 \in \left\{\bm{z}\in \Omega_q: \norm{\bm{z}-\bm{m}_k}_2 \leq r_0 \right\}$ and  Corollary~\ref{Riemannian_tri_update}, we deduce that
		\begin{align*}
			d^2(\bm{y}_{s+1},\bm{m}_k) &\leq d^2(\bm{y}_s,\bm{m}_k) - 2\eta \langle \grad f(\bm{y}_s), \Exp_{\bm{y}_s}^{-1}(\bm{m}_k) \rangle + \zeta(1, d(\bm{y}_s, \bm{m}_k)) \cdot \eta^2 \norm{\grad f(\bm{y}_s)}_2^2\\
			&\stackrel{\text{(i)}}{\leq} d^2(\bm{y}_s,\bm{m}_k) + 2\eta \left[f(\bm{y}_s) -f(\bm{m}_k) -\frac{\lambda_*}{4} d^2(\bm{y}_s,\bm{m}_k) \right] \\
			&\quad + \zeta(1,r_0)\cdot \eta^2 \cdot 2(q+1)C_3 \left[f(\bm{m}_k) -f(\bm{y}_s) \right]\\
			&=\left(1-\frac{\eta\lambda_*}{2} \right) d^2(\bm{y}_s, \bm{m}_k) -2\eta \left[1-\zeta(1,r_0) (q+1)C_3\eta \right]\cdot\underbrace{\left[f(\bm{m}_k) -f(\bm{y}_s) \right]}_{\geq 0}\\
			&\leq \left(1-\frac{\eta\lambda_*}{2} \right) d^2(\bm{y}_s, \bm{m}_k)
			\end{align*} 
		whenever $\eta \leq \min\left\{\frac{2}{\lambda_*}, \frac{1}{(q+1)C_3\cdot\zeta(1,r_0)} \right\}$, where we use the second inequality of (\ref{f_cond}), the monotonicity of $\zeta(1,c)$ with respect to $c$, and (\ref{grad_bound_f}) to obtain (i). By telescoping, we conclude that when $\eta \leq \min\left\{\frac{2}{\lambda_*}, \frac{1}{(q+1)C_3\cdot\zeta(1,r_0)} \right\}$ and $\bm{y}_0 \in \left\{\bm{z}\in \Omega_q: \norm{\bm{z}-\bm{m}_k}_2 \leq r_0 \right\}$,
		$$d(\bm{y}_s,\bm{m}_k) = \norm{\Exp_{\bm{y}_s}^{-1}(\bm{m}_k)}_2 \leq \left(1-\frac{\eta\lambda_*}{2} \right)^{\frac{s}{2}} \cdot d(\bm{y}_0,\bm{m}_k) = \left(1-\frac{\eta\lambda_*}{2} \right)^{\frac{s}{2}} \norm{\Exp_{\bm{y}_0}^{-1}(\bm{m}_k)}_2.$$
		The result follows.\\
			
		\noindent (b) \textbf{Linear convergence of gradient ascent with $\hat{f}_h$}: The proof here is partially adopted from the proof of Theorem 2 in \cite{EM2017}. By Theorem \ref{unif_conv_tang} and the continuity of exponential map, when $h$ is sufficiently small and $\frac{nh^{q+2}}{|\log h|}$ is sufficiently large, we have that for any $\delta \in (0,1)$,
		\begin{align}
			\label{Unif_grad_bound2}
			\begin{split}
			d\left(\Exp_{\bm{x}}(\eta \cdot \grad \hat{f}_h(\bm{x})), \Exp_{\bm{x}} (\eta \cdot \grad f(\bm{x}))\right) &\leq \eta\bar{C}_4 \cdot \sup_{\bm{x} \in \Omega_q} \norm{\grad \hat{f}_h(\bm{x}) -\grad f(\bm{x})}_{\max}\\
			&\equiv \epsilon_{n,h} \\
			&\leq (1-\Upsilon) \cdot  \arccos\left(1-\frac{r_0^2}{2} \right)
			\end{split}
		\end{align}
		with probability at least $1-\delta$, where $\bar{C}_4$ is some constant independent of $\bm{x} \in \Omega_q$, and $\epsilon_{n,h}=\eta\bar{C}_4 \cdot \sup_{\bm{x} \in \Omega_q} \norm{\grad \hat{f}_h(\bm{x}) -\grad f(\bm{x})}_{\max} = O(h^2) + O_P\left(\sqrt{\frac{|\log h|}{nh^{q+2}}} \right)$.\\
		We now claim that $d(\hat{\bm{y}}_s,\bm{m}_k) \leq \arccos\left(1-\frac{r_0^2}{2} \right)$ and
		\begin{equation}
			\label{claim1}
			d(\hat{\bm{y}}_{s+1}, \bm{m}_k) \leq \Upsilon \cdot d(\hat{\bm{y}}_s, \bm{m}_k) + \epsilon_{n,h}
		\end{equation}
		for any fixed $s=0,1,2,...$ with probability at least $1-\delta$. We will prove this claim by induction on the iteration number. Recall that 
		$$\hat{\bm{y}}_{s+1} = \Exp_{\hat{\bm{y}}_s}\left(\eta \cdot \grad \hat{f}_h(\hat{\bm{y}}_s) \right).$$
		Then with $s=1$, we have that
		\begin{align*}
			&d(\hat{\bm{y}}_1, \bm{m}_k) \\
			&= d\left(\Exp_{\hat{\bm{y}}_0}\left(\eta \cdot \grad \hat{f}_h(\hat{\bm{y}}_0) \right), \bm{m}_k \right)\\
			&\leq d\left(\Exp_{\hat{\bm{y}}_0}\left(\eta \cdot \grad f(\hat{\bm{y}}_0) \right), \bm{m}_k \right) + d\left(\Exp_{\hat{\bm{y}}_0}\left(\eta \cdot \grad \hat{f}_h(\hat{\bm{y}}_0) \right),\, \Exp_{\hat{\bm{y}}_0}\left(\eta \cdot \grad f(\hat{\bm{y}}_0) \right) \right)\\
			&\leq \Upsilon \cdot d(\hat{\bm{y}}_0, \bm{m}_k) + \eta\bar{C}_4 \cdot \sup_{\bm{x} \in \Omega_q} \norm{\grad \hat{f}_h(\bm{x}) -\grad f(\bm{x})}_{\max}\\
			&= \Upsilon \cdot d(\hat{\bm{y}}_0, \bm{m}_k) + \epsilon_{n,h},
		\end{align*}
		where the first inequality follows by the triangle inequality, the second one is from our result in (a), whereas the third equality is by \eqref{Unif_grad_bound2}. The triangle inequality is valid in this context because a geodesic measures the minimal distance between two points on $\Omega_q$. In addition, the bound in \eqref{Unif_grad_bound2} and our initialization $\hat{\bm{y}}_0 \in \left\{\bm{z}\in \Omega_q: \norm{\bm{z}-\bm{m}_k}_2 \leq r_0 \right\}$ ensure that $d(\hat{\bm{y}}_1, \bm{m}_k) \leq \arccos\left(1-\frac{r_0^2}{2} \right)$. In the induction from $s \to s+1$, suppose that $d(\hat{\bm{y}}_s, \bm{m}_k) \leq \arccos\left(1-\frac{r_0^2}{2} \right)$ and the claim (\ref{claim1}) holds at step $s$. With the fact proved in (a) that 
		$$d\left(\Exp_{\hat{\bm{y}}_s}\left(\eta\cdot \grad f(\hat{\bm{y}}_s) \right),\, \bm{m}_k \right) \leq \Upsilon \cdot d(\hat{\bm{y}}_s, \bm{m}_k),$$
		the same argument implies that the claim (\ref{claim1}) holds for step $s+1$ and $d(\hat{\bm{y}}_{s+1}, \bm{m}_k) \leq \arccos\left(1-\frac{r_0^2}{2} \right)$. The claim (\ref{claim1}) is thus proved.\\
		As a result, $\hat{\bm{y}}_s$ always lies within $\left\{\bm{z}\in \Omega_q: \norm{\bm{z}-\bm{m}_k}_2 \leq r_0 \right\}$ for all $s=0,1,...$.
		Now, with this claim and $\Upsilon=\sqrt{1- \frac{\eta\lambda_*}{2}} <1$, we iterate it to show that
		\begin{align*}
			d(\hat{\bm{y}}_s, \bm{m}_k) &\leq \Upsilon \cdot d(\hat{\bm{y}}_{s-1}, \bm{m}_k) + \epsilon_{n,h}\\
			&\leq \Upsilon \cdot \left[\Upsilon \cdot d(\hat{\bm{y}}_{s-2}, \bm{m}_k) + \epsilon_{n,h} \right] + \epsilon_{n,h}\\
			&\leq \Upsilon^s \cdot d(\hat{\bm{y}}_0, \bm{m}_k) + \left\{\sum_{k=0}^{s-1} \Upsilon^k\right\} \cdot \epsilon_{n,h}\\
			&\leq \Upsilon^s \cdot d(\hat{\bm{y}}_0, \bm{m}_k) + \frac{\epsilon_{n,h}}{1-\Upsilon} \\
			&\leq \Upsilon^s \cdot d(\hat{\bm{y}}_0, \bm{m}_k) + O(h^2) + O_P\left(\frac{|\log h|}{nh^{q+2}} \right),
		\end{align*}
		where the fourth inequality follows by summing the geometric series, and the last one follows from our notation that $\epsilon_{n,h}=O(h^2) + O_P\left(\sqrt{\frac{|\log h|}{nh^{q+2}}} \right)$. It completes the proof.
		\end{proof}
		
	\end{appendices}
	
\vskip 0.2in
\bibliography{Bibliography}
	
\end{document}